\documentclass[conference]{IEEEtran}
\usepackage{times}
\usepackage[numbers]{natbib}
\usepackage{multicol}
\usepackage[bookmarks=true]{hyperref}

\usepackage{amsmath,amsthm,amssymb}
\usepackage{mathtools}
\usepackage{graphicx}
\usepackage{adjustbox}
\usepackage{tikz}
\usepackage{tikz-3dplot}
\usepackage{pgfplots}
\usepackage{pgfplotstable}
\usepackage{float}
\usepackage{acro}
\usepackage{wrapfig}
\let\labelindent\relax
\usepackage{enumitem}
\usepackage{bm}
\usepackage{etoolbox}
\usepackage{makecell}
\usepackage{upgreek}
\usepackage{stfloats}
\usepackage[capitalize]{cleveref}
\usepackage[caption=false,font=footnotesize]{subfig}
\usepackage[ruled,linesnumbered]{algorithm2e}

\renewcommand{\vec}[1]{\mathbf{\boldsymbol{#1}}}
\newcommand{\mat}[1]{\boldsymbol{\mathbf{#1}}}

\newcommand{\diag}{\operatorname{diag}}
\newcommand{\blkdiag}{\operatorname{blkdiag}}
\newcommand{\trace}{\operatorname{Tr}}

\newcommand{\minimize}{\mathop{\text{minimize}}}

\newcommand{\subjectto}{\mathop{\text{subject to}}}

\usetikzlibrary{arrows.meta}
\usetikzlibrary{3d}
\usetikzlibrary{calc}
\usetikzlibrary{positioning}
\usetikzlibrary{fit}
\usetikzlibrary{angles}
\usetikzlibrary{quotes}
\usetikzlibrary{patterns}
\usetikzlibrary{pgfplots.statistics}
\usetikzlibrary{pgfplots.fillbetween}

\usepackage{xcolor}
\definecolor{gtcolor}{RGB}{0, 48, 87} 

\usepgfplotslibrary{colorbrewer}
\pgfplotsset{compat=1.9}
\pgfplotsset{legend cell align={left}}
\pgfplotsset{every axis plot/.append style={line width=0.8pt}}
\pgfplotsset{cycle list/Set1-5}

\theoremstyle{definition}
\newtheorem{theorem}{Theorem}
\newtheorem{lemma}{Lemma}
\newtheorem{corollary}{Corollary}
\newtheorem{example}{Example}

\renewcommand{\figurename}{Figure}

\AtBeginDocument{%
  \setlength{\floatsep}{12pt}       
  \setlength{\textfloatsep}{14pt}   
  \setlength{\intextsep}{0pt}       
}

\AtBeginDocument{%
  \thinmuskip=3mu
  \medmuskip=4mu
  \thickmuskip=5mu
  \binoppenalty=10000
  \relpenalty=10000
  \linepenalty=1000
}

\makeatletter
\pretocmd{\appendices}{%
  \addtocontents{toc}{\protect\setcounter{tocdepth}{0}}%
  \crefalias{section}{appendix}%
}{}{}
\crefname{appendix}{Appendix}{Appendices}
\Crefname{appendix}{Appendix}{Appendices}
\makeatother

\DeclareAcronym{RL}{
  short = RL,
  long  = reinforcement learning
}
\DeclareAcronym{MPPI}{
  short = MPPI,
  long  = model predictive path integral
}
\DeclareAcronym{CMA-ES}{
  short = CMA-ES,
  long  = covariance matrix adaptation evolution strategy
}
\DeclareAcronym{CEM}{
  short = CEM,
  long  = cross-entropy method
}
\DeclareAcronym{KKT}{
  short = KKT,
  long  = Karush-Kuhn-Tucker
}
\DeclareAcronym{DoF}{
  short = {},
  long  = {},
  plural = degrees of freedom,
  short-plural = DoFs
}
\DeclareAcronym{QP}{
  short = QP,
  long  = quadratic program
}
\DeclareAcronym{MPC}{
  short = MPC,
  long  = model predictive control
}
\DeclareAcronym{SCP}{
  short = SCP,
  long  = sequential convex programming
}
\DeclareAcronym{iLQR}{
  short = iLQR,
  long  = iterative linear quadratic regulator
}
\DeclareAcronym{DDP}{
  short = DDP,
  long  = differential dynamic programming
}
\DeclareAcronym{CTR}{
  short = CTR,
  long  = contact trust region
}
\DeclareAcronym{TO-CTR}{
  short = TO-CTR,
  long  = trajectory optimization with contact trust region
}
\DeclareAcronym{LCP}{
  short = LCP,
  long  = linear complementarity program
}
\DeclareAcronym{NCP}{
  short = NCP,
  long  = nonlinear complementarity program
}
\DeclareAcronym{RRT}{
  short = RRT,
  long  = rapidly exploring random tree
}
\DeclareAcronym{MIP}{
  short = MIP,
  long  = mixed integer programming
}
\DeclareAcronym{GCS}{
  short = GCS,
  long  = graph of convex sets
}
\DeclareAcronym{LTV}{
  short = LTV,
  long  = linear time-varying
}
\DeclareAcronym{AM-GM}{
  short = AM-GM,
  long  = arithmetic mean-geometric mean
}
\DeclareAcronym{MPCC}{
  short = MPCC,
  long  = mathmatical program with complementarity constraints
}
\DeclareAcronym{DDF}{
  short = DDF,
  long  = diagonal dominance factor
}

\begin{document}

\title{\looseness-1Certified Gradient-Based Contact-Rich Manipulation via Smoothing-Error Reachable Tubes}

\author{\authorblockN{Wei-Chen Li and Glen Chou}
\authorblockA{Georgia Institute of Technology, Atlanta, Georgia 30308\\
Email: \texttt{\{wli777, chou\}@gatech.edu}\\
\textcolor{gtcolor}{\textbf{Project Website}: \href{https://trustworthyrobotics.github.io/quasistatic-contact-SLS}{(Link)}} \quad $\mid$ \quad \textcolor{gtcolor}{\textbf{Code}: \href{https://github.com/trustworthyrobotics/quasistatic-contact-SLS}{(Github)}} \quad $\mid$ \quad \textcolor{gtcolor}{\textbf{Video}: \href{https://youtu.be/LFxi2hkQJyo}{(YouTube)}}}
}


%

\maketitle

\begin{abstract}
\looseness-1Gradient-based methods can efficiently optimize controllers by leveraging differentiable simulation and physical priors. However, contact-rich manipulation remains challenging because hybrid contact dynamics often produce discontinuous or vanishing gradients. Although smoothing the dynamics can restore informative gradients, the resulting model mismatch can cause controller failures when deployed on real systems.
We address this trade-off by planning with smoothed dynamics while explicitly quantifying and compensating for the induced error, providing formal guarantees on safety and task completion under the original nonsmooth dynamics. Our approach applies smoothing to both contact dynamics and contact geometry within a differentiable simulator based on convex optimization, allowing us to characterize the deviation from the nonsmooth dynamics as a set-valued discrepancy.
We incorporate this discrepancy into the optimization of time-varying affine feedback policies through analytical reachable sets, enabling robust constraint satisfaction for the closed-loop hybrid system while relying solely on the informative gradients of the smoothed model. By bridging differentiable simulation with set-valued robust control, our method produces affine feedback policies that respect the unilateral nature of contact.
We evaluate our method on several contact-rich tasks, including planar pushing, object rotation, and in-hand dexterous manipulation, achieving certified constraint satisfaction with lower safety violations and smaller goal errors than baseline approaches.
\end{abstract}

\section{Introduction}
Contact-rich manipulation involves controlling unactuated objects to reach a desired configuration, which requires making discrete decisions about when and how to make or break contact. Among current approaches, \ac{RL} \cite{schulman2017proximal} and zero-order planning methods such as \ac{CEM} \cite{howell2022predictive, li2025drop}, \ac{CMA-ES} \cite{hansen2016cma, jankowski2022vp}, and \ac{MPPI} control \cite{pezzato2025sampling} have achieved notable success. However, these methods are fundamentally sampling-based and effectively require rediscovering the underlying physics through trial-and-error, resulting in low sample efficiency.

In contrast, gradient-based trajectory optimization \cite{betts1998survey, schulman2014motion, malyuta2022convex} can exploit the problem's physical structure via differentiable simulation \cite{newbury2024review, murthy2021gradsim}. 
However, hybrid dynamics in contact-rich manipulation violate key assumptions of these methods. In particular, nonsmooth hybrid dynamics create fundamental challenges: gradients are often discontinuous across contact modes and frequently zero. For example, before the robot contacts the object, its actions do not affect the object or reward. Since the contact manifold occupies a zero-measure subset of the configuration space, standard gradients provide little guidance for improving actions. As such, while differentiable simulators exist \cite{freeman2021brax}, the practical utility of the gradients they provide remains limited \cite{suh2022do} for manipulation.

\begin{figure}[!t]
    \centering
    \vspace{1mm}
    \subfloat[]{%
        \adjincludegraphics[height=2.2cm, trim={{.34\width} {.085\height} {.345\width} {.057\height}}, clip]{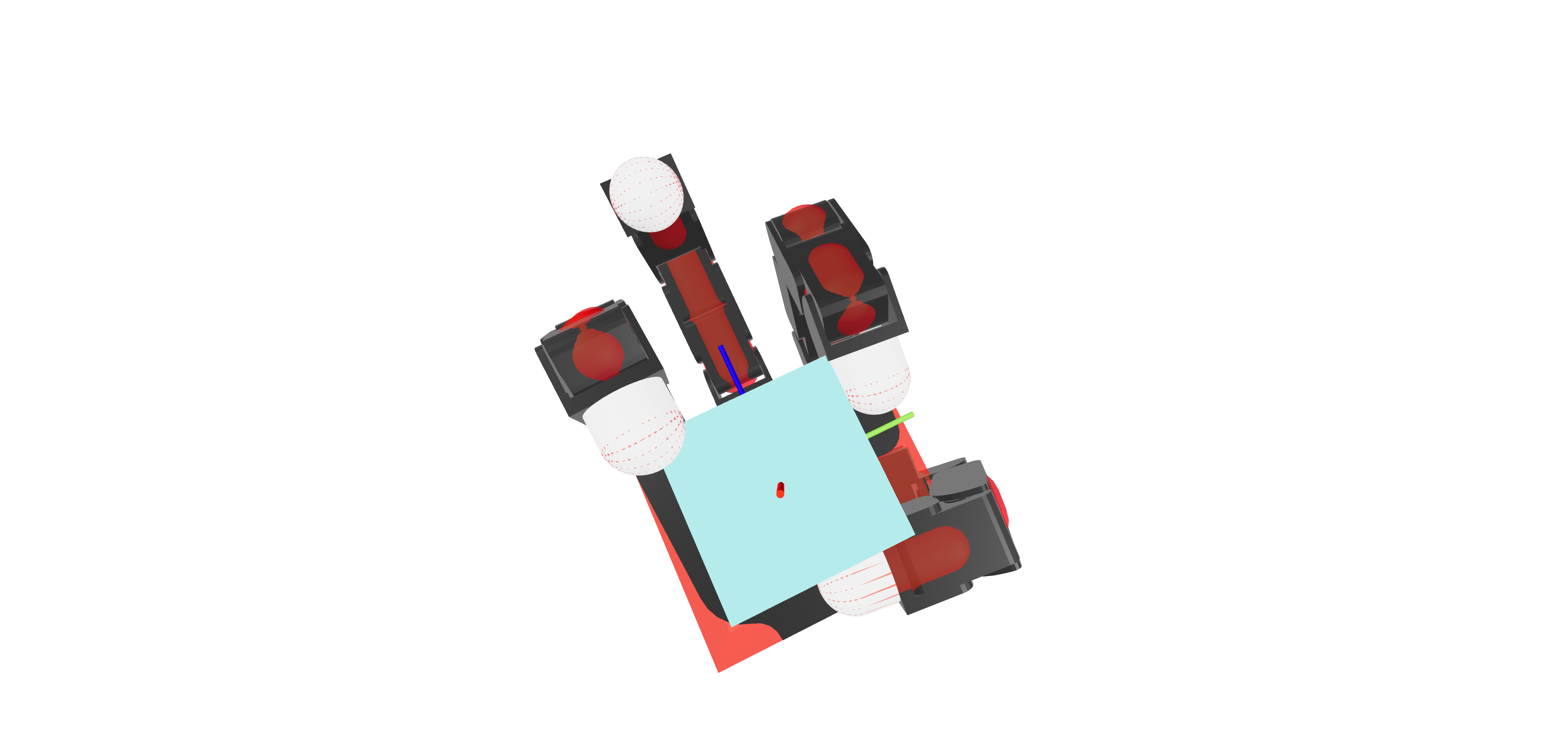}
        \adjincludegraphics[height=2.2cm, trim={{.34\width} {.085\height} {.345\width} {.057\height}}, clip]{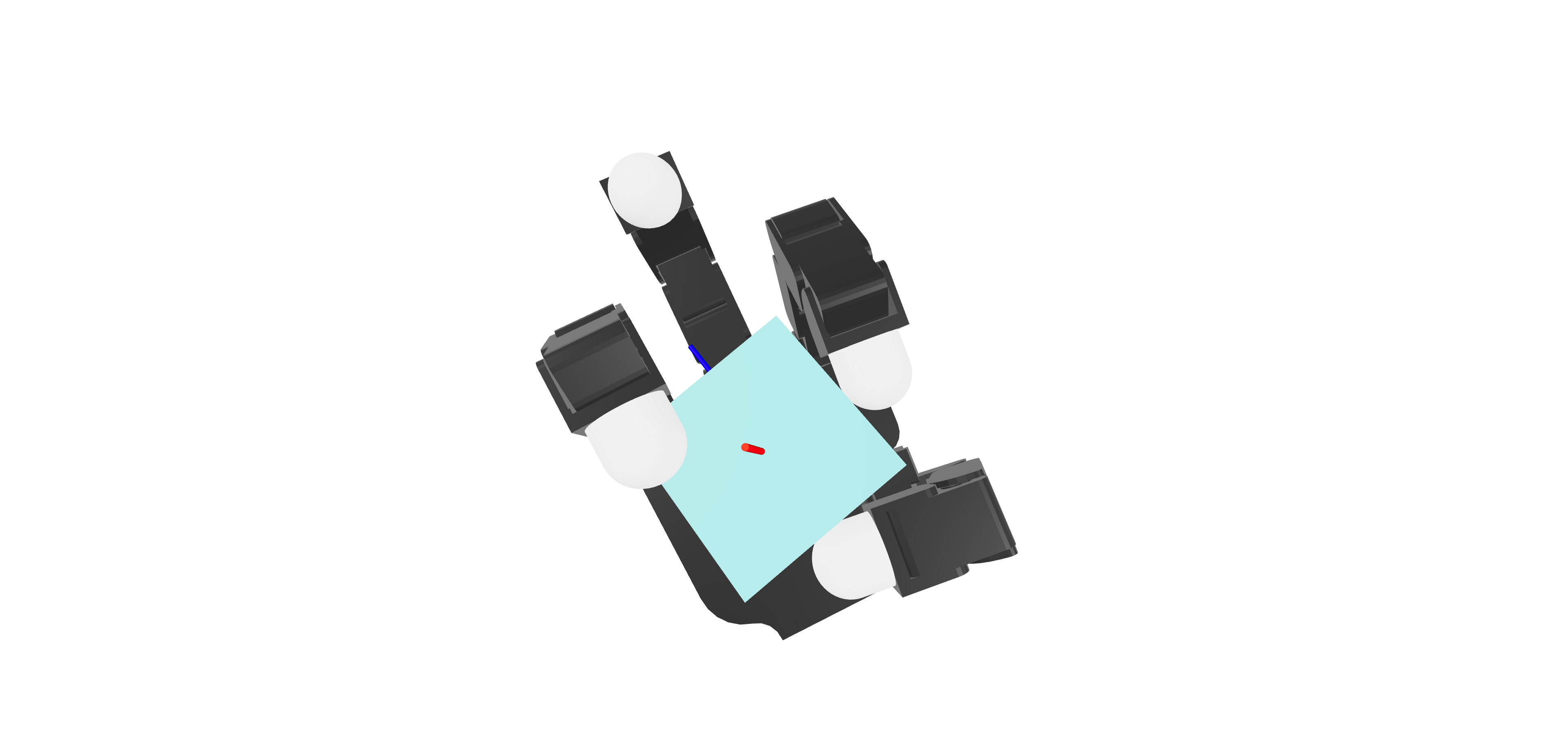}
        \adjincludegraphics[height=2.2cm, trim={{.34\width} {.085\height} {.345\width} {.057\height}}, clip]{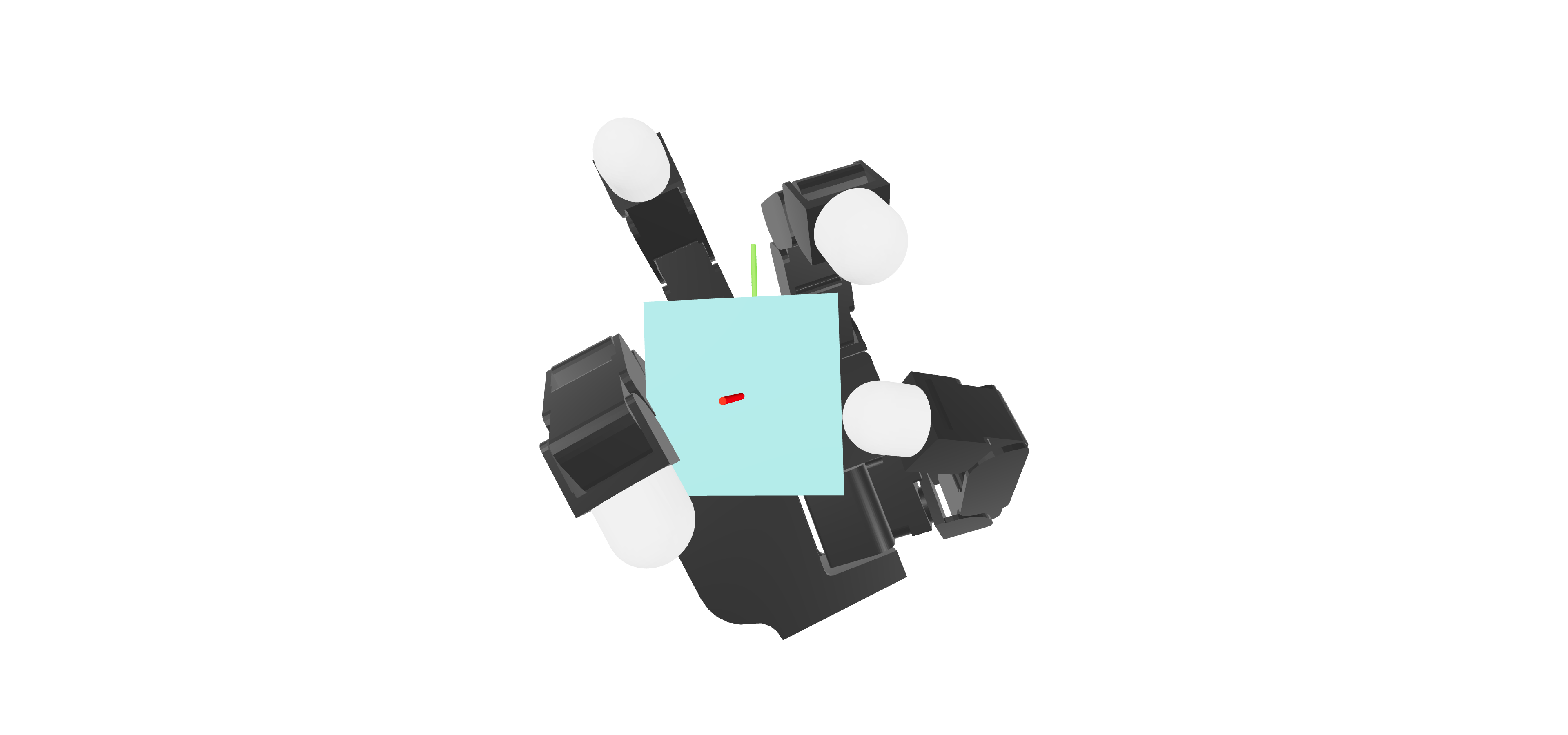}
        \adjincludegraphics[height=2.2cm, trim={{.34\width} {.085\height} {.345\width} {.057\height}}, clip]{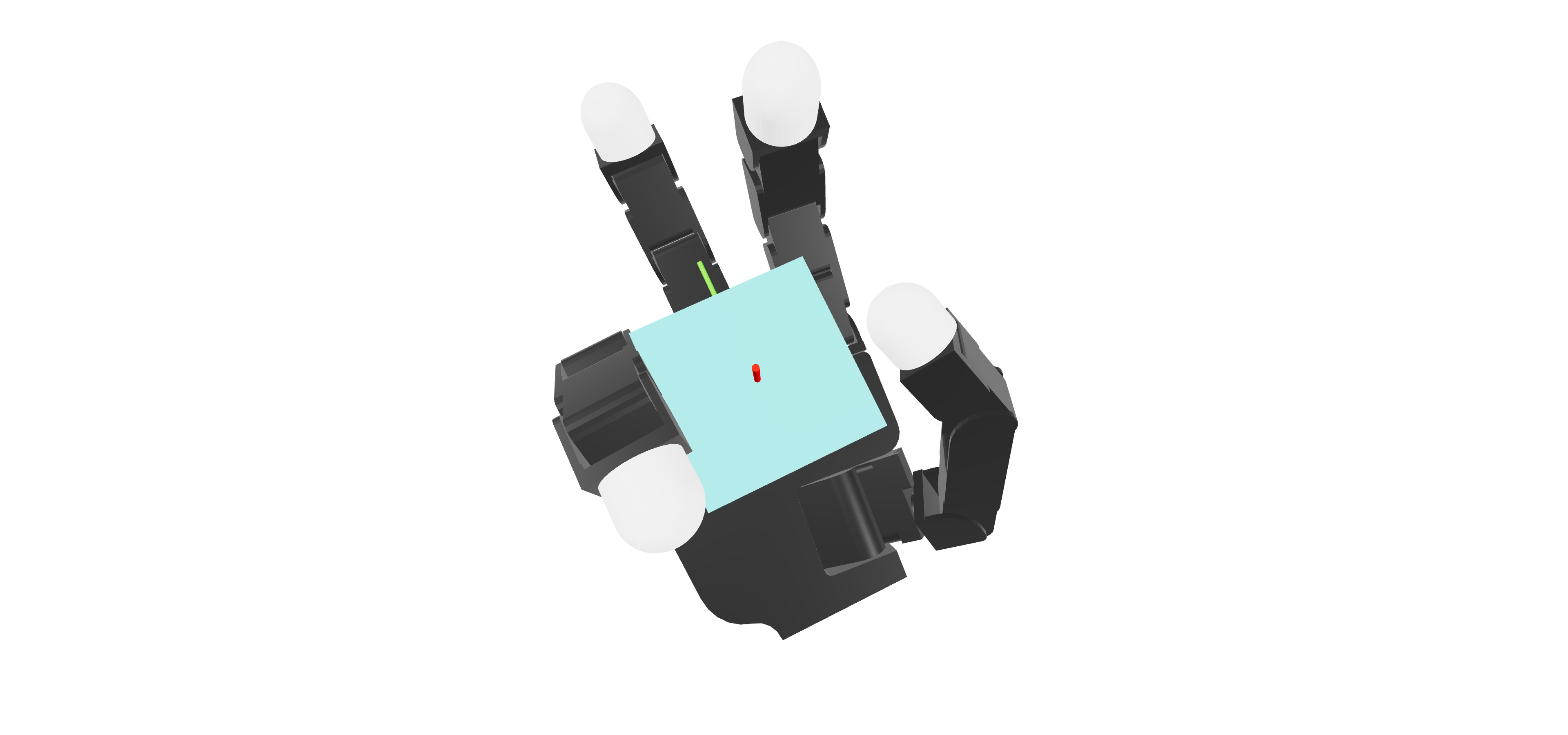}
        \adjincludegraphics[height=2.2cm, trim={{.34\width} {.085\height} {.345\width} {.057\height}}, clip]{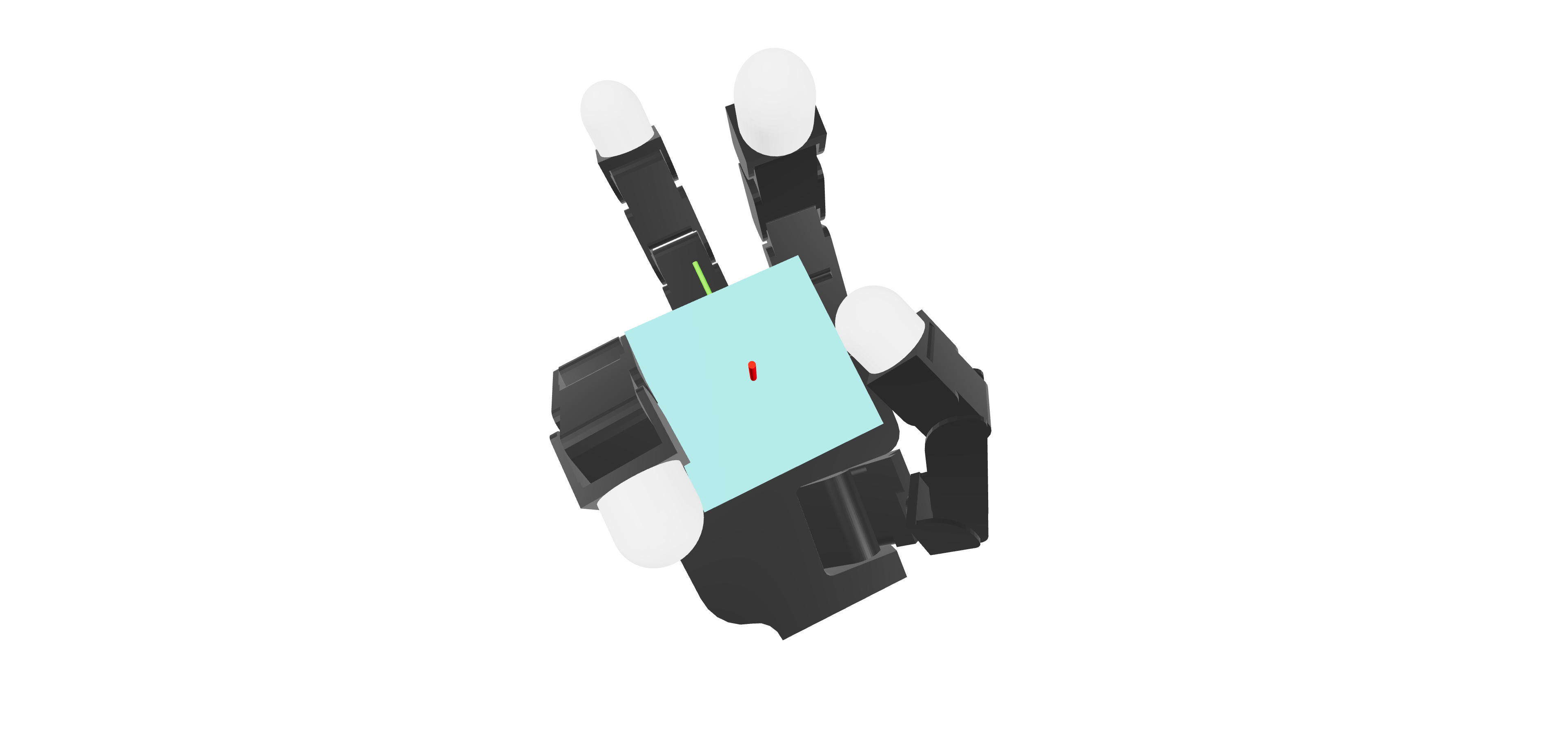}
    }\\
    \subfloat[]{%
      \shortstack{
        \hspace*{-0mm}
        \begin{tikzpicture}

\definecolor{color1}{RGB}{228,26,28}
\definecolor{color2}{RGB}{55,126,184}
\definecolor{color3}{RGB}{77,175,74}
\definecolor{color4}{RGB}{152,78,163}

\begin{axis}[
width=3.6cm, height=3.4cm,
xlabel={\footnotesize $t$},
ylabel={\footnotesize roll},
xlabel shift=-1mm,
ylabel shift=-1mm,
xmin=0, xmax=0.9,
tick label style={font=\scriptsize},
legend cell align={left},
legend style={
    font=\footnotesize,
    at={(-0.2,1.25)},
    anchor=north west,
    draw=none,
    fill=none,
    row sep=-2pt,
    inner sep=0pt,
},
legend image post style={scale=0.4},
]

\pgfplotstableread[row sep=crcr]{
t     Xo_nominal  Xo_true     Xo_lb       Xo_ub      \\
0.00  0.00000000  0.00000000  0.00000000  0.00000000 \\
0.02  0.14060123  0.07547210  0.17772230  0.04133807 \\
0.04  0.65062964  0.44805896  0.88616793  0.37182721 \\
0.06  0.83067282  0.68080950  1.23926773  0.38613393 \\
0.08  0.90062293  0.85158977  1.39965069  0.38183159 \\
0.10  0.92865189  0.95379366  1.10245515  0.74815941 \\
0.12  0.97753484  0.98733145  1.17088581  0.78206496 \\
0.14  1.04390466  1.06680235  1.40806992  0.67409528 \\
0.16  1.07213536  1.15193567  1.36551077  0.77657533 \\
0.18  1.07867366  1.22943037  1.75719874  0.39594387 \\
0.20  1.01889188  1.28270038  1.47755816  0.55513011 \\
0.22  1.07160986  1.32204958  1.91410285  0.22584301 \\
0.24  1.08342366  1.37190605  2.61716539 -0.45507043 \\
0.26  1.07480736  1.41373890  1.77875187  0.36994198 \\
0.28  1.34209965  1.43049563  2.00467455  0.68194973 \\
0.30  1.33976318  1.44243543  2.52525483  0.15259676 \\
0.32  1.35597211  1.45919685  1.75658555  0.95809204 \\
0.34  1.35279622  1.46983520  1.63918290  1.07347418 \\
0.36  1.62547854  1.48280300  1.80679090  1.45400969 \\
0.38  1.92890173  1.49027402  3.02637873  0.89432892 \\
0.40  1.96686726  1.50140073  2.26821989  1.64778862 \\
0.42  1.85493530  1.52226785  3.14143680  0.52366528 \\
0.44  1.84648892  1.54810868  3.63360025 -0.37449356 \\
0.46  1.80746318  1.55461533  2.78971650  0.25872717 \\
0.48  1.62925741  1.56630397  2.28700315  0.47724573 \\
0.50  1.61310156  1.57313092  3.00386153 -0.45564997 \\
0.52  1.83396582  1.57212247  3.22869155 -0.15925868 \\
0.54  1.88768324  1.57163298  3.44727800 -0.31325232 \\
0.56  1.67527220  1.57147560  2.97510376 -0.14679914 \\
0.58  1.63369162  1.57162231  2.71443850  0.03355537 \\
0.60  1.19377735  1.57202127  2.16058389  0.15355510 \\
0.62  1.31999166  1.57233277  2.36929954  0.21887068 \\
0.64  1.40137765  1.57277236  2.54620929  0.23664063 \\
0.66  1.43925366  1.57377805  2.79112618  0.34111546 \\
0.68  1.47333071  1.57503446  2.35548201  0.37476200 \\
0.70  1.49547882  1.57887186  2.56285336  0.15824101 \\
0.72  0.65643813  1.58147918  1.94802826 -1.32136593 \\
0.74  0.66861540  1.58702055  1.55617819 -0.68674913 \\
0.76  0.67719286  1.59031591  1.42152425 -0.78916986 \\
0.78  1.50814536  1.60138806  2.46395214 -0.72759415 \\
0.80  1.54191348  1.61526754  2.64291844 -1.20488800 \\
0.82  2.89627657  1.62581555  3.89548739  0.18889894 \\
0.84  2.74960846  1.62905720  3.53847158  0.50110371 \\
0.86  2.75957043  1.63031245  3.75403174 -0.41143078 \\
0.88  2.76764233  1.63297406  3.92314693 -0.47794137 \\
0.90  2.78042029  1.63481436  4.14776057 -0.23473048 \\
0.92  2.79216372  1.63598409  4.41216820  0.13307342 \\
0.94  2.80366701  1.63237315  4.80682477  0.47380592 \\
0.96  2.81405264  1.62899023  5.36243534  0.63686355 \\
0.98  2.82236182  1.58447749  6.14352666  0.63198157 \\
1.00  2.82850156  1.57749284  7.22230173  0.41544104 \\
}\datatable

\addplot[semithick, densely dashed, forget plot] {1.5707963};

\addplot[thick, color2]
    table[x=t, y=Xo_nominal]{\datatable};
\addlegendentry{Nominal trajectory}

\addplot[thick, color1]
    table[x=t, y=Xo_true]{\datatable};

\addplot[name path=lower, draw=none, forget plot]
    table[x=t, y=Xo_lb]{\datatable};

\addplot[name path=upper, draw=none, forget plot]
    table[x=t, y=Xo_ub]{\datatable};

\addplot[fill=color3, fill opacity=0.4, draw=none, legend image code/.code={
        \draw[fill=color3, fill opacity=0.4, draw=none] (0cm,-0.1cm) rectangle (0.6cm,0.1cm);
    }]
    fill between[of=lower and upper];

\end{axis}

\end{tikzpicture} \hspace*{-4mm}
        \begin{tikzpicture}

\definecolor{color1}{RGB}{228,26,28}
\definecolor{color2}{RGB}{55,126,184}
\definecolor{color3}{RGB}{77,175,74}
\definecolor{color4}{RGB}{152,78,163}

\begin{axis}[
width=3.6cm, height=3.4cm,
xlabel={\footnotesize $t$},
ylabel={\footnotesize pitch},
xlabel shift=-1mm,
ylabel shift=-3mm,
xmin=0, xmax=0.9,
tick label style={font=\scriptsize},
legend cell align={left},
legend style={
    font=\footnotesize,
    at={(-0.2,1.25)},
    anchor=north west,
    draw=none,
    fill=none,
    row sep=-2pt,
    inner sep=1pt,
},
legend image post style={scale=0.4},
]

\pgfplotstableread[row sep=crcr]{
t     Xo_nominal  Xo_true     Xo_lb       Xo_ub      \\
0.00  0.00000000  0.00000000  0.00000000  0.00000000 \\
0.02 -0.17673859 -0.07555672  0.00069633 -0.27283342 \\
0.04 -0.12343371 -0.13895533  0.39139666 -0.65010642 \\
0.06 -0.07571825 -0.03613510  0.47454709 -0.60858044 \\
0.08 -0.01410285  0.04818513  0.80554233 -0.82214961 \\
0.10 -0.01360059  0.06547300  0.19269577 -0.22117600 \\
0.12  0.13463394  0.03152827  0.54793796 -0.27677067 \\
0.14 -0.02578607  0.06835325  0.34938023 -0.40441394 \\
0.16 -0.02727212  0.08397178  0.57139956 -0.62559324 \\
0.18 -0.00489767  0.06053563  0.93787418 -0.94844242 \\
0.20 -0.09522814  0.05858721  0.68016233 -0.86881501 \\
0.22 -0.04167032  0.06630784  1.40477717 -1.48814993 \\
0.24 -0.00279069  0.04661688  0.84586287 -0.84850128 \\
0.26 -0.04433673  0.02126394  0.76129771 -0.85601100 \\
0.28  0.16005472  0.01649361  1.71221036 -1.39160734 \\
0.30  0.14064529  0.02239773  1.77977337 -1.48835493 \\
0.32  0.08370122  0.02268423  0.56455163 -0.40185922 \\
0.34  0.03919357  0.02987493  0.18977386 -0.11693588 \\
0.36  0.00285570  0.03778376  0.09604175 -0.09831604 \\
0.38  0.09391031  0.03805400  0.54528245 -0.42033080 \\
0.40 -0.17138019  0.03586421  0.95099439 -1.43044466 \\
0.42 -0.00721697  0.03004341  1.88141582 -2.04035530 \\
0.44  0.01246474  0.02576658  0.49049735 -0.60758168 \\
0.46 -0.23713942  0.02362877  0.42647481 -1.36040874 \\
0.48  0.02412309  0.01819250  1.21650873 -0.57551463 \\
0.50 -0.27904749  0.01330681  0.07395374 -0.73198026 \\
0.52  0.27835396  0.01534056  3.21740828 -1.77936698 \\
0.54  0.15707978  0.01830315  1.93825284 -1.43052398 \\
0.56  0.15589432  0.01986549  0.73076291 -0.68215564 \\
0.58  0.12817014  0.02013379  0.92213687 -1.36927202 \\
0.60  0.10271838  0.01999717  0.60476406 -1.10730411 \\
0.62  0.10062705  0.02061436  0.81231236 -0.88401360 \\
0.64  0.12204143  0.02052978  0.43443444 -0.92962961 \\
0.66  0.14979791  0.02012909  0.39077002 -1.06637717 \\
0.68  0.18698450  0.01865575  0.60986088 -1.74184708 \\
0.70  0.22942526  0.01435241  0.81006971 -2.06335015 \\
0.72  0.19619757  0.00934729  0.83371787 -1.96316839 \\
0.74  0.19930371  0.00104489  1.70920486 -2.43672544 \\
0.76  0.20661990 -0.00598022  1.69841415 -2.67719959 \\
0.78  0.15261183 -0.01428662  0.44576902 -0.63886671 \\
0.80  0.14684707 -0.02431005  0.46447650 -0.65234348 \\
0.82  0.17581969 -0.03229417  1.58333194 -0.67719643 \\
0.84  0.25822608 -0.03632358  1.28912422 -0.69421987 \\
0.86  0.24239820 -0.03770874  3.16944471 -1.28751048 \\
0.88  0.23195662 -0.04044846  1.67250419 -1.46952213 \\
0.90  0.21577627 -0.03872601  1.47028357 -2.66046009 \\
0.92  0.19788196 -0.03636060  1.49272060 -3.71647926 \\
0.94  0.17793139 -0.02834211  1.54712049 -4.37305848 \\
0.96  0.15626182 -0.01870050  1.60923665 -4.81026955 \\
0.98  0.13306360  0.01347806  1.62195867 -4.68354898 \\
1.00  0.10870430  0.01319885  1.60197664 -4.54210136 \\
}\datatable

\addplot[semithick, densely dashed, forget plot] {0};

\addplot[thick, color2, forget plot]
    table[x=t, y=Xo_nominal]{\datatable};

\addplot[thick, color1]
    table[x=t, y=Xo_true]{\datatable};
\addlegendentry{Closed-loop rollout}

\addplot[name path=lower, draw=none, forget plot]
    table[x=t, y=Xo_lb]{\datatable};

\addplot[name path=upper, draw=none, forget plot]
    table[x=t, y=Xo_ub]{\datatable};

\addplot[fill=color3, fill opacity=0.4, draw=none, legend image code/.code={
        \draw[fill=color3, fill opacity=0.4, draw=none] (0cm,-0.1cm) rectangle (0.6cm,0.1cm);
    }]
    fill between[of=lower and upper];

\end{axis}

\end{tikzpicture} \hspace*{-5mm}
        \begin{tikzpicture}

\definecolor{color1}{RGB}{228,26,28}
\definecolor{color2}{RGB}{55,126,184}
\definecolor{color3}{RGB}{77,175,74}
\definecolor{color4}{RGB}{152,78,163}

\begin{axis}[
width=3.6cm, height=3.4cm,
xlabel={\footnotesize $t$},
ylabel={\footnotesize yaw},
xlabel shift=-1mm,
ylabel shift=-3mm,
xmin=0, xmax=0.9,
tick label style={font=\scriptsize},
legend cell align={left},
legend style={
    font=\footnotesize,
    at={(-0.15,1.25)},
    anchor=north west,
    draw=none,
    fill=none,
    row sep=-2pt,
    inner sep=1pt,
},
legend image post style={scale=0.4},
]

\pgfplotstableread[row sep=crcr]{
t     Xo_nominal  Xo_true     Xo_lb       Xo_ub      \\
0.00  0.00000000  0.00000000  0.00000000  0.00000000 \\
0.02 -0.10817828 -0.01499433  0.04676520 -0.18066278 \\
0.04 -0.17132584 -0.17740796  0.26729981 -0.51791085 \\
0.06 -0.27339385 -0.27641002  0.58690364 -1.08206302 \\
0.08 -0.24203474 -0.30261023  0.48742983 -0.94974324 \\
0.10 -0.16798809 -0.27627989  0.13132231 -0.46142155 \\
0.12 -0.17939358 -0.21316796  0.50552840 -0.86192891 \\
0.14 -0.11122732 -0.19340900  0.56009961 -0.78014398 \\
0.16 -0.11966048 -0.15669289  1.82297747 -2.05700194 \\
0.18 -0.09494648 -0.10701526  1.03852957 -1.22547225 \\
0.20 -0.20700705 -0.09878132  2.31032382 -2.72043699 \\
0.22 -0.16567668 -0.08864714  1.28099572 -1.60905285 \\
0.24 -0.01174094 -0.07981696  3.28324263 -3.30038287 \\
0.26 -0.07399372 -0.07045262  1.48762125 -1.62420367 \\
0.28 -0.28613152 -0.05916727  1.44316392 -2.00092314 \\
0.30 -0.00258191 -0.04507064  1.99166935 -1.99004957 \\
0.32 -0.06998532 -0.03484730  0.68250812 -0.81505177 \\
0.34 -0.00018897 -0.02288134  0.28225655 -0.27225027 \\
0.36  0.06873224 -0.01011155  0.28902247 -0.13651467 \\
0.38  0.05755787 -0.00332652  1.42766944 -1.20444911 \\
0.40 -0.07839896  0.00129572  0.82772358 -0.89162020 \\
0.42  0.22153405  0.00858698  1.88745557 -1.34079773 \\
0.44  0.08880576  0.01163250  1.79991183 -1.76824950 \\
0.46 -0.02452208  0.01896040  0.87035402 -1.04765598 \\
0.48  0.21732713  0.02833235  1.13454949 -0.55557544 \\
0.50 -0.33802138  0.03777880  0.21361063 -1.13343324 \\
0.52 -0.37944638  0.03603383  1.37573619 -1.31056760 \\
0.54 -0.51426432  0.03077728  2.36618567 -2.59869326 \\
0.56 -0.37454927  0.02937382  0.02747074 -0.89124265 \\
0.58 -0.58577951  0.02942998  1.10579805 -1.40139823 \\
0.60 -0.30993428  0.03074575  1.18706432 -1.17253471 \\
0.62 -0.35780641  0.03024507  0.11054416 -0.93090657 \\
0.64 -0.40080746  0.03029028  0.06421545 -0.59818336 \\
0.66 -0.48542480  0.03357000  0.20584439 -0.85397261 \\
0.68 -0.57900538  0.03552771  1.04168277 -1.28741547 \\
0.70 -0.70553494  0.04509363  0.91968638 -1.46746096 \\
0.72 -0.92020966  0.05275558  0.69473205 -1.89284891 \\
0.74 -0.89245444  0.06742354  0.28853331 -1.72224443 \\
0.76 -0.85187577  0.07642943  0.73072974 -1.91214355 \\
0.78 -0.08556569  0.08268028  0.68895029 -1.99257388 \\
0.80 -0.04881128  0.08825656  1.36270102 -3.27373582 \\
0.82  1.14642283  0.09212728  1.93044264 -0.46780876 \\
0.84  1.23994257  0.09448137  1.98094716 -0.07464982 \\
0.86  1.22986310  0.09571581  2.48171794 -1.41264135 \\
0.88  1.23229944  0.09766845  3.12849753 -1.75757019 \\
0.90  1.21584862  0.09803564  3.49413365 -1.95630639 \\
0.92  1.19775250  0.09688021  3.24252996 -1.64262430 \\
0.94  1.17719498  0.09302264  2.87995617 -1.12298623 \\
0.96  1.15593367  0.08071953  2.69748168 -0.64388667 \\
0.98  1.13533638  0.06526674  1.86419422  0.01292364 \\
1.00  1.11602062  0.04495582  2.00774396  0.34275365 \\
}\datatable

\addplot[semithick, densely dashed, forget plot] {0};

\addplot[thick, color2, forget plot]
    table[x=t, y=Xo_nominal]{\datatable};

\addplot[thick, color1, forget plot]
    table[x=t, y=Xo_true]{\datatable};

\addplot[name path=lower, draw=none, forget plot]
    table[x=t, y=Xo_lb]{\datatable};

\addplot[name path=upper, draw=none, forget plot]
    table[x=t, y=Xo_ub]{\datatable};

\addplot[fill=color3, fill opacity=0.4, draw=none, legend image code/.code={
        \draw[fill=color3, fill opacity=0.4, draw=none] (0cm,-0.1cm) rectangle (0.6cm,0.1cm);
    }]
    fill between[of=lower and upper];
\addlegendentry{Predicted tube}

\end{axis}

\end{tikzpicture} \\[-4pt]    
        \hspace*{-2.5mm}
        \begin{tikzpicture}

\definecolor{color1}{RGB}{228,26,28}
\definecolor{color2}{RGB}{55,126,184}
\definecolor{color3}{RGB}{77,175,74}
\definecolor{color4}{RGB}{152,78,163}

\begin{axis}[
width=3.6cm, height=3.4cm,
xlabel={\footnotesize $t$},
ylabel={\footnotesize $x$},
xlabel shift=-1mm,
ylabel shift=-2.2mm,
xmin=0, xmax=0.9,
tick label style={font=\scriptsize},
legend cell align={left},
legend style={
    font=\footnotesize,
    at={(-0.3,1.25)},
    anchor=north west,
    draw=none,
    fill=none,
    row sep=-2pt,
    inner sep=1pt,
},
legend image post style={scale=0.6},
]

\pgfplotstableread[row sep=crcr]{
t     Xo_nominal  Xo_true     Xo_lb       Xo_ub      \\
0.00  0.04100000  0.04100000  0.04100000  0.04100000 \\
0.02  0.03693828  0.04187333  0.03514517  0.04158276 \\
0.04  0.03561444  0.04616793  0.02574170  0.04685533 \\
0.06  0.03305542  0.04753385 -0.00073410  0.06696183 \\
0.08  0.03390155  0.04799977 -0.01553147  0.08291399 \\
0.10  0.03618252  0.04739835  0.02464794  0.04769451 \\
0.12  0.03455442  0.04554131  0.02375970  0.04522441 \\
0.14  0.03714824  0.04527863  0.01223539  0.06196972 \\
0.16  0.03684299  0.04474776  0.01735269  0.05631461 \\
0.18  0.03711347  0.04397240 -0.00058666  0.07472924 \\
0.20  0.03405190  0.04409300  0.01997072  0.04808490 \\
0.22  0.03417779  0.04414406 -0.00519131  0.07347556 \\
0.24  0.04007088  0.04400688  0.02206606  0.05810049 \\
0.26  0.03730059  0.04366884 -0.01723096  0.09201942 \\
0.28  0.03194257  0.04331688 -0.03024933  0.09443315 \\
0.30  0.03419387  0.04303157 -0.04445493  0.11298762 \\
0.32  0.03485363  0.04281748  0.02600804  0.04387059 \\
0.34  0.03521318  0.04263785  0.02838713  0.04231659 \\
0.36  0.03556779  0.04240666  0.03274590  0.03880881 \\
0.38  0.03440568  0.04228941  0.02109332  0.05098404 \\
0.40  0.03840909  0.04226796  0.02496569  0.05649907 \\
0.42  0.03943270  0.04241506 -0.00185574  0.08659217 \\
0.44  0.03821961  0.04248674 -0.05049727  0.14445419 \\
0.46  0.03893090  0.04244057  0.02188431  0.07254408 \\
0.48  0.03971818  0.04254639  0.01954655  0.06024077 \\
0.50  0.03304687  0.04261960  0.01979205  0.05127673 \\
0.52  0.00317186  0.04255768 -0.06704718  0.05718152 \\
0.54  0.01803530  0.04245330 -0.03059133  0.05795200 \\
0.56  0.02851064  0.04242745  0.02034602  0.05106500 \\
0.58  0.03686077  0.04241376  0.01668047  0.05265277 \\
0.60  0.04162354  0.04242638  0.02488946  0.05812284 \\
0.62  0.04199679  0.04241336  0.03177324  0.05426495 \\
0.64  0.04150008  0.04240507  0.04038544  0.05401659 \\
0.66  0.04099268  0.04247254  0.00302476  0.05651645 \\
0.68  0.04059524  0.04248698  0.01062529  0.05869749 \\
0.70  0.04106428  0.04263481  0.00549860  0.06560599 \\
0.72  0.00705523  0.04271585 -0.03440790  0.03638002 \\
0.74  0.02671380  0.04291003 -0.01978558  0.05210568 \\
0.76  0.04652982  0.04301921 -0.00715700  0.07956035 \\
0.78  0.06451553  0.04322462 -0.02493176  0.07802147 \\
0.80  0.08445259  0.04343698 -0.00187521  0.12947556 \\
0.82  0.10073571  0.04359656  0.00533171  0.16242939 \\
0.84  0.11833051  0.04365844  0.01341968  0.19905016 \\
0.86  0.13849250  0.04368695  0.02029488  0.24037823 \\
0.88  0.15863814  0.04372954  0.03157957  0.22206914 \\
0.90  0.17882815  0.04372753  0.00579159  0.27132344 \\
0.92  0.19905157  0.04368845  0.03948580  0.34043027 \\
0.94  0.21929496  0.04359764  0.01594993  0.39838884 \\
0.96  0.23956585  0.04328102  0.01764139  0.43485114 \\
0.98  0.25987701  0.04308487  0.01596549  0.47045077 \\
1.00  0.28022362  0.04270601  0.03397032  0.49321571 \\
}\datatable

\addplot[semithick, densely dashed, forget plot] {0.041};

\addplot[thick, color2]
    table[x=t, y=Xo_nominal]{\datatable};

\addplot[thick, color1]
    table[x=t, y=Xo_true]{\datatable};

\addplot[name path=lower, draw=none, forget plot]
    table[x=t, y=Xo_lb]{\datatable};

\addplot[name path=upper, draw=none, forget plot]
    table[x=t, y=Xo_ub]{\datatable};

\addplot[fill=color3, fill opacity=0.4, draw=none, legend image code/.code={
        \draw[fill=color3, fill opacity=0.4, draw=none] (0cm,-0.1cm) rectangle (0.6cm,0.1cm);
    }]
    fill between[of=lower and upper];

\end{axis}

\end{tikzpicture} \hspace*{-4.5mm}
        \begin{tikzpicture}

\definecolor{color1}{RGB}{228,26,28}
\definecolor{color2}{RGB}{55,126,184}
\definecolor{color3}{RGB}{77,175,74}
\definecolor{color4}{RGB}{152,78,163}

\begin{axis}[
width=3.6cm, height=3.4cm,
xlabel={\footnotesize $t$},
ylabel={\footnotesize $y$},
xlabel shift=-1mm,
ylabel shift=-3mm,
xmin=0, xmax=0.9,
tick label style={font=\scriptsize},
legend cell align={left},
legend style={
    font=\footnotesize,
    at={(-0.3,1.25)},
    anchor=north west,
    draw=none,
    fill=none,
    row sep=-2pt,
    inner sep=1pt,
},
legend image post style={scale=0.6},
]

\pgfplotstableread[row sep=crcr]{
t     Xo_nominal  Xo_true     Xo_lb       Xo_ub      \\
0.00  0.00000000  0.00000000  0.00000000  0.00000000 \\
0.02 -0.00323268 -0.00000703 -0.00061448 -0.00340577 \\
0.04  0.00462748  0.00397249  0.01434583 -0.00085136 \\
0.06  0.00228971  0.00377257  0.02362969 -0.01558163 \\
0.08  0.00385135  0.00351811  0.03133717 -0.02133902 \\
0.10  0.00672095  0.00418415  0.01340421  0.00054684 \\
0.12  0.00744910  0.00638949  0.02609350 -0.01071418 \\
0.14  0.00993129  0.00623101  0.03840018 -0.01803226 \\
0.16  0.00968843  0.00526960  0.03143849 -0.01147367 \\
0.18  0.01162163  0.00351433  0.06276617 -0.03900687 \\
0.20  0.00868746  0.00253968  0.03865937 -0.02074705 \\
0.22  0.01394256  0.00174012  0.07306833 -0.04469826 \\
0.24  0.01922140  0.00154757  0.05687087 -0.01784881 \\
0.26  0.02271229  0.00161943  0.11479735 -0.06869227 \\
0.28  0.01574403  0.00183492  0.12428634 -0.09275457 \\
0.30  0.01570198  0.00211798  0.06020050 -0.02901740 \\
0.32  0.01551145  0.00247451  0.03629758 -0.00531309 \\
0.34  0.01556642  0.00298618  0.03684387 -0.00547191 \\
0.36  0.02005880  0.00319767  0.03665931  0.00386176 \\
0.38  0.01226893  0.00352669  0.06343697 -0.03627376 \\
0.40  0.01105143  0.00393527  0.07230850 -0.04842782 \\
0.42  0.01214977  0.00459856  0.07913278 -0.05211683 \\
0.44  0.01329404  0.00514179  0.09871224 -0.06679744 \\
0.46  0.01427295  0.00555417  0.09330804 -0.04295125 \\
0.48  0.01550346  0.00604610  0.10701009 -0.04351744 \\
0.50 -0.00288642  0.00639295  0.08744956 -0.05498341 \\
0.52 -0.04531974  0.00640137  0.03053753 -0.08947872 \\
0.54 -0.04491099  0.00628359  0.04778548 -0.09857078 \\
0.56 -0.04213525  0.00625880  0.04273205 -0.08419790 \\
0.58 -0.04614846  0.00626873  0.03611138 -0.08384699 \\
0.60 -0.04958286  0.00629890  0.04409007 -0.08748351 \\
0.62 -0.05363601  0.00627921  0.05022809 -0.09445739 \\
0.64 -0.05707481  0.00630116  0.05464051 -0.09603210 \\
0.66 -0.06072198  0.00637491  0.04919656 -0.09872104 \\
0.68 -0.06475547  0.00646048  0.05727312 -0.10373875 \\
0.70 -0.06958267  0.00669220  0.06845373 -0.10832636 \\
0.72 -0.04942891  0.00692152  0.08314631 -0.08200537 \\
0.74 -0.05088531  0.00728996  0.07783557 -0.07983818 \\
0.76 -0.05234783  0.00761097  0.02151046 -0.08622545 \\
0.78 -0.01635670  0.00811449  0.07494067 -0.06015217 \\
0.80 -0.01730264  0.00864582  0.07211626 -0.05677799 \\
0.82 -0.04753707  0.00907320  0.02316729 -0.08008394 \\
0.84 -0.05710513  0.00920022  0.00925041 -0.09404106 \\
0.86 -0.05808090  0.00921111  0.02088245 -0.10837075 \\
0.88 -0.05908324  0.00928061  0.01330300 -0.13396691 \\
0.90 -0.05989941  0.00929643  0.01427021 -0.15427943 \\
0.92 -0.06065613  0.00924465  0.00782885 -0.16069732 \\
0.94 -0.06138542  0.00895035 -0.00717129 -0.15698917 \\
0.96 -0.06208661  0.00838245 -0.01349490 -0.15711840 \\
0.98 -0.06275565  0.00643407 -0.01526299 -0.16233213 \\
1.00 -0.06338828  0.00496316 -0.01245295 -0.16698220 \\
}\datatable

\addplot[semithick, densely dashed, forget plot] {0.0};

\addplot[thick, color2]
    table[x=t, y=Xo_nominal]{\datatable};

\addplot[thick, color1]
    table[x=t, y=Xo_true]{\datatable};

\addplot[name path=lower, draw=none, forget plot]
    table[x=t, y=Xo_lb]{\datatable};

\addplot[name path=upper, draw=none, forget plot]
    table[x=t, y=Xo_ub]{\datatable};

\addplot[fill=color3, fill opacity=0.4, draw=none, legend image code/.code={
        \draw[fill=color3, fill opacity=0.4, draw=none] (0cm,-0.1cm) rectangle (0.6cm,0.1cm);
    }]
    fill between[of=lower and upper];

\end{axis}

\end{tikzpicture} \hspace*{-3.3mm}
        \begin{tikzpicture}

\definecolor{color1}{RGB}{228,26,28}
\definecolor{color2}{RGB}{55,126,184}
\definecolor{color3}{RGB}{77,175,74}
\definecolor{color4}{RGB}{152,78,163}

\begin{axis}[
width=3.6cm, height=3.4cm,
xlabel={\footnotesize $t$},
ylabel={\footnotesize $z$},
xlabel shift=-1mm,
ylabel shift=-2mm,
xmin=0, xmax=0.9,
tick label style={font=\scriptsize},
legend cell align={left},
legend style={
    font=\footnotesize,
    at={(-0.3,1.25)},
    anchor=north west,
    draw=none,
    fill=none,
    row sep=-2pt,
    inner sep=1pt,
},
legend image post style={scale=0.6},
]

\pgfplotstableread[row sep=crcr]{
t     Xo_nominal  Xo_true     Xo_lb       Xo_ub      \\
0.00  0.06000000  0.06000000  0.06000000  0.06000000 \\
0.02  0.06426057  0.06384628  0.06794977  0.06060915 \\
0.04  0.06764077  0.07365314  0.08081080  0.05633747 \\
0.06  0.07293758  0.07848514  0.08776622  0.05932455 \\
0.08  0.07351115  0.08241007  0.11363856  0.03554495 \\
0.10  0.07446413  0.08572214  0.08038218  0.06899459 \\
0.12  0.06990287  0.08793386  0.08829608  0.05202567 \\
0.14  0.07743542  0.09027706  0.09415606  0.06127463 \\
0.16  0.07869928  0.09181727  0.09378790  0.06419294 \\
0.18  0.07941490  0.09206469  0.11014119  0.04923683 \\
0.20  0.07990116  0.09210852  0.10866513  0.05158374 \\
0.22  0.07852040  0.09216346  0.11608508  0.04153054 \\
0.24  0.07785621  0.09219871  0.11449346  0.04164044 \\
0.26  0.08085587  0.09207634  0.12587697  0.03663826 \\
0.28  0.08195505  0.09215969  0.11682693  0.04765883 \\
0.30  0.08172474  0.09217053  0.12907428  0.03467355 \\
0.32  0.08250971  0.09220275  0.11937459  0.04584995 \\
0.34  0.08197323  0.09223887  0.11213290  0.05203448 \\
0.36  0.07740634  0.09254592  0.10445144  0.05070290 \\
0.38  0.06258526  0.09277196  0.18063159 -0.05308506 \\
0.40  0.06000154  0.09291614  0.15597993 -0.03228397 \\
0.42  0.05702827  0.09277685  0.12732265 -0.00806621 \\
0.44  0.05417733  0.09294961  0.19452454 -0.06648862 \\
0.46  0.05170072  0.09298597  0.20146261 -0.06204280 \\
0.48  0.04903541  0.09294735  0.21126517 -0.05880344 \\
0.50  0.04829272  0.09297285  0.20326533 -0.04808730 \\
0.52  0.04373369  0.09293192  0.16741690 -0.02661752 \\
0.54  0.03910593  0.09284544  0.13814673 -0.01040117 \\
0.56  0.03195278  0.09279693  0.08300540  0.02169556 \\
0.58  0.02831919  0.09277827  0.09595153  0.01316152 \\
0.60  0.02481027  0.09277068  0.10398974  0.01027389 \\
0.62  0.02005237  0.09274001  0.10645217  0.00664057 \\
0.64  0.01639478  0.09271033  0.10934109  0.00826255 \\
0.66  0.01338508  0.09269325  0.11132624  0.00563446 \\
0.68  0.01068570  0.09269910  0.11732336  0.00469863 \\
0.70  0.00851831  0.09271226  0.13079746  0.00415655 \\
0.72  0.00833982  0.09278043  0.14800280  0.00444196 \\
0.74  0.00771188  0.09285763  0.15062745  0.00330626 \\
0.76  0.00710481  0.09296987  0.11133093 -0.00979886 \\
0.78 -0.01429637  0.09295218  0.11044307 -0.04661615 \\
0.80 -0.01461700  0.09293143  0.12211633 -0.04696428 \\
0.82 -0.02083762  0.09293457  0.11233038 -0.04666084 \\
0.84 -0.02611259  0.09298074  0.09944276 -0.04405309 \\
0.86 -0.02648884  0.09301778  0.10274391 -0.04021353 \\
0.88 -0.02689019  0.09302101  0.09695639 -0.04269527 \\
0.90 -0.02725520  0.09291473  0.11851224 -0.05238348 \\
0.92 -0.02759432  0.09281710  0.14267476 -0.05842756 \\
0.94 -0.02791704  0.09271560  0.16169896 -0.06184803 \\
0.96 -0.02822256  0.09261468  0.17434112 -0.06362944 \\
0.98 -0.02851267  0.09321164  0.18796222 -0.06577452 \\
1.00 -0.02878764  0.09340040  0.20155368 -0.06868265 \\
}\datatable

\addplot[semithick, densely dashed, forget plot] {0.06};

\addplot[thick, color2]
    table[x=t, y=Xo_nominal]{\datatable};

\addplot[thick, color1]
    table[x=t, y=Xo_true]{\datatable};

\addplot[name path=lower, draw=none, forget plot]
    table[x=t, y=Xo_lb]{\datatable};

\addplot[name path=upper, draw=none, forget plot]
    table[x=t, y=Xo_ub]{\datatable};

\addplot[fill=color3, fill opacity=0.4, draw=none, legend image code/.code={
        \draw[fill=color3, fill opacity=0.4, draw=none] (0cm,-0.1cm) rectangle (0.6cm,0.1cm);
    }]
    fill between[of=lower and upper];

\end{axis}

\end{tikzpicture}
      }
    }
    \caption{
        We evaluate on an in-hand cube reorientation task. Our method produces an affine feedback policy that respects the unilateral nature of contact and can be executed online.
        \textbf{(a)} Keyframes from the executed rollout.
        \textbf{(b)} The closed-loop trajectory remains within the reachable tube for all object degrees of freedom, thereby certifying constraint satisfaction.
    }
    \label{fig:inhand-cube-rollout}
\end{figure}

From the \ac{RL} policy gradient formula \cite{sutton1999policy} with policy $\pi_\theta$,
\begin{equation*}
    \setlength{\abovedisplayskip}{2pt}
    \setlength{\belowdisplayskip}{2pt}
    \nabla_\theta J(\theta) := \mathbb{E}_{\tau\sim\pi_\theta} \left[ \nabla_\theta \log p_\theta(\tau)\, R(\tau) \right],
\end{equation*}
it might appear that the policy gradient $\nabla_\theta J(\theta)$ would also vanish, as most actions do not affect the manipulated object and yield zero episodic reward $R(\tau)$ for a sampled trajectory $\tau$. 
\ac{RL} addresses this through stochasticity in the policy and environment, often via domain randomization \cite{peng2018simtoreal, rajeswaran2017epopt}, so that some rollouts accidentally induce contact, producing nonzero rewards and gradients. Domain randomization can be seen as randomized smoothing, where the dynamics are convolved with a probabilistic kernel to improve the optimization landscape \cite{suh2022bundled, duchi2012randomized}. Later work showed that this is equivalent to analytical smoothing \cite{pang2023global, zhang2023adaptive}, where the environment is deterministic but contact dynamics are smoothed explicitly.

Despite the connection between \ac{RL} and analytically-smoothed contact dynamics, na\"ively planning with smoothed dynamics often fails due to premature contact loss \cite{shirai2025is} arising from the unilateral nature of contact, i.e., contact can only apply normal forces in one direction. Later work mitigates this issue by creating a local \ac{CTR} \cite{suh2025dexterous}, replacing conventional elliptical trust regions \cite{sorensen1982newtons} to keep smoothed dynamics close to the nonsmooth hybrid system. However, contact loss can still occur, and the CTR is heuristic, as the allowable deviation that ensures a sufficiently accurate model for planning is not formally characterized. Quantifying and propagating smoothing error through the dynamics produces reachable sets that can constrain controllers to ensure goal reachability and constraint satisfaction. However, no methods can quantify smoothing error, perform reachability analysis for high-dimensional hybrid contact dynamics, and perform efficient gradient-based policy optimization using these bounds.

To close these gaps, we propose a scalable \textit{robust manipulation planner} that tightly bounds smoothing error, propagates it through the dynamics, and leverages the resulting reachable tubes to guarantee task constraint satisfaction and goal reaching. A key ingredient is a differentiable simulator that smooths contact dynamics and contact geometry in a modular manner via convex conic optimization. This simulator provides informative gradients for planning, as well as gradients with respect to the smoothing parameter $\kappa$, which are required to compute smoothing-error bounds. 
\textit{Our key insight is that the deviation between the smoothed dynamics and the nonsmooth hybrid dynamics can be treated as a structured model mismatch, which we show can be tightly bounded by a set-valued, state-control-dependent function.} Building on recent advances in set-valued robust control \cite{leeman2025robust}, we jointly optimize a nominal trajectory and an affine feedback policy that admit analytical predictions of the nonsmooth hybrid dynamics behavior under feedback control. We compare with \ac{CTR} in \cref{sec:baseline-ctr}, obtaining similar solve times while reducing constraint violations and certifying constraint satisfaction on \textit{the nonsmooth hybrid system}.
Our contributions are:
\begin{enumerate}[leftmargin=*]
\item We present a differentiable simulator with contact dynamics and contact geometry smoothing building on modular differentiable convex conic programs, providing accurate gradients for planning, without differentiating through the entire computation graph.
\item We derive tight analytical bounds on the deviation between smoothed dynamics and nonsmooth hybrid dynamics, which depend on gradients provided by our differentiable simulator.
\item We propose a robust policy synthesis method that uses gradients from the smoothed dynamics for efficient optimization while using tube-valued predictions of the nonsmooth hybrid dynamics under feedback control to ensure robust constraint satisfaction.
\item \looseness-1We show the efficacy of our method on contact-rich manipulation problems both in simulation and on hardware, including planar pushing (2DOF), bimanual non-prehensile manipulation (9DOF), and dexterous in-hand reorientation (22DOF), reducing constraint violation rate relative to baselines.
\end{enumerate}

\section{Related Work}
Differentiable simulators have been used to compute policy gradients via backpropagation through time \cite{mora2021pods, qiao2021efficient, xu2022accelerated, georgiev2023adaptive}. However, these approaches often suffer from vanishing or exploding gradients caused by repeatedly chaining derivatives over long horizons.

Trajectory optimization offers an alternative by avoiding long-horizon gradient chaining and instead using derivatives of the system dynamics to iteratively refine trajectories. For contact-rich systems, contact dynamics are often modeled using the Signorini condition, Coulomb’s law, and the maximum dissipation principle, yielding \ac{NCP} embedded within the trajectory optimization problem \cite{le2024fast, kim2025contact}.  However, solution to \ac{NCP} contact dynamics can be multi-valued \cite{lidec2024contact}, causing the dynamics to be inconsistent with its local linear approximations and undermining gradient reliability.

Convex formulations of contact dynamics address this issue by ensuring single-valued solutions, which ensures consistency of the gradients for trajectory optimization \cite{suh2025dexterous}. In settings where soft, spring-like contact models are sufficient, these convex formulations can even admit closed-form solutions \cite{jin2025complementarity, kurtz2026inverse}. Regardless of the specific contact model, trajectory optimization is typically performed either via shooting methods, where only controls are decision variables, or via direct transcription, where both controls and states are optimized. In both cases, smoothing of the contact dynamics is typically required to mitigate discontinuous or vanishing gradients.

Another class of methods formulates trajectory optimization as \ac{MPCC}, jointly optimizing controls, states, and contact forces. Early approaches restore constraint qualifications through Scholtes relaxation \cite{scholtes2001convergence}, which effectively smooths the contact complementarity constraints \cite{posa2013direct, manchester2019contact}. More recent work instead avoids such relaxation by using solvers that do not introduce dual variables for the complementarity constraints \cite{li2025surprising, aydinoglu2024consensus}.

Gradient-based trajectory optimization can only guarantee convergence to local minima. Although outside the scope of this work, several approaches aim to address this limitation through global optimization, including graph of convex sets \cite{graesdal2024towards}, polynomial optimization \cite{kang2025global}, mixed-integer programming \cite{hogan2020reactive}, and rapidly exploring random tree \cite{pang2023global}. While these methods can perform global search, they are typically restricted to low-dimensional systems and remain vulnerable to combinatorial growth in the number of possible contact modes.

Finally, trajectory optimization produces an open-loop trajectory and therefore typically requires continual replanning in a receding-horizon fashion to compensate for modeling errors and disturbances. In contrast, our method produces affine feedback policies that respect the unilateral nature of contact, enabling online deployment without repeated replanning.

\section{Preliminaries and Problem Statement}
\label{sec:conic-program}

\noindent\textbf{Conic Programs.\quad} Throughout this paper, we make extensive use of convex conic programs of the form
\begin{equation}  \label{eq:conic-program}
\begin{aligned}
    & \minimize\limits_{x}  && \tfrac{1}{2} x^\top P x + q^\top x  \\
    & \subjectto && s := b - A x,  \quad s \in \mathcal{K},
\end{aligned}
\end{equation}
where $\mathcal{K}$ is a Cartesian product of symmetric cones, i.e., $\mathcal{K} = \mathcal{K}_1 \times \mathcal{K}_2 \times \cdots \times \mathcal{K}_I$.
Each $\mathcal{K}_i$ is either a non-negative cone or a second-order cone.
A variety of numerical solvers can be used to solve \eqref{eq:conic-program}, including \cite{goulart2024clarabel}.

The \ac{KKT} conditions of \eqref{eq:conic-program} are
\vspace{-3pt}
\begin{subequations}\label{eq:conic-program-kkt}
    \noindent
    \begin{minipage}{0.48\linewidth}
    \begin{equation}
    P x + q + A^\top z = 0\vspace{-5pt}
    \end{equation}
    \begin{equation}
    A x + s - b = 0
    \end{equation}
    \end{minipage}
    \hfill
    \begin{minipage}{0.48\linewidth}
    \vspace{7pt}
    \begin{equation}\label{eq:conic-program-kkt-c}
    s \circ z = 0\hspace{2.5em}\vspace{-5pt}
    \end{equation}
    \begin{equation}
    s \in \mathcal{K}, \; z \in \mathcal{K}^*
    \end{equation}
    \end{minipage}
\end{subequations}
\vspace{8pt}\\
\noindent where $\mathcal{K}^*$ denotes the dual cone of $\mathcal{K}$, and $\kappa = 0$ corresponds to the complementarity condition.

Relaxing the complementarity condition \eqref{eq:conic-program-kkt-c} to a nonzero value $\kappa > 0$ yields the \ac{KKT} conditions of the following convex barrier-penalized problem \cite[\S11.2]{boyd2004convex}:
\begin{equation}  \label{eq:conic-logbarrier-program}
\begin{aligned}
    & \minimize_{x}  && \tfrac{1}{2} x^\top P x + q^\top x + \kappa \textstyle\sum_{i=1}^{I} \psi_i(s_i),
\end{aligned}
\end{equation}
where $s_i := (b - A x)_i$ and $\psi_i(\cdot)$ is the logarithmic barrier associated with the symmetric cone $\mathcal{K}_i$.

Convex conic programs of the form \eqref{eq:conic-program} or \eqref{eq:conic-logbarrier-program} can be embedded as differentiable layers within computational graphs \cite{agrawal2019differentiable}. When $P$, $q$, $A$, and $b$ depend on problem data $\theta$, the sensitivity of the solution with respect to $\theta$ can be computed \cite{agrawal2019differentiating}, as can its sensitivity with respect to the complementarity parameter $\kappa$. The corresponding implicit differentiation formulas are summarized in \cref{app:conic-program-implicit-differentiation}.

\vspace{2mm}
\noindent\textbf{Contact Model.\quad} We adopt the quasistatic model of contact dynamics from \cite{pang2021convex, pang2023global}, which assumes velocities are sufficiently small to be neglected and employs a convex formulation of contact. This assumption is appropriate for manipulation tasks in which departures from static equilibrium are brief, so that transient accelerations do not accumulate into significant velocities \cite[\S10.1]{mason2021mechanics}. Although the convex contact formulation introduces an artifact whereby slipping contact occurs with a nonzero offset \cite{anitescu2006optimization}, this offset is proportional to the tangential velocity and is therefore negligible under the quasistatic assumption.

\looseness-1In this model, the nonsmooth contact dynamics $f_0:\mathcal{X} \times \mathcal{U}\rightarrow \mathcal{X}$, which maps the current state and control input $(x,u)$ to the next state $x^+$, is defined implicitly as the argmin of the conic program
\begin{subequations} \label{eq:dynamics-conic-prog}
\begin{align}
    & \minimize_{x^+} && \tfrac{1}{2} (x^+)^\top\, P(x)\, x^+ + q(x,u)^\top\, x^+\\
    & \subjectto && \nu_i := J_i(x)\, (x^+ - x) + \begin{bmatrix} \phi_i(x) & 0 & 0 \end{bmatrix}^\top,  \label{eq:nu-definition} \\
    &            && \nu_i \in \mathcal{F}_i^*, \qquad \forall i = 1,\dots,n_c,
\end{align}
\end{subequations}
where $P$ includes the mass matrix, $J_i$ is the contact Jacobian for the $i$th contact, and $n_c$ denotes the total number of contact pairs. Additional details of the model are provided in \cref{app:quasistatic-contact-model}. 
The $\kappa$-smoothed dynamics $f_\kappa:\mathcal{X} \times \mathcal{U} \to \mathcal{X}$ is defined analogously as the argmin of the relaxed problem
\begin{equation} \label{eq:smoothed-dynamics}
\begin{aligned}
    & \minimize\limits_{x^+}  && \tfrac{1}{2} (x^+)^\top\, P(x)\, x^+ + q(x,u)^\top\, x^+  + \kappa \sum_{i=1}^{n_c} \psi_i(\nu_i),
\end{aligned}\hspace{-15pt}
\end{equation}
where $\nu_i$ is defined in \eqref{eq:nu-definition}.
This corresponds to the same complementarity relaxation used to transform \eqref{eq:conic-program} into \eqref{eq:conic-logbarrier-program}.
The primal solutions of \eqref{eq:dynamics-conic-prog} and \eqref{eq:smoothed-dynamics}, denoted by $x^+$ and $x_\kappa^+$, represent the next state under the nonsmooth and smoothed dynamics, respectively. The corresponding dual solutions, denoted by $\lambda$ and $\lambda_\kappa$, are the associated contact forces.

\vspace{2mm}
\noindent\textbf{Problem Statement.\quad} We aim to 1) build a reachability-friendly differentiable simulator and bound the smoothing-error for reachability analysis, and 2) use it for certified contact-rich planning.


\noindent \textbf{Problem 1}: \textit{Reachability for contact dynamics}. Find a tight, efficiently computable, and differentiable bound on the smoothing error $f_0(x,u) - f_\kappa(x,u)$, and propagate this error through the system dynamics to obtain reachable state and control tubes $\{\mathcal{R}_k^x\}_{k=0}^N$ and $\{\mathcal{R}_k^u\}_{k=0}^{N-1}$. These tubes must guarantee that the \textit{true} executed trajectory, evolving under the nonsmooth hybrid dynamics \eqref{eq:dynamics-conic-prog}, satisfies $x_k \in \mathcal{R}_k^x \subseteq \mathcal{X}$ and $u_k \in \mathcal{R}_k^u \subseteq \mathcal{U}$ for all $k$.

\noindent \textbf{Problem 2}: \textit{Robust feedback motion planning}. Optimize a nominal state $\textbf{z}:=\left[z_k\right]_{k=0}^N$ and control $\textbf{v}:=\left[v_k\right]_{k=0}^{N-1}$ trajectory that is feasible under the $\kappa$-smoothed dynamics \eqref{eq:smoothed-dynamics} and a causal feedback controller $\pi :=(\pi_0, \ldots, \pi_{N-1})$, where $\pi_k:(\mathbb{R}^{n_x})^{k+1} \rightarrow \mathbb{R}^{n_u}$, that stabilizes the nonsmooth contact dynamics \eqref{eq:dynamics-conic-prog} about $(\textbf{z}, \textbf{v})$ and ensures closed-loop satisfaction of state and control constraints.

\vspace{2mm}
\noindent\textbf{Outline.\quad} Next, we will present our differentiable simulator and smoothing-error bound (Sec. \ref{sec:simulator}) and our reachability-certified contact-rich feedback motion planner (Sec. \ref{sec:policy-synthesis}).
\section{Differentiable Simulator with Contact Smoothing}\label{sec:simulator}

\subsection{Differentiable Simulator}

Utilizing the ability to differentiate through convex conic programs, we can differentiate the dynamics $x_\kappa^+ = f_\kappa(x, u)$ implicitly defined by \eqref{eq:smoothed-dynamics} to obtain $\partial x_\kappa^+ / \partial x$ and $\partial x_\kappa^+ / \partial u$. We can also compute the sensitivity with respect to $\kappa$, which is required for the smoothing-error bound, as derived next.

\subsection{Contact Dynamics Smoothing and Smoothing-Error Bound}
Transforming the nonsmooth hybrid dynamics $x^{+} = f_{0}(x,u)$ defined by \eqref{eq:dynamics-conic-prog} into the smoothed dynamics $x^{+}_{\kappa} = f_{\kappa}(x,u)$ defined by \eqref{eq:smoothed-dynamics} via complementarity relaxation inevitably introduces a deviation between the two. Although this deviation may be small over a single time step, it can accumulate over time, leading to significant errors.
This creates a tension: the smoothed dynamics is necessary for well-conditioned trajectory optimization, yet our ultimate objective is to provide guarantees with respect to the original nonsmooth hybrid dynamics. To address this issue, we first quantify the deviation between the nonsmooth and smoothed dynamics, as formalized in the following theorem.

\begin{theorem} \label{thm:smoothing-error-bounds}
Under the assumption that $J_i P^{-1} J_j^\top = 0$ for all $i \neq j$,
the nonsmooth hybrid dynamics $f_0$ and the smoothed dynamics $f_\kappa$ satisfy
\begin{equation} \label{eq:smoothing-bounds}
    f_0(x,u) = f_\kappa(x,u) - P(x)^{-1}\, \sum_{i=1}^{n_c} J_i(x)^\top\, \frac{\partial\lambda_{\kappa,i}}{\partial\kappa}\, \kappa\, w_i,
\end{equation}
where $\lambda_{\kappa,i}$ denotes the contact force associated with contact $i$ under the smoothed dynamics, and $w_i \in [1,2]$. Moreover, the bound is tight: for every choice of $w_i \in \{1,2\}$, there exist $(x,u)$ pairs that satisfy \eqref{eq:smoothing-bounds}.
\end{theorem}

The proof of \cref{thm:smoothing-error-bounds} proceeds by integrating $\frac{\partial f_\upkappa}{\partial \upkappa}$ from $0$ to $\kappa$; details are provided in \cref{app:proof-smoothing-error-bounds}.
Inspection of \eqref{eq:smoothing-bounds} reveals that the deviation bound is induced by the terms $\frac{\partial \lambda_{\kappa,i}}{\partial \kappa}\,\kappa\,w_i$, which represent deviations in the contact forces. Because $w_i$ is one-sided (i.e., not centered around zero), the resulting bound naturally reflects the unilateral nature of contact. The contact Jacobian $J_i^\top$ maps these contact force deviations into generalized coordinates, and the resulting contributions are aggregated and weighted by $P^{-1}$.

Equation \eqref{eq:smoothing-bounds} can be written compactly as
\begin{equation}  \label{eq:dynamics-bounds}
    x^+ = f_\kappa(x,u) + E_\kappa(x,u)\, w,
\end{equation}
where $w \in [1,2]^{n_c}$ and the $i$th column of $E_\kappa$ is given by
\begin{equation*}
    E_{\kappa,:i} = - P(x)^{-1}\, J_i(x)^\top\, \frac{\partial\lambda_{\kappa,i}}{\partial\kappa}\, \kappa.
\end{equation*}
Both $f_\kappa$ and $E_\kappa$ are continuously differentiable.
The following example illustrates this set-valued bound.
\begin{example}
\begin{figure}[!t]
    \centering
    \subfloat[]{%
      \begin{tikzpicture}[>=Stealth, scale=4]

\draw[dashed, semithick] (0,-0.05) -- (0,0.25);

\fill[red!60!black, draw=black] (0.3,0.1) circle (0.1);

\draw[densely dashed] (0.3,0) -- (0.3,0.2);
\draw[-{Latex[scale=0.8]}, densely dotted, semithick] (0,0.1) -- (0.3,0.1) node[midway, above, inner sep=1.5pt] {$x_a$};

\draw[fill=gray!30, , draw=black] (0.6,0) rectangle (0.8,0.2);

\draw[densely dashed] (0.7,-0.05) -- (0.7,0.22);
\draw[-{Latex[scale=0.8]}, densely dotted, semithick] (0,-0.04) -- (0.7,-0.04) node[midway, below, inner sep=1.5pt] {$x_o$};

\draw[red, thick]
(0.30,0.10)
-- (0.34,0.10)
-- (0.35,0.14)
-- (0.37,0.06)
-- (0.39,0.14)
-- (0.41,0.06)
-- (0.43,0.14)
-- (0.45,0.06)
-- (0.46,0.10)
-- (0.50,0.10);

\draw[semithick, red] (0.5,0.03) -- (0.5,0.25);
\draw[-{Latex[scale=0.8]}, densely dotted, semithick, red] (0,0.23) -- (0.5,0.23) node[midway, above, inner sep=1.5pt] {$u$};

\draw[thick] (0,0) -- (0.85,0);

\end{tikzpicture}
      \label{fig:1d-pusher-schematic}
    }
    \subfloat[]{%
      \begin{tikzpicture}

\definecolor{color1}{RGB}{228,26,28}
\definecolor{color2}{RGB}{55,126,184}
\definecolor{color3}{RGB}{77,175,74}

\begin{axis}[
width=6cm, height=3.6cm,
xlabel={\small $u$},
ylabel={\small $x_o^+$},
xmin=-0.1, xmax=0.3,
ymin=0.1985, ymax=0.209,
xtick=\empty,
ytick=\empty,
legend cell align={left},
legend style={
    font=\footnotesize,
    at={(0.00,1.00)},
    anchor=north west,
    draw=none,
    fill=none,
    row sep=-2pt,
    inner sep=1pt,
},
legend image post style={scale=0.6},
]

\pgfplotstableread[row sep=crcr]{
U        Xo_true      Xo           Eo             \\
-0.100   0.200000000  0.200195054  -0.000190346747\\
-0.096   0.200000000  0.200198837  -0.000193855248\\
-0.092   0.200000000  0.200202766  -0.000197488126\\
-0.088   0.200000000  0.200206849  -0.000201251652\\
-0.084   0.200000000  0.200211095  -0.000205152485\\
-0.080   0.200000000  0.200215513  -0.000209197696\\
-0.076   0.200000000  0.200220115  -0.000213394800\\
-0.072   0.200000000  0.200224912  -0.000217751778\\
-0.068   0.200000000  0.200229914  -0.000222277112\\
-0.064   0.200000001  0.200235137  -0.000226979812\\
-0.060   0.200000001  0.200240594  -0.000231869448\\
-0.056   0.200000001  0.200246300  -0.000236956180\\
-0.052   0.200000001  0.200252272  -0.000242250792\\
-0.048   0.200000001  0.200258529  -0.000247764715\\
-0.044   0.200000001  0.200265090  -0.000253510063\\
-0.040   0.200000001  0.200271977  -0.000259499648\\
-0.036   0.200000002  0.200279213  -0.000265747002\\
-0.032   0.200000002  0.200286826  -0.000272266385\\
-0.028   0.200000002  0.200294842  -0.000279072786\\
-0.024   0.200000003  0.200303293  -0.000286181904\\
-0.020   0.200000003  0.200312213  -0.000293610111\\
-0.016   0.200000000  0.200321640  -0.000301374389\\
-0.012   0.200000000  0.200331615  -0.000309492231\\
-0.008   0.200000000  0.200342182  -0.000317981503\\
-0.004   0.200000000  0.200353394  -0.000326860229\\
 0.000   0.200000000  0.200365304  -0.000336146325\\
 0.004   0.200000000  0.200377974  -0.000345857232\\
 0.008   0.200000000  0.200391472  -0.000356009443\\
 0.012   0.200000000  0.200405874  -0.000366617788\\
 0.016   0.200000000  0.200421262  -0.000377694787\\
 0.020   0.200000000  0.200437728  -0.000389249472\\
 0.024   0.200000000  0.200455375  -0.000401286177\\
 0.028   0.200000001  0.200474315  -0.000413802944\\
 0.032   0.200000001  0.200494673  -0.000426789600\\
 0.036   0.200000001  0.200516587  -0.000440225453\\
 0.040   0.200000001  0.200540209  -0.000454076609\\
 0.044   0.200000002  0.200565704  -0.000468292906\\
 0.048   0.200000003  0.200593255  -0.000482804550\\
 0.052   0.200000005  0.200623058  -0.000497518595\\
 0.056   0.200000008  0.200655325  -0.000512315520\\
 0.060   0.200000000  0.200690282  -0.000527046277\\
 0.064   0.200000000  0.200728168  -0.000541530361\\
 0.068   0.200000000  0.200769231  -0.000555555556\\
 0.072   0.200000001  0.200813723  -0.000568880124\\
 0.076   0.200000002  0.200861896  -0.000581238194\\
 0.080   0.200000005  0.200913996  -0.000592348878\\
 0.084   0.200000016  0.200970250  -0.000601929265\\
 0.088   0.200000001  0.201030863  -0.000609710761\\
 0.092   0.200000006  0.201096006  -0.000615457455\\
 0.096   0.200000002  0.201165807  -0.000618984461\\
 0.100   0.200002073  0.201240347  -0.000620173673\\
 0.104   0.200153848  0.201319653  -0.000618984461\\
 0.108   0.200307698  0.201403698  -0.000615457455\\
 0.112   0.200461539  0.201492401  -0.000609710761\\
 0.116   0.200615401  0.201585634  -0.000601929265\\
 0.120   0.200769236  0.201683226  -0.000592348874\\
 0.124   0.200923079  0.201784973  -0.000581238170\\
 0.128   0.201076924  0.201890646  -0.000568880124\\
 0.132   0.201230770  0.202000000  -0.000555555556\\
 0.136   0.201384616  0.202112783  -0.000541530361\\
 0.140   0.201538462  0.202228744  -0.000527046277\\
 0.144   0.201692316  0.202347632  -0.000512315519\\
 0.148   0.201846159  0.202469212  -0.000497518592\\
 0.152   0.202000003  0.202593255  -0.000482804538\\
 0.156   0.202153848  0.202719550  -0.000468292815\\
 0.160   0.202307694  0.202847901  -0.000454076609\\
 0.164   0.202461540  0.202978126  -0.000440225453\\
 0.168   0.202615385  0.203110058  -0.000426789600\\
 0.172   0.202769231  0.203243546  -0.000413802944\\
 0.176   0.202923077  0.203378452  -0.000401286175\\
 0.180   0.203076923  0.203514651  -0.000389249472\\
 0.184   0.203230769  0.203652031  -0.000377694629\\
 0.188   0.203384616  0.203790489  -0.000366617788\\
 0.192   0.203538462  0.203929934  -0.000356009427\\
 0.196   0.203692308  0.204070282  -0.000345857232\\
 0.200   0.203846154  0.204211458  -0.000336146322\\
 0.204   0.204000000  0.204353394  -0.000326860225\\
 0.208   0.204153846  0.204496028  -0.000317981456\\
 0.212   0.204307692  0.204639307  -0.000309492230\\
 0.216   0.204461539  0.204783178  -0.000301374387\\
 0.220   0.204615385  0.204927598  -0.000293610110\\
 0.224   0.204769234  0.205072524  -0.000286181903\\
 0.228   0.204923079  0.205217919  -0.000279072786\\
 0.232   0.205076925  0.205363749  -0.000272266326\\
 0.236   0.205230771  0.205509983  -0.000265746926\\
 0.240   0.205384617  0.205656592  -0.000259499648\\
 0.244   0.205538463  0.205803551  -0.000253510063\\
 0.248   0.205692309  0.205950836  -0.000247764715\\
 0.252   0.205846155  0.206098426  -0.000242250792\\
 0.256   0.206000001  0.206246300  -0.000236956180\\
 0.260   0.206153847  0.206394440  -0.000231869448\\
 0.264   0.206307693  0.206542829  -0.000226979812\\
 0.268   0.206461539  0.206691453  -0.000222277112\\
 0.272   0.206615385  0.206840296  -0.000217751776\\
 0.276   0.206769231  0.206989346  -0.000213394790\\
 0.280   0.206923077  0.207138590  -0.000209197652\\
 0.284   0.207076923  0.207288018  -0.000205152417\\
 0.288   0.207230770  0.207437618  -0.000201251652\\
 0.292   0.207384616  0.207587381  -0.000197488126\\
 0.296   0.207538462  0.207737299  -0.000193855248\\
 0.300   0.207692308  0.207887362  -0.000190346747\\
}\datatable

\addplot[thick, color1]
    table[x=U, y=Xo_true]{\datatable};
\addlegendentry{Nonsmooth dynamics}

\addplot[thick, color2]
    table[x=U, y=Xo]{\datatable};
\addlegendentry{Smoothed dynamics}

\addplot[name path=lower, draw=none, forget plot]
    table[x=U, y expr=\thisrow{Xo} + \thisrow{Eo}]{\datatable};

\addplot[name path=upper, draw=none, forget plot]
    table[x=U, y expr=\thisrow{Xo} + 2*\thisrow{Eo}]{\datatable};

\addplot[fill=color3, fill opacity=0.4, draw=none, legend image code/.code={
        \draw[fill=color3, fill opacity=0.4, draw=none] (0cm,-0.1cm) rectangle (0.6cm,0.1cm);
    }]
    fill between[of=lower and upper];
\addlegendentry{\shortstack[l]{Smoothed dynamics\\[-4pt]+ deviation bounds}}

\addplot[thin, densely dashed, color3]
    table[x=U, y expr=\thisrow{Xo} + 1.5*\thisrow{Eo}]{\datatable};

\addplot[dotted, gray] coordinates {(0.1,0.1) (0.1,0.3)};
\node[anchor=south west, gray, inner sep=0] at (rel axis cs:0,0) {\footnotesize no contact};
\node[anchor=south east, gray, inner sep=0] at (rel axis cs:1,0) {\footnotesize in contact};

\end{axis}

\end{tikzpicture}
      \label{fig:1d-pusher-dynamics}
    }
    \caption{1D pusher. \textbf{(a)} System schematic. \textbf{(b)} Discrete-time dynamics $x^+ = f(x,u)$, with $x_o^+$ plotted versus the input $u$.}
    \label{fig:1d-pusher}
\end{figure}
Consider the 1D system of a stiffness-controlled red pusher pushing a gray object, shown in \cref{fig:1d-pusher-schematic}. The corresponding discrete-time dynamics are plotted in \cref{fig:1d-pusher-dynamics}. Under the nonsmooth hybrid dynamics, when the pusher is not in contact with the object, the object position remains constant with respect to the input, resulting in flat regions and nonsmooth transitions between the non-contact and contact modes.

By smoothing the dynamics, the gradient $\partial x_o^+ / \partial u$ becomes positive for all values of $u$. These gradients are informative, as they guide the optimizer toward actions that induce contact, and they remain continuous across the entire input domain. The shaded region in \cref{fig:1d-pusher-dynamics} represents the bounds induced by the smoothing-error: its upper and lower envelopes correspond to $f_\kappa(x,u) + E_\kappa(x,u)$ and $f_\kappa(x,u) + 2\, E_\kappa(x,u)$, respectively. As predicted by \eqref{eq:dynamics-bounds}, the original nonsmooth dynamics lie entirely within this shaded region. 
One may also wonder if the nominal dynamics could instead be chosen as $f_\kappa(x,u) + 1.5\, E_\kappa(x,u)$, shown as the dashed line in \cref{fig:1d-pusher-dynamics}, to eliminate the bias. The answer is no. Such a choice yields zero or even negative $\partial x_o^+ / \partial u$ gradients over parts of the input domain, rendering them non-informative for action refinement.
\end{example}

\begin{wrapfigure}{r}{0.42\linewidth}
    \hspace*{-3mm}
    \begin{tikzpicture}

\definecolor{color1}{RGB}{228,26,28}
\definecolor{color3}{RGB}{77,175,74}

\begin{axis}[
width=4.375cm, height=3.75cm,
xlabel={\footnotesize $\frac{|J_i P^{-1} J_i^\top|}{\sum_{j \neq i} |J_i P^{-1} J_j^\top|}$},
xlabel shift=-1mm,
ylabel={\footnotesize Value of terms in (\ref*{eq:smoothing-bounds})},
ylabel shift=-1mm,
xmin=0.3, xmax=4.0,
ymin=1.5e-5, ymax=2.8e-5,
tick label style={font=\scriptsize},
legend cell align={left},
legend style={
    font=\footnotesize,
    at={(1.0,0.0)},
    anchor=south east,
    draw=none,
    fill=none,
    row sep=-2pt,
    inner sep=0pt,
},
legend image post style={scale=0.4},
every y tick scale label/.style={
    at={(axis description cs:0.15,1.1)},
    anchor=north east,
    inner sep=0,
},
]

\pgfplotstableread[row sep=crcr]{
ratio    value        lb           ub           \\
0.280985 1.485948e-05 1.848247e-05 3.696494e-05 \\
0.300137 1.677018e-05 1.924241e-05 3.848483e-05 \\
0.326183 1.825513e-05 1.985634e-05 3.971269e-05 \\
0.361204 1.938558e-05 2.033787e-05 4.067575e-05 \\
0.407596 2.023045e-05 2.070670e-05 4.141340e-05 \\
0.467885 2.085191e-05 2.098401e-05 4.196803e-05 \\
0.544391 2.130759e-05 2.118959e-05 4.237919e-05 \\
0.638764 2.163557e-05 2.134039e-05 4.268077e-05 \\
0.751550 2.187012e-05 2.145013e-05 4.290026e-05 \\
0.882027 2.203661e-05 2.152954e-05 4.305909e-05 \\
1.028578 2.215380e-05 2.158677e-05 4.317353e-05 \\
1.189605 2.223576e-05 2.162788e-05 4.325576e-05 \\
1.364592 2.229249e-05 2.165735e-05 4.331470e-05 \\
1.554726 2.233198e-05 2.167844e-05 4.335689e-05 \\
1.762561 2.235833e-05 2.169353e-05 4.338705e-05 \\
1.990658 2.237576e-05 2.170430e-05 4.340860e-05 \\
2.239588 2.238541e-05 2.171199e-05 4.342398e-05 \\
2.506081 2.238551e-05 2.171748e-05 4.343496e-05 \\
2.782251 2.237731e-05 2.172140e-05 4.344280e-05 \\
3.056524 2.240722e-05 2.172419e-05 4.344839e-05 \\
3.316139 2.240868e-05 2.172619e-05 4.345237e-05 \\
3.550135 2.240428e-05 2.172761e-05 4.345521e-05 \\
3.751521 2.241047e-05 2.172862e-05 4.345724e-05 \\
3.917893 2.239165e-05 2.172934e-05 4.345869e-05 \\
4.050679 2.241082e-05 2.172986e-05 4.345972e-05 \\
}\datatable

\addplot[thick, color1]
    table[x=ratio, y=value]{\datatable};
\addlegendentry{$f_0 - f_\kappa$}

\addplot[name path=lower, draw=none, forget plot]
    table[x=ratio, y=lb]{\datatable};

\addplot[name path=upper, draw=none, forget plot]
    table[x=ratio, y=ub]{\datatable};

\addplot[fill=color3, fill opacity=0.4, draw=none, legend image code/.code={
        \draw[fill=color3, fill opacity=0.4, draw=none] (0cm,-0.1cm) rectangle (0.6cm,0.1cm);
    }]
    fill between[of=lower and upper];
\addlegendentry{$E_\kappa\, w$}

\end{axis}

\end{tikzpicture}
    \vspace*{-4mm}
\end{wrapfigure}
\noindent\textbf{Remark 1.}
On the right, we plot terms in \eqref{eq:smoothing-bounds} against the \ac{DDF} of $J P^{-1} J^\top$ for the bimanual planar bucket manipulation example in Sec.~\ref{sec:planar-bucket}. 
Although \cref{thm:smoothing-error-bounds} requires $J_i P^{-1} J_j^\top = 0$ for all $i \neq j$, the plot shows that \eqref{eq:smoothing-bounds} holds if the \ac{DDF} exceeds approximately 0.5. This suggests, in practice, diagonal dominance is sufficient for \eqref{eq:smoothing-bounds} to hold numerically.
In the Sec.~\ref{sec:results} examples, the \ac{DDF} ranges from 1.3 to 2.1, exceeding the 0.5 cutoff.

\subsection{Contact Geometry Smoothing}
The complementarity between signed distance and contact force is not the only source of nonsmoothness in the mapping from $(x,u)$ to $x^+$. The geometry of contact primitives can also introduce nonsmoothness. For example, in contact between a sphere and a polyhedron, nonsmoothness arises from discontinuous changes in surface normals across the faces of the polyhedron. This results in discontinuities in the contact Jacobian and, consequently, in a nonsmooth dynamics. This issue has also been identified in prior work \cite{montaut2023differentiable, le2023single}.

In this work, we utilize the differentiability of conic programs by formulating contact point queries as \acp{QP}. For instance, the contact points between a sphere $\{p_1 \,|\, \|p_1 - p_c\|_2 \leq r\}$ and a polyhedron $\{p_2 \,|\, A\, p_2 \leq b\}$ are obtained by solving a relaxed version of the \ac{QP}
\begin{equation*}
\begin{aligned}
    & \minimize_{p_2}  && \tfrac{1}{2} \|p_2 - p_c\|_2^2 & \subjectto && A\, p_2 \leq b,
\end{aligned}
\end{equation*}
analogous to how \eqref{eq:conic-program} is relaxed to \eqref{eq:conic-logbarrier-program}. The resulting contact points are then used to compute the contact Jacobian $J_i(x)$ and signed distance $\phi_i(x)$, which enter into \eqref{eq:dynamics-conic-prog}.

\looseness-1Ideally, one would characterize how such geometry smoothing induces deviations in the contact Jacobian and signed distance, and thus in the dynamics mapping. We defer this analysis to future work, as the error introduced by contact geometry smoothing is less severe than that arising from contact dynamics smoothing. In particular, contact dynamics smoothing introduces a force-at-a-distance effect whose error can accumulate over time, whereas contact geometry smoothing only induces a local perturbation that manifests as a constant offset.

\section{Trajectory Optimization and Policy Synthesis}
\label{sec:policy-synthesis}

We aim to solve the trajectory optimization problem:
\begin{subequations} \label{eq:nominal-traj-opt}
\begin{alignat}{2}
    & \minimize\quad && J(\vec{z}, \vec{v}) \\
    & \subjectto && z_{k+1} = f_\kappa(z_k, v_k), \quad \forall k = 0,\dots,N-1, \label{eq:nominal-traj-opt-dynamics-constraint}  \\
    &            && z_0 = x_0,
\end{alignat}
\end{subequations}
where $\vec{z} = \begin{bmatrix} z_0^\top & \cdots & z_N^\top \end{bmatrix}^\top$ is the nominal state trajectory and $\vec{v} = \begin{bmatrix} v_0^\top & \cdots & v_{N-1}^\top \end{bmatrix}^\top$ is the nominal control sequence. The cost $J(\vec{z}, \vec{v})$ may include terms such as $\sum_{k=0}^{N} \tfrac{1}{2} \|z_k - x_{\text{goal}}\|_Q^2$.
The dynamics constraint \eqref{eq:nominal-traj-opt-dynamics-constraint} is based on the smoothed dynamics $f_\kappa$ rather than the nonsmooth dynamics $f_0$. This choice is necessary because $f_0$ does not provide informative gradients: its associated optimization landscape contains flat regions and sharp discontinuities that hinder gradient-based optimization.



However, when executing the nominal control sequence $\vec{v}$ on the nonsmooth hybrid dynamics, the resulting state trajectory $\vec{x} = \begin{bmatrix} x_0^\top & \cdots & x_N^\top \end{bmatrix}^\top$ deviates from $\vec{z}$. As a consequence, there is no guarantee that $\vec{x}$ will reach the goal state $x_{\text{goal}}$. More critically, in the presence of state or control constraints, the true trajectory may violate constraints even if the nominal trajectory is feasible. 
The key idea is that, instead of executing nominal controls on the nonsmooth hybrid dynamics, we execute a closed-loop feedback \textit{policy}. Specifically, we consider a class of affine feedback policies of the form
\begin{equation}  \label{eq:affine-policy}
    u_k = \pi(x_0, \dots, x_k) = v_k + \textstyle\sum_{j=0}^{k} K_{kj}\, (x_j - z_j),
\end{equation}
where $K_{kj}$ are the feedback gains. Under such policies, the closed-loop behavior of the nonsmooth hybrid dynamics admits analytically tractable tube-valued predictions. This result is formalized in the following theorem.

\begin{theorem} \label{thm:policy-parameterization}
For the dynamics written as \eqref{eq:dynamics-bounds}, consider strictly lower-triangular block matrices
{\footnotesize
\begin{equation*}
    \mat{\Phi}^x\hspace{-3pt} =\hspace{-3pt}
    \left[\setlength{\arraycolsep}{2pt}\begin{array}{l l l l}
        0 \\
        \Phi^x_{1,0} \\
        \:\vdots & \ddots \\
        \Phi^x_{N,0} & \cdots & \Phi^x_{N,N-1}
    \end{array}\right],\hspace{-3pt} \
    \mat{\Phi}^u\hspace{-3pt} =\hspace{-3pt}
    \left[\setlength{\arraycolsep}{2pt}\begin{array}{l l l l}
        0 \\
        \Phi^u_{1,0} \\
        \:\vdots & \ddots \\
        \Phi^u_{N-1,0} & \cdots & \Phi^u_{N-1,N-2} & 0
    \end{array}\right],
\end{equation*}}
that satisfy, for all $j = 0, \dots, N-1$,
\begin{equation*}
\begin{aligned}
    \Phi_{k+1,j}^x &= A_k \Phi_{k,j}^x + B_k \Phi_{k,j}^u,
    \quad \forall k = j+1, \dots, N-1, \\
    \Phi_{j+1,j}^x &= E_j,
\end{aligned}
\end{equation*}
where $A_k = \partial f_\kappa(z_k, v_k) / \partial z$,
$B_k = \partial f_\kappa(z_k, v_k) / \partial v$,
and $E_j = E_\kappa(z_j, v_j)$.
The set of all such block matrices $\mat{\Phi}^x$ and $\mat{\Phi}^u$ parameterize all possible closed-loop system responses
\begin{subequations} \label{eq:system-response}
\begin{alignat}{3}
    \vec{x} &= \vec{z} + \mat{\Phi}^x\, \vec{w} \quad&\text{or}\quad x_k &= z_k + \textstyle\sum_{j=0}^{k-1} \Phi_{k,j}^x\, w, \\
    \vec{u} &= \vec{v} + \mat{\Phi}^u\, \vec{w} \quad&\text{or}\quad u_k &= v_k + \textstyle\sum_{j=0}^{k-1} \Phi_{k,j}^u\, w,
\end{alignat}
\end{subequations}
with $\vec{w} \in [1,2]^{N \cdot n_c}$.
Moreover, a feedback policy of the form
\begin{equation*}
    \vec{u}
    = \vec{v}
    + \mat{\Phi}^u (\mat{\Phi}^x)^{-1}
    (\vec{x} - \vec{z}) := \vec{v}
    + \mat{K}
    (\vec{x} - \vec{z})
\end{equation*}
achieves such a closed-loop response.
\end{theorem}

The proof of \cref{thm:policy-parameterization} is provided in \cref{app:proof-policy-parameterization}. 
By leveraging \cref{thm:policy-parameterization}, we can formulate an optimization problem in which the closed-loop response  of the nonsmooth hybrid dynamics is guaranteed to satisfy the stage and terminal constraints
\begin{subequations}  \label{eq:constraint} 
\begin{align}
    G_k \begin{bmatrix} x_k \\ u_k \end{bmatrix} + g_k &\leq 0, \quad \forall k = 0, \dots, N-1,  \label{eq:stage-constraint} \\
    G_f\, x_N + g_f &\leq 0.  \label{eq:terminal-constraint} 
\end{align}
\end{subequations}
The optimization problem that jointly optimizes the nominal trajectory and affine feedback policy is formulated as
\begin{subequations} \label{eq:robust-traj-opt}
\begin{alignat}{2}
    & \min_{\substack{\vec{z}, \vec{v},\\\mat{\Phi}^x, \mat{\Phi}^u}}\ && J(\vec{z}, \vec{v}) + \mathcal{J}(\mat{\Phi}^x, \mat{\Phi}^u)  \label{eq:robust-traj-opt-cost} \\
    & \hspace{2.2mm}\text{s.t.} 
        && z_{k+1} = f_\kappa(z_k, v_k), \quad z_0 = x_0, \\
    & 
        && \Phi_{k+1,j}^x = A_k \Phi_{k,j}^x + B_k \Phi_{k,j}^u, \quad 
           \Phi_{j+1,j}^x = E_j,  \\
    &
        && \underbrace{\sum_{j=0}^{k-1} G_k 
           \begin{bmatrix} \Phi_{k,j}^x \\ \Phi_{k,j}^u \end{bmatrix} w_c 
           + w_r\, \big\| G_k \begin{bmatrix} \Phi_{k,j}^x \\ \Phi_{k,j}^u \end{bmatrix} \big\|_{\text{row},q}}_{h_k(\mat{\Phi}^x, \mat{\Phi}^u)}  \nonumber \\[-3mm]
    &   &&\hspace{10em} + G_k \begin{bmatrix} z_k \\ v_k \end{bmatrix} + g_k \leq 0,  \label{eq:robust-traj-opt-stage-constraint} \\
    & 
        && \underbrace{\sum_{j=0}^{N-1} G_f\, \Phi_{N,j}^x\, w_c 
           + w_r\, \big\| G_f \Phi_{N,j}^x \big\|_{\text{row},q}}_{h_f(\mat{\Phi}^x)}  \nonumber \\[-3mm]
    &   &&\hspace{10.7em} + G_f\, z_N + g_f \leq 0,  \label{eq:robust-traj-opt-terminal-constraint}
\end{alignat}
\end{subequations}
where $\|\cdot\|_{\text{row},q}$ denotes the row-wise $q$-norm.
The constraints \eqref{eq:robust-traj-opt-stage-constraint} and \eqref{eq:robust-traj-opt-terminal-constraint} correspond to the stage constraint \eqref{eq:stage-constraint} and the terminal constraint \eqref{eq:terminal-constraint}, respectively. Their form follows from the closed-loop system response representation in \eqref{eq:system-response} and a standard robustification argument. 
Specifically, recall
\begin{equation*}
    w \in [1,2]^{n_c}
    = \bigl\{ w \in \mathbb{R}^{n_c} \,\big|\, \|w - w_c\|_p \leq w_r \bigr\},
\end{equation*}
where $w_c = 1.5\,\bm{1}$, $w_r = 0.5$, and $p = \infty$.
For any vector $a \in \mathbb{R}^{n_c}$, the worst-case value of the linear
function $a^\top w$ over this uncertainty set admits the closed-form expression
\begin{equation*}
    \max_{w \in [1,2]^{n_c}} a^\top w
    = a^\top w_c + w_r\, \|a\|_q,
\end{equation*}
where $q$ is the dual norm exponent satisfying $1/p + 1/q = 1$.
Applying this property row-wise to the affine dependence of $(x_k,u_k)$ on $w$ in \eqref{eq:system-response} yields the robust constraints in \eqref{eq:robust-traj-opt}.

We do not explicitly account for model mismatch introduced by estimating the nonlinear tracking error dynamics as a \ac{LTV} system. This can be addressed via linearization error bounds \cite{leeman2025robust} computed via interval arithmetic \cite{limon2005robust} or statistical methods \cite{knuth2023statistical, knuth2021planning}, which inflate the tubes to maintain guarantees.

\subsection{Cost and Constraints}
The cost \eqref{eq:robust-traj-opt-cost} includes two components:
\begin{itemize}[leftmargin=*]
\item {
Nominal trajectory cost, encouraging goal reaching and smooth controls.
\begin{equation*}
    \textstyle J(\vec{v}, \vec{z}) 
    = \sum_{k=0}^N \tfrac{1}{2} \| z_k - x_{\text{goal}} \|_Q^2
      + \sum_{k=0}^{N-2} \tfrac{1}{2} \| v_k - v_{k+1} \|_R^2.
\end{equation*}
}
\item {
Tube cost, which penalizes the size of the tube.
\begin{equation*} 
    \mathcal{J}(\mat{\Phi}^x, \mat{\Phi}^u)
    = \sum_{j=0}^{N-1} \biggl( \sum_{k=j+1}^{N}\!\! \| \bar{Q}^\frac{1}{2} \Phi^x_{k,j} \|_\mathcal{F}^2
      +\!\! \sum_{k=j+1}^{N-1}\!\! \| \bar{R}^\frac{1}{2} \Phi^u_{k,j} \|_\mathcal{F}^2 \biggr),
\end{equation*}
}
\end{itemize}

The constraints \eqref{eq:robust-traj-opt-stage-constraint} and \eqref{eq:robust-traj-opt-terminal-constraint}, or equivalently \eqref{eq:constraint}, enforce the following:
\begin{itemize}[leftmargin=*]
\item Joint angle and torque limits.
\item Non-collision constraints, locally approximated by affine inequalities.
\item Object position constraints, which require the object state to lie within a polytopic set.
\end{itemize}

\subsection{Solver Algorithm}
We solve the optimization problem \eqref{eq:robust-traj-opt} using a \ac{SCP} approach. At each iteration, we linearize around the current nominal trajectory $(\vec{z}, \vec{v})$ and solve the resulting convex optimization problem:
\begin{subequations}  \label{eq:inner-cvx-robust}
\begin{alignat}{2}
    & \min_{\substack{\delta\vec{z}, \delta\vec{v},\\\mat{\Phi}^x, \mat{\Phi}^u}}\ && J(\vec{z} + \delta\vec{z}, \vec{v} + \delta\vec{v}) + \mathcal{J}(\mat{\Phi}^x, \mat{\Phi}^u)  \\
    & \hspace{2.2mm}\text{s.t.}
        && \delta z_{k+1} = A_k\, \delta z_k + B_k\, \delta v_k, \quad \delta z_0 = 0,  \label{eq:inner-cvx-robust-dynamics} \\
    & 
        && \Phi_{k+1,j}^x = A_k \Phi_{k,j}^x + B_k \Phi_{k,j}^u,  \nonumber \\
    &
        &&\hspace{2.2em} \Phi_{j+1,j}^x = E_j + \tfrac{\partial E_j}{\partial z}\, \delta z_j + \tfrac{\partial E_j}{\partial v}\, \delta v_j,  \\
    &
        && h_k(\mat{\Phi}^x, \mat{\Phi}^u)
           + G_k \begin{bmatrix} z_k + \delta z_k \\ v_k + \delta v_k \end{bmatrix} + g_k \leq 0,  \\
    & 
        && h_f(\mat{\Phi}^x)
           + G_f\, (z_N + \delta z_N) + g_f \leq 0,  \\
    &
        && \|\delta\vec{z}\| \leq \varepsilon, \quad
           \|\delta\vec{v}\| \leq \varepsilon.  \label{eq:inner-cvx-robust-trust-region}
\end{alignat}
\end{subequations}
The constraints \eqref{eq:inner-cvx-robust-trust-region} ensure the validity of the local linearization.
The resulting  increments $\delta \vec{z}$ and $\delta \vec{v}$ are then used to update the nominal trajectory, and the procedure is repeated until convergence. The full algorithm is summarized in Alg.~\ref{alg:scp}.
\begin{algorithm}  \label{alg:scp}
\caption{\mbox{Certifiable Gradient-Based Policy Synthesis}}
\KwIn{Initial control sequence $\vec{v}$, initial state $x_0$}
\KwOut{Nominal trajectory $(\vec{z}, \vec{v})$, feedback gains $\mat{K}$}

\tcp{\ac{SCP} iterations}
\For{$\text{iter} = 1$ \KwTo maxIters}{
    $z_0 \gets x_0$\\
    \For{$k = 0$ \KwTo $N-1$}{
        $z_{k+1} \gets f_\kappa(z_k, v_k)$
    }
    $\mat{A}, \mat{B} \gets \text{LinearizeDynamics}(f_\kappa, \vec{z}, \vec{v})$\\
    $\mat{E}, \frac{\partial \mat{E}}{\partial \vec{z}}, \frac{\partial \mat{E}}{\partial \vec{v}} 
       \gets \text{EvaluateDisturbance}(E_\kappa, \vec{z}, \vec{v})$\\
    $\delta \vec{z}, \delta \vec{v}, \mat{\Phi}^x, \mat{\Phi}^u 
       \gets \text{Solve\eqref{eq:inner-cvx-robust}}(\vec{z}, \vec{v}, \mat{A}, \mat{B}, \mat{E}, \frac{\partial \mat{E}}{\partial \vec{z}}, \frac{\partial \mat{E}}{\partial \vec{v}})$\\ 
    $\vec{v} \gets \vec{v} + \delta \vec{v}$
}

\tcp{Post-processing}
$\mat{K} \gets \mat{\Phi}^u (\mat{\Phi}^x)^{-1}$\\
\For{$k = 0$ \KwTo $N-1$}{
    $v_k \gets v_k + \sum_{j=0}^{k} K_{kj}\, (x_j - z_j)$\\
    $z_{k+1} \gets f_\kappa(z_k, v_k)$\\
    $x_{k+1} \gets f_0(x_k, v_k)$
}

\Return $\vec{z}, \vec{v}, \mat{K}$
\end{algorithm}

While the optimization problem \eqref{eq:inner-cvx-robust} is convex and can be solved by standard solvers, its structure can be exploited to drastically improve computational efficiency. A specialized solver is detailed in \cref{app:fast-sls}.

\subsection{Baseline}
\label{sec:baseline-ctr}
We compare our method against a baseline \ac{TO-CTR} proposed in \cite{suh2025dexterous}. The method is closely related to \ac{iLQR}, but augments the backward pass with a heuristic \ac{CTR} constraint. 
Specifically, during the backward pass, instead of solving \eqref{eq:inner-cvx-robust}, \ac{TO-CTR} solves the following convex optimization problem:
\begin{subequations} \label{eq:inner-cvx-ctr}
\begin{alignat}{2}
    & \min_{\delta\vec{z}, \delta\vec{v}}\ && J(\vec{z} + \delta\vec{z}, \vec{v} + \delta\vec{v})  \\
    & \hspace{1mm}\text{s.t.}
        && \delta z_{k+1} = A_k\, \delta z_k + B_k\, \delta v_k, \quad \delta z_0 = 0,  \label{eq:inner-cvx-ctr-dynamics} \\
    &
        && G_k \begin{bmatrix} z_k + \delta z_k \\ v_k + \delta v_k \end{bmatrix} + g_k \leq 0,  \\
    & 
        && G_f\, (z_N + \delta z_N) + g_f \leq 0,  \\
    &
        && \|\delta \vec{z}\| \leq \varepsilon, \quad
           \|\delta \vec{v}\| \leq \varepsilon,  \\
    &
        && \lambda_\kappa 
           + \tfrac{\partial \lambda_\kappa}{\partial z_k}\, \delta z_k 
           + \tfrac{\partial \lambda_\kappa}{\partial v_k}\, \delta v_k 
           \in \mathcal{F}.  \label{eq:inner-cvx-ctr-constraint}
\end{alignat}
\end{subequations}
The constraint \eqref{eq:inner-cvx-ctr-constraint} enforces the \ac{CTR} by restricting the linearized contact force to remain within the friction cone $\mathcal{F}$.
An intuitive explanation of the key difference between \ac{TO-CTR} and our method is best illustrated through an example.

\begin{example}
\begin{figure}[!t]
    \centering
    \vspace*{-2.5mm}
    \subfloat[]{%
      \hspace*{-3mm}
      \begin{tikzpicture}

\definecolor{color1}{RGB}{228,26,28}
\definecolor{color2}{RGB}{55,126,184}
\definecolor{color3}{RGB}{77,175,74}
\definecolor{color4}{RGB}{152,78,163}

\begin{axis}[
width=5.4cm, height=3.3cm,
xlabel={\small $v$},
ylabel={\small $z_o^+$},
xmin=-0.1, xmax=0.3,
ymin=0.1995, ymax=0.2080,
xtick=\empty,
ytick=\empty,
legend cell align={left},
legend style={
    font=\footnotesize,
    at={(0.00,1.00)},
    anchor=north west,
    draw=none,
    fill=none,
    row sep=-2pt,
    inner sep=0.5pt,
},
legend image post style={scale=0.6},
]

\pgfplotstableread[row sep=crcr]{
U        Xo_true      Xo           Eo             \\
-0.100   0.200000000  0.200195054  -0.000190346747\\
-0.096   0.200000000  0.200198837  -0.000193855248\\
-0.092   0.200000000  0.200202766  -0.000197488126\\
-0.088   0.200000000  0.200206849  -0.000201251652\\
-0.084   0.200000000  0.200211095  -0.000205152485\\
-0.080   0.200000000  0.200215513  -0.000209197696\\
-0.076   0.200000000  0.200220115  -0.000213394800\\
-0.072   0.200000000  0.200224912  -0.000217751778\\
-0.068   0.200000000  0.200229914  -0.000222277112\\
-0.064   0.200000001  0.200235137  -0.000226979812\\
-0.060   0.200000001  0.200240594  -0.000231869448\\
-0.056   0.200000001  0.200246300  -0.000236956180\\
-0.052   0.200000001  0.200252272  -0.000242250792\\
-0.048   0.200000001  0.200258529  -0.000247764715\\
-0.044   0.200000001  0.200265090  -0.000253510063\\
-0.040   0.200000001  0.200271977  -0.000259499648\\
-0.036   0.200000002  0.200279213  -0.000265747002\\
-0.032   0.200000002  0.200286826  -0.000272266385\\
-0.028   0.200000002  0.200294842  -0.000279072786\\
-0.024   0.200000003  0.200303293  -0.000286181904\\
-0.020   0.200000003  0.200312213  -0.000293610111\\
-0.016   0.200000000  0.200321640  -0.000301374389\\
-0.012   0.200000000  0.200331615  -0.000309492231\\
-0.008   0.200000000  0.200342182  -0.000317981503\\
-0.004   0.200000000  0.200353394  -0.000326860229\\
 0.000   0.200000000  0.200365304  -0.000336146325\\
 0.004   0.200000000  0.200377974  -0.000345857232\\
 0.008   0.200000000  0.200391472  -0.000356009443\\
 0.012   0.200000000  0.200405874  -0.000366617788\\
 0.016   0.200000000  0.200421262  -0.000377694787\\
 0.020   0.200000000  0.200437728  -0.000389249472\\
 0.024   0.200000000  0.200455375  -0.000401286177\\
 0.028   0.200000001  0.200474315  -0.000413802944\\
 0.032   0.200000001  0.200494673  -0.000426789600\\
 0.036   0.200000001  0.200516587  -0.000440225453\\
 0.040   0.200000001  0.200540209  -0.000454076609\\
 0.044   0.200000002  0.200565704  -0.000468292906\\
 0.048   0.200000003  0.200593255  -0.000482804550\\
 0.052   0.200000005  0.200623058  -0.000497518595\\
 0.056   0.200000008  0.200655325  -0.000512315520\\
 0.060   0.200000000  0.200690282  -0.000527046277\\
 0.064   0.200000000  0.200728168  -0.000541530361\\
 0.068   0.200000000  0.200769231  -0.000555555556\\
 0.072   0.200000001  0.200813723  -0.000568880124\\
 0.076   0.200000002  0.200861896  -0.000581238194\\
 0.080   0.200000005  0.200913996  -0.000592348878\\
 0.084   0.200000016  0.200970250  -0.000601929265\\
 0.088   0.200000001  0.201030863  -0.000609710761\\
 0.092   0.200000006  0.201096006  -0.000615457455\\
 0.096   0.200000002  0.201165807  -0.000618984461\\
 0.100   0.200002073  0.201240347  -0.000620173673\\
 0.104   0.200153848  0.201319653  -0.000618984461\\
 0.108   0.200307698  0.201403698  -0.000615457455\\
 0.112   0.200461539  0.201492401  -0.000609710761\\
 0.116   0.200615401  0.201585634  -0.000601929265\\
 0.120   0.200769236  0.201683226  -0.000592348874\\
 0.124   0.200923079  0.201784973  -0.000581238170\\
 0.128   0.201076924  0.201890646  -0.000568880124\\
 0.132   0.201230770  0.202000000  -0.000555555556\\
 0.136   0.201384616  0.202112783  -0.000541530361\\
 0.140   0.201538462  0.202228744  -0.000527046277\\
 0.144   0.201692316  0.202347632  -0.000512315519\\
 0.148   0.201846159  0.202469212  -0.000497518592\\
 0.152   0.202000003  0.202593255  -0.000482804538\\
 0.156   0.202153848  0.202719550  -0.000468292815\\
 0.160   0.202307694  0.202847901  -0.000454076609\\
 0.164   0.202461540  0.202978126  -0.000440225453\\
 0.168   0.202615385  0.203110058  -0.000426789600\\
 0.172   0.202769231  0.203243546  -0.000413802944\\
 0.176   0.202923077  0.203378452  -0.000401286175\\
 0.180   0.203076923  0.203514651  -0.000389249472\\
 0.184   0.203230769  0.203652031  -0.000377694629\\
 0.188   0.203384616  0.203790489  -0.000366617788\\
 0.192   0.203538462  0.203929934  -0.000356009427\\
 0.196   0.203692308  0.204070282  -0.000345857232\\
 0.200   0.203846154  0.204211458  -0.000336146322\\
 0.204   0.204000000  0.204353394  -0.000326860225\\
 0.208   0.204153846  0.204496028  -0.000317981456\\
 0.212   0.204307692  0.204639307  -0.000309492230\\
 0.216   0.204461539  0.204783178  -0.000301374387\\
 0.220   0.204615385  0.204927598  -0.000293610110\\
 0.224   0.204769234  0.205072524  -0.000286181903\\
 0.228   0.204923079  0.205217919  -0.000279072786\\
 0.232   0.205076925  0.205363749  -0.000272266326\\
 0.236   0.205230771  0.205509983  -0.000265746926\\
 0.240   0.205384617  0.205656592  -0.000259499648\\
 0.244   0.205538463  0.205803551  -0.000253510063\\
 0.248   0.205692309  0.205950836  -0.000247764715\\
 0.252   0.205846155  0.206098426  -0.000242250792\\
 0.256   0.206000001  0.206246300  -0.000236956180\\
 0.260   0.206153847  0.206394440  -0.000231869448\\
 0.264   0.206307693  0.206542829  -0.000226979812\\
 0.268   0.206461539  0.206691453  -0.000222277112\\
 0.272   0.206615385  0.206840296  -0.000217751776\\
 0.276   0.206769231  0.206989346  -0.000213394790\\
 0.280   0.206923077  0.207138590  -0.000209197652\\
 0.284   0.207076923  0.207288018  -0.000205152417\\
 0.288   0.207230770  0.207437618  -0.000201251652\\
 0.292   0.207384616  0.207587381  -0.000197488126\\
 0.296   0.207538462  0.207737299  -0.000193855248\\
 0.300   0.207692308  0.207887362  -0.000190346747\\
}\datatable

\addplot[thick, color1]
    table[x=U, y=Xo_true]{\datatable};
\addlegendentry{Nonsmooth dynamics}

\addplot[thick, color2]
    table[x=U, y=Xo]{\datatable};
\addlegendentry{Smoothed dynamics}

\pgfplotstableread[row sep=crcr]{
U        Xo        \\
0.0555   0.20000000\\
0.2050    0.20390551\\
}\datatable

\addplot[thick, color4]
    table[x=U, y=Xo]{\datatable};
\addlegendentry{\shortstack[l]{Smoothed dynamics\\[-4pt]linearized}}

\end{axis}

\end{tikzpicture}
      \label{fig:1d-pusher-dynamics-linearized-TOCTR}
    }
    \subfloat[]{%
      \hspace*{-3.6mm}
      \begin{tikzpicture}

\definecolor{color1}{RGB}{228,26,28}
\definecolor{color2}{RGB}{55,126,184}
\definecolor{color3}{RGB}{77,175,74}
\definecolor{color4}{RGB}{152,78,163}

\begin{axis}[
width=5.4cm, height=3.3cm,
xlabel={\small $v$},
ylabel={\small $z_o^+$},
xmin=-0.1, xmax=0.3,
ymin=0.1995, ymax=0.2080,
xtick=\empty,
ytick=\empty,
legend cell align={left},
legend style={
    font=\footnotesize,
    at={(0.00,1.00)},
    anchor=north west,
    draw=none,
    fill=none,
    row sep=-2pt,
    inner sep=0.5pt,
},
legend image post style={scale=0.6},
]

\pgfplotstableread[row sep=crcr]{
U        Xo_true      Xo           Eo             \\
-0.100   0.200000000  0.200195054  -0.000190346747\\
-0.096   0.200000000  0.200198837  -0.000193855248\\
-0.092   0.200000000  0.200202766  -0.000197488126\\
-0.088   0.200000000  0.200206849  -0.000201251652\\
-0.084   0.200000000  0.200211095  -0.000205152485\\
-0.080   0.200000000  0.200215513  -0.000209197696\\
-0.076   0.200000000  0.200220115  -0.000213394800\\
-0.072   0.200000000  0.200224912  -0.000217751778\\
-0.068   0.200000000  0.200229914  -0.000222277112\\
-0.064   0.200000001  0.200235137  -0.000226979812\\
-0.060   0.200000001  0.200240594  -0.000231869448\\
-0.056   0.200000001  0.200246300  -0.000236956180\\
-0.052   0.200000001  0.200252272  -0.000242250792\\
-0.048   0.200000001  0.200258529  -0.000247764715\\
-0.044   0.200000001  0.200265090  -0.000253510063\\
-0.040   0.200000001  0.200271977  -0.000259499648\\
-0.036   0.200000002  0.200279213  -0.000265747002\\
-0.032   0.200000002  0.200286826  -0.000272266385\\
-0.028   0.200000002  0.200294842  -0.000279072786\\
-0.024   0.200000003  0.200303293  -0.000286181904\\
-0.020   0.200000003  0.200312213  -0.000293610111\\
-0.016   0.200000000  0.200321640  -0.000301374389\\
-0.012   0.200000000  0.200331615  -0.000309492231\\
-0.008   0.200000000  0.200342182  -0.000317981503\\
-0.004   0.200000000  0.200353394  -0.000326860229\\
 0.000   0.200000000  0.200365304  -0.000336146325\\
 0.004   0.200000000  0.200377974  -0.000345857232\\
 0.008   0.200000000  0.200391472  -0.000356009443\\
 0.012   0.200000000  0.200405874  -0.000366617788\\
 0.016   0.200000000  0.200421262  -0.000377694787\\
 0.020   0.200000000  0.200437728  -0.000389249472\\
 0.024   0.200000000  0.200455375  -0.000401286177\\
 0.028   0.200000001  0.200474315  -0.000413802944\\
 0.032   0.200000001  0.200494673  -0.000426789600\\
 0.036   0.200000001  0.200516587  -0.000440225453\\
 0.040   0.200000001  0.200540209  -0.000454076609\\
 0.044   0.200000002  0.200565704  -0.000468292906\\
 0.048   0.200000003  0.200593255  -0.000482804550\\
 0.052   0.200000005  0.200623058  -0.000497518595\\
 0.056   0.200000008  0.200655325  -0.000512315520\\
 0.060   0.200000000  0.200690282  -0.000527046277\\
 0.064   0.200000000  0.200728168  -0.000541530361\\
 0.068   0.200000000  0.200769231  -0.000555555556\\
 0.072   0.200000001  0.200813723  -0.000568880124\\
 0.076   0.200000002  0.200861896  -0.000581238194\\
 0.080   0.200000005  0.200913996  -0.000592348878\\
 0.084   0.200000016  0.200970250  -0.000601929265\\
 0.088   0.200000001  0.201030863  -0.000609710761\\
 0.092   0.200000006  0.201096006  -0.000615457455\\
 0.096   0.200000002  0.201165807  -0.000618984461\\
 0.100   0.200002073  0.201240347  -0.000620173673\\
 0.104   0.200153848  0.201319653  -0.000618984461\\
 0.108   0.200307698  0.201403698  -0.000615457455\\
 0.112   0.200461539  0.201492401  -0.000609710761\\
 0.116   0.200615401  0.201585634  -0.000601929265\\
 0.120   0.200769236  0.201683226  -0.000592348874\\
 0.124   0.200923079  0.201784973  -0.000581238170\\
 0.128   0.201076924  0.201890646  -0.000568880124\\
 0.132   0.201230770  0.202000000  -0.000555555556\\
 0.136   0.201384616  0.202112783  -0.000541530361\\
 0.140   0.201538462  0.202228744  -0.000527046277\\
 0.144   0.201692316  0.202347632  -0.000512315519\\
 0.148   0.201846159  0.202469212  -0.000497518592\\
 0.152   0.202000003  0.202593255  -0.000482804538\\
 0.156   0.202153848  0.202719550  -0.000468292815\\
 0.160   0.202307694  0.202847901  -0.000454076609\\
 0.164   0.202461540  0.202978126  -0.000440225453\\
 0.168   0.202615385  0.203110058  -0.000426789600\\
 0.172   0.202769231  0.203243546  -0.000413802944\\
 0.176   0.202923077  0.203378452  -0.000401286175\\
 0.180   0.203076923  0.203514651  -0.000389249472\\
 0.184   0.203230769  0.203652031  -0.000377694629\\
 0.188   0.203384616  0.203790489  -0.000366617788\\
 0.192   0.203538462  0.203929934  -0.000356009427\\
 0.196   0.203692308  0.204070282  -0.000345857232\\
 0.200   0.203846154  0.204211458  -0.000336146322\\
 0.204   0.204000000  0.204353394  -0.000326860225\\
 0.208   0.204153846  0.204496028  -0.000317981456\\
 0.212   0.204307692  0.204639307  -0.000309492230\\
 0.216   0.204461539  0.204783178  -0.000301374387\\
 0.220   0.204615385  0.204927598  -0.000293610110\\
 0.224   0.204769234  0.205072524  -0.000286181903\\
 0.228   0.204923079  0.205217919  -0.000279072786\\
 0.232   0.205076925  0.205363749  -0.000272266326\\
 0.236   0.205230771  0.205509983  -0.000265746926\\
 0.240   0.205384617  0.205656592  -0.000259499648\\
 0.244   0.205538463  0.205803551  -0.000253510063\\
 0.248   0.205692309  0.205950836  -0.000247764715\\
 0.252   0.205846155  0.206098426  -0.000242250792\\
 0.256   0.206000001  0.206246300  -0.000236956180\\
 0.260   0.206153847  0.206394440  -0.000231869448\\
 0.264   0.206307693  0.206542829  -0.000226979812\\
 0.268   0.206461539  0.206691453  -0.000222277112\\
 0.272   0.206615385  0.206840296  -0.000217751776\\
 0.276   0.206769231  0.206989346  -0.000213394790\\
 0.280   0.206923077  0.207138590  -0.000209197652\\
 0.284   0.207076923  0.207288018  -0.000205152417\\
 0.288   0.207230770  0.207437618  -0.000201251652\\
 0.292   0.207384616  0.207587381  -0.000197488126\\
 0.296   0.207538462  0.207737299  -0.000193855248\\
 0.300   0.207692308  0.207887362  -0.000190346747\\
}\datatable

\addplot[thick, color1, forget plot]
    table[x=U, y=Xo_true]{\datatable};

\addplot[thick, color2, forget plot]
    table[x=U, y=Xo]{\datatable};

\pgfplotstableread[row sep=crcr]{
U        Xo           Eo       \\
0.083   0.20071143  -0.00070514\\
0.167   0.20291063  -0.00045136\\
}\datatable

\addplot[thick, color4, forget plot]
    table[x=U, y=Xo]{\datatable};

\addplot[name path=lower, draw=none, forget plot]
    table[x=U, y expr=\thisrow{Xo} + \thisrow{Eo}]{\datatable};

\addplot[name path=upper, draw=none, forget plot]
    table[x=U, y expr=\thisrow{Xo} + 2*\thisrow{Eo}]{\datatable};

\addplot[fill=color3, fill opacity=0.4, draw=none, legend image code/.code={
        \draw[fill=color3, fill opacity=0.4, draw=none] (0cm,-0.1cm) rectangle (0.6cm,0.1cm);
    }]
    fill between[of=lower and upper];
\addlegendentry{\shortstack[l]{Smoothed dynamics\\[-4pt]linearized\\[-2pt]+ deviation bounds}}

\end{axis}

\end{tikzpicture}
      \label{fig:1d-pusher-dynamics-linearized-iSLS}
    }
    \caption{Linearization of the smoothed dynamics $z^+ = f_\kappa(z,v)$ for a 1D pusher as viewed by \textbf{(a)} \ac{TO-CTR} and \textbf{(b)} our method.}
    \label{fig:1d-pusher-dynamics-linearized}
\end{figure}
\begin{figure}[!t]
    \centering
    \vspace*{-2.5mm}
    \subfloat[]{%
      \hspace*{-1.5mm}
      \begin{tikzpicture}

\definecolor{color1}{RGB}{228,26,28}
\definecolor{color2}{RGB}{55,126,184}
\definecolor{color3}{RGB}{77,175,74}
\definecolor{color4}{RGB}{152,78,163}

\begin{axis}[
width=5.4cm, height=3.1cm,
xlabel={\small $t$},
ylabel={\small $x_o$},
xmin=0, xmax=4,
ymin=0.2, ymax=0.53,
xtick=\empty,
ytick=\empty,
legend cell align={left},
legend style={
    font=\footnotesize,
    at={(1.00,0.00)},
    anchor=south east,
    draw=none,
    fill=none,
    row sep=-2pt,
    inner sep=0.5pt,
},
legend image post style={scale=0.6},
]

\pgfplotstableread[row sep=crcr]{
t          Xo_nominal Xo_true   \\
0.         0.2        0.2       \\
0.02       0.2003653  0.2       \\
0.04       0.20737958 0.20680899\\
0.06       0.22128527 0.220626  \\
0.08       0.23940491 0.23868608\\
0.1        0.25981618 0.25904963\\
0.12       0.28119443 0.2803854 \\
0.14       0.30265348 0.30180387\\
0.16       0.32362234 0.32273204\\
0.18       0.34375211 0.34281963\\
0.2        0.36284572 0.36186853\\
0.22       0.38080433 0.37977905\\
0.24       0.39758538 0.39650786\\
0.26       0.41316848 0.41203367\\
0.28       0.42752554 0.42632723\\
0.3        0.44059273 0.43932277\\
0.32       0.4522416  0.45088841\\
0.34       0.46224789 0.460793  \\
0.36       0.47025906 0.4686681 \\
0.38       0.4757758  0.47396716\\
0.4        0.47829677 0.47594743\\
0.42       0.47887617 0.47594743\\
0.44       0.47915306 0.47594743\\
0.46       0.47934834 0.47594744\\
0.48       0.4795082  0.47594744\\
0.5        0.4796491  0.47594744\\
0.52       0.47977864 0.47594744\\
0.54       0.47990091 0.47594744\\
0.56       0.48001833 0.47594744\\
0.58       0.4801324  0.47594744\\
0.6        0.48024414 0.47594744\\
0.62       0.48035419 0.47594744\\
0.64       0.48046304 0.47594744\\
0.66       0.48057101 0.47594744\\
0.68       0.48067832 0.47594744\\
0.7        0.48078516 0.47594744\\
0.72       0.48089163 0.47594744\\
0.74       0.48099782 0.47594744\\
0.76       0.4811038  0.47594744\\
0.78       0.48120961 0.47594744\\
0.8        0.48131529 0.47594744\\
0.82       0.48142085 0.47594744\\
0.84       0.48152633 0.47594744\\
0.86       0.48163172 0.47594744\\
0.88       0.48173705 0.47594744\\
0.9        0.48184231 0.47594744\\
0.92       0.48194752 0.47594744\\
0.94       0.48205267 0.47594744\\
0.96       0.48215777 0.47594744\\
0.98       0.48226283 0.47594744\\
1.         0.48236784 0.47594744\\
1.02       0.4824728  0.47594744\\
1.04       0.48257772 0.47594744\\
1.06       0.48268259 0.47594744\\
1.08       0.48278742 0.47594744\\
1.1        0.48289221 0.47594744\\
1.12       0.48299696 0.47594744\\
1.14       0.48310166 0.47594744\\
1.16       0.48320632 0.47594744\\
1.18       0.48331095 0.47594744\\
1.2        0.48341552 0.47594744\\
1.22       0.48352006 0.47594744\\
1.24       0.48362456 0.47594744\\
1.26       0.48372901 0.47594744\\
1.28       0.48383343 0.47594744\\
1.3        0.48393781 0.47594744\\
1.32       0.48404214 0.47594744\\
1.34       0.48414643 0.47594744\\
1.36       0.48425069 0.47594744\\
1.38       0.4843549  0.47594744\\
1.4        0.48445907 0.47594744\\
1.42       0.48456321 0.47594744\\
1.44       0.4846673  0.47594744\\
1.46       0.48477136 0.47594744\\
1.48       0.48487537 0.47594744\\
1.5        0.48497935 0.47594744\\
1.52       0.48508329 0.47594744\\
1.54       0.48518718 0.47594744\\
1.56       0.48529104 0.47594744\\
1.58       0.48539487 0.47594744\\
1.6        0.48549865 0.47594744\\
1.62       0.48560239 0.47594744\\
1.64       0.4857061  0.47594744\\
1.66       0.48580976 0.47594744\\
1.68       0.48591339 0.47594744\\
1.7        0.48601698 0.47594744\\
1.72       0.48612054 0.47594744\\
1.74       0.48622405 0.47594744\\
1.76       0.48632753 0.47594744\\
1.78       0.48643097 0.47594744\\
1.8        0.48653438 0.47594744\\
1.82       0.48663774 0.47594744\\
1.84       0.48674107 0.47594744\\
1.86       0.48684436 0.47594744\\
1.88       0.48694762 0.47594744\\
1.9        0.48705084 0.47594744\\
1.92       0.48715402 0.47594744\\
1.94       0.48725717 0.47594744\\
1.96       0.48736028 0.47594744\\
1.98       0.48746335 0.47594744\\
2.         0.48756639 0.47594744\\
2.02       0.48766939 0.47594744\\
2.04       0.48777235 0.47594744\\
2.06       0.48787528 0.47594745\\
2.08       0.48797818 0.47594745\\
2.1        0.48808103 0.47594745\\
2.12       0.48818386 0.47594745\\
2.14       0.48828664 0.47594745\\
2.16       0.4883894  0.47594745\\
2.18       0.48849211 0.47594745\\
2.2        0.4885948  0.47594745\\
2.22       0.48869744 0.47594745\\
2.24       0.48880006 0.47594745\\
2.26       0.48890264 0.47594745\\
2.28       0.48900518 0.47594745\\
2.3        0.48910769 0.47594745\\
2.32       0.48921016 0.47594745\\
2.34       0.4893126  0.47594745\\
2.36       0.48941501 0.47594745\\
2.38       0.48951738 0.47594745\\
2.4        0.48961972 0.47594745\\
2.42       0.48972202 0.47594745\\
2.44       0.48982429 0.47594745\\
2.46       0.48992653 0.47594745\\
2.48       0.49002874 0.47594745\\
2.5        0.49013091 0.47594745\\
2.52       0.49023305 0.47594745\\
2.54       0.49033515 0.47594745\\
2.56       0.49043722 0.47594745\\
2.58       0.49053926 0.47594745\\
2.6        0.49064126 0.47594745\\
2.62       0.49074324 0.47594745\\
2.64       0.49084518 0.47594745\\
2.66       0.49094709 0.47594745\\
2.68       0.49104896 0.47594745\\
2.7        0.49115081 0.47594745\\
2.72       0.49125262 0.47594745\\
2.74       0.4913544  0.47594745\\
2.76       0.49145614 0.47594745\\
2.78       0.49155786 0.47594745\\
2.8        0.49165954 0.47594745\\
2.82       0.49176119 0.47594745\\
2.84       0.49186281 0.47594745\\
2.86       0.4919644  0.47594745\\
2.88       0.49206596 0.47594745\\
2.9        0.49216749 0.47594745\\
2.92       0.49226898 0.47594745\\
2.94       0.49237044 0.47594745\\
2.96       0.49247188 0.47594745\\
2.98       0.49257328 0.47594745\\
3.         0.49267465 0.47594745\\
3.02       0.49277599 0.47594745\\
3.04       0.4928773  0.47594745\\
3.06       0.49297858 0.47594745\\
3.08       0.49307983 0.47594745\\
3.1        0.49318105 0.47594745\\
3.12       0.49328223 0.47594745\\
3.14       0.49338339 0.47594745\\
3.16       0.49348452 0.47594745\\
3.18       0.49358562 0.47594745\\
3.2        0.49368669 0.47594745\\
3.22       0.49378773 0.47594745\\
3.24       0.49388874 0.47594745\\
3.26       0.49398972 0.47594745\\
3.28       0.49409067 0.47594745\\
3.3        0.49419159 0.47594745\\
3.32       0.49429248 0.47594745\\
3.34       0.49439334 0.47594745\\
3.36       0.49449418 0.47594745\\
3.38       0.49459498 0.47594745\\
3.4        0.49469576 0.47594745\\
3.42       0.4947965  0.47594745\\
3.44       0.49489722 0.47594745\\
3.46       0.49499791 0.47594745\\
3.48       0.49509857 0.47594745\\
3.5        0.4951992  0.47594745\\
3.52       0.49529981 0.47594745\\
3.54       0.49540038 0.47594745\\
3.56       0.49550093 0.47594745\\
3.58       0.49560145 0.47594745\\
3.6        0.49570194 0.47594745\\
3.62       0.49580241 0.47594745\\
3.64       0.49590284 0.47594745\\
3.66       0.49600325 0.47594745\\
3.68       0.49610363 0.47594745\\
3.7        0.49620398 0.47594745\\
3.72       0.49630431 0.47594745\\
3.74       0.49640461 0.47594745\\
3.76       0.49650488 0.47594745\\
3.78       0.49660512 0.47594745\\
3.8        0.49670534 0.47594745\\
3.82       0.49680553 0.47594745\\
3.84       0.49690569 0.47594745\\
3.86       0.49700582 0.47594745\\
3.88       0.49710593 0.47594745\\
3.9        0.49720602 0.47594745\\
3.92       0.49730607 0.47594745\\
3.94       0.4974061  0.47594745\\
3.96       0.49750611 0.47594745\\
3.98       0.49760609 0.47594745\\
4.         0.49770604 0.47594745\\
}\datatable

\addplot[semithick, densely dashed, forget plot] {0.5};

\addplot[thick, color2]
    table[x=t, y=Xo_nominal]{\datatable};
\addlegendentry{Nominal trajectory}

\addplot[thick, color1]
    table[x=t, y=Xo_true]{\datatable};
\addlegendentry{\shortstack[l]{Nonsmooth dynamics\\[-4pt]rollout}}

\end{axis}

\end{tikzpicture}
      \label{fig:1d-rollout-TOCTR}
    }
    \subfloat[]{%
      \hspace*{-2.1mm}
      \input{figures/1d-pusher-rollout-iSLS}
      \label{fig:1d-rollout-iSLS}
    }
    \caption{Rollouts of the 1D pusher example for pushing the object to reach a goal position (dashed line) using \textbf{(a)} \ac{TO-CTR} and \textbf{(b)} our method.}
    \label{fig:1d-pusher-rollout}
\end{figure}
Consider the 1D pusher–object system (\cref{fig:1d-pusher}). The linearized dynamics used by \ac{TO-CTR} and our method are shown in \cref{fig:1d-pusher-dynamics-linearized}. 
\ac{TO-CTR} assumes that the linearization of the smoothed dynamics accurately approximates the nonsmooth hybrid dynamics over control perturbations $\delta v$ satisfying the \ac{CTR} constraint $\lambda_\kappa + \tfrac{\partial \lambda_\kappa}{\partial v}\delta v \geq 0$ and $|\delta v| \leq \varepsilon$. As illustrated, this approximation error remains non-negligible within the assumed region.
In contrast, our method explicitly bounds the deviation between the smoothed and nonsmooth hybrid dynamics, guaranteeing that for all $|\delta v| \leq \varepsilon'$, the nonsmooth dynamics lies within a known deviation of the smoothed model.

This difference is evident in the rollouts (\cref{fig:1d-pusher-rollout}). \ac{TO-CTR} exploits a force-at-a-distance artifact of the smoothed dynamics to reach the goal, but fails when executed on the nonsmooth dynamics (\cref{fig:1d-rollout-TOCTR}). Our method compensates for this discrepancy during optimization, overshooting the nominal trajectory so that rollout on the nonsmooth hybrid dynamics reaches the goal (\cref{fig:1d-rollout-iSLS}).
\end{example}

\section{Experiment Results and Discussion}
\label{sec:results}

\begin{table}[t]
    \caption{Task parameters}
    \label{tab:task-parameters}
    \centering
    \setlength{\tabcolsep}{5pt}
    \vspace*{-3mm}
    \begin{tabular}{l | c c c c}
      \hline
                             & \makecell{Time step\\size (s)} & \makecell{Time\\steps} & \makecell{Dynamics\\smoothing} & \makecell{Geometry\\smoothing} \\
      \hline
      Bimanual Planar Bucket & 0.02 & 100 & $10^{-3}$ & $0$ \\
      Bimanual Planar Box    & 0.02 & 100 & $10^{-3}$ & $10^{-3}$ \\
      In-hand cube           & 0.02 &  50 & $10^{-4}$ & $0$ \\
      \hline
    \end{tabular}
\end{table}
Through a series of simulated and real-world experiments, we seek to answer the following research questions (RQs):
\begin{enumerate}[label=\textbf{RQ\arabic*.}, leftmargin=*]
\item Do closed-loop rollouts of the nonsmooth hybrid dynamics remain within the predicted tube?
\item Does our method achieve lower goal-reaching error than the baseline while strictly enforcing constraints?
\item Does performance degrade due to strict constraint enforcement?
\item Is contact geometry smoothing necessary in practice?
\item Can our method scale to high-dimensional systems?
\end{enumerate}
Parameters for the experimented tasks are listed in \cref{tab:task-parameters}. We perform all computations on a desktop computer with an Intel Core Ultra 9 285K processor and 64 GB RAM.

\subsection{Bimanual Planar Bucket Manipulation}

\begin{figure}[!t]
    \centering
    \vspace{5mm}
    \subfloat[]{%
        \adjincludegraphics[height=2.0cm, trim={{.04\width} {.16\height} {.14\width} {.15\height}}, clip]{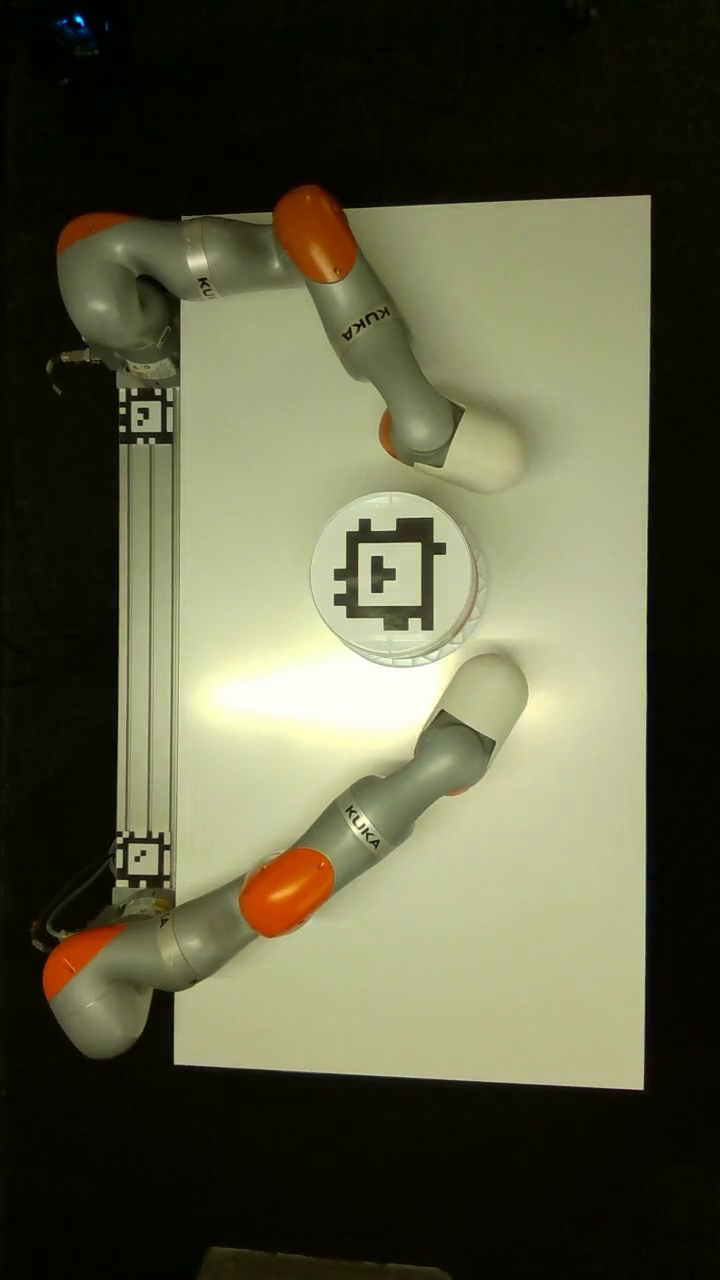}
        \adjincludegraphics[height=2.0cm, trim={{.04\width} {.16\height} {.14\width} {.15\height}}, clip]{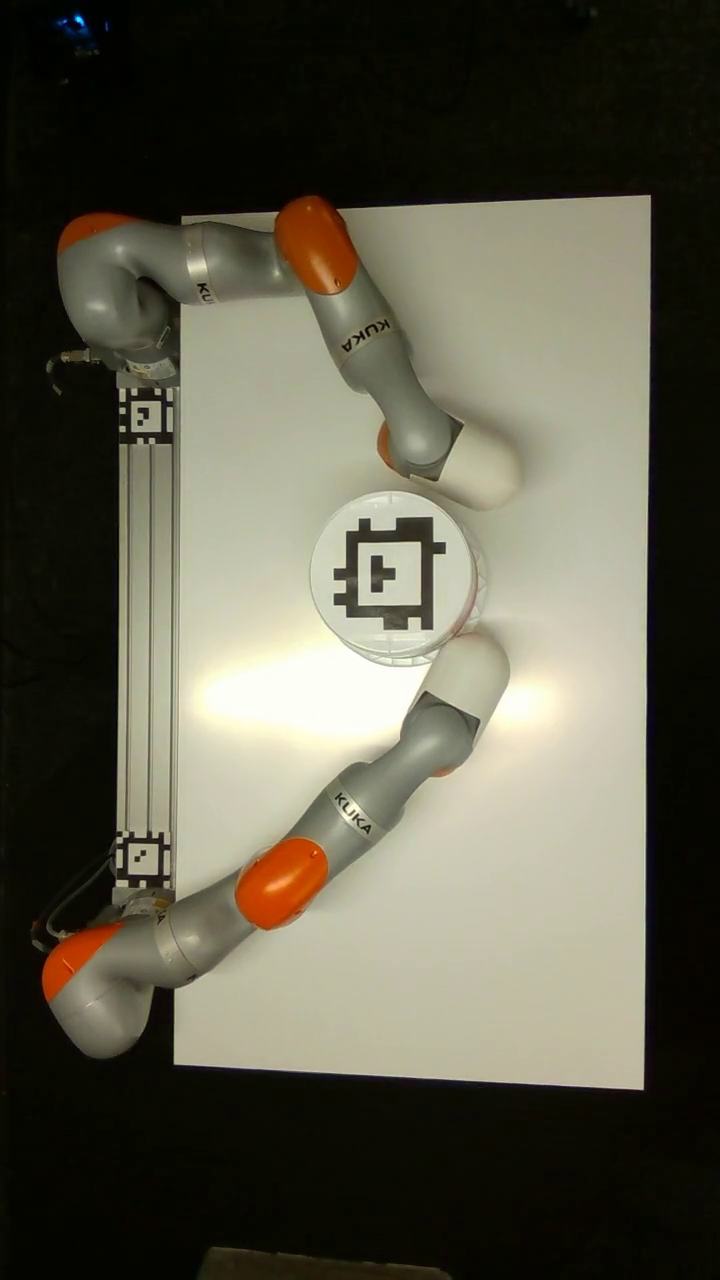}
        \adjincludegraphics[height=2.0cm, trim={{.04\width} {.16\height} {.14\width} {.15\height}}, clip]{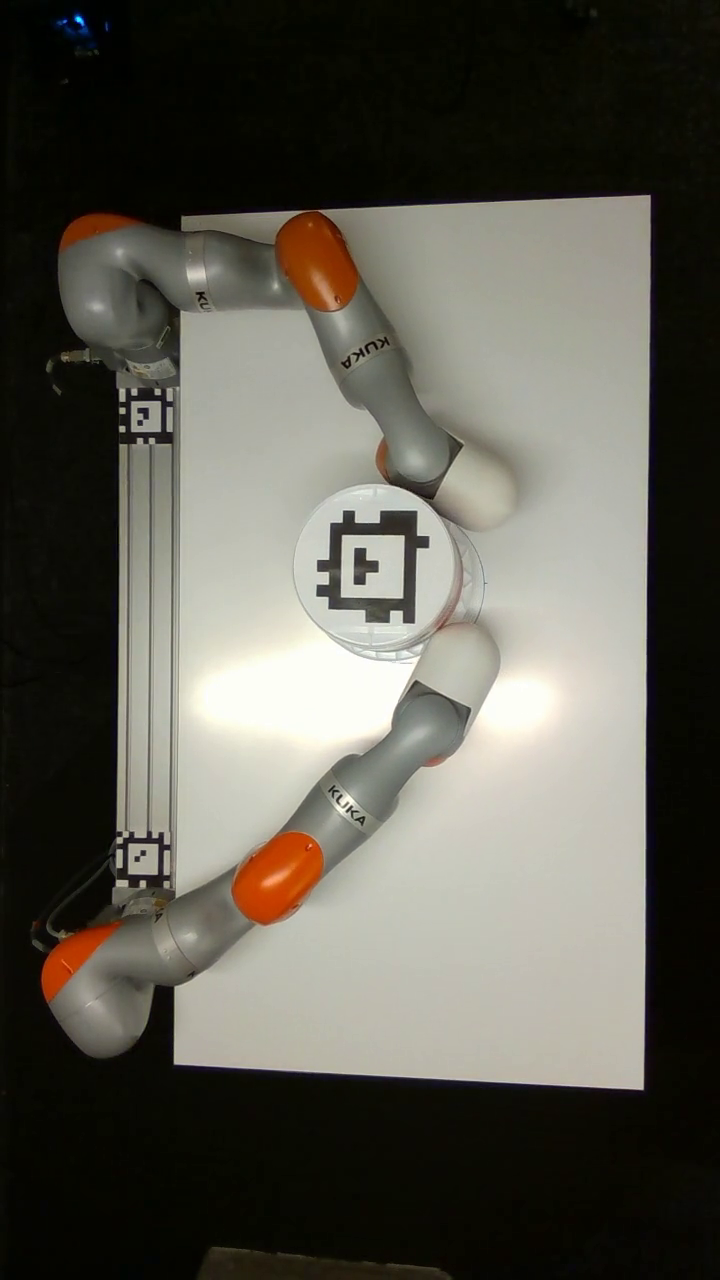}
        \adjincludegraphics[height=2.0cm, trim={{.04\width} {.16\height} {.14\width} {.15\height}}, clip]{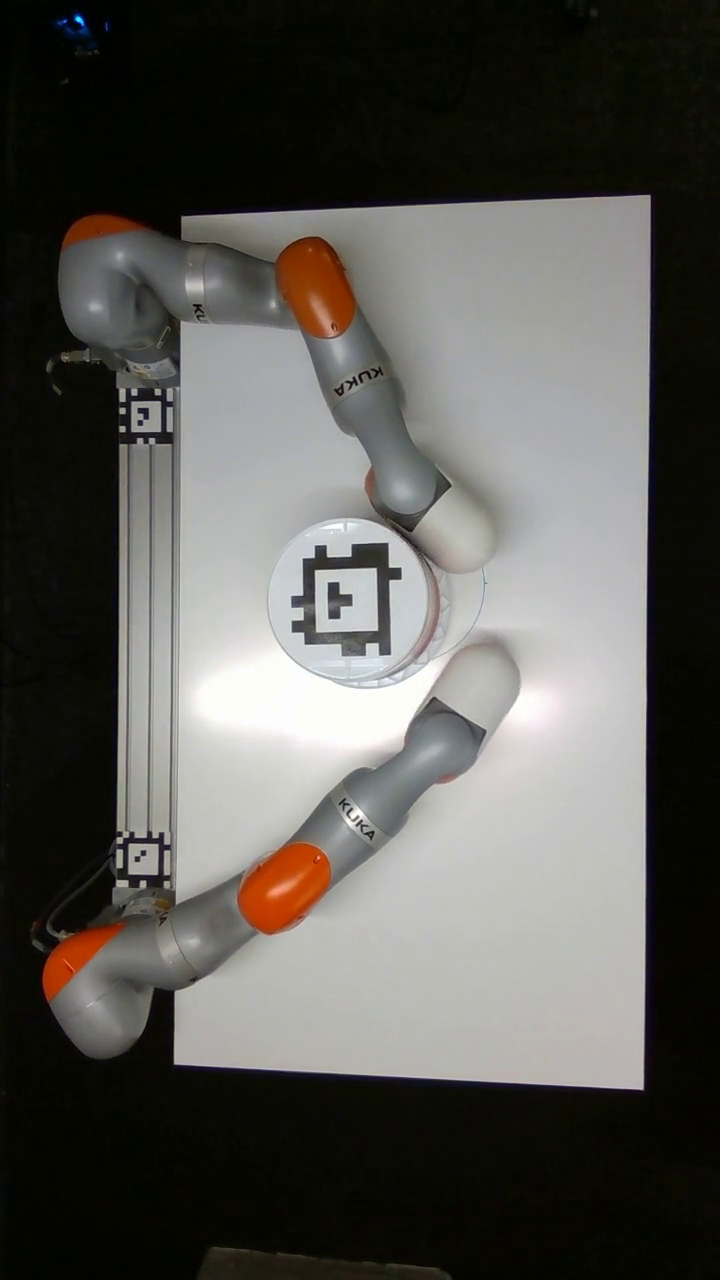}
        \adjincludegraphics[height=2.0cm, trim={{.04\width} {.16\height} {.14\width} {.15\height}}, clip]{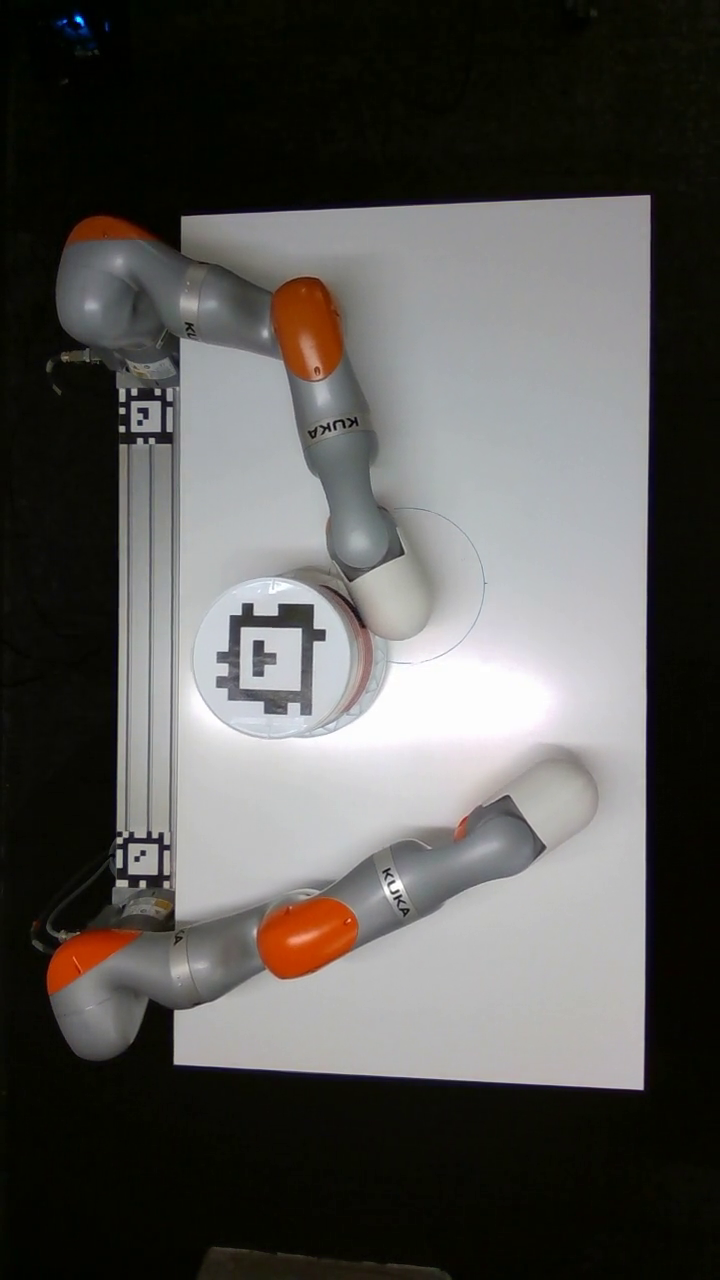}
        \adjincludegraphics[height=2.0cm, trim={{.04\width} {.16\height} {.14\width} {.15\height}}, clip]{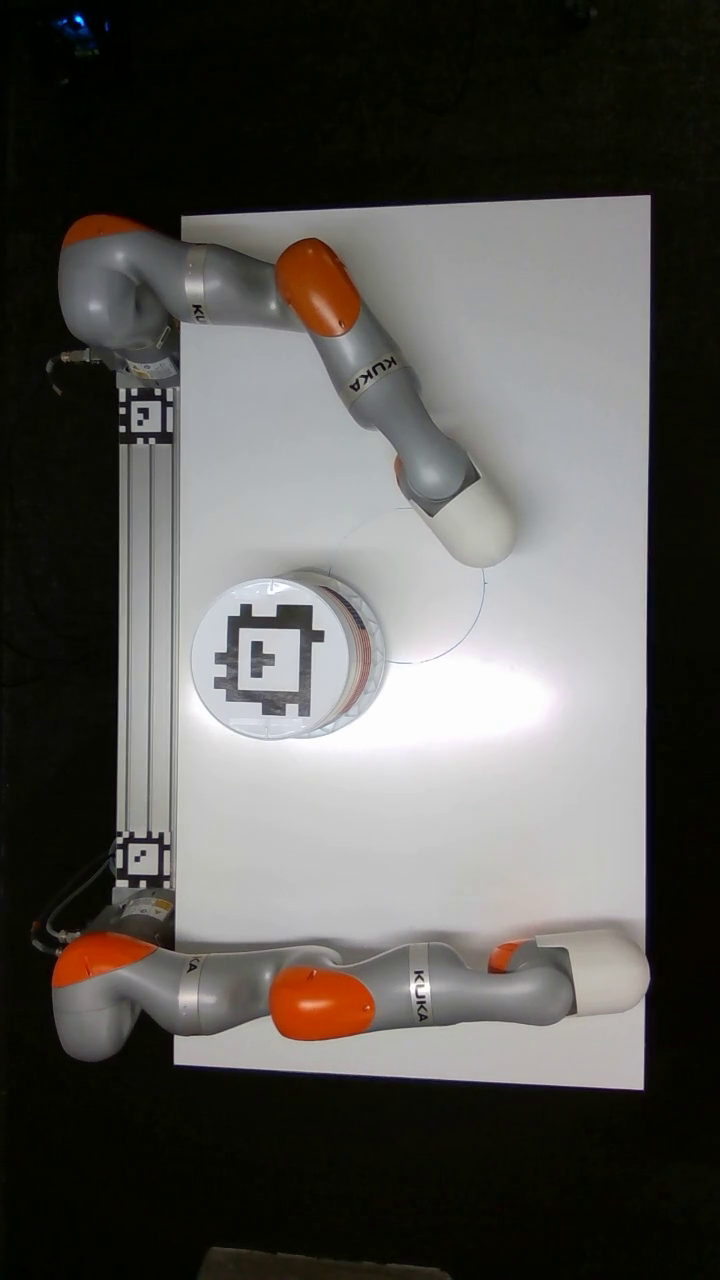}
    } \\
    \subfloat[]{%
        \hspace*{-2mm}
        \input{figures/iiwa-bucket-x} \hspace*{-5mm}
        \input{figures/iiwa-bucket-y} \hspace*{-5mm}
        \input{figures/iiwa-bucket-theta}
    }
    \caption{Bimanual planar bucket manipulation on hardware. \textbf{(a)} Keyframes from the executed rollout. \textbf{(b)} Bucket state trajectory and predicted tube.}
    \label{fig:iiwa-bucket-rollout}
\end{figure}

\label{sec:planar-bucket}
\begin{table}[t]
    \caption{Comparison of algorithms on bimanual bucket manipulation}
    \label{tab:algorithms-comparison}
    \centering
    \setlength{\tabcolsep}{3pt}
    \vspace*{-3mm}
    \begin{tabular}{l c c c c}
      \hline
      & \textbf{TO-CTR} & \textbf{Ours} & \textbf{\makecell{TO-CTR\\with MPC}} & \textbf{\makecell{Ours\\with MPC}} \\
      \hline
      MPC horizon                            &    -   &    -   &    1   &    1   \\
      Goal position error (m)                & 0.1522 & 0.1447 & 0.1558 & 0.1538 \\
      Goal angle error (rad)                 & 0.1034 & 0.0946 & 0.0985 & 0.0987 \\
      Nominal traj. $(\vec{z},\vec{v})$ cost & 313.51 & 299.43 &    -   &    -   \\
      Real traj. $(\vec{x},\vec{u})$ cost    & 314.66 & 296.41 & 318.02 & 309.39 \\
      Constraint violating rollouts          & 3.1\%  & 0.0\%  & 1.2\%  & 0.0\%  \\
      Compute time (s)                       & 88.2   & 85.9   & 6.36   & 6.99   \\
      \hline
    \end{tabular}
\end{table}

\begin{figure}[t]
    \centering
    \subfloat[\hspace*{-8.3mm}]{%
      \begin{tikzpicture}

\definecolor{color1}{RGB}{127,201,127}
\definecolor{color2}{RGB}{190,174,212}

\begin{axis}[
width=5cm,
height=4cm,
xmin=0.0, xmax=0.6,
ymin=0.0, ymax=80,
xlabel={\footnotesize Position + angle goal error},
ylabel={\footnotesize Number of samples},     
tick label style={font=\footnotesize},
legend cell align={left},
legend style={
    font=\footnotesize,
    at={(1.00,1.00)},
    anchor=north east,
    draw=none,
    fill=none,
    row sep=-2pt,
    inner sep=0pt,
},
ylabel shift=-1mm,
]

\addplot+[
    ybar interval,
    fill=color1,
    draw=none,
    fill opacity=0.5,
    legend image code/.code={
        \draw[fill=color1, fill opacity=0.5, draw=none] (0cm,-0.05cm) rectangle (0.3cm,0.05cm);
    },
] 
coordinates {
(0.01, 1)
(0.03, 10)
(0.05, 19)
(0.07, 24)
(0.09, 34)
(0.11, 39)
(0.13, 48)
(0.15, 56)
(0.17, 73)
(0.19, 70)
(0.21, 62)
(0.23, 59)
(0.25, 68)
(0.27, 45)
(0.29, 55)
(0.31, 49)
(0.33, 36)
(0.35, 36)
(0.37, 34)
(0.39, 36)
(0.41, 28)
(0.43, 43)
(0.45, 18)
(0.47, 20)
(0.49, 13)
(0.51, 10)
(0.53, 5)
(0.55, 3)
(0.57, 4)
(0.59, 2)
};
\addlegendentry{TO-CTR}

\addplot+[
    ybar interval,
    fill=color2,
    draw=none,
    fill opacity=0.5,
    legend image code/.code={
        \draw[fill=color2, fill opacity=0.5, draw=none] (0cm,-0.05cm) rectangle (0.3cm,0.05cm);
    },
] 
coordinates {
(0.01, 2)
(0.03, 10)
(0.05, 25)
(0.07, 24)
(0.09, 39)
(0.11, 49)
(0.13, 62)
(0.15, 72)
(0.17, 68)
(0.19, 73)
(0.21, 61)
(0.23, 68)
(0.25, 54)
(0.27, 59)
(0.29, 48)
(0.31, 40)
(0.33, 37)
(0.35, 39)
(0.37, 36)
(0.39, 32)
(0.41, 32)
(0.43, 18)
(0.45, 20)
(0.47, 10)
(0.49, 10)
(0.51, 5)
(0.53, 2)
(0.55, 3)
(0.57, 2)
(0.59, 0)
};
\addlegendentry{Ours}

\addplot[color1, very thick, densely dashed] coordinates {(0.2586, 0) (0.2586, 1000)};
\addplot[color2, very thick, densely dashed] coordinates {(0.2393, 0) (0.2393, 1000)};

\end{axis}

\end{tikzpicture}
      \label{fig:planar-iiwa-noreplan-histogram}
    }
    \subfloat[]{%
      \hspace*{-4.5mm}
      \begin{tikzpicture}

\definecolor{color1}{RGB}{127,201,127}
\definecolor{color2}{RGB}{190,174,212}

\begin{axis}[
width=5cm,
height=4cm,
xmin=0.0, xmax=0.6,
ymin=0.0, ymax=80,
xlabel={\footnotesize Position + angle goal error},
tick label style={font=\footnotesize},
legend cell align={left},
legend style={
    font=\footnotesize,
    at={(1.00,1.00)},
    anchor=north east,
    draw=none,
    fill=none,
    row sep=-2pt,
    inner sep=0pt,
},
]

\addplot+[
    ybar interval,
    fill=color1,
    draw=none,
    fill opacity=0.5,
    legend image code/.code={
        \draw[fill=color1, fill opacity=0.5, draw=none] (0cm,-0.05cm) rectangle (0.3cm,0.05cm);
    },
] 
coordinates {
(0.01, 2)
(0.03, 10)
(0.05, 22)
(0.07, 24)
(0.09, 38)
(0.11, 44)
(0.13, 45)
(0.15, 58)
(0.17, 68)
(0.19, 77)
(0.21, 57)
(0.23, 56)
(0.25, 63)
(0.27, 50)
(0.29, 47)
(0.31, 40)
(0.33, 47)
(0.35, 39)
(0.37, 37)
(0.39, 36)
(0.41, 27)
(0.43, 37)
(0.45, 20)
(0.47, 16)
(0.49, 13)
(0.51, 10)
(0.53, 9)
(0.55, 3)
(0.57, 3)
(0.59, 2)
};
\addlegendentry{TO-CTR}

\addplot+[
    ybar interval,
    fill=color2,
    draw=none,
    fill opacity=0.5,
    legend image code/.code={
        \draw[fill=color2, fill opacity=0.5, draw=none] (0cm,-0.05cm) rectangle (0.3cm,0.05cm);
    },
] 
coordinates {
(0.01, 5)
(0.03, 10)
(0.05, 19)
(0.07, 23)
(0.09, 39)
(0.11, 43)
(0.13, 49)
(0.15, 54)
(0.17, 76)
(0.19, 70)
(0.21, 59)
(0.23, 57)
(0.25, 63)
(0.27, 47)
(0.29, 48)
(0.31, 46)
(0.33, 43)
(0.35, 39)
(0.37, 35)
(0.39, 37)
(0.41, 24)
(0.43, 37)
(0.45, 20)
(0.47, 18)
(0.49, 13)
(0.51, 11)
(0.53, 7)
(0.55, 3)
(0.57, 3)
(0.59, 2)
};
\addlegendentry{Ours}

\addplot[color1, very thick, densely dashed] coordinates {(0.2543, 0) (0.2543, 1000)};
\addplot[color2, very thick, densely dashed] coordinates {(0.2525, 0) (0.2525, 1000)};

\end{axis}

\end{tikzpicture}
      \label{fig:planar-iiwa-mpc-histogram}
    }
    \caption{Goal position error plus goal angle error for bimanual planar bucket manipulation. \textbf{(a)} Without replanning. \textbf{(b)} With MPC.}
    \label{fig:planar-iiwa-histogram}
\end{figure}

For this task, the goal is to reorient and translate the bucket to a desired target pose. We first evaluate our method on bimanual hardware, using two Kuka iiwa 7 arms, each acting as a 3-DoF finger (Fig.~\ref{fig:iiwa-bucket-rollout}).
In the hardware rollout, the bucket reaches the goal configuration with only a small steady-state error, in spite of sim-to-real modeling errors such as mass and joint stiffness. We believe this robustness is due in part to the deployed policy \eqref{eq:affine-policy} having an integral structure, which enables the rejection of constant disturbances.
Furthermore, aside from noise in the position estimates, the bucket remains largely within the computed tubes, affirming \textbf{RQ1}. Strict containment can be achieved by inflating $E_\kappa(x,u)$ with an additional bound to account for perception noise \cite{antoine2026visionsls, chou2022safe}.

We also use this task to compare against the baseline (Sec.~\ref{sec:baseline-ctr}) in simulation. We sampled a set of initial and goal states and computed a trajectory or policy using both the baseline and our method. Samples deemed infeasible by either algorithm were discarded. A quantitative comparison between our method and the baseline is presented in \cref{tab:algorithms-comparison} and \cref{fig:planar-iiwa-histogram}. All reported metrics are averaged over 1000 experiments.

Without replanning, our method achieves lower average goal position and orientation error than both the open-loop baseline and the baseline with \ac{MPC}. We therefore confirm \textbf{RQ2}: our method improves goal-reaching accuracy, although the improvement is modest. This is likely because the planning horizon is not long enough for smoothing errors to accumulate sufficiently to invalidate the \ac{TO-CTR} trajectory.

Although our method can also be executed in an \ac{MPC} fashion, it provides no clear advantage over \ac{TO-CTR} with \ac{MPC}, aside from reduced constraint violation. Because our method already yields an affine feedback policy that captures the unilateral nature of contact and enables inexpensive online rollouts, additional replanning through \ac{MPC} is largely unnecessary unless the system encounters additional disturbances not accounted for during planning.

\cref{tab:algorithms-comparison} also indicates that several \ac{TO-CTR} rollouts on the nonsmooth hybrid dynamics result in constraint violations. These violations manifest either as robot collisions (\cref{fig:iiwa-bucket-collision}) or as the bucket falling off the supporting table (\cref{fig:iiwa-bucket-fall}). In both cases, the nominal trajectories generated by \ac{TO-CTR} are feasible under the smoothed model; however, executing the nominal controls on the nonsmooth hybrid dynamics leads to constraint violations.

We also compare against a method that uses exact dynamics in the forward pass and smoothed gradients in the backward pass, as in \cite{kim2025contact} (\cref{fig:iiwa-bucket-fallfree-iLQR}), which is likewise insufficient to guarantee constraint satisfaction. This failure arises because gradient bias prevents state constraints from being accurately translated into control constraints. Furthermore, the resulting linearization error is first-order in $\|\delta \bm{u}\|$. As a result, the local model is never sufficiently accurate, and reducing $\|\delta \bm{u}\|$ alone cannot recover feasibility guarantees.

In contrast, our method certifies constraint satisfaction for the nonsmooth hybrid dynamics. Furthermore, as shown in \cref{fig:iiwa-bucket-fallfree-SLS}, our method optimizes a policy that minimizes the reachable tube, thereby keeping the nominal trajectory and the true rollout closer than \ac{TO-CTR}.

We also address \textbf{RQ3}. Strict constraint enforcement does not degrade performance. In fact, our method achieves lower nominal trajectory cost, requiring smaller control effort and incurring lower state cost than the baseline. This suggests that the tube-based reasoning enables the optimizer to select contacts less conservatively. In contrast, the baseline’s \ac{CTR} condition appears overly conservative relative to the tube characterization, prematurely marking certain transitions as infeasible.

Overall, this experiment validates \textbf{RQ1–RQ3}. Closed-loop executions of the nonsmooth hybrid dynamics remain within the predicted tubes (RQ1). We obtain modest improvements in goal-reaching error without sacrificing performance (RQ2–RQ3), while providing a formal certificate of robust constraint satisfaction.

\begin{figure}[!t]
    \centering
    \subfloat[]{%
        \adjincludegraphics[height=1.4cm, angle=90, trim={{.09\width} {.07\height} {.08\width} 0}, clip]{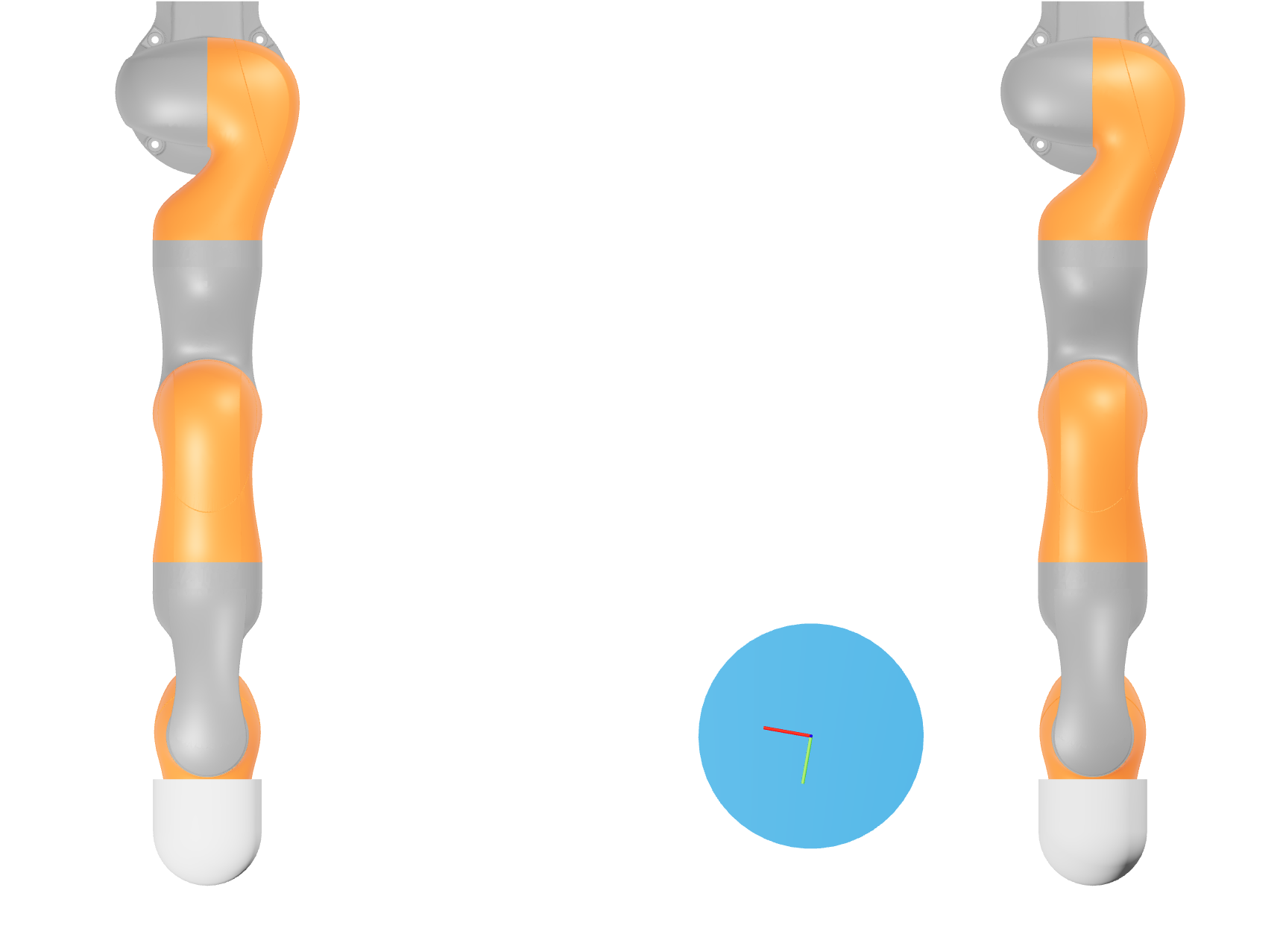}
        \adjincludegraphics[height=1.4cm, angle=90, trim={{.09\width} {.07\height} {.08\width} 0}, clip]{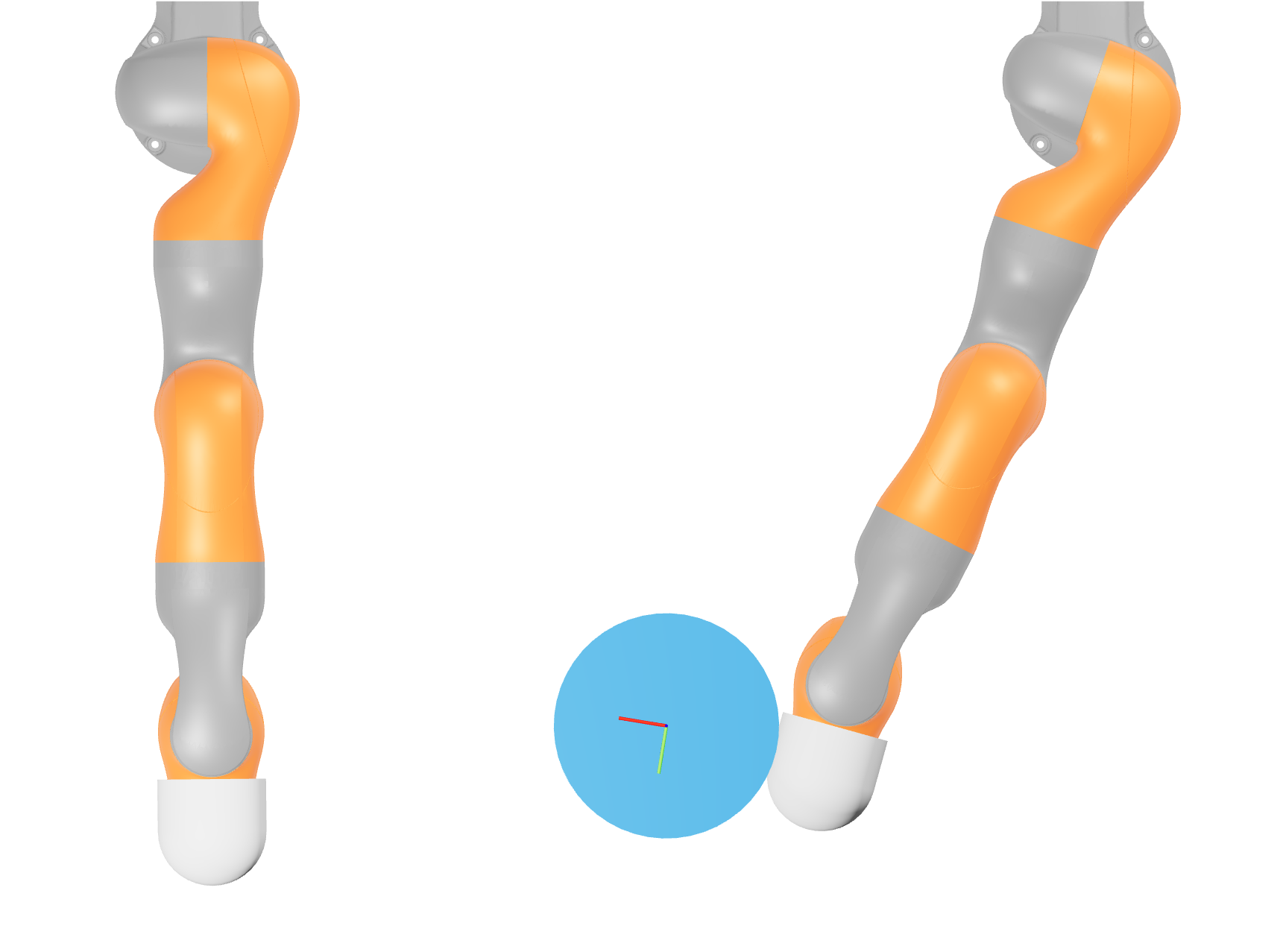}
        \adjincludegraphics[height=1.4cm, angle=90, trim={{.09\width} {.07\height} {.08\width} 0}, clip]{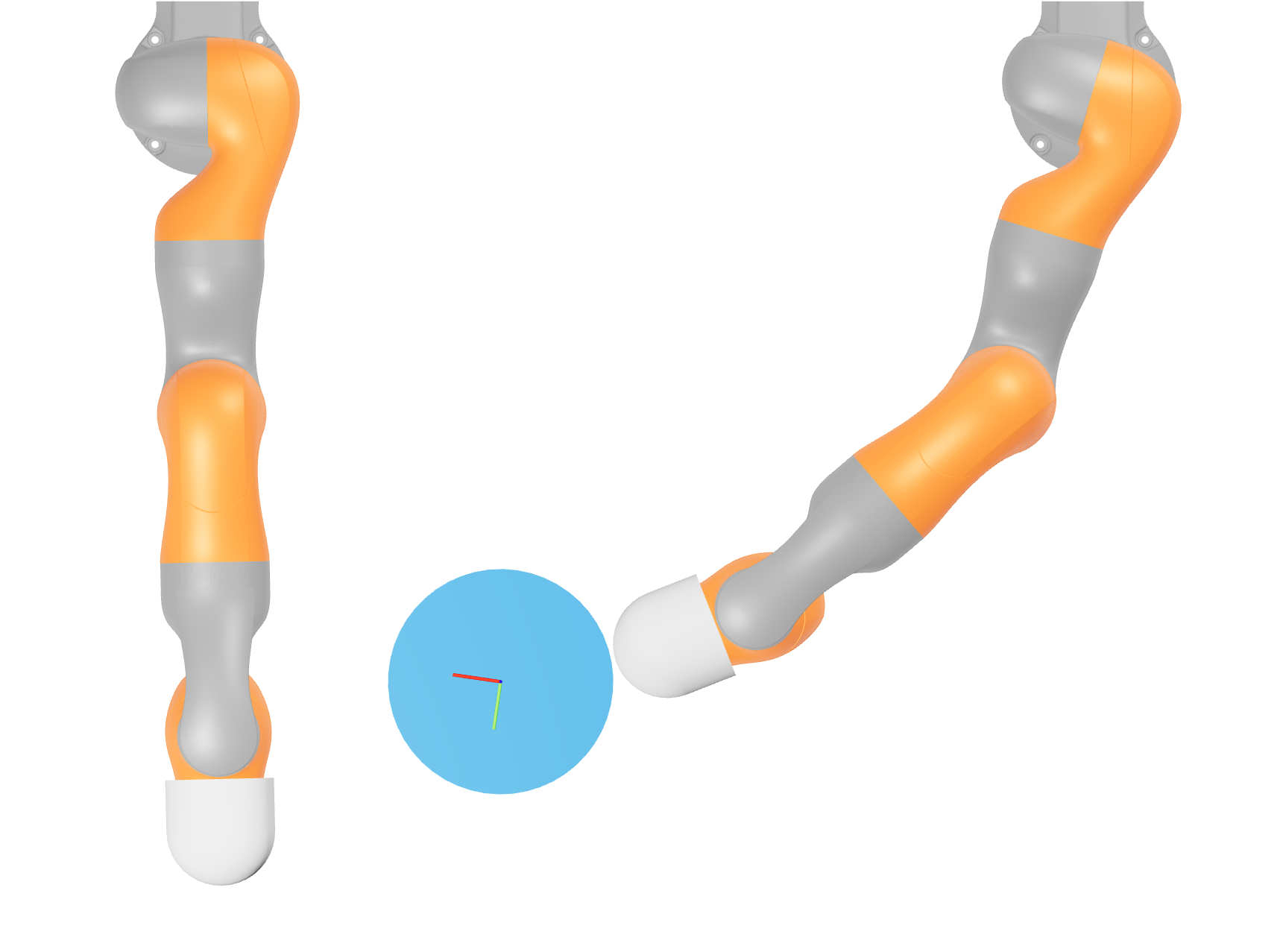}
        \adjincludegraphics[height=1.4cm, angle=90, trim={{.09\width} {.07\height} {.08\width} 0}, clip]{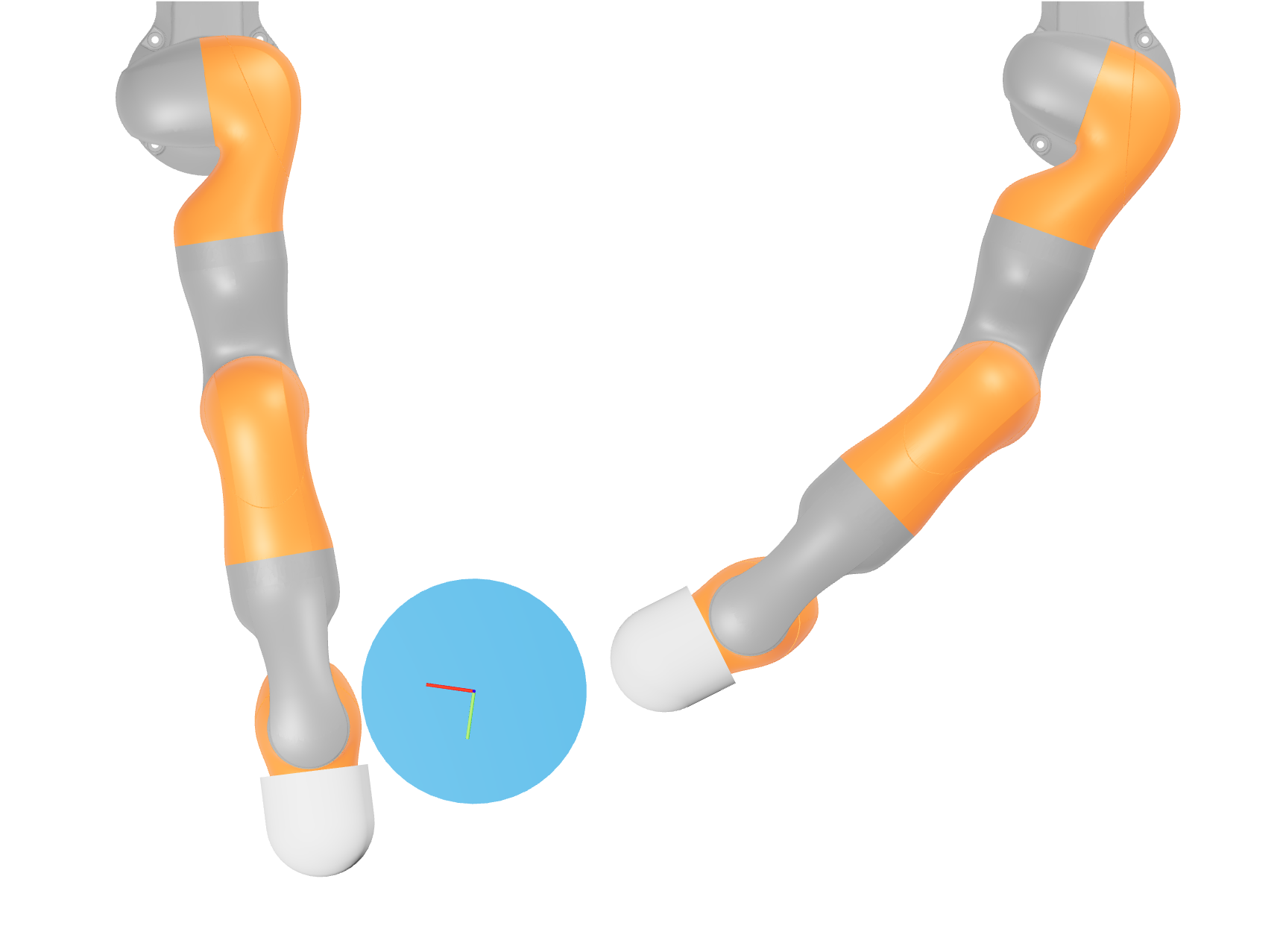}
        \adjincludegraphics[height=1.4cm, angle=90, trim={{.09\width} {.07\height} {.08\width} 0}, clip]{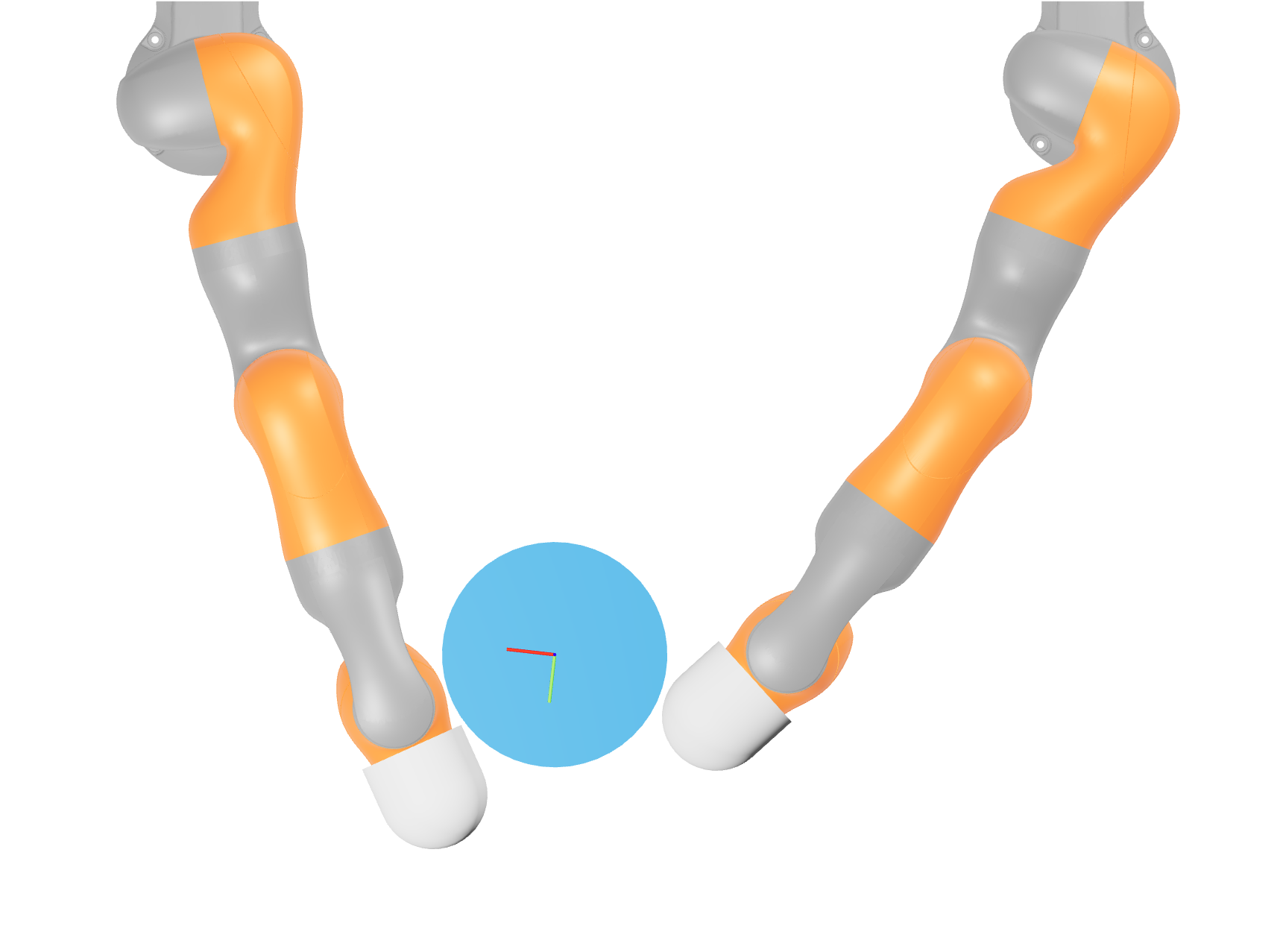}
        \adjincludegraphics[height=1.4cm, angle=90, trim={{.09\width} {.07\height} {.08\width} 0}, clip]{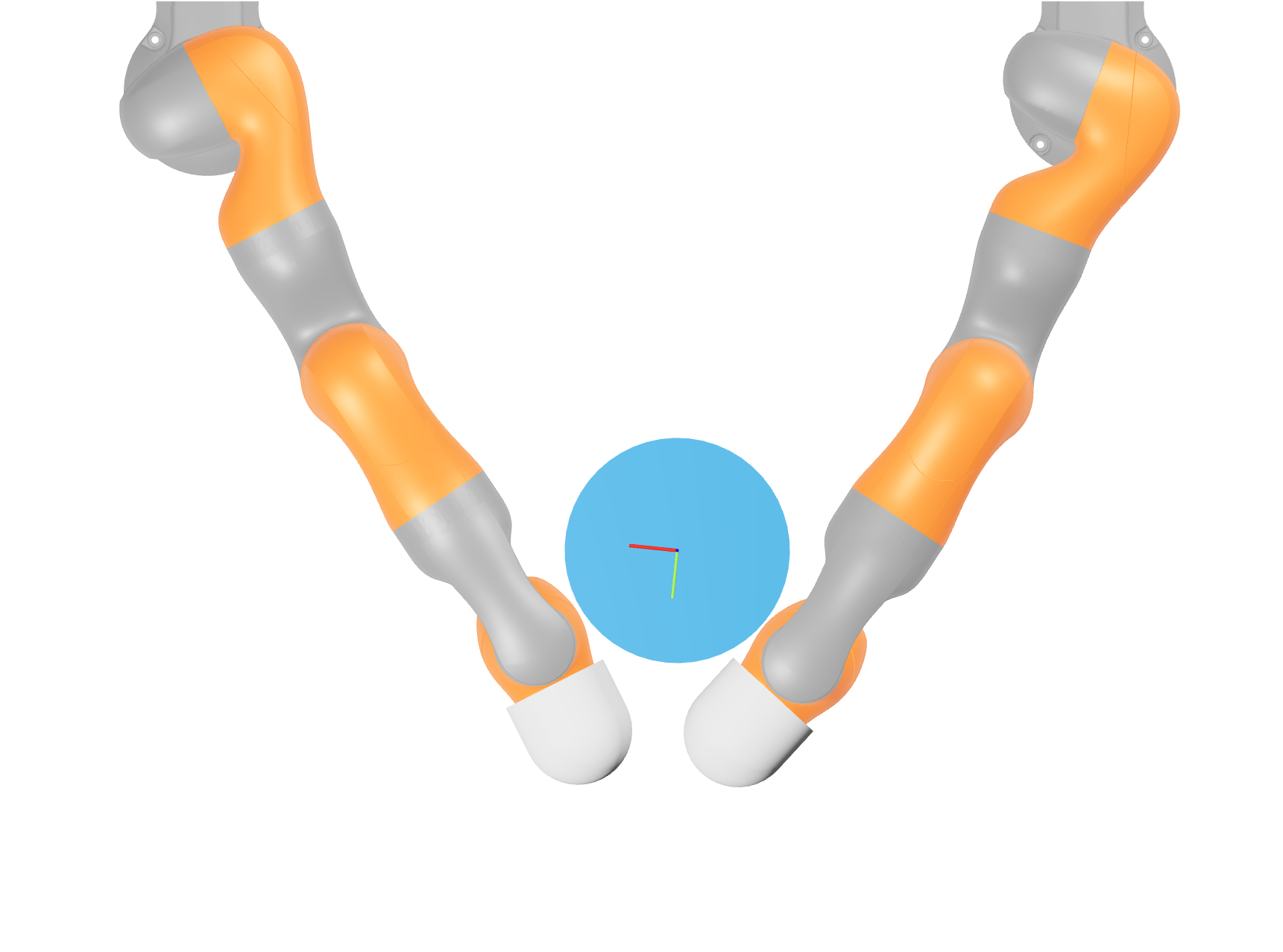}
    } \\    
    \subfloat[]{%
        \adjincludegraphics[height=1.4cm, angle=90, trim={{.09\width} {.07\height} {.08\width} 0}, clip]{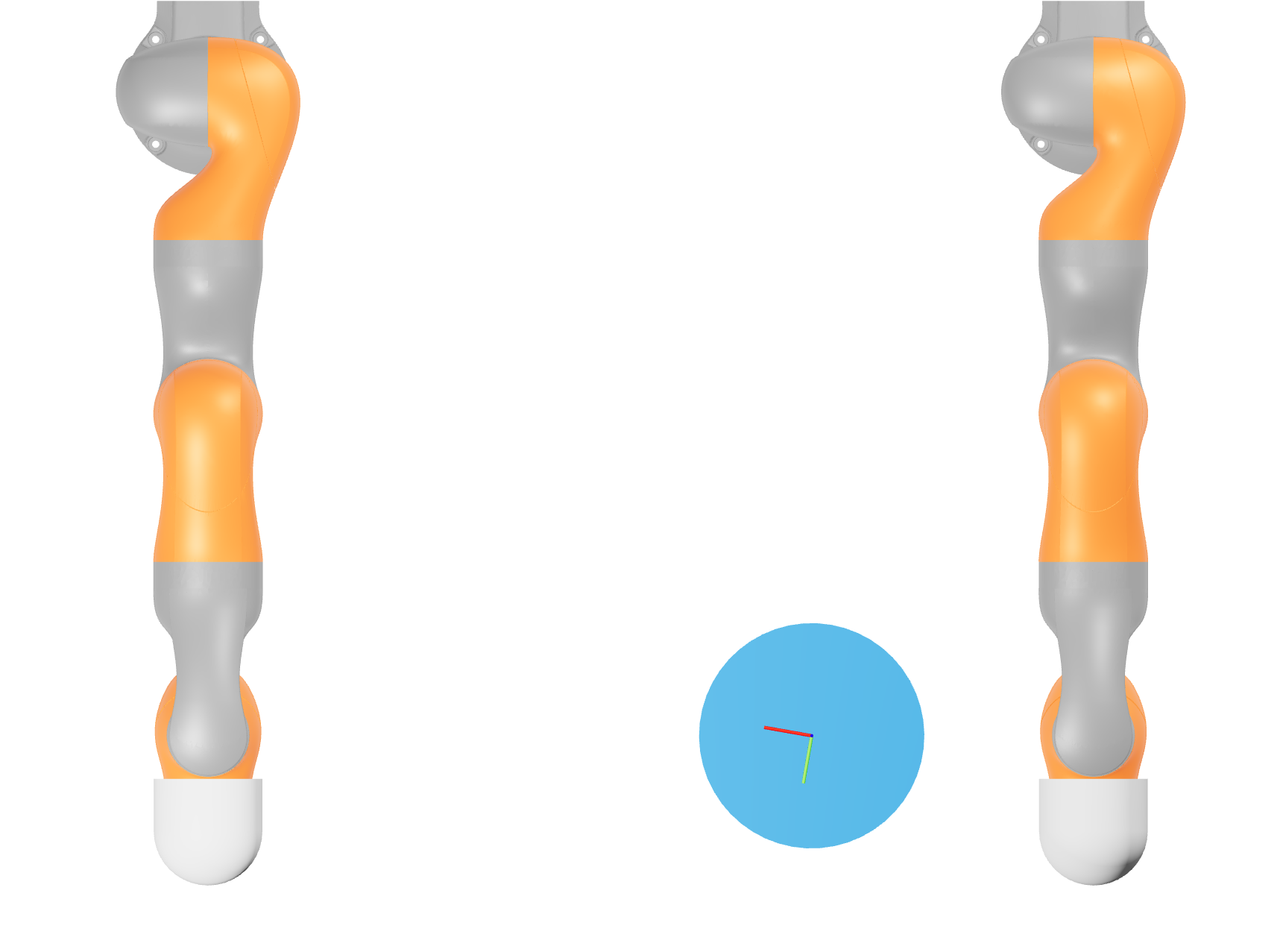}
        \adjincludegraphics[height=1.4cm, angle=90, trim={{.09\width} {.07\height} {.08\width} 0}, clip]{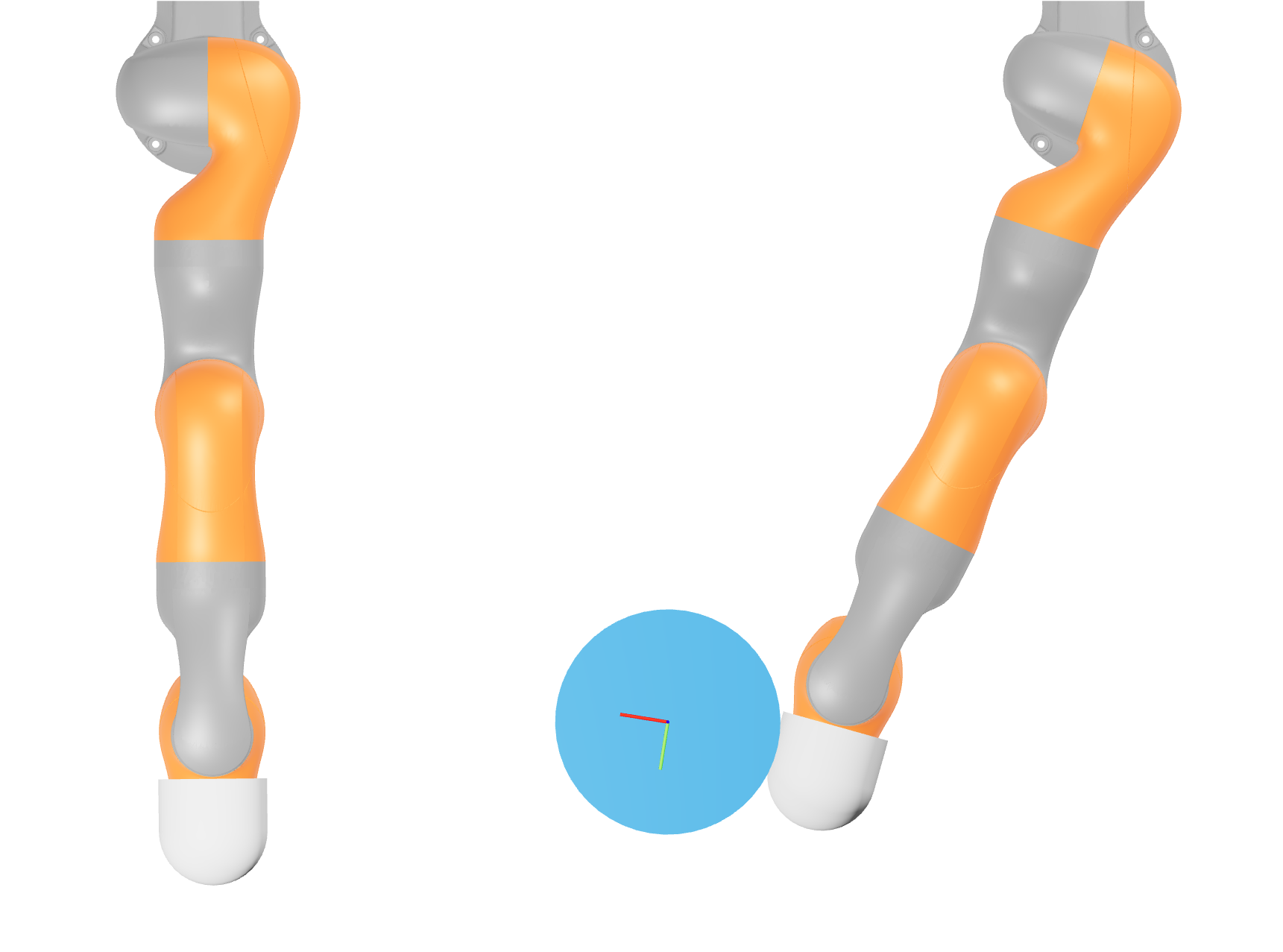}
        \adjincludegraphics[height=1.4cm, angle=90, trim={{.09\width} {.07\height} {.08\width} 0}, clip]{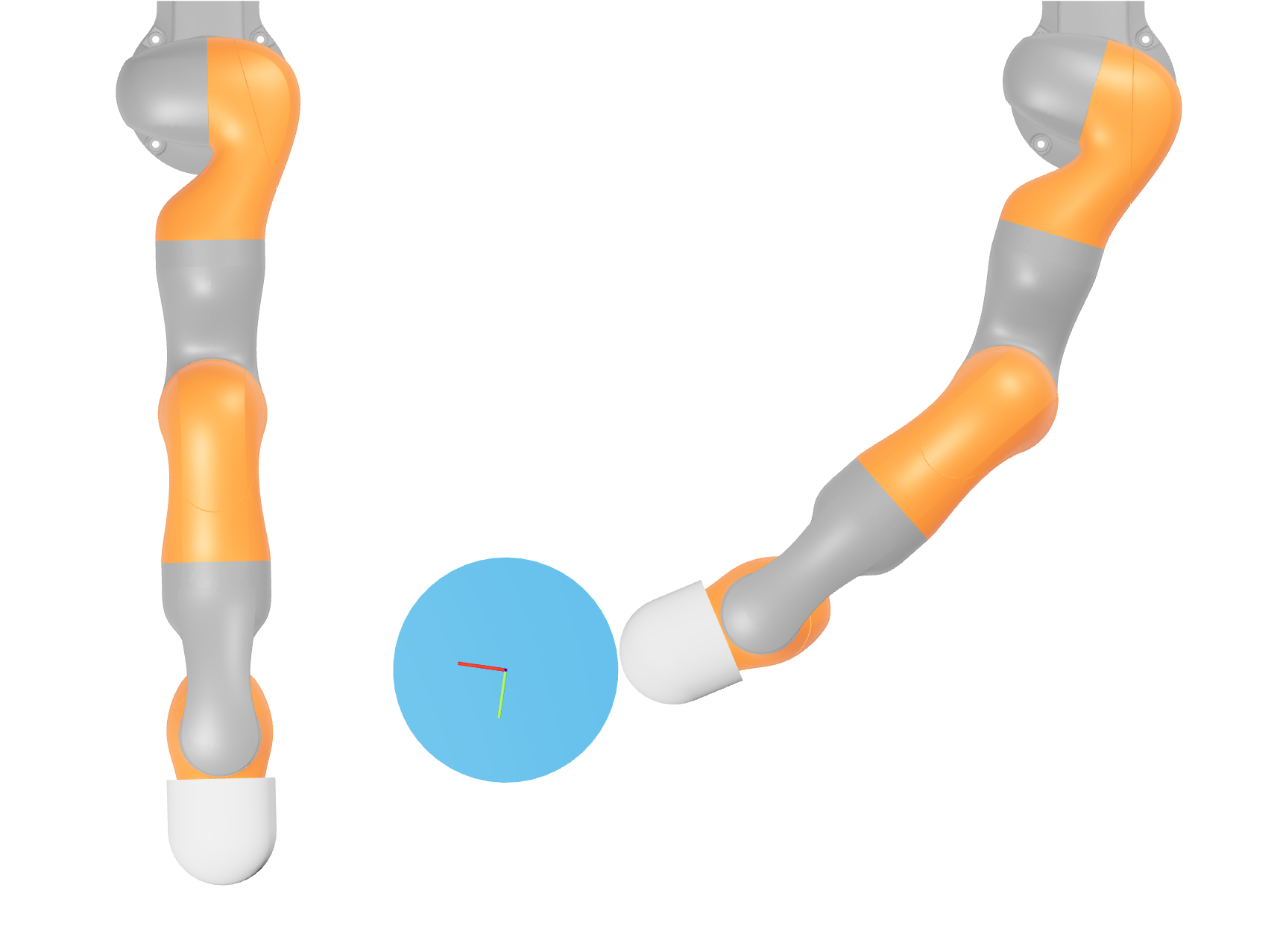}
        \adjincludegraphics[height=1.4cm, angle=90, trim={{.09\width} {.07\height} {.08\width} 0}, clip]{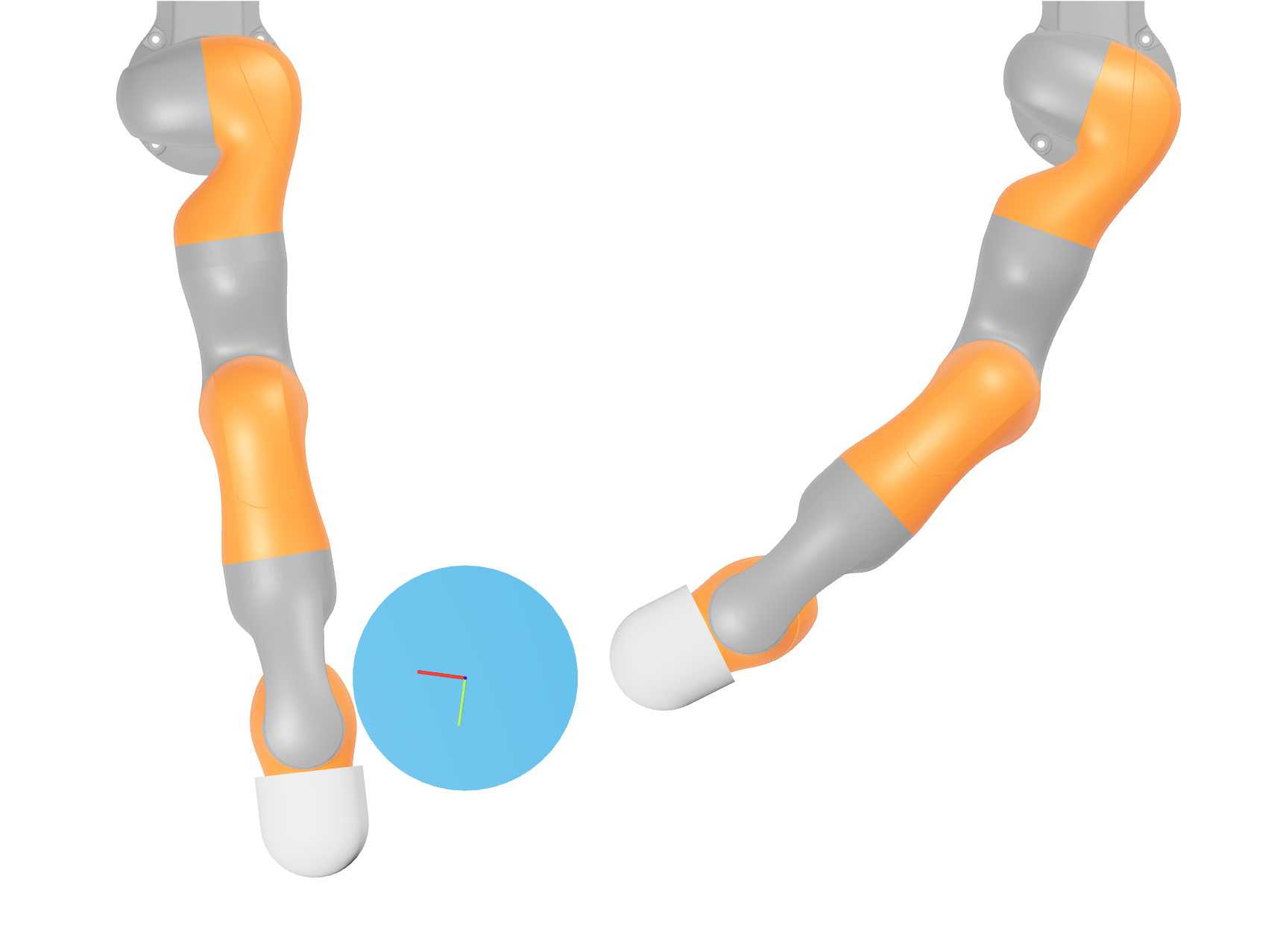}
        \adjincludegraphics[height=1.4cm, angle=90, trim={{.09\width} {.07\height} {.08\width} 0}, clip]{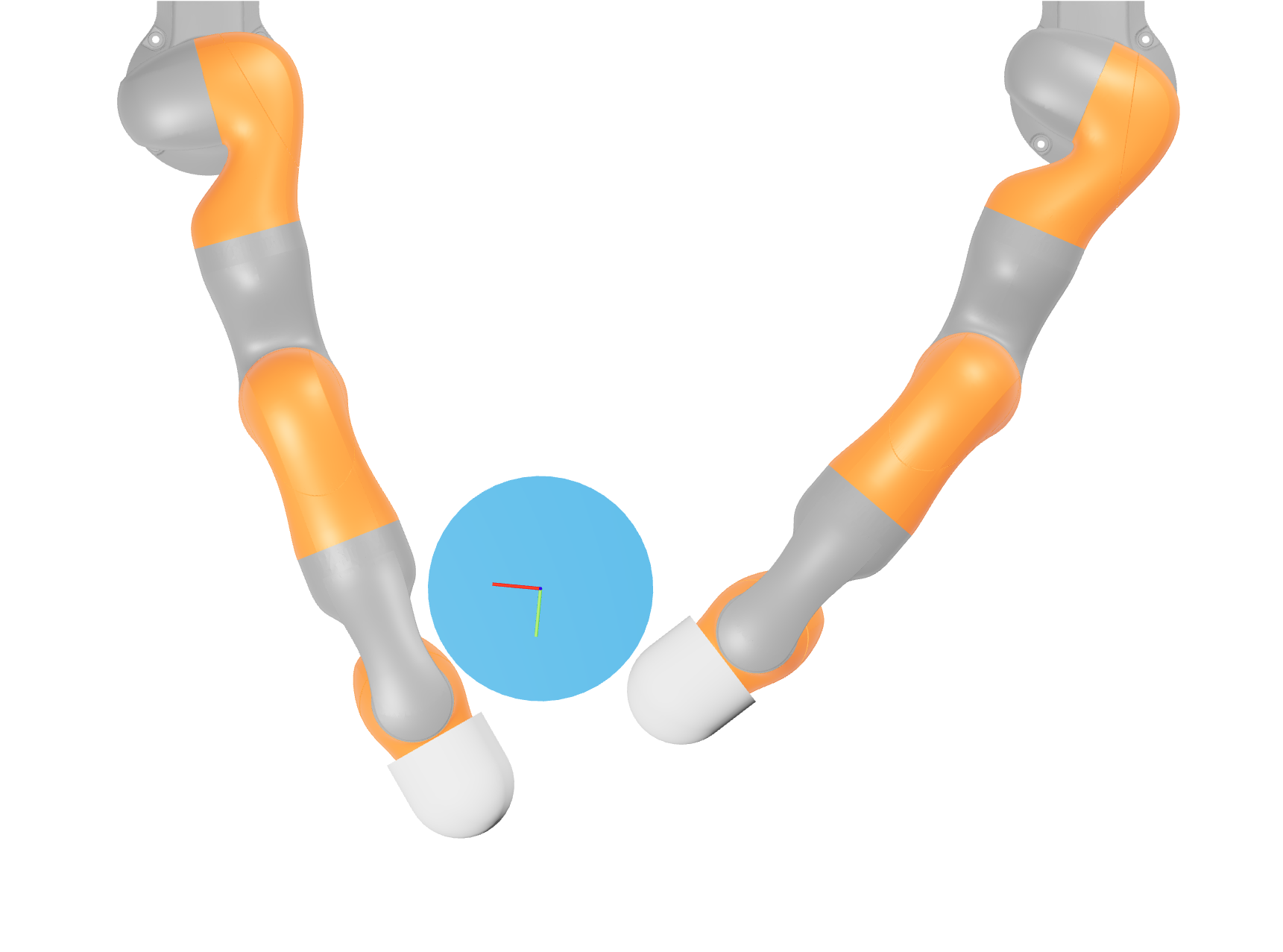}
        \adjincludegraphics[height=1.4cm, angle=90, trim={{.09\width} {.07\height} {.08\width} 0}, clip]{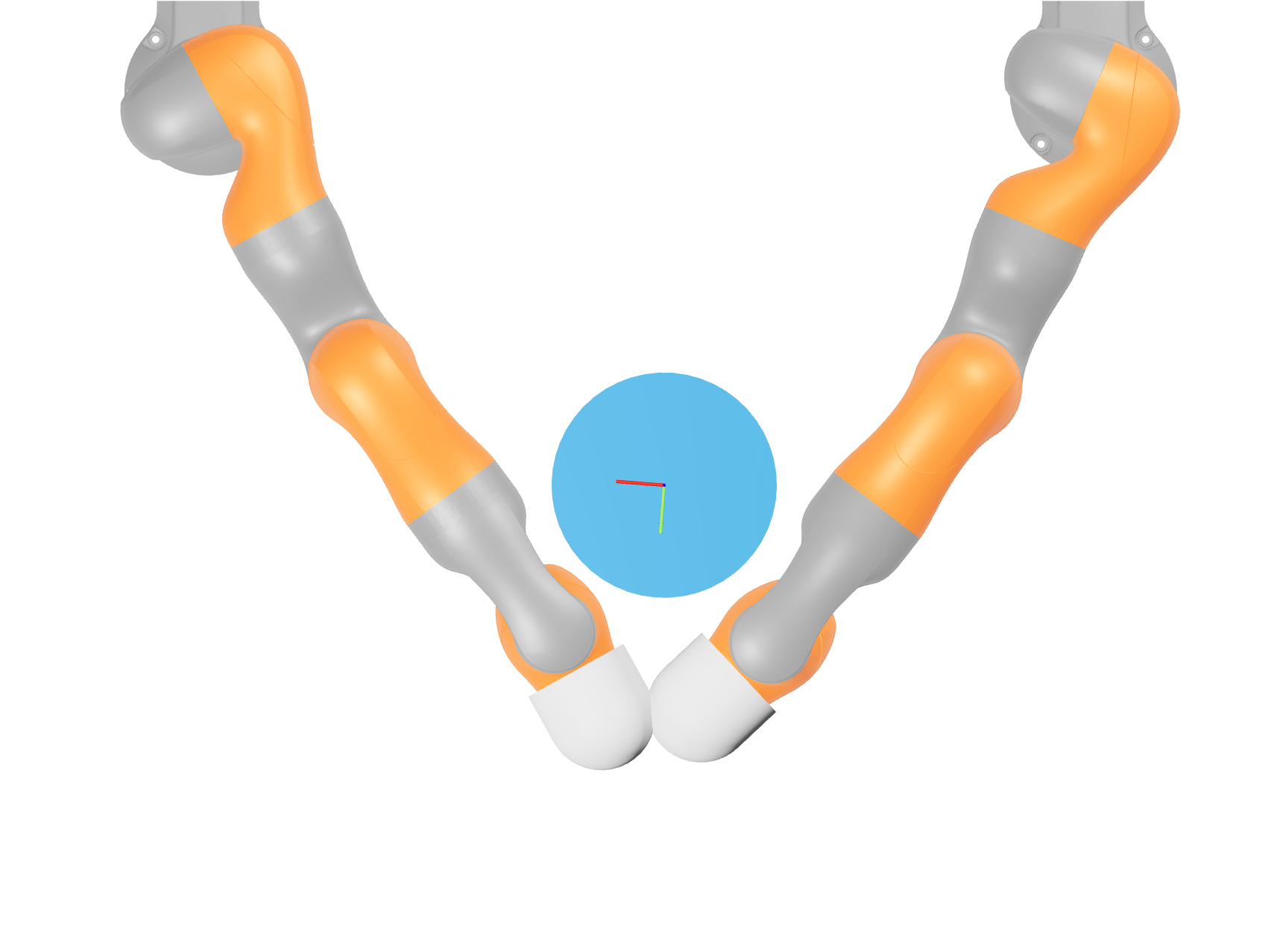}
    } \\    
    \subfloat[]{%
        \adjincludegraphics[height=1.4cm, angle=90, trim={{.09\width} {.07\height} {.08\width} 0}, clip]{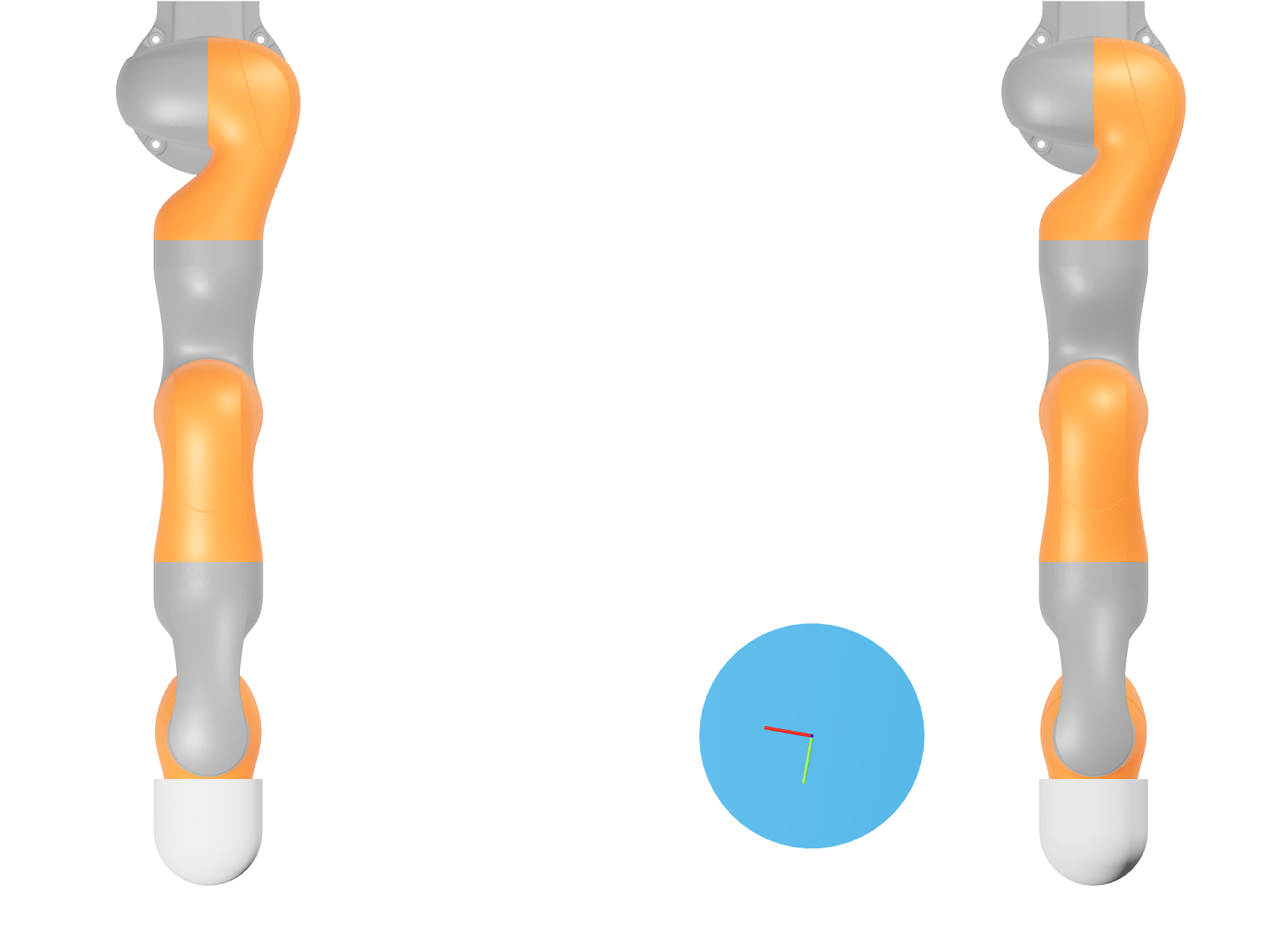}
        \adjincludegraphics[height=1.4cm, angle=90, trim={{.09\width} {.07\height} {.08\width} 0}, clip]{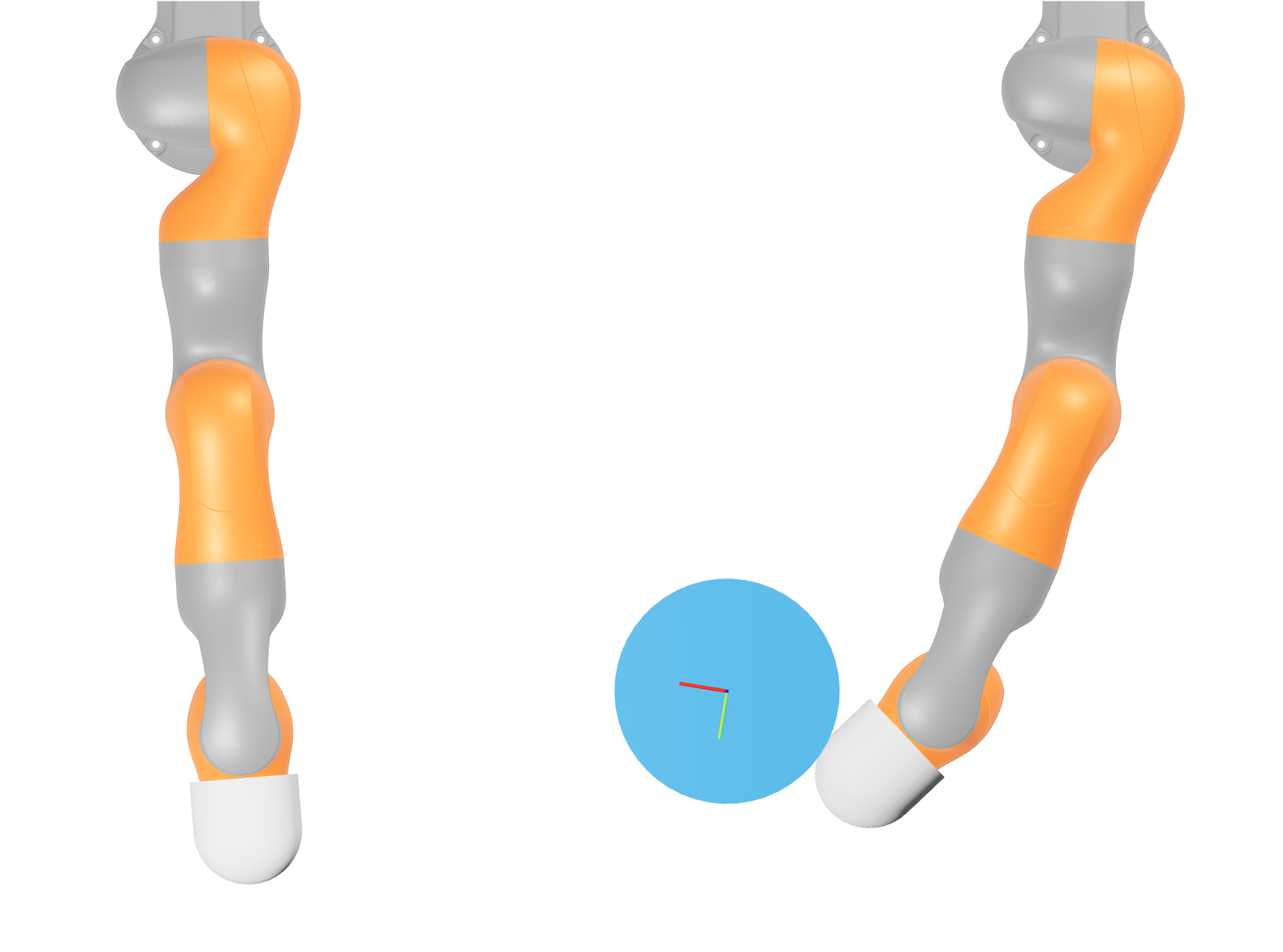}
        \adjincludegraphics[height=1.4cm, angle=90, trim={{.09\width} {.07\height} {.08\width} 0}, clip]{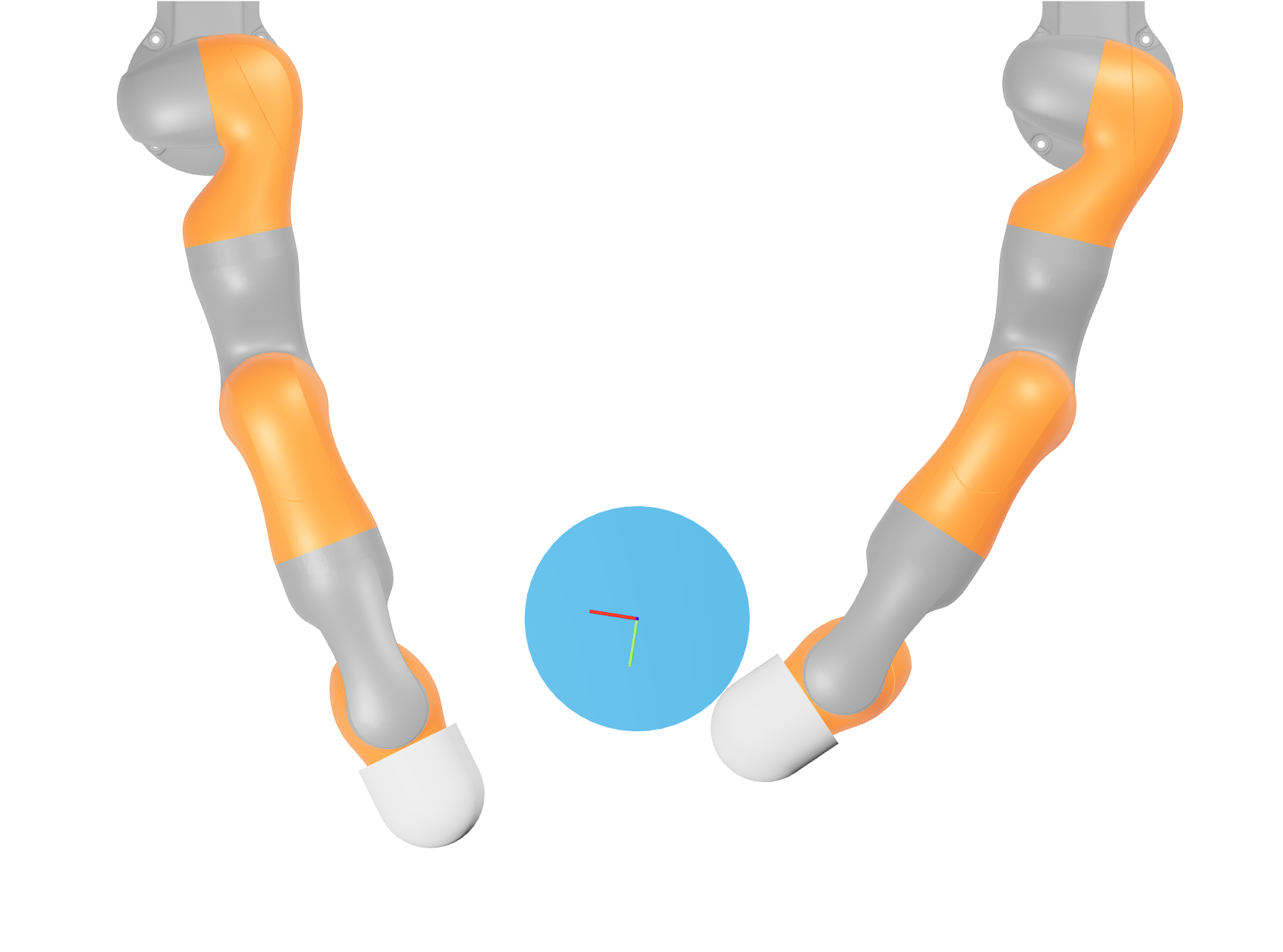}
        \adjincludegraphics[height=1.4cm, angle=90, trim={{.09\width} {.07\height} {.08\width} 0}, clip]{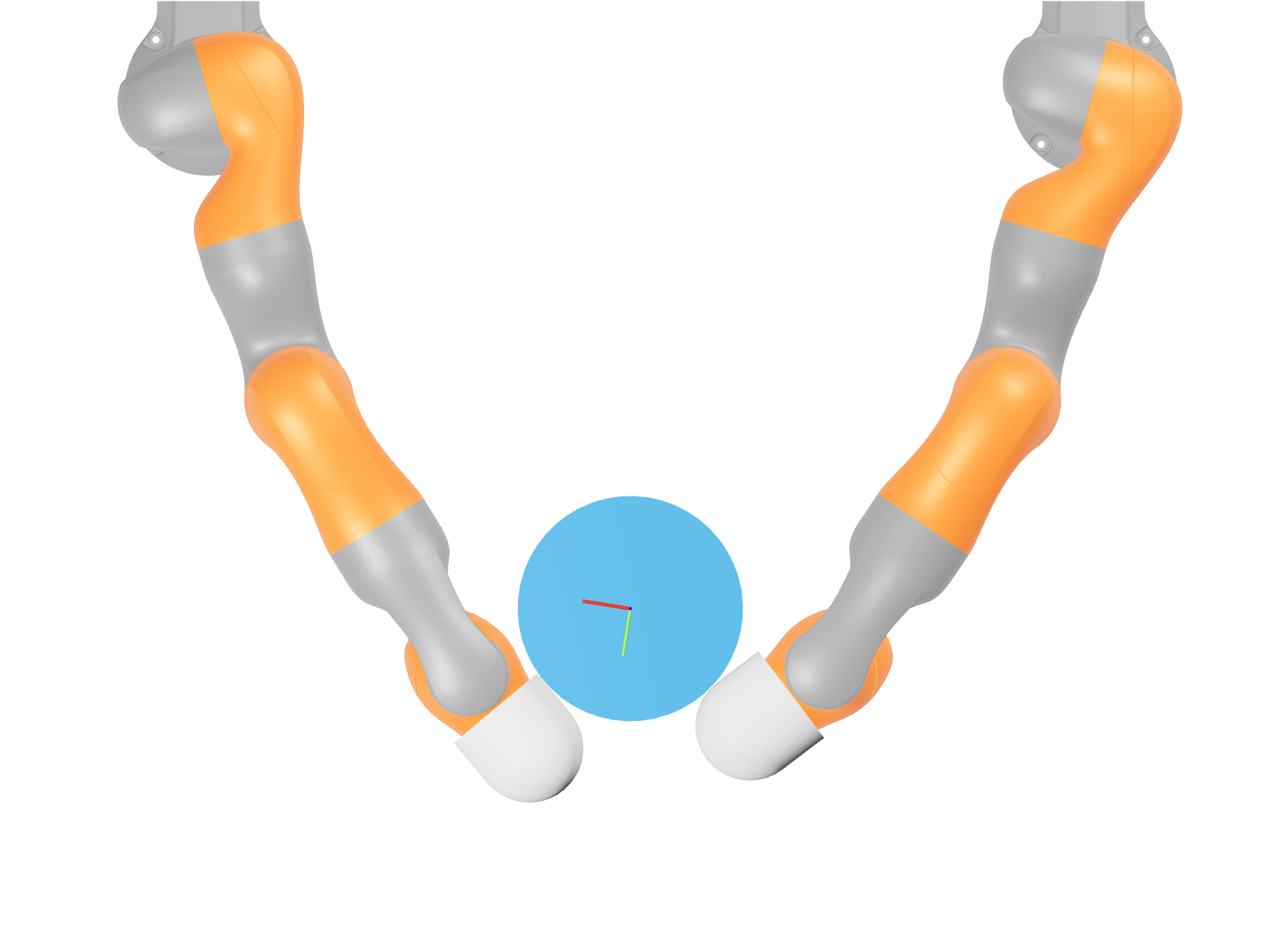}
        \adjincludegraphics[height=1.4cm, angle=90, trim={{.09\width} {.07\height} {.08\width} 0}, clip]{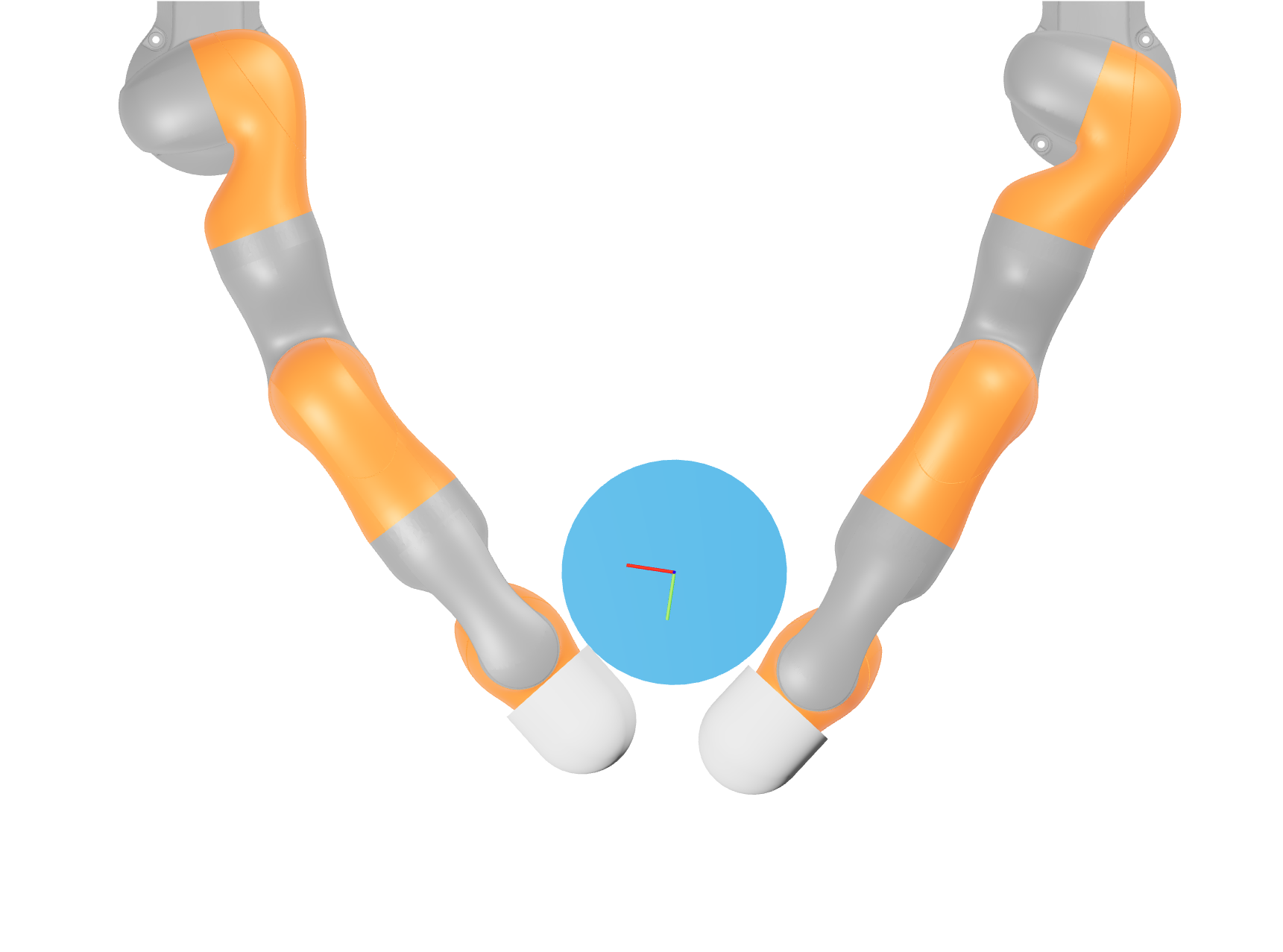}
        \adjincludegraphics[height=1.4cm, angle=90, trim={{.09\width} {.07\height} {.08\width} 0}, clip]{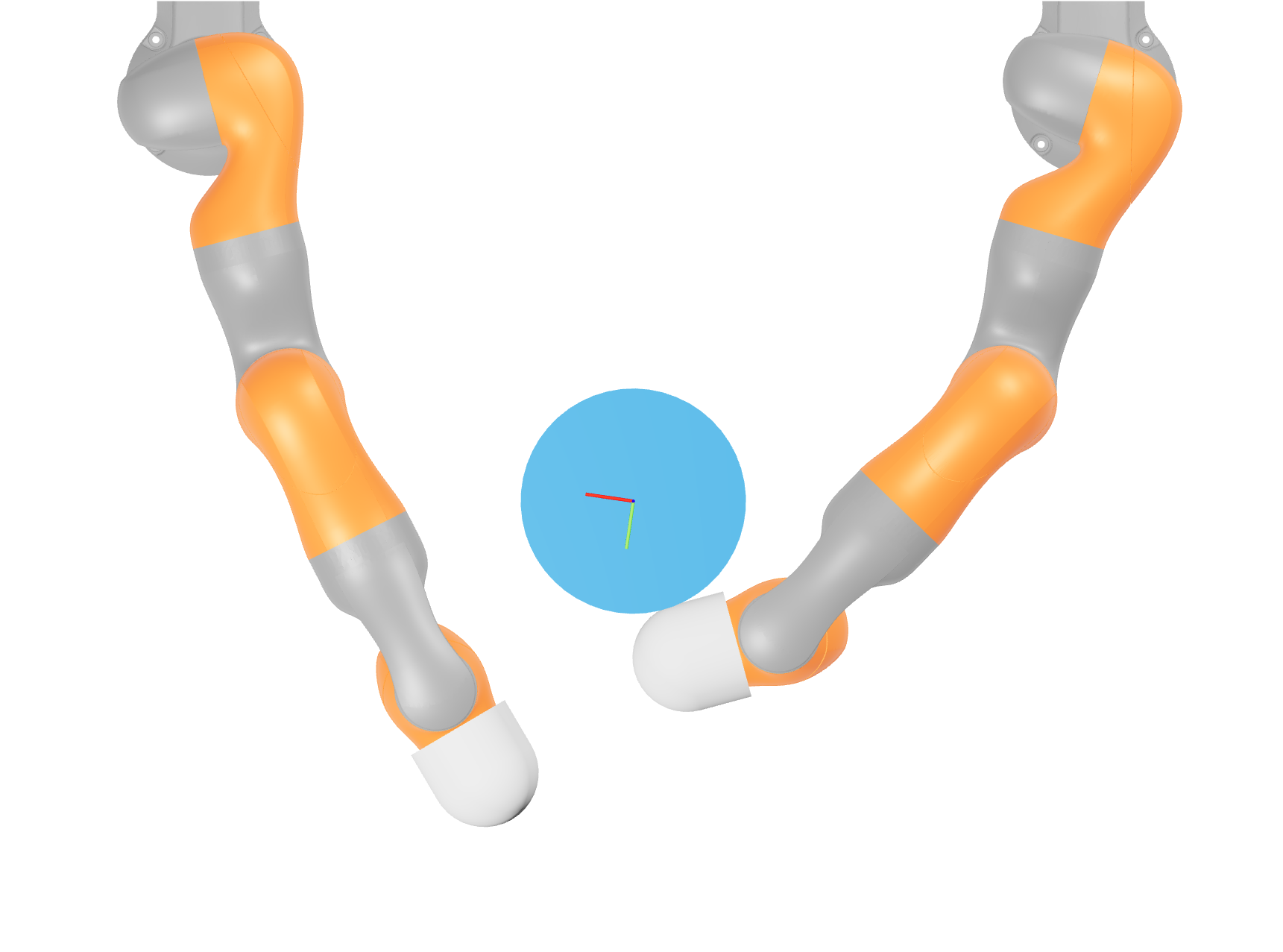}
    }
    \caption{Rollouts of bimanual planar bucket manipulation illustrating robot collisions under baselines and their absence in our method. 
    \textbf{(a)} The trajectory generated by \ac{TO-CTR} under smoothed dynamics exhibits no collision. 
    \textbf{(b)} Executing the \ac{TO-CTR} nominal control sequence on the nonsmooth hybrid dynamics leads to a robot collision. 
    \textbf{(c)} Executing the policy produced by our method on the nonsmooth hybrid dynamics avoids robot collision.}
    \label{fig:iiwa-bucket-collision}
\end{figure}

\begin{figure}[!t]
    \centering
    \subfloat[]{%
      \shortstack{
        \adjincludegraphics[height=1.4cm, angle=90, trim={{.01\width} {.07\height} {.04\width} 0}, clip]{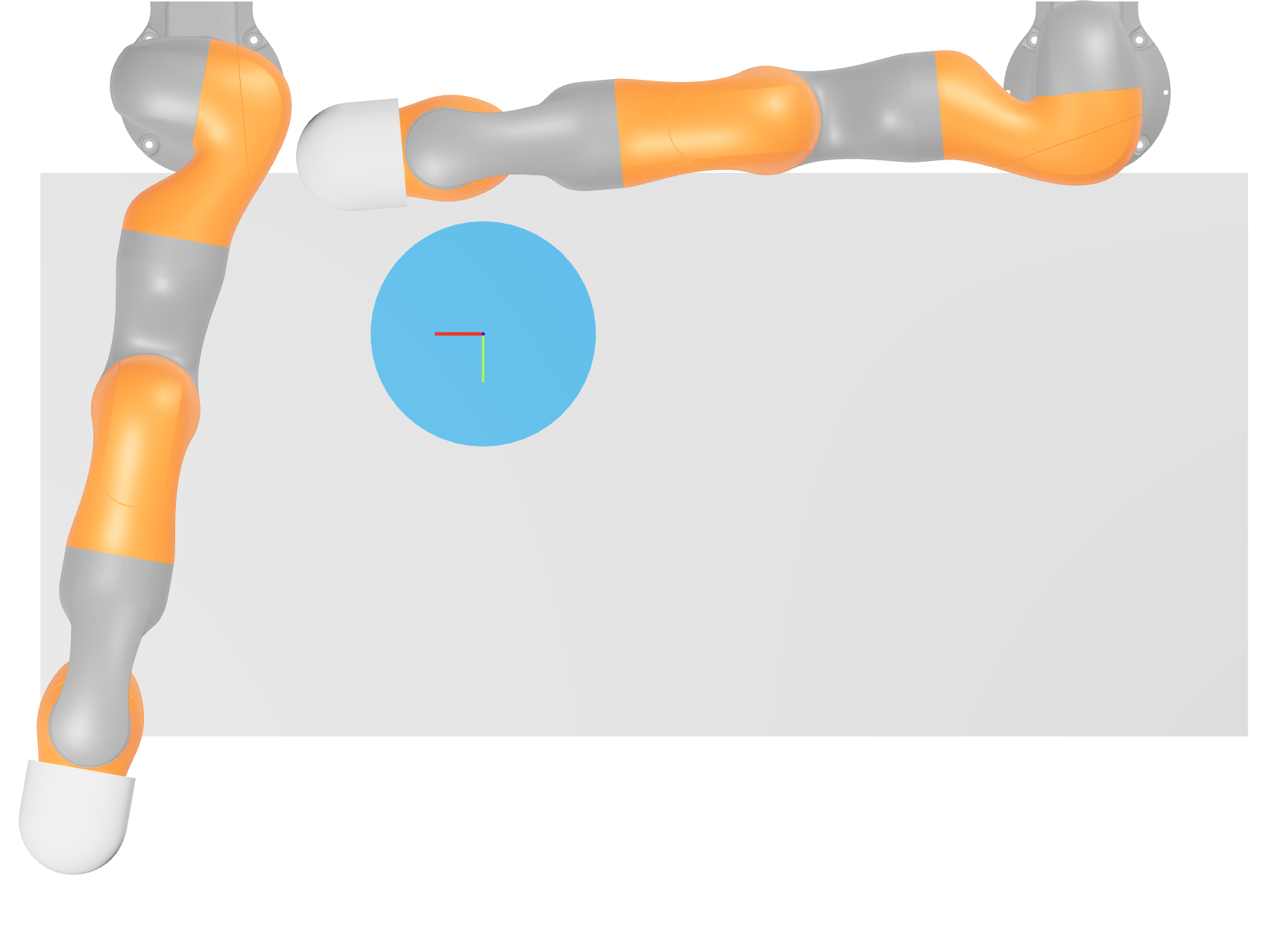}
        \adjincludegraphics[height=1.4cm, angle=90, trim={{.01\width} {.07\height} {.04\width} 0}, clip]{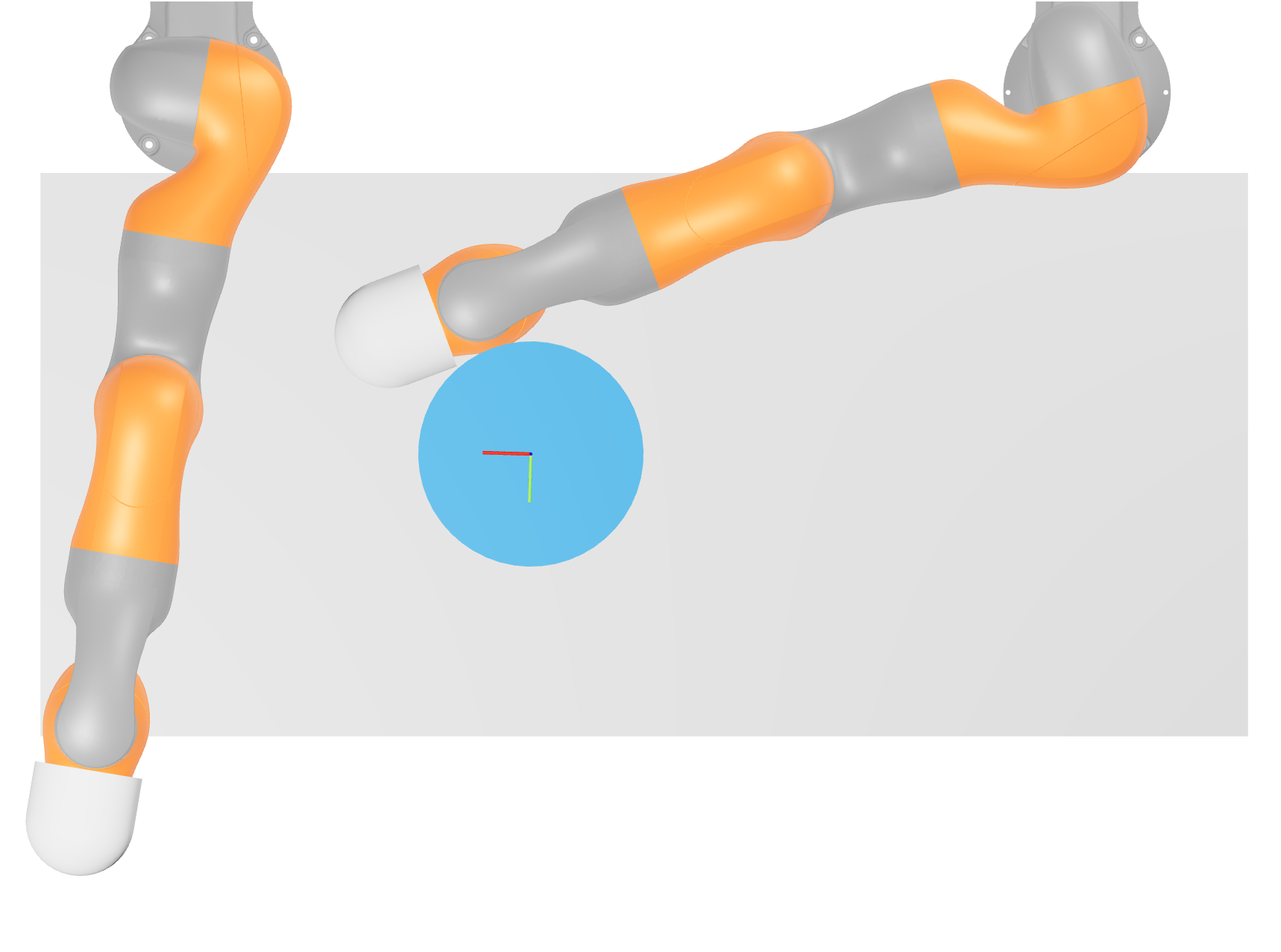}
        \adjincludegraphics[height=1.4cm, angle=90, trim={{.01\width} {.07\height} {.04\width} 0}, clip]{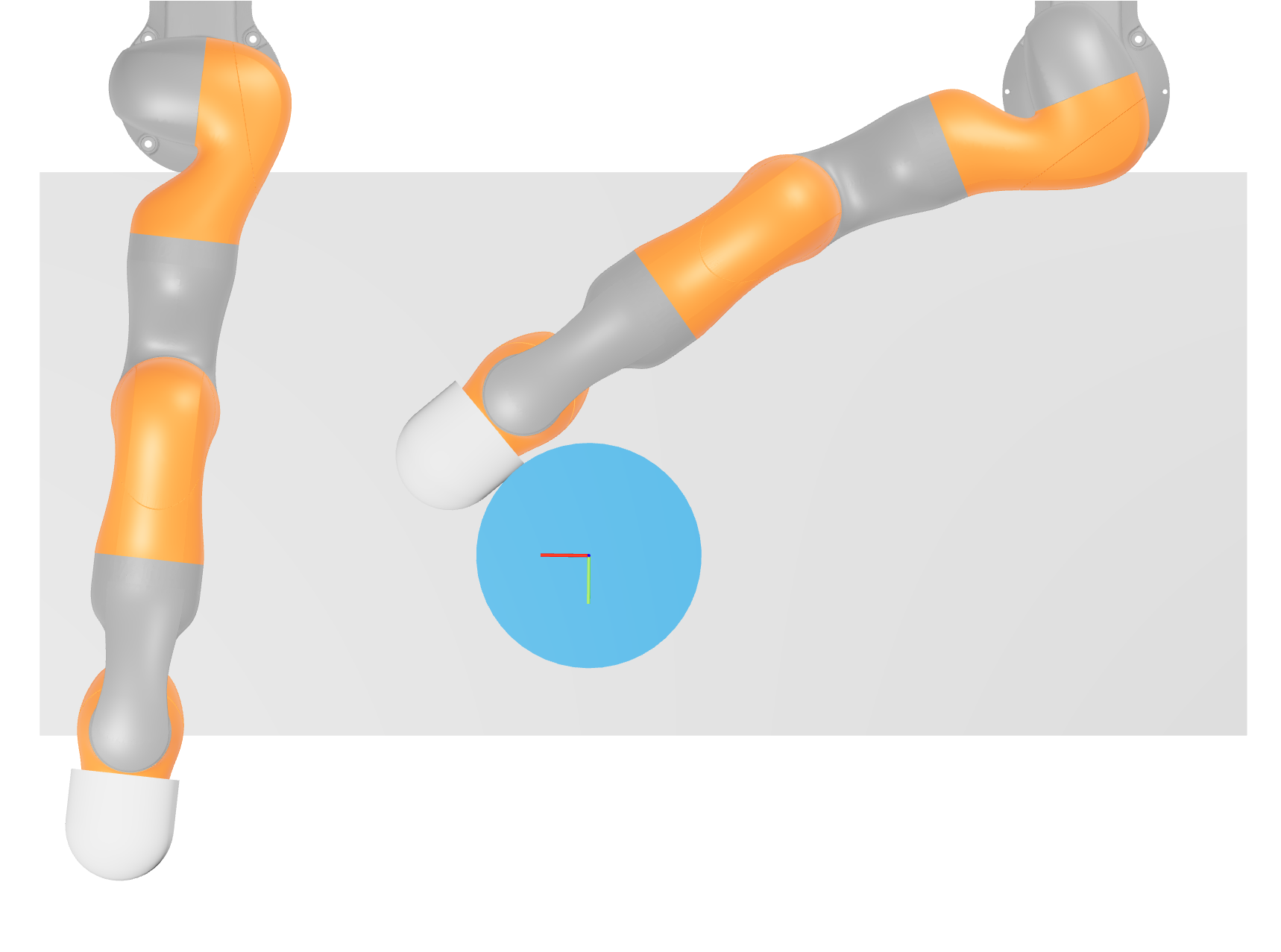}
        \adjincludegraphics[height=1.4cm, angle=90, trim={{.01\width} {.07\height} {.04\width} 0}, clip]{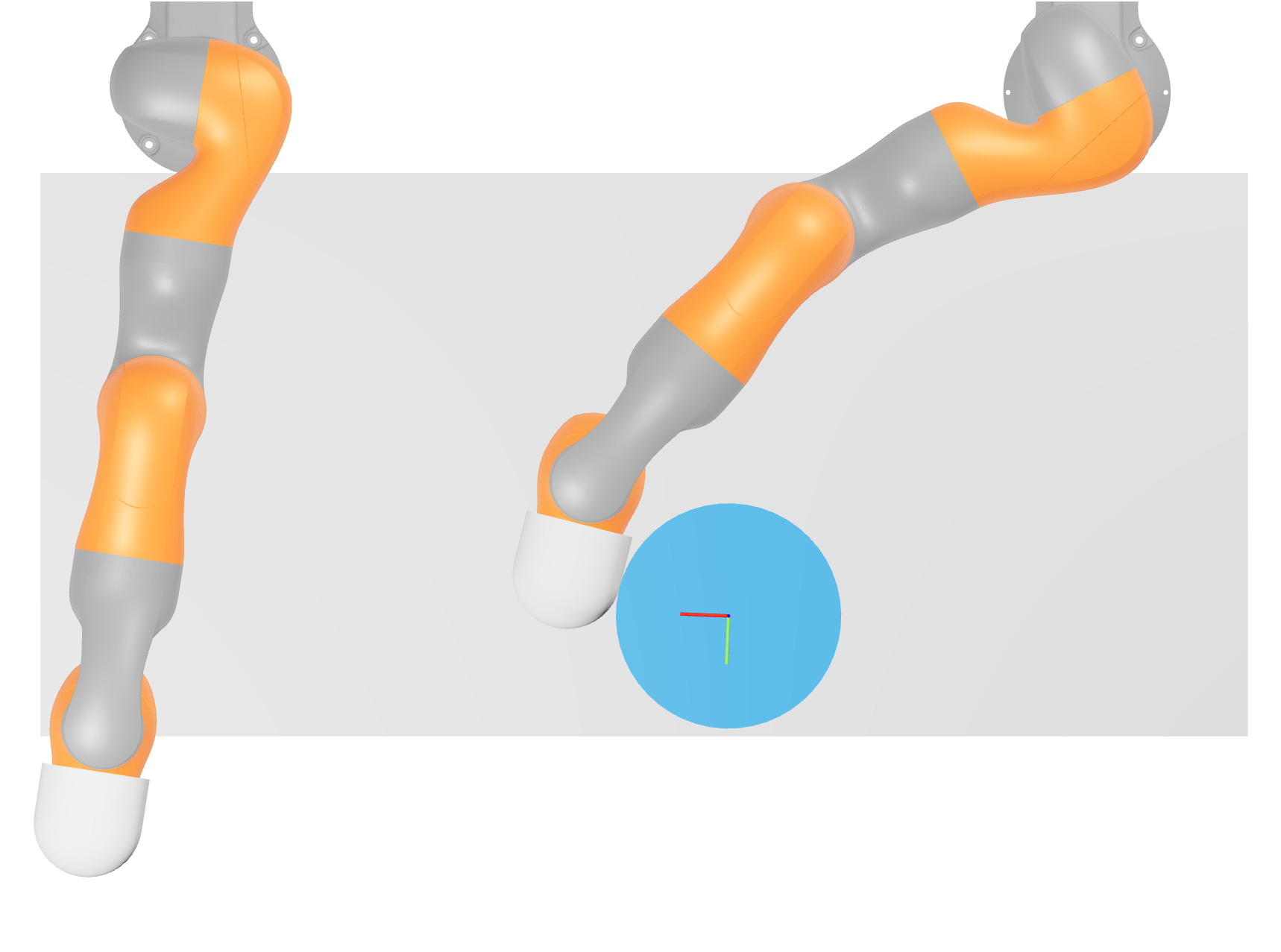}
        \adjincludegraphics[height=1.4cm, angle=90, trim={{.01\width} {.07\height} {.04\width} 0}, clip]{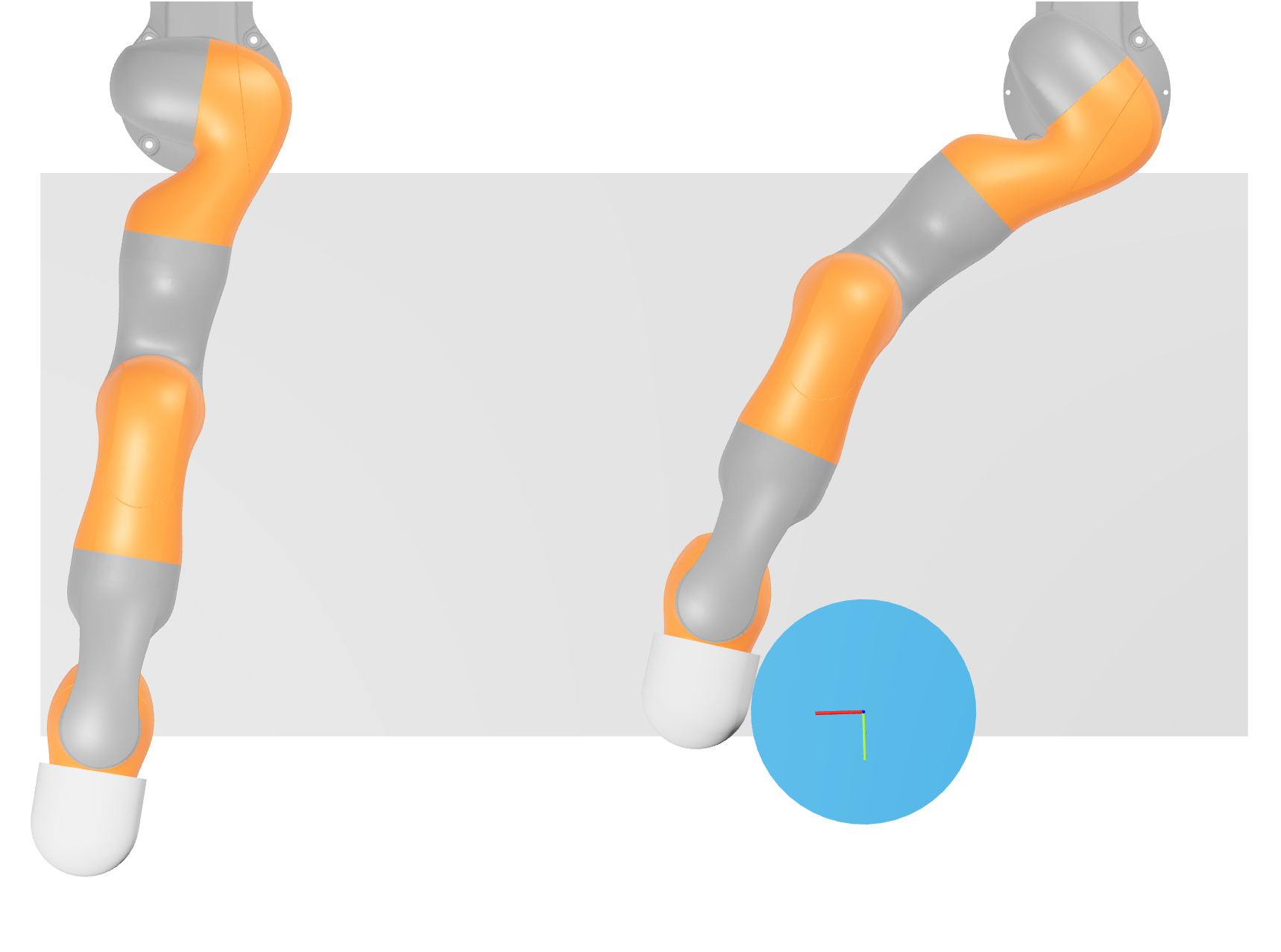}
        \adjincludegraphics[height=1.4cm, angle=90, trim={{.01\width} {.07\height} {.04\width} 0}, clip]{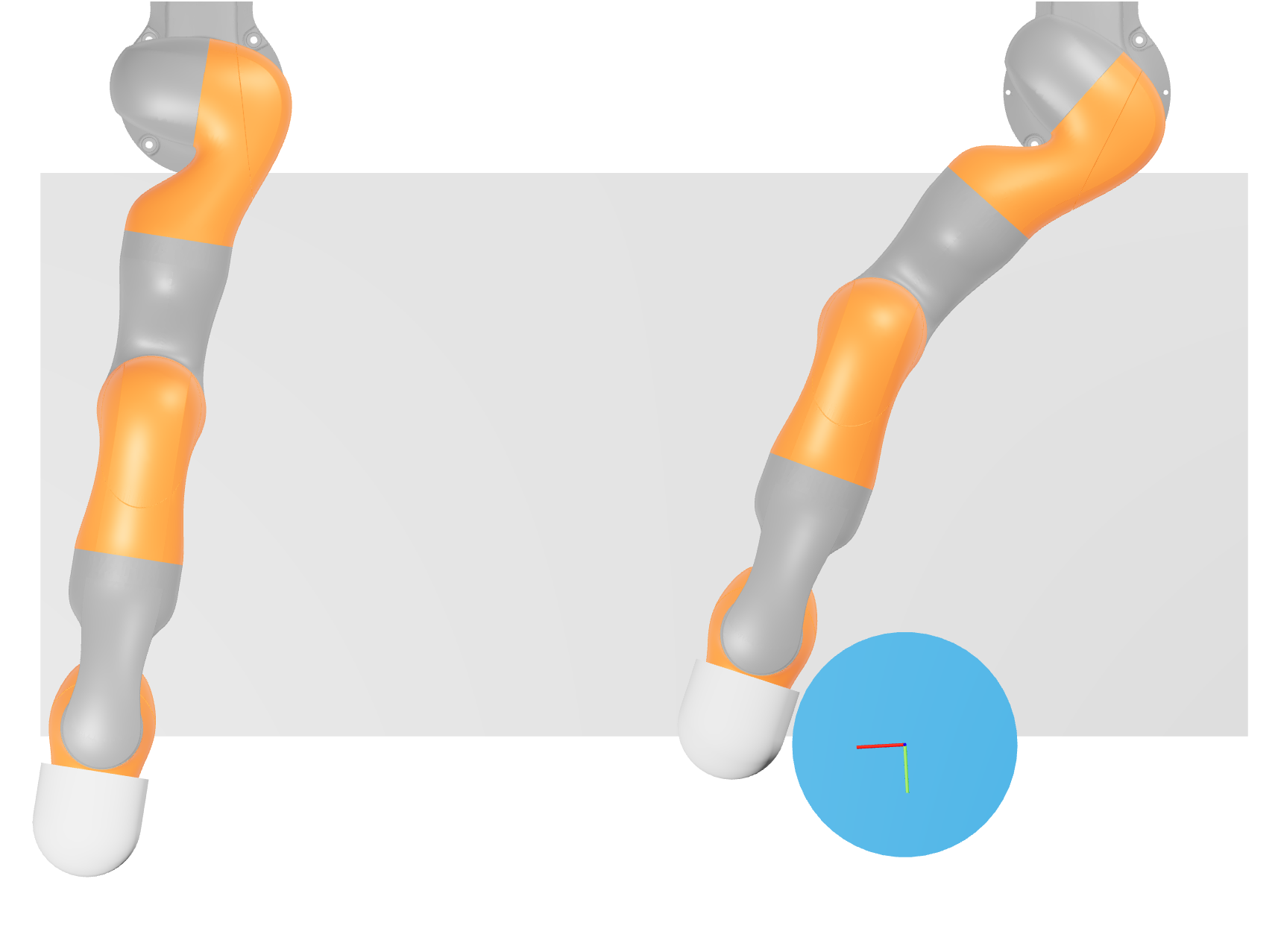}
        \\[2mm]
        \hspace*{-2mm}
        \begin{tikzpicture}

\definecolor{color1}{RGB}{228,26,28}
\definecolor{color2}{RGB}{55,126,184}
\definecolor{color3}{RGB}{77,175,74}
\definecolor{color4}{RGB}{152,78,163}

\begin{axis}[
width=3.6cm, height=3.4cm,
xlabel={\footnotesize $t$},
ylabel={\footnotesize $x$},
xlabel shift=-1mm,
ylabel shift=-3mm,
xmin=0, xmax=2,
tick label style={
    font=\scriptsize,
    /pgf/number format/fixed,
},
legend cell align={left},
legend style={
    font=\footnotesize,
    at={(-0.3,1.25)},
    anchor=north west,
    draw=none,
    fill=none,
    row sep=-2pt,
    inner sep=0pt,
},
legend image post style={scale=0.4},
]

\pgfplotstableread[row sep=crcr]{
t     Xo_nominal  Xo_true    \\
0.00  0.20000000  0.20000000 \\
0.02  0.19995382  0.20000000 \\
0.04  0.20030591  0.20059303 \\
0.06  0.19698237  0.19693787 \\
0.08  0.19090275  0.19058468 \\
0.10  0.18184042  0.18128292 \\
0.12  0.17005872  0.16926951 \\
0.14  0.15554037  0.15449617 \\
0.16  0.13875212  0.13744369 \\
0.18  0.11967920  0.11820940 \\
0.20  0.09860954  0.09700740 \\
0.22  0.07604756  0.07433771 \\
0.24  0.05252342  0.05072525 \\
0.26  0.02890046  0.02703195 \\
0.28  0.00535843  0.00340556 \\
0.30 -0.01798855 -0.02010363 \\
0.32 -0.03905093 -0.04155400 \\
0.34 -0.05436453 -0.05718203 \\
0.36 -0.06335498 -0.06548930 \\
0.38 -0.07144512 -0.07322934 \\
0.40 -0.07963274 -0.08127691 \\
0.42 -0.08802954 -0.08945457 \\
0.44 -0.09683578 -0.09795146 \\
0.46 -0.10614487 -0.10639412 \\
0.48 -0.11596631 -0.11554408 \\
0.50 -0.12612084 -0.12526048 \\
0.52 -0.13578760 -0.13441152 \\
0.54 -0.14490882 -0.14287906 \\
0.56 -0.15343891 -0.15066661 \\
0.58 -0.16151370 -0.15816959 \\
0.60 -0.16930946 -0.16548059 \\
0.62 -0.17689443 -0.17261349 \\
0.64 -0.18422454 -0.17947950 \\
0.66 -0.19126897 -0.18594396 \\
0.68 -0.19792938 -0.19199481 \\
0.70 -0.20419906 -0.19771714 \\
0.72 -0.21015584 -0.20320344 \\
0.74 -0.21590116 -0.20852406 \\
0.76 -0.22150829 -0.21373272 \\
0.78 -0.22703984 -0.21887250 \\
0.80 -0.23254430 -0.22396521 \\
0.82 -0.23806787 -0.22902332 \\
0.84 -0.24364949 -0.23405668 \\
0.86 -0.24928518 -0.23922773 \\
0.88 -0.25489780 -0.24447433 \\
0.90 -0.26038843 -0.24964170 \\
0.92 -0.26569732 -0.25463329 \\
0.94 -0.27079979 -0.25940727 \\
0.96 -0.27569385 -0.26394606 \\
0.98 -0.28040262 -0.26826397 \\
1.00 -0.28495696 -0.27240150 \\
1.02 -0.28937523 -0.27638932 \\
1.04 -0.29366594 -0.28024404 \\
1.06 -0.29783651 -0.28397927 \\
1.08 -0.30189200 -0.28760075 \\
1.10 -0.30583176 -0.29110634 \\
1.12 -0.30965028 -0.29448942 \\
1.14 -0.31333753 -0.29773948 \\
1.16 -0.31687750 -0.30084102 \\
1.18 -0.32024541 -0.30377105 \\
1.20 -0.32340541 -0.30649780 \\
1.22 -0.32631293 -0.30898542 \\
1.24 -0.32892521 -0.31121083 \\
1.26 -0.33121698 -0.31317106 \\
1.28 -0.33319343 -0.31487356 \\
1.30 -0.33488941 -0.31633745 \\
1.32 -0.33635401 -0.31758708 \\
1.34 -0.33763474 -0.31865015 \\
1.36 -0.33877072 -0.31955293 \\
1.38 -0.33979214 -0.32031931 \\
1.40 -0.34072185 -0.32096676 \\
1.42 -0.34157707 -0.32151302 \\
1.44 -0.34237091 -0.32197210 \\
1.46 -0.34311347 -0.32235555 \\
1.48 -0.34381262 -0.32267384 \\
1.50 -0.34447456 -0.32293533 \\
1.52 -0.34510430 -0.32314736 \\
1.54 -0.34570588 -0.32331585 \\
1.56 -0.34628262 -0.32344832 \\
1.58 -0.34683727 -0.32354868 \\
1.60 -0.34737215 -0.32362047 \\
1.62 -0.34788922 -0.32366135 \\
1.64 -0.34839015 -0.32368152 \\
1.66 -0.34887638 -0.32368823 \\
1.68 -0.34934915 -0.32368988 \\
1.70 -0.34980956 -0.32369025 \\
1.72 -0.35025855 -0.32369297 \\
1.74 -0.35069699 -0.32369439 \\
1.76 -0.35112561 -0.32369517 \\
1.78 -0.35154511 -0.32369561 \\
1.80 -0.35195606 -0.32369587 \\
1.82 -0.35235903 -0.32369602 \\
1.84 -0.35275448 -0.32369612 \\
1.86 -0.35314290 -0.32369619 \\
1.88 -0.35352464 -0.32369623 \\
1.90 -0.35390008 -0.32369626 \\
1.92 -0.35426952 -0.32369628 \\
1.94 -0.35463325 -0.32369630 \\
1.96 -0.35499142 -0.32369631 \\
1.98 -0.35534360 -0.32369632 \\
2.00 -0.35569129 -0.32369690 \\
}\datatable

\addplot[semithick, densely dashed, forget plot] {-0.355};

\addplot[thick, color2]
    table[x=t, y=Xo_nominal]{\datatable};
\addlegendentry{Nominal trajectory}

\addplot[thick, color1]
    table[x=t, y=Xo_true]{\datatable};

\end{axis}

\end{tikzpicture} \hspace*{-5mm}
        \begin{tikzpicture}

\definecolor{color1}{RGB}{228,26,28}
\definecolor{color2}{RGB}{55,126,184}
\definecolor{color3}{RGB}{77,175,74}
\definecolor{color4}{RGB}{152,78,163}

\begin{axis}[
width=3.6cm, height=3.4cm,
xlabel={\footnotesize $t$},
ylabel={\footnotesize $y$},
xlabel shift=-1mm,
ylabel shift=-2mm,
xmin=0, xmax=2,
ymax=0.85,
tick label style={
    font=\scriptsize,
    /pgf/number format/fixed,
},
legend cell align={left},
legend style={
    font=\footnotesize,
    at={(-0.3,1.25)},
    anchor=north west,
    draw=none,
    fill=none,
    row sep=-2pt,
    inner sep=0pt,
},
legend image post style={scale=0.4},
]

\pgfplotstableread[row sep=crcr]{
t     Xo_nominal  Xo_true    \\
0.00  0.30000000  0.30000000 \\
0.02  0.30038597  0.30000002 \\
0.04  0.30225106  0.30163658 \\
0.06  0.31304058  0.31228849 \\
0.08  0.33281308  0.33196132 \\
0.10  0.35915547  0.35822785 \\
0.12  0.39001800  0.38901957 \\
0.14  0.42340474  0.42232316 \\
0.16  0.45708744  0.45595156 \\
0.18  0.48926864  0.48810317 \\
0.20  0.51867198  0.51746537 \\
0.22  0.54437235  0.54310706 \\
0.24  0.56567471  0.56432701 \\
0.26  0.58234909  0.58089292 \\
0.28  0.59372494  0.59210028 \\
0.30  0.59941767  0.59748094 \\
0.32  0.60172268  0.59908187 \\
0.34  0.60504823  0.60175923 \\
0.36  0.61341771  0.61058803 \\
0.38  0.62199730  0.61938537 \\
0.40  0.63080673  0.62817826 \\
0.42  0.63899432  0.63642251 \\
0.44  0.64677260  0.64437672 \\
0.46  0.65345522  0.65195255 \\
0.48  0.65935622  0.65860515 \\
0.50  0.66421035  0.66396842 \\
0.52  0.66918840  0.66959578 \\
0.54  0.67421152  0.67551751 \\
0.56  0.67938156  0.68179530 \\
0.58  0.68438815  0.68775754 \\
0.60  0.68918669  0.69344225 \\
0.62  0.69372190  0.69883729 \\
0.64  0.69792380  0.70388448 \\
0.66  0.70161066  0.70862153 \\
0.68  0.70485108  0.71301981 \\
0.70  0.70787808  0.71716313 \\
0.72  0.71082629  0.72114233 \\
0.74  0.71373296  0.72502872 \\
0.76  0.71662502  0.72887201 \\
0.78  0.71952040  0.73271532 \\
0.80  0.72240102  0.73658302 \\
0.82  0.72521205  0.74048316 \\
0.84  0.72787978  0.74440636 \\
0.86  0.73034206  0.74800925 \\
0.88  0.73261318  0.75126507 \\
0.90  0.73476674  0.75434357 \\
0.92  0.73684614  0.75735746 \\
0.94  0.73886311  0.76035144 \\
0.96  0.74083898  0.76337559 \\
0.98  0.74279098  0.76646228 \\
1.00  0.74471567  0.76958551 \\
1.02  0.74660734  0.77270940 \\
1.04  0.74846870  0.77581383 \\
1.06  0.75030185  0.77888139 \\
1.08  0.75210463  0.78190195 \\
1.10  0.75387300  0.78486934 \\
1.12  0.75560164  0.78777632 \\
1.14  0.75728269  0.79061232 \\
1.16  0.75890390  0.79336122 \\
1.18  0.76044576  0.79599804 \\
1.20  0.76187865  0.79848413 \\
1.22  0.76316612  0.80076355 \\
1.24  0.76427850  0.80275770 \\
1.26  0.76521039  0.80439595 \\
1.28  0.76598475  0.80567372 \\
1.30  0.76663782  0.80664339 \\
1.32  0.76720232  0.80738280 \\
1.34  0.76770194  0.80795729 \\
1.36  0.76815284  0.80841346 \\
1.38  0.76856617  0.80878082 \\
1.40  0.76894980  0.80908537 \\
1.42  0.76930945  0.80933898 \\
1.44  0.76964935  0.80955111 \\
1.46  0.76997274  0.80972925 \\
1.48  0.77028210  0.80987729 \\
1.50  0.77057940  0.80999941 \\
1.52  0.77086619  0.81009841 \\
1.54  0.77114378  0.81017770 \\
1.56  0.77141317  0.81023390 \\
1.58  0.77167527  0.81027060 \\
1.60  0.77193079  0.81029268 \\
1.62  0.77218036  0.81031102 \\
1.64  0.77242449  0.81031772 \\
1.66  0.77266368  0.81031955 \\
1.68  0.77289830  0.81031995 \\
1.70  0.77312874  0.81032004 \\
1.72  0.77335527  0.81032068 \\
1.74  0.77357823  0.81032102 \\
1.76  0.77379784  0.81032121 \\
1.78  0.77401436  0.81032131 \\
1.80  0.77422800  0.81032137 \\
1.82  0.77443898  0.81032141 \\
1.84  0.77464747  0.81032144 \\
1.86  0.77485368  0.81032145 \\
1.88  0.77505775  0.81032146 \\
1.90  0.77525987  0.81032147 \\
1.92  0.77546013  0.81032148 \\
1.94  0.77565870  0.81032148 \\
1.96  0.77585562  0.81032148 \\
1.98  0.77605074  0.81032149 \\
2.00  0.77624463  0.81032164 \\
}\datatable

\addplot[semithick, densely dashed, forget plot] {0.8};

\addplot[thick, color2, forget plot]
    table[x=t, y=Xo_nominal]{\datatable};

\addplot[thick, color1]
    table[x=t, y=Xo_true]{\datatable};
\addlegendentry{True dynamics rollout}

\addplot[name path=lower, draw=none, forget plot] {0.8};
\addplot[name path=upper, draw=none, forget plot] {0.9};

\addplot[
    draw=none,
    pattern=crosshatch,
    pattern color=black!60,
]
fill between[
    of=upper and lower,
];

\end{axis}

\end{tikzpicture} \hspace*{-5mm}
        \begin{tikzpicture}

\definecolor{color1}{RGB}{228,26,28}
\definecolor{color2}{RGB}{55,126,184}
\definecolor{color3}{RGB}{77,175,74}
\definecolor{color4}{RGB}{152,78,163}

\begin{axis}[
width=3.6cm, height=3.4cm,
xlabel={\footnotesize $t$},
ylabel={\footnotesize $\theta$},
xlabel shift=-1mm,
ylabel shift=-2.8mm,
xmin=0, xmax=2,
tick label style={
    font=\scriptsize,
    /pgf/number format/fixed,
},
every y tick scale label/.style={
    at={(axis description cs:0.03,1.05)},
    anchor=north east,
    inner sep=0,
},
legend cell align={left},
legend style={
    font=\footnotesize,
    at={(-0.25,1.25)},
    anchor=north west,
    draw=none,
    fill=none,
    row sep=-2pt,
    inner sep=0pt,
},
legend image post style={scale=0.4},
]

\pgfplotstableread[row sep=crcr]{
t     Xo_nominal  Xo_true    \\
0.00  0.00000000  0.00000000 \\
0.02  0.00000124 -0.00000000 \\
0.04  0.00048885  0.00068113 \\
0.06 -0.00133380 -0.00142920 \\
0.08 -0.00448567 -0.00478964 \\
0.10 -0.00917068 -0.00961523 \\
0.12 -0.01511218 -0.01564481 \\
0.14 -0.02238318 -0.02297603 \\
0.16 -0.02863063 -0.02946539 \\
0.18 -0.03376399 -0.03467937 \\
0.20 -0.03797908 -0.03893385 \\
0.22 -0.04145292 -0.04241803 \\
0.24 -0.04456434 -0.04552702 \\
0.26 -0.04744502 -0.04840350 \\
0.28 -0.05121373 -0.05221490 \\
0.30 -0.05670101 -0.05788669 \\
0.32 -0.06292489 -0.06445319 \\
0.34 -0.06590547 -0.06770018 \\
0.36 -0.06245624 -0.06333155 \\
0.38 -0.05840551 -0.05874647 \\
0.40 -0.05399231 -0.05408217 \\
0.42 -0.04978018 -0.04953822 \\
0.44 -0.04567894 -0.04498380 \\
0.46 -0.04235101 -0.04054533 \\
0.48 -0.03961620 -0.03686589 \\
0.50 -0.03770080 -0.03432216 \\
0.52 -0.03545520 -0.03134849 \\
0.54 -0.03295677 -0.02792105 \\
0.56 -0.03011919 -0.02400709 \\
0.58 -0.02723265 -0.02025727 \\
0.60 -0.02435821 -0.01663340 \\
0.62 -0.02160722 -0.01313382 \\
0.64 -0.01913004 -0.00979954 \\
0.66 -0.01705700 -0.00666358 \\
0.68 -0.01529410 -0.00373857 \\
0.70 -0.01364582 -0.00097437 \\
0.72 -0.01202169  0.00167768 \\
0.74 -0.01041792  0.00425596 \\
0.76 -0.00883016  0.00678895 \\
0.78 -0.00726000  0.00930031 \\
0.80 -0.00573596  0.01180269 \\
0.82 -0.00431229  0.01430187 \\
0.84 -0.00304966  0.01679763 \\
0.86 -0.00198588  0.01898187 \\
0.88 -0.00108175  0.02084961 \\
0.90 -0.00025239  0.02257926 \\
0.92  0.00054634  0.02428431 \\
0.94  0.00132357  0.02600705 \\
0.96  0.00209342  0.02778989 \\
0.98  0.00286213  0.02965419 \\
1.00  0.00361765  0.03156700 \\
1.02  0.00434993  0.03348977 \\
1.04  0.00505858  0.03540081 \\
1.06  0.00574285  0.03728138 \\
1.08  0.00639901  0.03912080 \\
1.10  0.00702331  0.04091394 \\
1.12  0.00761225  0.04265618 \\
1.14  0.00816157  0.04434138 \\
1.16  0.00866521  0.04596043 \\
1.18  0.00911393  0.04749941 \\
1.20  0.00949400  0.04893643 \\
1.22  0.00979006  0.05023798 \\
1.24  0.00999411  0.05135017 \\
1.26  0.01011503  0.05222116 \\
1.28  0.01017697  0.05285083 \\
1.30  0.01020569  0.05328197 \\
1.32  0.01021841  0.05357551 \\
1.34  0.01022390  0.05377977 \\
1.36  0.01022616  0.05392718 \\
1.38  0.01022694  0.05403613 \\
1.40  0.01022705  0.05412373 \\
1.42  0.01022686  0.05419517 \\
1.44  0.01022653  0.05425461 \\
1.46  0.01022616  0.05430522 \\
1.48  0.01022577  0.05434751 \\
1.50  0.01022541  0.05438275 \\
1.52  0.01022506  0.05441136 \\
1.54  0.01022474  0.05443464 \\
1.56  0.01022444  0.05444781 \\
1.58  0.01022417  0.05445288 \\
1.60  0.01022392  0.05445302 \\
1.62  0.01022370  0.05445792 \\
1.64  0.01022349  0.05445837 \\
1.66  0.01022331  0.05445819 \\
1.68  0.01022314  0.05445810 \\
1.70  0.01022301  0.05445808 \\
1.72  0.01022288  0.05445791 \\
1.74  0.01022280  0.05445783 \\
1.76  0.01022272  0.05445778 \\
1.78  0.01022269  0.05445775 \\
1.80  0.01022267  0.05445774 \\
1.82  0.01022270  0.05445773 \\
1.84  0.01022274  0.05445772 \\
1.86  0.01022283  0.05445772 \\
1.88  0.01022294  0.05445772 \\
1.90  0.01022310  0.05445771 \\
1.92  0.01022326  0.05445771 \\
1.94  0.01022347  0.05445771 \\
1.96  0.01022366  0.05445771 \\
1.98  0.01022388  0.05445771 \\
2.00  0.01022407  0.05445768 \\
}\datatable

\addplot[semithick, densely dashed, forget plot] {0.0};

\addplot[thick, color2, forget plot]
    table[x=t, y=Xo_nominal]{\datatable};

\addplot[thick, color1, forget plot]
    table[x=t, y=Xo_true]{\datatable};

\end{axis}

\end{tikzpicture}        
      }
    } \\
    \subfloat[]{%
        \hspace*{-2mm}
        \begin{tikzpicture}

\definecolor{color1}{RGB}{228,26,28}
\definecolor{color2}{RGB}{55,126,184}
\definecolor{color3}{RGB}{77,175,74}
\definecolor{color4}{RGB}{152,78,163}

\begin{axis}[
width=3.6cm, height=3.4cm,
xlabel={\footnotesize $t$},
ylabel={\footnotesize $x$},
xlabel shift=-1mm,
ylabel shift=-3mm,
xmin=0, xmax=2,
tick label style={
    font=\scriptsize,
    /pgf/number format/fixed,
},
legend cell align={left},
legend style={
    font=\footnotesize,
    at={(-0.3,1.25)},
    anchor=north west,
    draw=none,
    fill=none,
    row sep=-2pt,
    inner sep=0pt,
},
legend image post style={scale=0.4},
]

\pgfplotstableread[row sep=crcr]{
t     Xo_true    \\
0.00  0.20000000 \\
0.02  0.20000000 \\
0.04  0.19929339 \\
0.06  0.19624413 \\
0.08  0.19231783 \\
0.10  0.18765343 \\
0.12  0.18198639 \\
0.14  0.17540321 \\
0.16  0.16785346 \\
0.18  0.15931226 \\
0.20  0.14977791 \\
0.22  0.13927388 \\
0.24  0.12778100 \\
0.26  0.11531600 \\
0.28  0.10192444 \\
0.30  0.08760198 \\
0.32  0.07259018 \\
0.34  0.05702264 \\
0.36  0.04102246 \\
0.38  0.02454927 \\
0.40  0.00777347 \\
0.42 -0.00917048 \\
0.44 -0.02617741 \\
0.46 -0.04315137 \\
0.48 -0.06000363 \\
0.50 -0.07665214 \\
0.52 -0.09302156 \\
0.54 -0.10904313 \\
0.56 -0.12465451 \\
0.58 -0.13979938 \\
0.60 -0.15442708 \\
0.62 -0.16849276 \\
0.64 -0.18195964 \\
0.66 -0.19480817 \\
0.68 -0.20705010 \\
0.70 -0.21869768 \\
0.72 -0.22974425 \\
0.74 -0.24016794 \\
0.76 -0.24993360 \\
0.78 -0.25899414 \\
0.80 -0.26729127 \\
0.82 -0.27473590 \\
0.84 -0.28120808 \\
0.86 -0.28659759 \\
0.88 -0.29082252 \\
0.90 -0.29384941 \\
0.92 -0.29575023 \\
0.94 -0.29692252 \\
0.96 -0.29766365 \\
0.98 -0.29847069 \\
1.00 -0.29927078 \\
1.02 -0.30002397 \\
1.04 -0.30070073 \\
1.06 -0.30127907 \\
1.08 -0.30174142 \\
1.10 -0.30207266 \\
1.12 -0.30225847 \\
1.14 -0.30228379 \\
1.16 -0.30228402 \\
1.18 -0.30228406 \\
1.20 -0.30228409 \\
1.22 -0.30228411 \\
1.24 -0.30228413 \\
1.26 -0.30228416 \\
1.28 -0.30228423 \\
1.30 -0.30228462 \\
1.32 -0.30228541 \\
1.34 -0.30239191 \\
1.36 -0.30265670 \\
1.38 -0.30309978 \\
1.40 -0.30374098 \\
1.42 -0.30460132 \\
1.44 -0.30570049 \\
1.46 -0.30703573 \\
1.48 -0.30848675 \\
1.50 -0.31008061 \\
1.52 -0.31178395 \\
1.54 -0.31357216 \\
1.56 -0.31542303 \\
1.58 -0.31731886 \\
1.60 -0.31924571 \\
1.62 -0.32119291 \\
1.64 -0.32315240 \\
1.66 -0.32511821 \\
1.68 -0.32708611 \\
1.70 -0.32905340 \\
1.72 -0.33101836 \\
1.74 -0.33297952 \\
1.76 -0.33493574 \\
1.78 -0.33688598 \\
1.80 -0.33882922 \\
1.82 -0.34076432 \\
1.84 -0.34269004 \\
1.86 -0.34460510 \\
1.88 -0.34650505 \\
1.90 -0.34837947 \\
1.92 -0.35021747 \\
1.94 -0.35200602 \\
1.96 -0.35372947 \\
1.98 -0.35536935 \\
2.00 -0.35690557 \\
}\datatable

\addplot[semithick, densely dashed, forget plot] {-0.355};

\addplot[thick, color1]
    table[x=t, y=Xo_true]{\datatable};

\end{axis}

\end{tikzpicture} \hspace*{-5mm}
        \begin{tikzpicture}

\definecolor{color1}{RGB}{228,26,28}
\definecolor{color2}{RGB}{55,126,184}
\definecolor{color3}{RGB}{77,175,74}
\definecolor{color4}{RGB}{152,78,163}

\begin{axis}[
width=3.6cm, height=3.4cm,
xlabel={\footnotesize $t$},
ylabel={\footnotesize $y$},
xlabel shift=-1mm,
ylabel shift=-2mm,
xmin=0, xmax=2,
ymax=0.85,
tick label style={
    font=\scriptsize,
    /pgf/number format/fixed,
},
legend cell align={left},
legend style={
    font=\footnotesize,
    at={(-0.3,1.25)},
    anchor=north west,
    draw=none,
    fill=none,
    row sep=-2pt,
    inner sep=0pt,
},
legend image post style={scale=0.4},
]

\pgfplotstableread[row sep=crcr]{
t     Xo_true    \\
0.0   0.30000000 \\
0.02  0.30000002 \\
0.04  0.30684394 \\
0.06  0.31863075 \\
0.08  0.33496399 \\
0.10  0.35496555 \\
0.12  0.37768863 \\
0.14  0.40248150 \\
0.16  0.42870226 \\
0.18  0.45580133 \\
0.20  0.48331108 \\
0.22  0.51084067 \\
0.24  0.53799661 \\
0.26  0.56443978 \\
0.28  0.58987675 \\
0.30  0.61394560 \\
0.32  0.63663630 \\
0.34  0.65785615 \\
0.36  0.67760149 \\
0.38  0.69587035 \\
0.40  0.71270334 \\
0.42  0.72809842 \\
0.44  0.74203148 \\
0.46  0.75449059 \\
0.48  0.76547838 \\
0.50  0.77501065 \\
0.52  0.78311393 \\
0.54  0.78982308 \\
0.56  0.79517985 \\
0.58  0.79923376 \\
0.60  0.80204862 \\
0.62  0.80372293 \\
0.64  0.80444399 \\
0.66  0.80461546 \\
0.68  0.80477434 \\
0.70  0.80491990 \\
0.72  0.80505210 \\
0.74  0.80517115 \\
0.76  0.80527734 \\
0.78  0.80537091 \\
0.80  0.80545204 \\
0.82  0.80536930 \\
0.84  0.80490360 \\
0.86  0.80406162 \\
0.88  0.80307413 \\
0.90  0.80243190 \\
0.92  0.80313286 \\
0.94  0.80486848 \\
0.96  0.80590137 \\
0.98  0.80694057 \\
1.00  0.80789264 \\
1.02  0.80872353 \\
1.04  0.80941705 \\
1.06  0.80996838 \\
1.08  0.81037879 \\
1.10  0.81065271 \\
1.12  0.81079587 \\
1.14  0.81081399 \\
1.16  0.81081401 \\
1.18  0.81081401 \\
1.20  0.81081401 \\
1.22  0.81081400 \\
1.24  0.81081400 \\
1.26  0.81081399 \\
1.28  0.81081398 \\
1.30  0.81081394 \\
1.32  0.81081415 \\
1.34  0.81085186 \\
1.36  0.81093951 \\
1.38  0.81107615 \\
1.40  0.81126053 \\
1.42  0.81149136 \\
1.44  0.81176672 \\
1.46  0.81210982 \\
1.48  0.81234214 \\
1.50  0.81237145 \\
1.52  0.81240396 \\
1.54  0.81243916 \\
1.56  0.81247514 \\
1.58  0.81251170 \\
1.60  0.81254824 \\
1.62  0.81258456 \\
1.64  0.81262056 \\
1.66  0.81265619 \\
1.68  0.81269136 \\
1.70  0.81272578 \\
1.72  0.81275950 \\
1.74  0.81279295 \\
1.76  0.81282619 \\
1.78  0.81285897 \\
1.80  0.81289098 \\
1.82  0.81292184 \\
1.84  0.81295120 \\
1.86  0.81297875 \\
1.88  0.81298532 \\
1.90  0.81292992 \\
1.92  0.81279959 \\
1.94  0.81259176 \\
1.96  0.81230794 \\
1.98  0.81194890 \\
2.00  0.81153036 \\
}\datatable

\addplot[semithick, densely dashed, forget plot] {0.8};

\addplot[thick, color1]
    table[x=t, y=Xo_true]{\datatable};
\addlegendentry{True dynamics rollout}

\addplot[name path=lower, draw=none, forget plot] {0.8};
\addplot[name path=upper, draw=none, forget plot] {0.9};

\addplot[
    draw=none,
    pattern=crosshatch,
    pattern color=black!60,
]
fill between[
    of=upper and lower,
];

\end{axis}

\end{tikzpicture} \hspace*{-5mm}
        \begin{tikzpicture}

\definecolor{color1}{RGB}{228,26,28}
\definecolor{color2}{RGB}{55,126,184}
\definecolor{color3}{RGB}{77,175,74}
\definecolor{color4}{RGB}{152,78,163}

\begin{axis}[
width=3.6cm, height=3.4cm,
xlabel={\footnotesize $t$},
ylabel={\footnotesize $\theta$},
xlabel shift=-1mm,
ylabel shift=-2.8mm,
xmin=0, xmax=2,
tick label style={
    font=\scriptsize,
    /pgf/number format/fixed,
},
every y tick scale label/.style={
    at={(axis description cs:0.03,1.05)},
    anchor=north east,
    inner sep=0,
},
legend cell align={left},
legend style={
    font=\footnotesize,
    at={(-0.25,1.25)},
    anchor=north west,
    draw=none,
    fill=none,
    row sep=-2pt,
    inner sep=0pt,
},
legend image post style={scale=0.4},
]

\pgfplotstableread[row sep=crcr]{
t      Xo_true    \\
0.00   0.00000000 \\
0.02  -0.00000000 \\
0.04  -0.00030288 \\
0.06  -0.00337745 \\
0.08  -0.00656407 \\
0.10  -0.00907595 \\
0.12  -0.01105399 \\
0.14  -0.01209782 \\
0.16  -0.01222828 \\
0.18  -0.01149116 \\
0.20  -0.00995820 \\
0.22  -0.00770201 \\
0.24  -0.00490260 \\
0.26  -0.00166500 \\
0.28   0.00193960 \\
0.30   0.00572216 \\
0.32   0.00997731 \\
0.34   0.01466699 \\
0.36   0.01962938 \\
0.38   0.02372504 \\
0.40   0.02708023 \\
0.42   0.02972539 \\
0.44   0.03164914 \\
0.46   0.03283885 \\
0.48   0.03328694 \\
0.50   0.03299191 \\
0.52   0.03195793 \\
0.54   0.03019430 \\
0.56   0.02771533 \\
0.58   0.02454207 \\
0.60   0.02070847 \\
0.62   0.01627888 \\
0.64   0.01139295 \\
0.66   0.00636968 \\
0.68   0.00163351 \\
0.70  -0.00283685 \\
0.72  -0.00705166 \\
0.74  -0.01101222 \\
0.76  -0.01471243 \\
0.78  -0.01813983 \\
0.80  -0.02127624 \\
0.82  -0.02421805 \\
0.84  -0.02710608 \\
0.86  -0.02989409 \\
0.88  -0.03235685 \\
0.90  -0.03406778 \\
0.92  -0.03421342 \\
0.94  -0.03320613 \\
0.96  -0.03258603 \\
0.98  -0.03193210 \\
1.00  -0.03130237 \\
1.02  -0.03072432 \\
1.04  -0.03021632 \\
1.06  -0.02979061 \\
1.08  -0.02945615 \\
1.10  -0.02922022 \\
1.12  -0.02908970 \\
1.14  -0.02907218 \\
1.16  -0.02907214 \\
1.18  -0.02907214 \\
1.20  -0.02907214 \\
1.22  -0.02907214 \\
1.24  -0.02907214 \\
1.26  -0.02907214 \\
1.28  -0.02907213 \\
1.30  -0.02907205 \\
1.32  -0.02907162 \\
1.34  -0.02900274 \\
1.36  -0.02883130 \\
1.38  -0.02854517 \\
1.40  -0.02813199 \\
1.42  -0.02757865 \\
1.44  -0.02687293 \\
1.46  -0.02601531 \\
1.48  -0.02520795 \\
1.50  -0.02449440 \\
1.52  -0.02372187 \\
1.54  -0.02289908 \\
1.56  -0.02203689 \\
1.58  -0.02114292 \\
1.60  -0.02022379 \\
1.62  -0.01928457 \\
1.64  -0.01832925 \\
1.66  -0.01736083 \\
1.68  -0.01638191 \\
1.70  -0.01539635 \\
1.72  -0.01440746 \\
1.74  -0.01341599 \\
1.76  -0.01242265 \\
1.78  -0.01142820 \\
1.80  -0.01043359 \\
1.82  -0.00943980 \\
1.84  -0.00844788 \\
1.86  -0.00745888 \\
1.88  -0.00649086 \\
1.90  -0.00558269 \\
1.92  -0.00475044 \\
1.94  -0.00400286 \\
1.96  -0.00334681 \\
1.98  -0.00279140 \\
2.00  -0.00233449 \\
}\datatable

\addplot[semithick, densely dashed, forget plot] {0.0};

\addplot[thick, color1, forget plot]
    table[x=t, y=Xo_true]{\datatable};

\end{axis}

\end{tikzpicture}
      \label{fig:iiwa-bucket-fallfree-iLQR}
    } \\
    \subfloat[]{%
      \shortstack{
        \adjincludegraphics[height=1.4cm, angle=90, trim={{.01\width} {.07\height} {.04\width} 0}, clip]{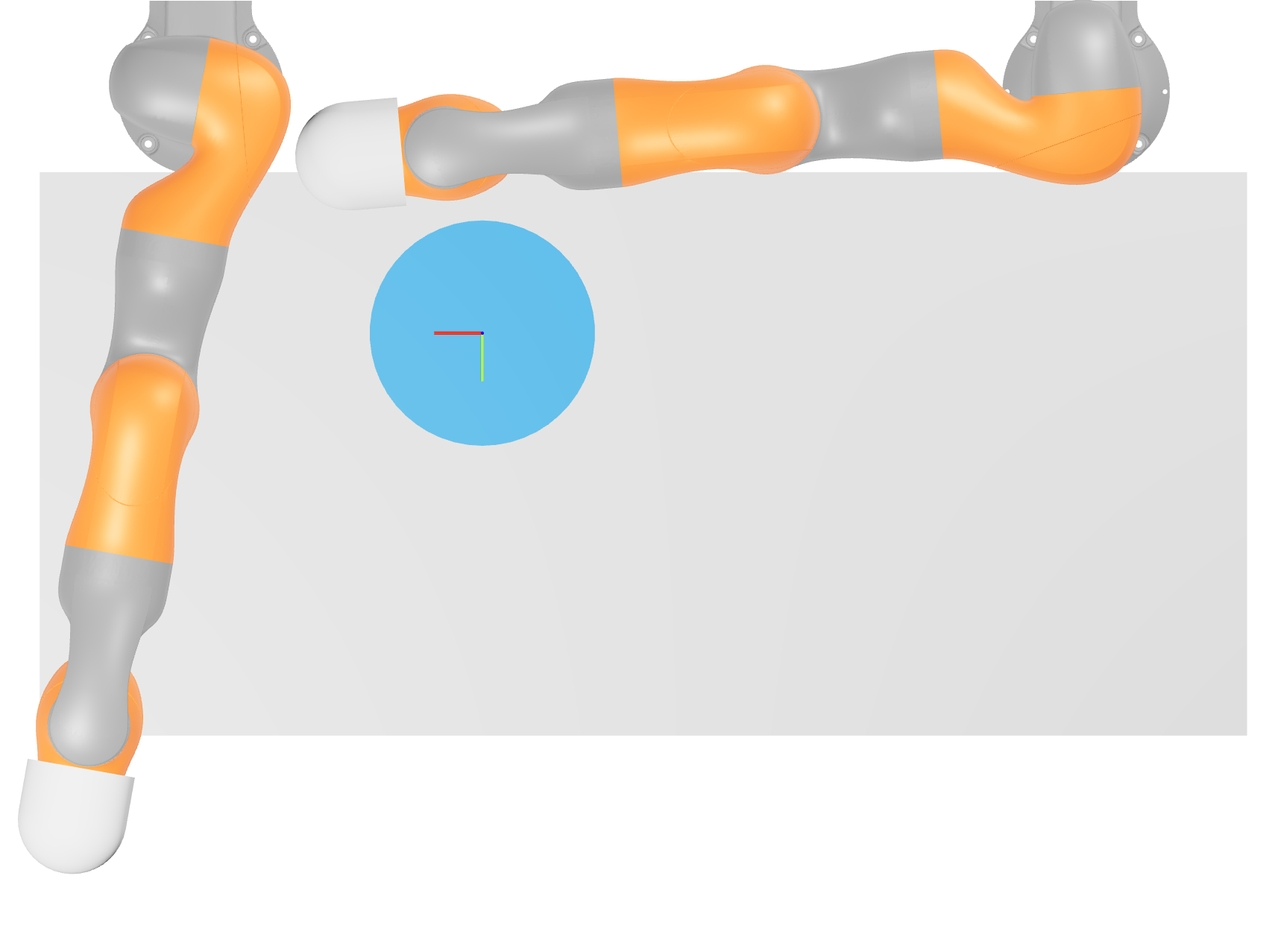}
        \adjincludegraphics[height=1.4cm, angle=90, trim={{.01\width} {.07\height} {.04\width} 0}, clip]{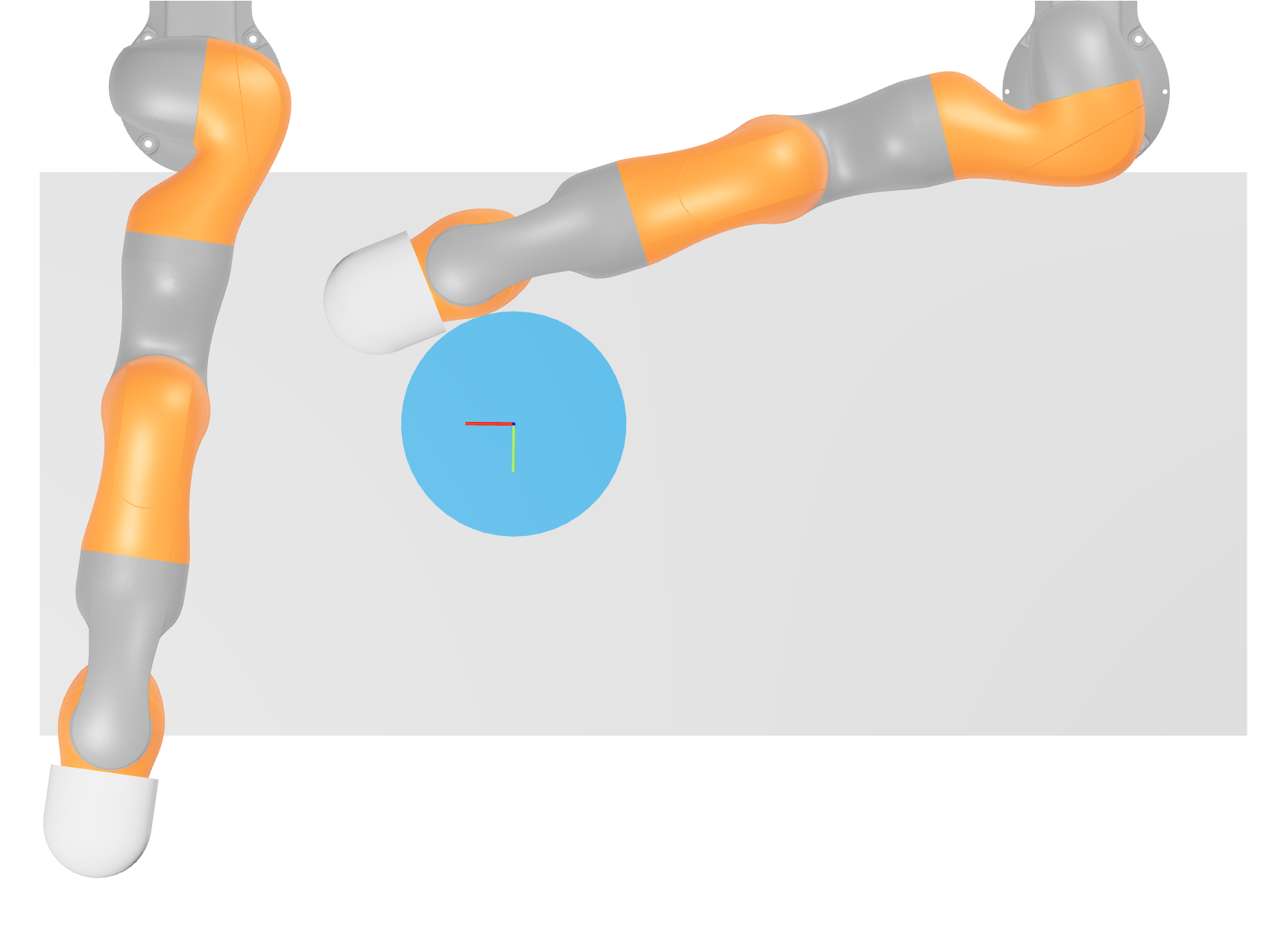}
        \adjincludegraphics[height=1.4cm, angle=90, trim={{.01\width} {.07\height} {.04\width} 0}, clip]{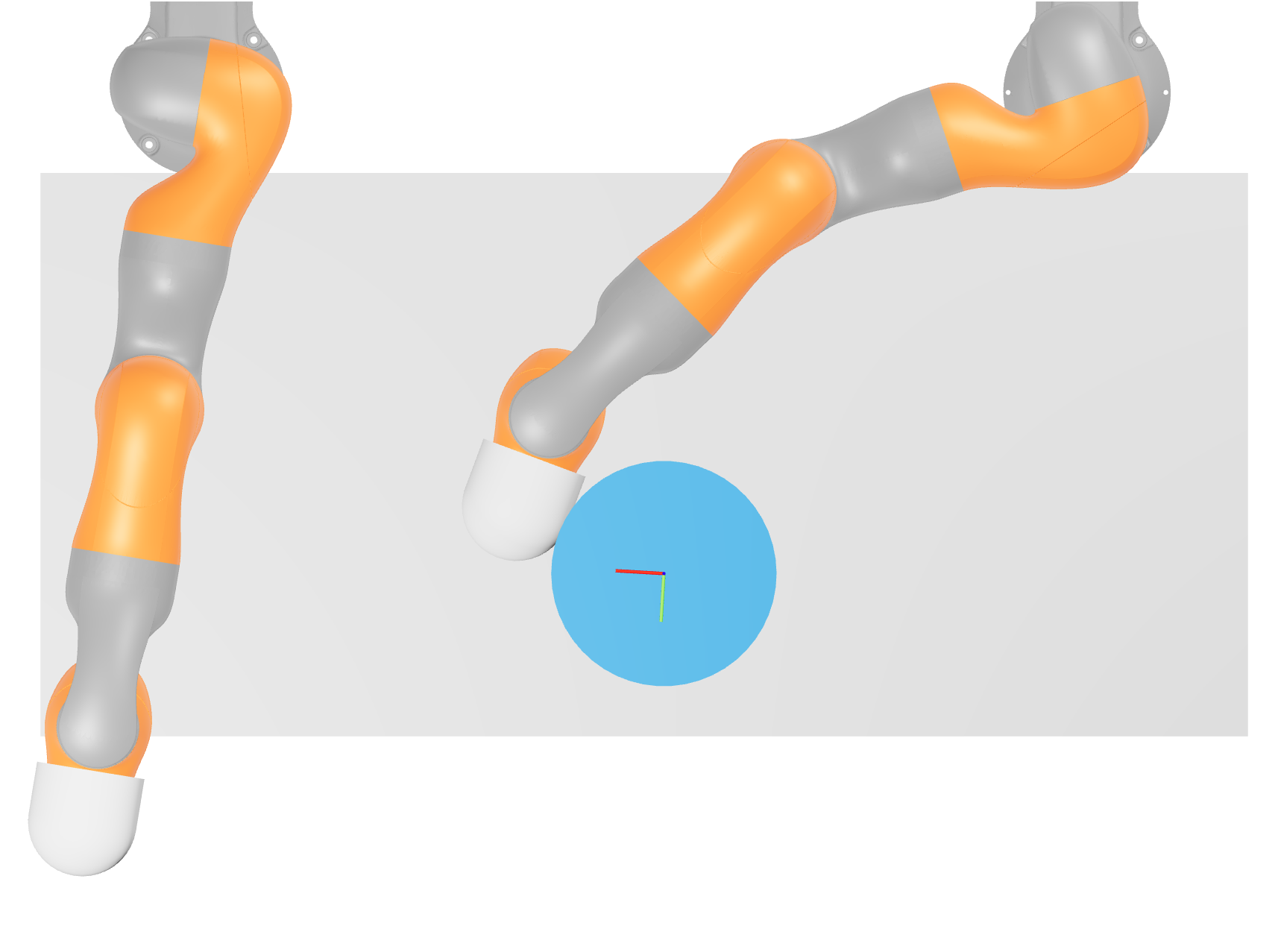}
        \adjincludegraphics[height=1.4cm, angle=90, trim={{.01\width} {.07\height} {.04\width} 0}, clip]{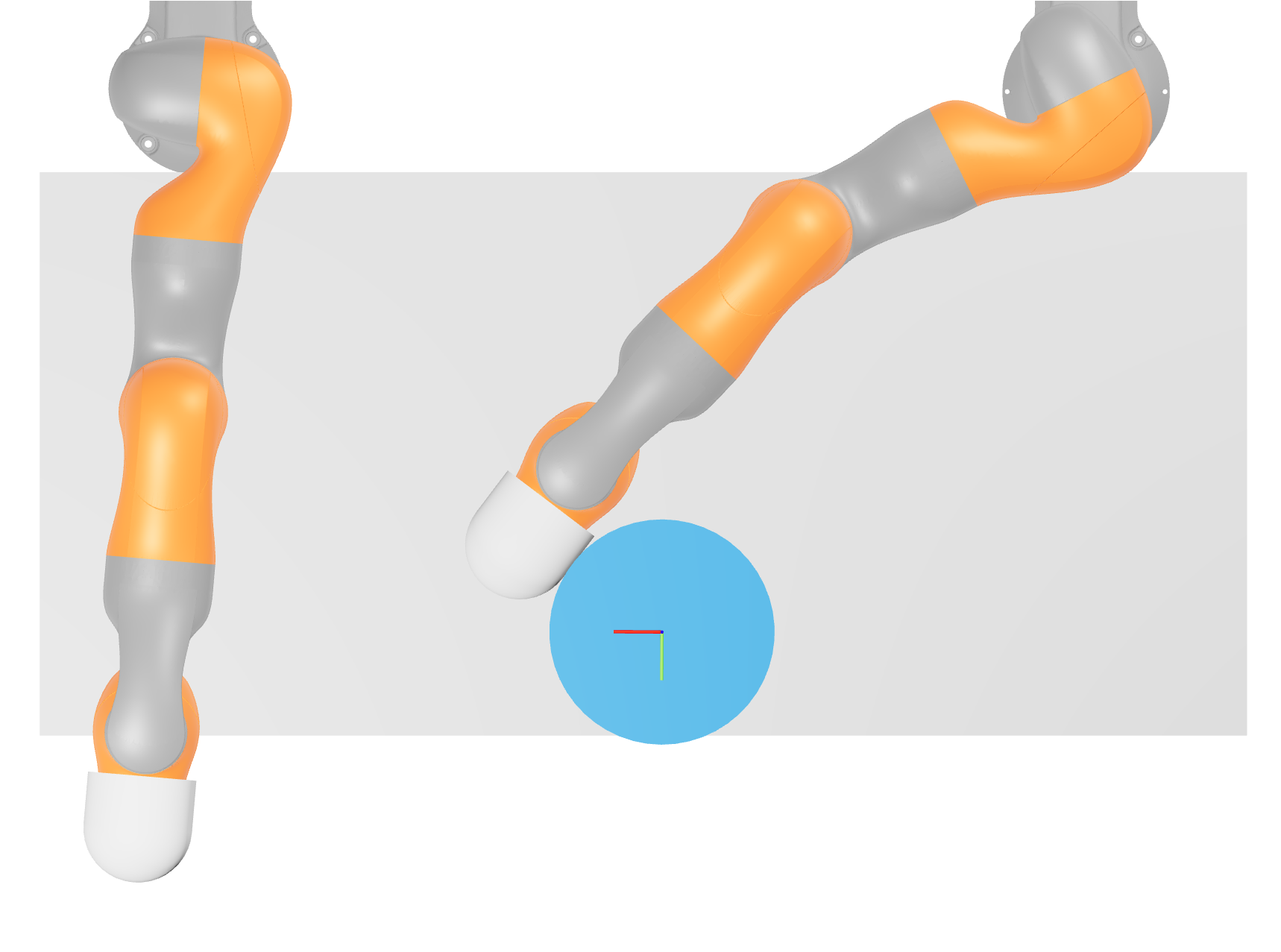}
        \adjincludegraphics[height=1.4cm, angle=90, trim={{.01\width} {.07\height} {.04\width} 0}, clip]{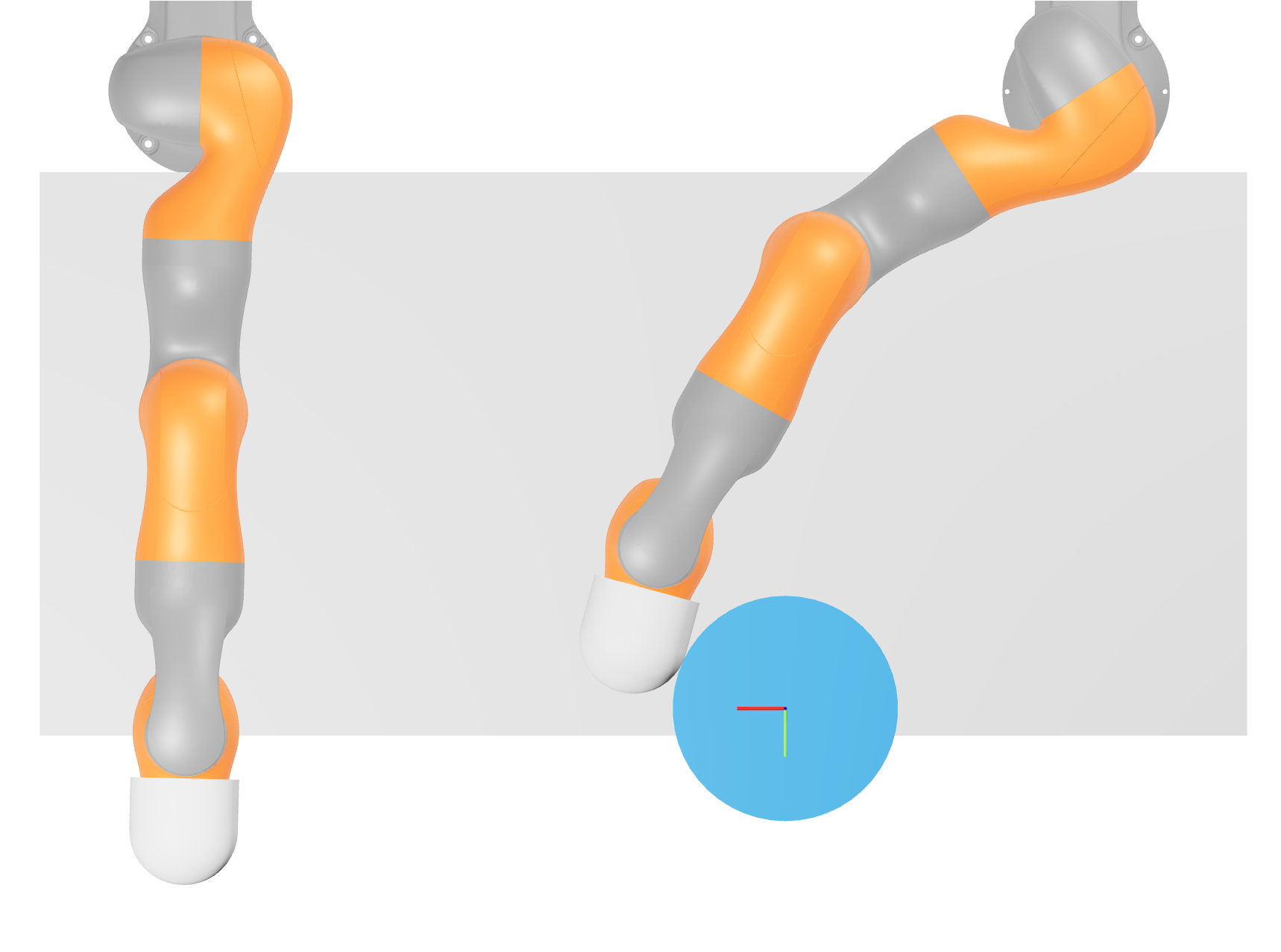}
        \adjincludegraphics[height=1.4cm, angle=90, trim={{.01\width} {.07\height} {.04\width} 0}, clip]{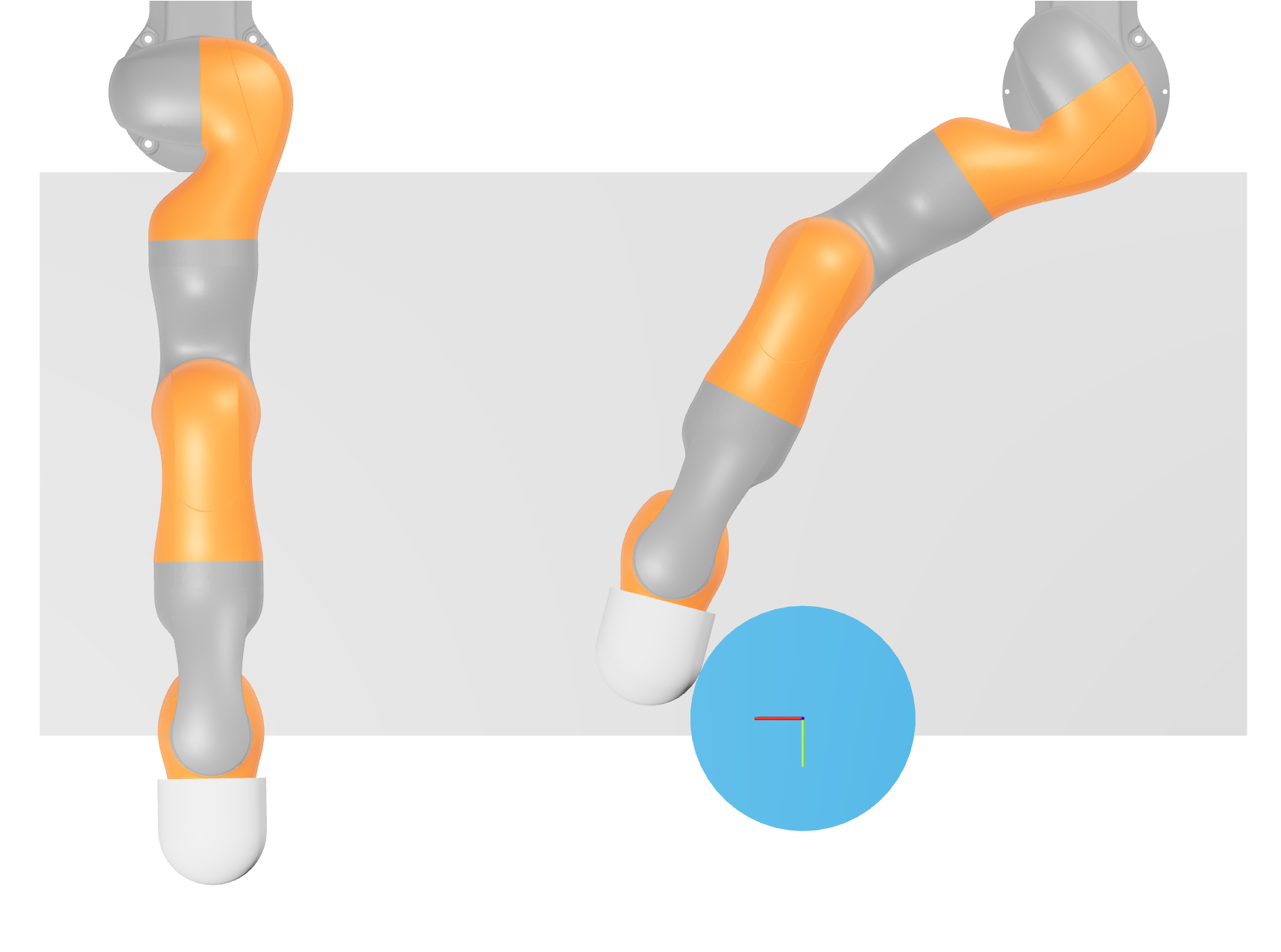}
        \\[2mm]
        \hspace*{-2mm}
        \begin{tikzpicture}

\definecolor{color1}{RGB}{228,26,28}
\definecolor{color2}{RGB}{55,126,184}
\definecolor{color3}{RGB}{77,175,74}
\definecolor{color4}{RGB}{152,78,163}

\begin{axis}[
width=3.6cm, height=3.4cm,
xlabel={\footnotesize $t$},
ylabel={\footnotesize $x$},
xlabel shift=-1mm,
ylabel shift=-3mm,
xmin=0, xmax=2,
tick label style={
    font=\scriptsize,
    /pgf/number format/fixed,
},
legend cell align={left},
legend style={
    font=\footnotesize,
    at={(-0.3,1.25)},
    anchor=north west,
    draw=none,
    fill=none,
    row sep=-2pt,
    inner sep=0pt,
},
legend image post style={scale=0.4},
]

\pgfplotstableread[row sep=crcr]{
t     Xo_nominal  Xo_true     Xo_lb       Xo_ub      \\
0.00  0.20000000  0.20000000  0.20000000  0.20000000 \\
0.02  0.19995382  0.20000000  0.20012113  0.19991634 \\
0.04  0.19831473  0.19713575  0.19837348  0.19697785 \\
0.06  0.19540338  0.19374830  0.19541179  0.19326937 \\
0.08  0.19112700  0.18932067  0.19112537  0.18850036 \\
0.10  0.18553789  0.18363787  0.18554588  0.18258442 \\
0.12  0.17856359  0.17652233  0.17857861  0.17529329 \\
0.14  0.17062068  0.16888471  0.17092236  0.16688530 \\
0.16  0.16102397  0.15905869  0.16149123  0.15739702 \\
0.18  0.15043960  0.14841055  0.15101559  0.14670191 \\
0.20  0.13883211  0.13678724  0.13952560  0.13498626 \\
0.22  0.12635510  0.12430635  0.12717582  0.12241592 \\
0.24  0.11321644  0.11117284  0.11416914  0.10920130 \\
0.26  0.09965307  0.09761971  0.10073744  0.09557874 \\
0.28  0.08592634  0.08391013  0.08713813  0.08179877 \\
0.30  0.07230632  0.07024621  0.07363033  0.06816930 \\
0.32  0.05924174  0.05741109  0.06147687  0.05402358 \\
0.34  0.04602266  0.04438933  0.04949007  0.03943967 \\
0.36  0.03288008  0.03091875  0.03639738  0.02678446 \\
0.38  0.02042142  0.01854304  0.02407356  0.01443547 \\
0.40  0.00855059  0.00677621  0.01238549  0.00262028 \\
0.42 -0.00274206 -0.00441032  0.00131322 -0.00862015 \\
0.44 -0.01346876 -0.01502633 -0.00916069 -0.01929522 \\
0.46 -0.02368086 -0.02512464 -0.01909996 -0.02945790 \\
0.48 -0.03339899 -0.03472456 -0.02853128 -0.03912802 \\
0.50 -0.04264279 -0.04385117 -0.03746881 -0.04833694 \\
0.52 -0.05143024 -0.05252097 -0.04592954 -0.05710312 \\
0.54 -0.05977931 -0.06074742 -0.05393741 -0.06543651 \\
0.56 -0.06770921 -0.06854933 -0.06151395 -0.07335500 \\
0.58 -0.07523939 -0.07594612 -0.06868170 -0.08087678 \\
0.60 -0.08238927 -0.08295714 -0.07546475 -0.08801931 \\
0.62 -0.08917808 -0.08960264 -0.08188497 -0.09480210 \\
0.64 -0.09562537 -0.09590288 -0.08796427 -0.10124494 \\
0.66 -0.10174656 -0.10188952 -0.09372638 -0.10735886 \\
0.68 -0.10755356 -0.10757183 -0.09919509 -0.11314887 \\
0.70 -0.11307054 -0.11298544 -0.10437147 -0.11867289 \\
0.72 -0.11829135 -0.11810019 -0.10938055 -0.12388053 \\
0.74 -0.12323844 -0.12294091 -0.11416013 -0.12879918 \\
0.76 -0.12791791 -0.12752500 -0.11872296 -0.13343866 \\
0.78 -0.13232438 -0.13185460 -0.12308148 -0.13778771 \\
0.80 -0.13644988 -0.13592572 -0.12723659 -0.14183240 \\
0.82 -0.14029297 -0.13973531 -0.13117444 -0.14557180 \\
0.84 -0.14386937 -0.14329055 -0.13487304 -0.14903299 \\
0.86 -0.14721260 -0.14661427 -0.13832424 -0.15226634 \\
0.88 -0.15036916 -0.14974287 -0.14153849 -0.15533735 \\
0.90 -0.15336820 -0.15271167 -0.14456855 -0.15826708 \\
0.92 -0.15622280 -0.15554027 -0.14744354 -0.16106145 \\
0.94 -0.15894226 -0.15824077 -0.15017339 -0.16372893 \\
0.96 -0.16153520 -0.16082259 -0.15276524 -0.16627820 \\
0.98 -0.16400715 -0.16328994 -0.15522201 -0.16871475 \\
1.00 -0.16635656 -0.16563617 -0.15753754 -0.17103666 \\
1.02 -0.16858665 -0.16786437 -0.15971456 -0.17324710 \\
1.04 -0.17070073 -0.16997187 -0.16175293 -0.17534883 \\
1.06 -0.17270581 -0.17196431 -0.16366010 -0.17734863 \\
1.08 -0.17461212 -0.17385405 -0.16544840 -0.17925655 \\
1.10 -0.17643079 -0.17565440 -0.16713104 -0.18108339 \\
1.12 -0.17817244 -0.17737700 -0.16872040 -0.18283936 \\
1.14 -0.17983521 -0.17901252 -0.17021212 -0.18452159 \\
1.16 -0.18142466 -0.18056528 -0.17161311 -0.18613518 \\
1.18 -0.18294516 -0.18203932 -0.17292818 -0.18768423 \\
1.20 -0.18438713 -0.18343021 -0.17415053 -0.18915134 \\
1.22 -0.18574543 -0.18473142 -0.17526883 -0.19053090 \\
1.24 -0.18701989 -0.18593864 -0.17627802 -0.19182313 \\
1.26 -0.18821445 -0.18705133 -0.17717993 -0.19303237 \\
1.28 -0.18933638 -0.18807307 -0.17798272 -0.19416631 \\
1.30 -0.19039559 -0.18901003 -0.17870015 -0.19523541 \\
1.32 -0.19140400 -0.18987396 -0.17935048 -0.19625268 \\
1.34 -0.19237328 -0.19067828 -0.17995219 -0.19723124 \\
1.36 -0.19330457 -0.19142989 -0.18050831 -0.19816958 \\
1.38 -0.19419912 -0.19213155 -0.18101976 -0.19906835 \\
1.40 -0.19505815 -0.19278472 -0.18148749 -0.19992837 \\
1.42 -0.19588290 -0.19339054 -0.18191250 -0.20075058 \\
1.44 -0.19667465 -0.19395022 -0.18229599 -0.20153601 \\
1.46 -0.19743469 -0.19446514 -0.18263929 -0.20228583 \\
1.48 -0.19816436 -0.19493141 -0.18294394 -0.20300131 \\
1.50 -0.19886503 -0.19535351 -0.18321164 -0.20368386 \\
1.52 -0.19953810 -0.19573415 -0.18344425 -0.20433494 \\
1.54 -0.20018500 -0.19607669 -0.18364377 -0.20495611 \\
1.56 -0.20080713 -0.19638454 -0.18381229 -0.20554896 \\
1.58 -0.20140593 -0.19666036 -0.18395194 -0.20611509 \\
1.60 -0.20198281 -0.19690620 -0.18406491 -0.20665616 \\
1.62 -0.20253916 -0.19712384 -0.18415337 -0.20717376 \\
1.64 -0.20307628 -0.19731489 -0.18421939 -0.20766942 \\
1.66 -0.20359546 -0.19748073 -0.18426500 -0.20814467 \\
1.68 -0.20409792 -0.19762304 -0.18429214 -0.20860095 \\
1.70 -0.20458482 -0.19774444 -0.18430263 -0.20903964 \\
1.72 -0.20505724 -0.19784648 -0.18429818 -0.20946205 \\
1.74 -0.20551623 -0.19792971 -0.18428037 -0.20986940 \\
1.76 -0.20596275 -0.19799492 -0.18425066 -0.21026284 \\
1.78 -0.20639755 -0.19804414 -0.18421014 -0.21064329 \\
1.80 -0.20682072 -0.19807704 -0.18415881 -0.21101078 \\
1.82 -0.20723203 -0.19809190 -0.18409629 -0.21136502 \\
1.84 -0.20763153 -0.19809580 -0.18402266 -0.21170600 \\
1.86 -0.20802016 -0.19809662 -0.18393944 -0.21203490 \\
1.88 -0.20839965 -0.19809686 -0.18384917 -0.21235382 \\
1.90 -0.20877171 -0.19809696 -0.18375423 -0.21266485 \\
1.92 -0.20913824 -0.19809704 -0.18365719 -0.21297030 \\
1.94 -0.20950101 -0.19809712 -0.18356037 -0.21327233 \\
1.96 -0.20986195 -0.19809730 -0.18346629 -0.21357326 \\
1.98 -0.21022312 -0.19809808 -0.18337768 -0.21387564 \\
2.00 -0.21059523 -0.19816486 -0.18331848 -0.21419315 \\
}\datatable

\addplot[semithick, densely dashed, forget plot] {-0.355};

\addplot[thick, color2]
    table[x=t, y=Xo_nominal]{\datatable};
\addlegendentry{Nominal trajectory}

\addplot[thick, color1]
    table[x=t, y=Xo_true]{\datatable};

\addplot[name path=lower, draw=none, forget plot]
    table[x=t, y=Xo_lb]{\datatable};

\addplot[name path=upper, draw=none, forget plot]
    table[x=t, y=Xo_ub]{\datatable};

\addplot[fill=color3, fill opacity=0.4, draw=none, legend image code/.code={
        \draw[fill=color3, fill opacity=0.4, draw=none] (0cm,-0.1cm) rectangle (0.6cm,0.1cm);
    }]
    fill between[of=lower and upper];

\end{axis}

\end{tikzpicture} \hspace*{-5mm}
        \begin{tikzpicture}

\definecolor{color1}{RGB}{228,26,28}
\definecolor{color2}{RGB}{55,126,184}
\definecolor{color3}{RGB}{77,175,74}
\definecolor{color4}{RGB}{152,78,163}

\begin{axis}[
width=3.6cm, height=3.4cm,
xlabel={\footnotesize $t$},
ylabel={\footnotesize $y$},
xlabel shift=-1mm,
ylabel shift=-2mm,
xmin=0, xmax=2,
ymax=0.85,
tick label style={
    font=\scriptsize,
    /pgf/number format/fixed,
},
legend cell align={left},
legend style={
    font=\footnotesize,
    at={(-0.3,1.25)},
    anchor=north west,
    draw=none,
    fill=none,
    row sep=-2pt,
    inner sep=0pt,
},
legend image post style={scale=0.4},
]

\pgfplotstableread[row sep=crcr]{
t     Xo_nominal  Xo_true     Xo_lb       Xo_ub      \\
0.00  0.30000000  0.30000000  0.30000000  0.30000000 \\
0.02  0.30038597  0.30000002  0.30018233  0.29953232 \\
0.04  0.30517557  0.30437717  0.30483246  0.30375165 \\
0.06  0.31574814  0.31466246  0.31534694  0.31387622 \\
0.08  0.33099204  0.32979131  0.33058498  0.32883133 \\
0.10  0.34984012  0.34857771  0.34946250  0.34746757 \\
0.12  0.37129191  0.36996041  0.37094259  0.36872589 \\
0.14  0.39460452  0.39348504  0.39451279  0.39245496 \\
0.16  0.41889744  0.41770515  0.41887956  0.41666551 \\
0.18  0.44365943  0.44244547  0.44368101  0.44126611 \\
0.20  0.46826321  0.46703771  0.46832928  0.46569991 \\
0.22  0.49221523  0.49097348  0.49233466  0.48947283 \\
0.24  0.51510254  0.51383666  0.51528549  0.51216559 \\
0.26  0.53657587  0.53526940  0.53683294  0.53341513 \\
0.28  0.55634754  0.55498201  0.55668937  0.55292382 \\
0.30  0.57449086  0.57298091  0.57490977  0.57079715 \\
0.32  0.59137281  0.59007630  0.59209550  0.58826914 \\
0.34  0.60659840  0.60552578  0.60741587  0.60431759 \\
0.36  0.62025798  0.61888031  0.62115164  0.61797839 \\
0.38  0.63307153  0.63166563  0.63411044  0.63049532 \\
0.40  0.64486680  0.64347162  0.64608871  0.64197872 \\
0.42  0.65566578  0.65429231  0.65710075  0.65249084 \\
0.44  0.66552166  0.66418039  0.66719457  0.66208565 \\
0.46  0.67449501  0.67319368  0.67641865  0.67081991 \\
0.48  0.68266807  0.68141571  0.68484722  0.67877626 \\
0.50  0.69011689  0.68892030  0.69255213  0.68602754 \\
0.52  0.69691068  0.69577505  0.69960294  0.69264046 \\
0.54  0.70311145  0.70204192  0.70606193  0.69867543 \\
0.56  0.70877393  0.70777427  0.71198402  0.70418294 \\
0.58  0.71394720  0.71302021  0.71741696  0.70920787 \\
0.60  0.71867537  0.71782306  0.72240293  0.71379253 \\
0.62  0.72299814  0.72222152  0.72697913  0.71797425 \\
0.64  0.72695990  0.72625049  0.73118911  0.72179443 \\
0.66  0.73061011  0.72994642  0.73508378  0.72529814 \\
0.68  0.73398460  0.73335042  0.73869882  0.72851740 \\
0.70  0.73709110  0.73645105  0.74210204  0.73138176 \\
0.72  0.73999738  0.73935440  0.74527149  0.73395696 \\
0.74  0.74271252  0.74206842  0.74826288  0.73631832 \\
0.76  0.74522307  0.74456195  0.75102953  0.73846577 \\
0.78  0.74752646  0.74681924  0.75357110  0.74037831 \\
0.80  0.74962847  0.74883757  0.75590804  0.74204255 \\
0.82  0.75154816  0.75063600  0.75807190  0.74347214 \\
0.84  0.75331855  0.75226007  0.76008652  0.74472652 \\
0.86  0.75497449  0.75376359  0.76197248  0.74587454 \\
0.88  0.75654955  0.75520298  0.76373812  0.74698853 \\
0.90  0.75804393  0.75656742  0.76540101  0.74804187 \\
0.92  0.75945818  0.75784393  0.76697068  0.74901711 \\
0.94  0.76079803  0.75903249  0.76845307  0.74991868 \\
0.96  0.76206925  0.76013666  0.76985428  0.75075311 \\
0.98  0.76327712  0.76116294  0.77118030  0.75152753 \\
1.00  0.76442515  0.76212194  0.77243562  0.75224859 \\
1.02  0.76551451  0.76301407  0.77362123  0.75291647 \\
1.04  0.76655083  0.76385498  0.77474421  0.75354081 \\
1.06  0.76753716  0.76465058  0.77580780  0.75412472 \\
1.08  0.76847716  0.76540182  0.77681592  0.75467069 \\
1.10  0.76937569  0.76611124  0.77777395  0.75518356 \\
1.12  0.77023794  0.76678412  0.77868780  0.75566938 \\
1.14  0.77106551  0.76743354  0.77955978  0.75613096 \\
1.16  0.77186068  0.76806532  0.78039232  0.75657016 \\
1.18  0.77262494  0.76868223  0.78118615  0.75698824 \\
1.20  0.77335081  0.76925483  0.78193756  0.75736582 \\
1.22  0.77403713  0.76977920  0.78264338  0.75770245 \\
1.24  0.77468484  0.77026058  0.78330336  0.75800006 \\
1.26  0.77529643  0.77070820  0.78391942  0.75826191 \\
1.28  0.77587563  0.77113358  0.78449539  0.75849251 \\
1.30  0.77642717  0.77155227  0.78503675  0.75869770 \\
1.32  0.77695665  0.77197907  0.78555028  0.75888487 \\
1.34  0.77746925  0.77242665  0.78604225  0.75906114 \\
1.36  0.77796390  0.77287613  0.78651280  0.75922103 \\
1.38  0.77844114  0.77332219  0.78696241  0.75936442 \\
1.40  0.77890147  0.77376183  0.78739163  0.75949109 \\
1.42  0.77934540  0.77419277  0.78780103  0.75960100 \\
1.44  0.77977347  0.77461279  0.78819119  0.75969429 \\
1.46  0.78018628  0.77501963  0.78856276  0.75977123 \\
1.48  0.78058442  0.77542027  0.78891643  0.75983229 \\
1.50  0.78096853  0.77580866  0.78925295  0.75987806 \\
1.52  0.78133928  0.77618099  0.78957309  0.75990923 \\
1.54  0.78169734  0.77653288  0.78987765  0.75992661 \\
1.56  0.78204336  0.77686014  0.79016746  0.75993105 \\
1.58  0.78237805  0.77716025  0.79044332  0.75992372 \\
1.60  0.78270208  0.77743228  0.79070611  0.75990569 \\
1.62  0.78301613  0.77767613  0.79095673  0.75987757 \\
1.64  0.78332077  0.77789146  0.79119593  0.75984025 \\
1.66  0.78361665  0.77807963  0.79142455  0.75979461 \\
1.68  0.78390438  0.77824158  0.79164336  0.75974151 \\
1.70  0.78418455  0.77837664  0.79185310  0.75968174 \\
1.72  0.78445770  0.77848614  0.79205448  0.75961605 \\
1.74  0.78472435  0.77857326  0.79224815  0.75954510 \\
1.76  0.78498496  0.77864102  0.79243471  0.75946952 \\
1.78  0.78523992  0.77868829  0.79261466  0.75938978 \\
1.80  0.78548931  0.77871251  0.79278807  0.75930590 \\
1.82  0.78573299  0.77871994  0.79295474  0.75921763 \\
1.84  0.78597095  0.77872138  0.79311466  0.75912501 \\
1.86  0.78620370  0.77872165  0.79326851  0.75902877 \\
1.88  0.78643209  0.77872172  0.79341739  0.75893002 \\
1.90  0.78665697  0.77872176  0.79356241  0.75882979 \\
1.92  0.78687925  0.77872178  0.79370477  0.75872913 \\
1.94  0.78709980  0.77872181  0.79384557  0.75862900 \\
1.96  0.78731951  0.77872187  0.79398601  0.75853041 \\
1.98  0.78753934  0.77872214  0.79412734  0.75843444 \\
2.00  0.78776425  0.77878066  0.79427643  0.75835209 \\
}\datatable

\addplot[semithick, densely dashed, forget plot] {0.8};

\addplot[thick, color2, forget plot]
    table[x=t, y=Xo_nominal]{\datatable};

\addplot[thick, color1]
    table[x=t, y=Xo_true]{\datatable};
\addlegendentry{Closed-loop rollout}

\addplot[name path=lower, draw=none, forget plot]
    table[x=t, y=Xo_lb]{\datatable};

\addplot[name path=upper, draw=none, forget plot]
    table[x=t, y=Xo_ub]{\datatable};

\addplot[fill=color3, fill opacity=0.4, draw=none, legend image code/.code={
        \draw[fill=color3, fill opacity=0.4, draw=none] (0cm,-0.1cm) rectangle (0.6cm,0.1cm);
    }]
    fill between[of=lower and upper];

\addplot[name path=lower, draw=none, forget plot] {0.8};
\addplot[name path=upper, draw=none, forget plot] {0.9};

\addplot[
    draw=none,
    pattern=crosshatch,
    pattern color=black!60,
]
fill between[
    of=upper and lower,
];

\end{axis}

\end{tikzpicture} \hspace*{-5mm}
        \begin{tikzpicture}

\definecolor{color1}{RGB}{228,26,28}
\definecolor{color2}{RGB}{55,126,184}
\definecolor{color3}{RGB}{77,175,74}
\definecolor{color4}{RGB}{152,78,163}

\begin{axis}[
width=3.6cm, height=3.4cm,
xlabel={\footnotesize $t$},
ylabel={\footnotesize $\theta$},
xlabel shift=-1mm,
ylabel shift=-3.5mm,
xmin=0, xmax=2,
tick label style={
    font=\scriptsize,
    /pgf/number format/fixed,
},
every y tick scale label/.style={
    at={(axis description cs:0.03,1.05)},
    anchor=north east,
    inner sep=0,
},
legend cell align={left},
legend style={
    font=\footnotesize,
    at={(-0.25,1.25)},
    anchor=north west,
    draw=none,
    fill=none,
    row sep=-2pt,
    inner sep=0pt,
},
legend image post style={scale=0.4},
]

\pgfplotstableread[row sep=crcr]{
t     Xo_nominal  Xo_true     Xo_lb       Xo_ub      \\
0.00  0.00000000  0.00000000  0.00000000  0.00000000 \\
0.02  0.00000124 -0.00000000 -0.00000008 -0.00000404 \\
0.04 -0.00095996 -0.00204616 -0.00104278 -0.00212564 \\
0.06 -0.00218439 -0.00369711 -0.00226205 -0.00405164 \\
0.08 -0.00356942 -0.00518851 -0.00362088 -0.00584343 \\
0.10 -0.00489988 -0.00658029 -0.00493032 -0.00742178 \\
0.12 -0.00612776 -0.00791737 -0.00613777 -0.00889953 \\
0.14 -0.00677334 -0.00822013 -0.00645182 -0.00962244 \\
0.16 -0.00767113 -0.00939715 -0.00725865 -0.01017520 \\
0.18 -0.00809318 -0.00983647 -0.00764224 -0.01058579 \\
0.20 -0.00832475 -0.01003380 -0.00781143 -0.01081693 \\
0.22 -0.00842758 -0.01009139 -0.00782929 -0.01092562 \\
0.24 -0.00845479 -0.01006915 -0.00775650 -0.01096263 \\
0.26 -0.00846323 -0.01003291 -0.00765428 -0.01099250 \\
0.28 -0.00850811 -0.01003856 -0.00758803 -0.01107685 \\
0.30 -0.00848675 -0.01006729 -0.00738648 -0.01115289 \\
0.32 -0.00800507 -0.00926261 -0.00622281 -0.01081821 \\
0.34 -0.00800577 -0.00902863 -0.00585579 -0.00920767 \\
0.36 -0.00824468 -0.00971444 -0.00597921 -0.00893492 \\
0.38 -0.00798182 -0.00932230 -0.00552827 -0.00865438 \\
0.40 -0.00754473 -0.00873120 -0.00486201 -0.00823689 \\
0.42 -0.00700536 -0.00804009 -0.00404498 -0.00771724 \\
0.44 -0.00642139 -0.00730184 -0.00314116 -0.00714861 \\
0.46 -0.00585000 -0.00657427 -0.00222465 -0.00658912 \\
0.48 -0.00530278 -0.00586619 -0.00131677 -0.00604843 \\
0.50 -0.00478547 -0.00518697 -0.00042568 -0.00553794 \\
0.52 -0.00429996 -0.00453784  0.00044597 -0.00505959 \\
0.54 -0.00384700 -0.00391670  0.00129438 -0.00461064 \\
0.56 -0.00342780 -0.00332533  0.00211645 -0.00419350 \\
0.58 -0.00304301 -0.00276503  0.00290868 -0.00380958 \\
0.60 -0.00269284 -0.00223649  0.00366628 -0.00345935 \\
0.62 -0.00237704 -0.00174074  0.00438567 -0.00314369 \\
0.64 -0.00211240 -0.00128033  0.00504745 -0.00288095 \\
0.66 -0.00191388 -0.00091126  0.00563340 -0.00268679 \\
0.68 -0.00176485 -0.00061378  0.00615473 -0.00254434 \\
0.70 -0.00167481 -0.00041584  0.00664060 -0.00251985 \\
0.72 -0.00159361 -0.00022645  0.00685571 -0.00240905 \\
0.74 -0.00152783 -0.00005386  0.00704293 -0.00232108 \\
0.76 -0.00149417  0.00006726  0.00716476 -0.00227118 \\
0.78 -0.00149632  0.00012065  0.00720549 -0.00226724 \\
0.80 -0.00153041  0.00010291  0.00717033 -0.00230989 \\
0.82 -0.00158451  0.00002752  0.00708503 -0.00238952 \\
0.84 -0.00164195 -0.00007501  0.00699256 -0.00249120 \\
0.86 -0.00169068 -0.00017331  0.00692167 -0.00259458 \\
0.88 -0.00172382 -0.00023870  0.00687199 -0.00266881 \\
0.90 -0.00175317 -0.00029449  0.00681475 -0.00272344 \\
0.92 -0.00178428 -0.00035948  0.00675050 -0.00278098 \\
0.94 -0.00181719 -0.00043920  0.00668300 -0.00284617 \\
0.96 -0.00185152 -0.00053522  0.00661376 -0.00291857 \\
0.98 -0.00188600 -0.00064468  0.00654650 -0.00299651 \\
1.00 -0.00191794 -0.00075738  0.00648674 -0.00307518 \\
1.02 -0.00194814 -0.00087476  0.00642936 -0.00315198 \\
1.04 -0.00197390 -0.00098404  0.00638303 -0.00322459 \\
1.06 -0.00199587 -0.00108327  0.00634217 -0.00328891 \\
1.08 -0.00201540 -0.00117689  0.00630422 -0.00334677 \\
1.10 -0.00203320 -0.00126827  0.00626922 -0.00340008 \\
1.12 -0.00204944 -0.00135789  0.00623767 -0.00344922 \\
1.14 -0.00206244 -0.00143237  0.00621322 -0.00349174 \\
1.16 -0.00207265 -0.00148902  0.00619287 -0.00352617 \\
1.18 -0.00208062 -0.00152741  0.00617505 -0.00355314 \\
1.20 -0.00208880 -0.00156977  0.00615519 -0.00357801 \\
1.22 -0.00209652 -0.00161719  0.00613758 -0.00360243 \\
1.24 -0.00210318 -0.00166424  0.00612277 -0.00362444 \\
1.26 -0.00210844 -0.00170360  0.00611091 -0.00364272 \\
1.28 -0.00211212 -0.00172760  0.00610198 -0.00365662 \\
1.30 -0.00211416 -0.00172619  0.00609589 -0.00366601 \\
1.32 -0.00211454 -0.00169186  0.00609260 -0.00367099 \\
1.34 -0.00211340 -0.00161939  0.00609173 -0.00367196 \\
1.36 -0.00211190 -0.00152642  0.00609133 -0.00367158 \\
1.38 -0.00211011 -0.00141827  0.00609171 -0.00367052 \\
1.40 -0.00210813 -0.00129792  0.00609263 -0.00366897 \\
1.42 -0.00210603 -0.00116766  0.00609396 -0.00366713 \\
1.44 -0.00210387 -0.00102973  0.00609557 -0.00366515 \\
1.46 -0.00210172 -0.00088646  0.00609736 -0.00366317 \\
1.48 -0.00209961 -0.00073090  0.00609923 -0.00366127 \\
1.50 -0.00209758 -0.00056955  0.00610113 -0.00365951 \\
1.52 -0.00209564 -0.00040641  0.00610299 -0.00365793 \\
1.54 -0.00209382 -0.00024619  0.00610478 -0.00365655 \\
1.56 -0.00209212 -0.00009343  0.00610648 -0.00365538 \\
1.58 -0.00209055  0.00004893  0.00610806 -0.00365443 \\
1.60 -0.00208911  0.00017943  0.00610951 -0.00365368 \\
1.62 -0.00208781  0.00029738  0.00611084 -0.00365312 \\
1.64 -0.00208668  0.00040193  0.00611201 -0.00365278 \\
1.66 -0.00208570  0.00049370  0.00611303 -0.00365260 \\
1.68 -0.00208486  0.00057284  0.00611390 -0.00365256 \\
1.70 -0.00208413  0.00063791  0.00611464 -0.00365262 \\
1.72 -0.00208351  0.00068938  0.00611527 -0.00365279 \\
1.74 -0.00208298  0.00072964  0.00611578 -0.00365304 \\
1.76 -0.00208254  0.00076079  0.00611621 -0.00365336 \\
1.78 -0.00208215  0.00078121  0.00611654 -0.00365373 \\
1.80 -0.00208181  0.00078897  0.00611681 -0.00365412 \\
1.82 -0.00208154  0.00078967  0.00611700 -0.00365466 \\
1.84 -0.00208137  0.00078944  0.00611708 -0.00365532 \\
1.86 -0.00208128  0.00078936  0.00611707 -0.00365604 \\
1.88 -0.00208124  0.00078934  0.00611700 -0.00365678 \\
1.90 -0.00208121  0.00078933  0.00611692 -0.00365748 \\
1.92 -0.00208116  0.00078932  0.00611686 -0.00365809 \\
1.94 -0.00208108  0.00078931  0.00611685 -0.00365860 \\
1.96 -0.00208093  0.00078930  0.00611690 -0.00365897 \\
1.98 -0.00208070  0.00078925  0.00611707 -0.00365921 \\
2.00 -0.00208042  0.00081259  0.00611741 -0.00365929 \\
}\datatable

\addplot[semithick, densely dashed, forget plot] {0.0};

\addplot[thick, color2, forget plot]
    table[x=t, y=Xo_nominal]{\datatable};

\addplot[thick, color1, forget plot]
    table[x=t, y=Xo_true]{\datatable};

\addplot[name path=lower, draw=none, forget plot]
    table[x=t, y=Xo_lb]{\datatable};

\addplot[name path=upper, draw=none, forget plot]
    table[x=t, y=Xo_ub]{\datatable};

\addplot[fill=color3, fill opacity=0.4, draw=none, legend image code/.code={
        \draw[fill=color3, fill opacity=0.4, draw=none] (0cm,-0.1cm) rectangle (0.6cm,0.1cm);
    }]
    fill between[of=lower and upper];
\addlegendentry{Predicted tube}

\end{axis}

\end{tikzpicture}
      }
      \label{fig:iiwa-bucket-fallfree-SLS}
    }
    \caption{Rollouts of bimanual planar bucket manipulation illustrating state constraint violations under baselines and their absence in our method. Infeasible states are indicated by a cross-hatched pattern.
    \textbf{(a)} Executing the \ac{TO-CTR} nominal control sequence on the nonsmooth hybrid dynamics leads to bucket falling out of the gray table.
    \textbf{(b)} Using exact dynamics for rollouts with smoothed gradients also violates the constraint, causing the bucket to fall off the table in the same manner as shown in the keyframes in (a).
    \textbf{(c)} Executing the policy produced by our method on the nonsmooth hybrid dynamics prevents the bucket from falling off the table.}
    \label{fig:iiwa-bucket-fall}
\end{figure}

\subsection{Bimanual Planar Box Manipulation}
\begin{figure}[!t]
    \centering
    \vspace*{-2.5mm}
    \subfloat[]{%
        \adjincludegraphics[height=1.7cm, trim={0 0 {.17\width} 0}, clip]{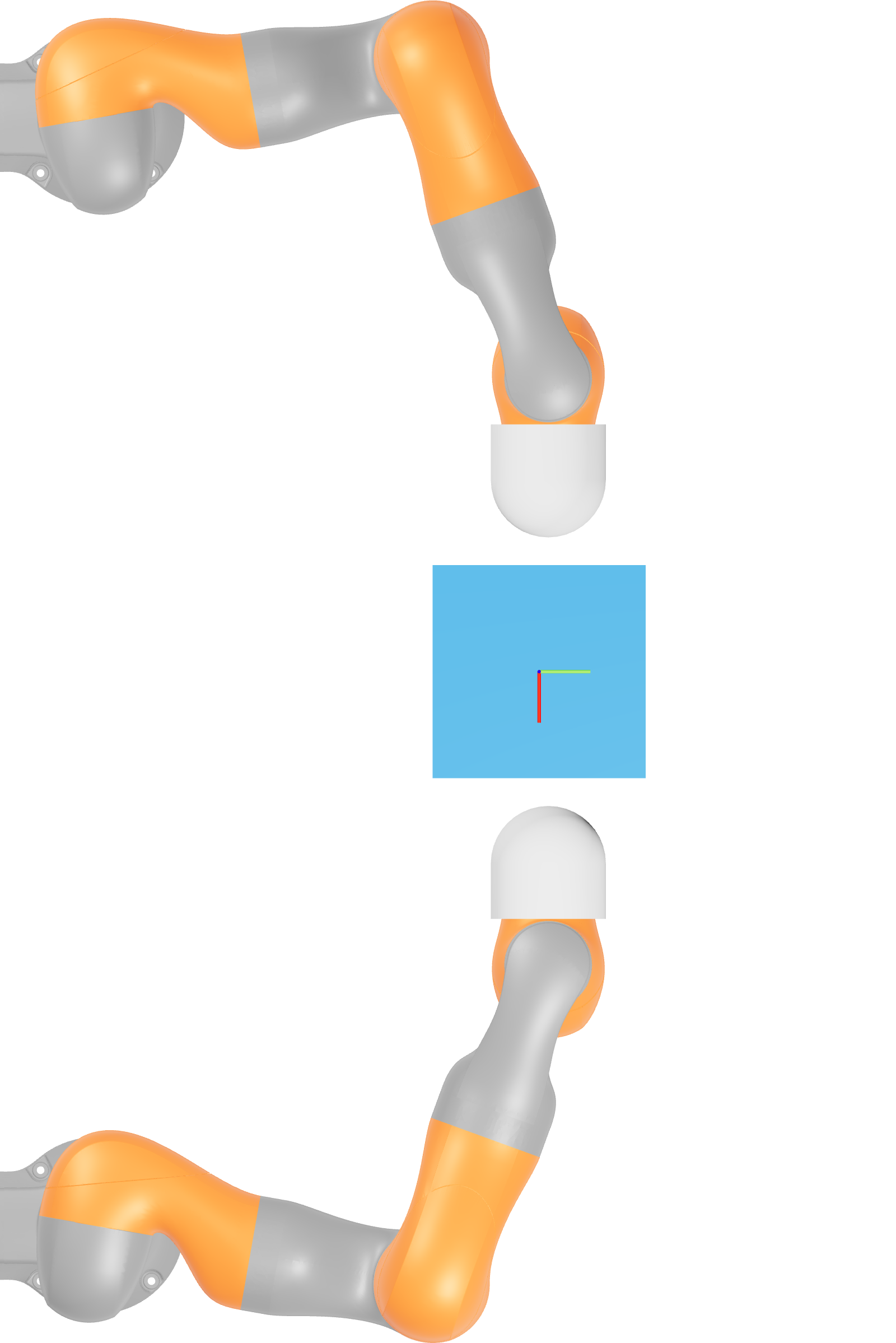}
        \adjincludegraphics[height=1.7cm, trim={0 0 {.17\width} 0}, clip]{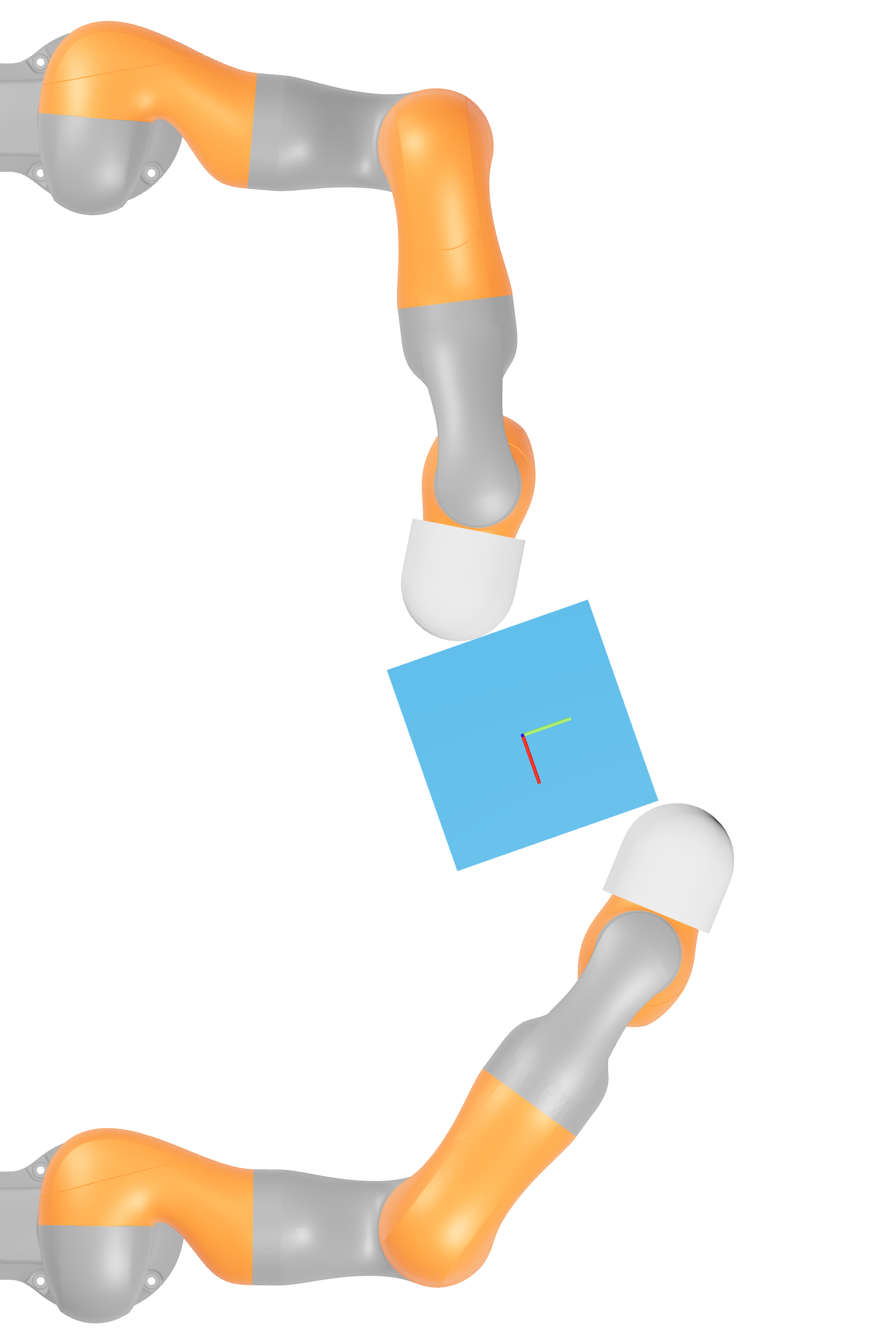}
        \adjincludegraphics[height=1.7cm, trim={0 0 {.17\width} 0}, clip]{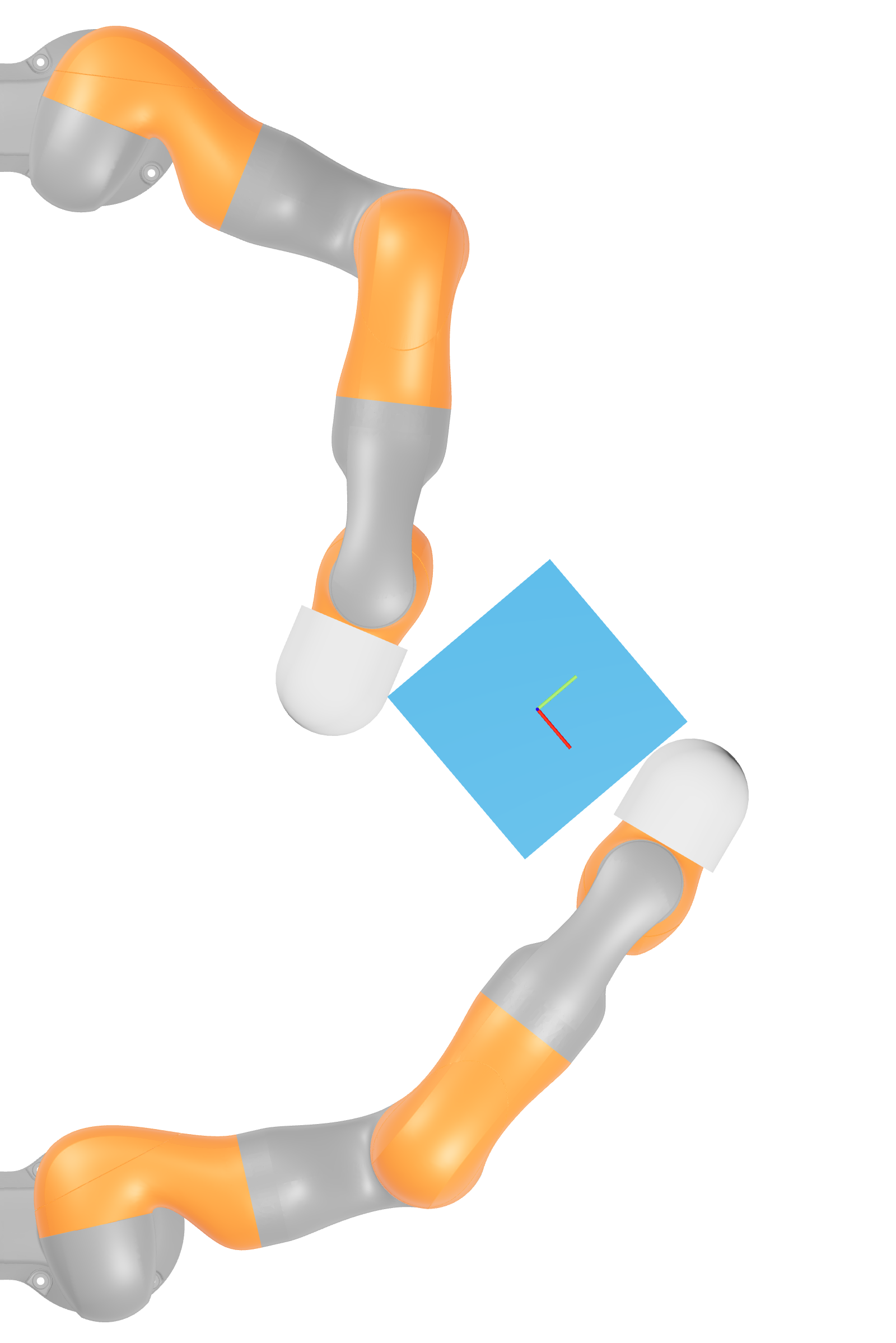}
        \adjincludegraphics[height=1.7cm, trim={0 0 {.17\width} 0}, clip]{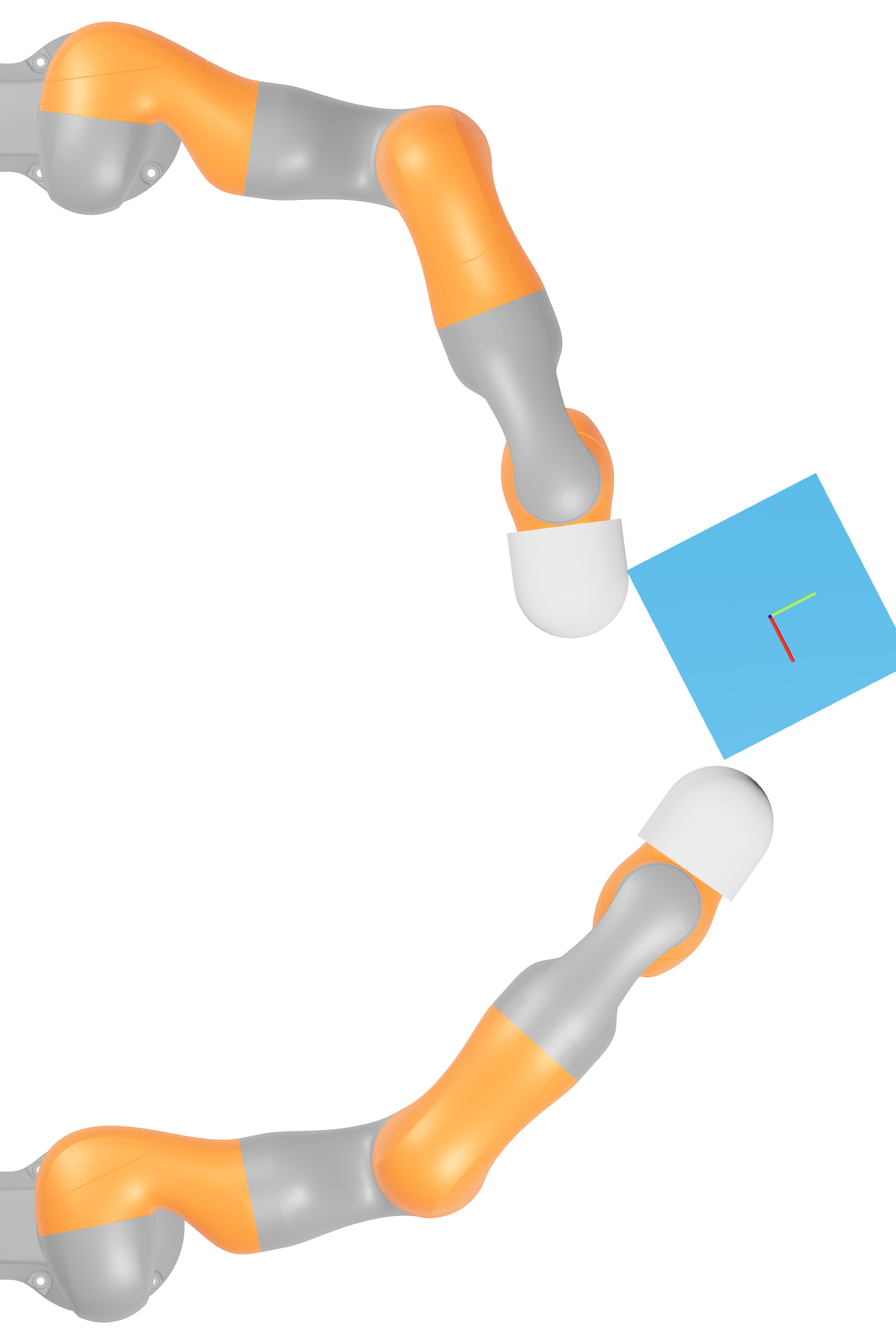}
    } \hfill
    \subfloat[]{%
        \adjincludegraphics[height=1.7cm, trim={0 0 {.17\width} 0}, clip]{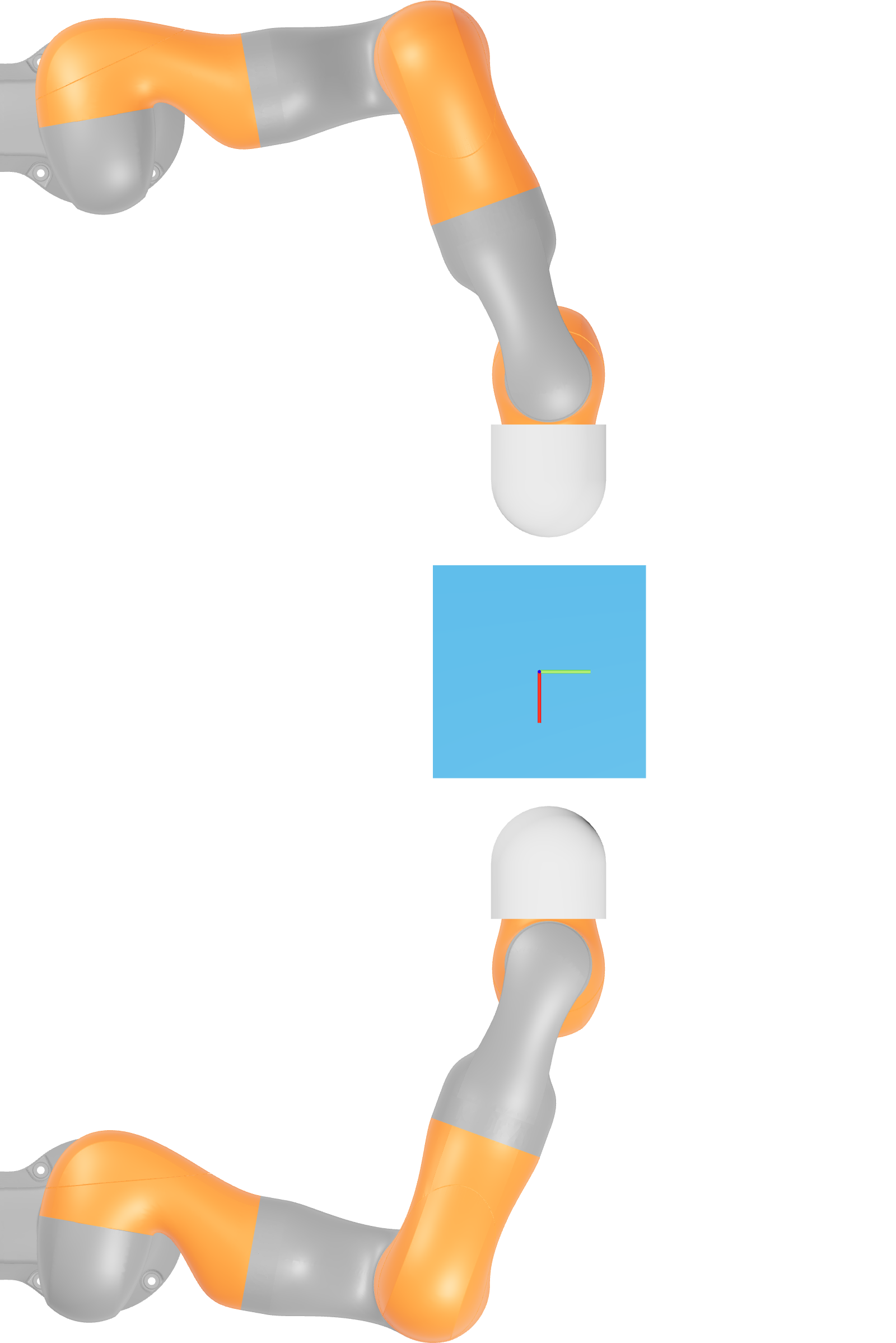}
        \adjincludegraphics[height=1.7cm, trim={0 0 {.17\width} 0}, clip]{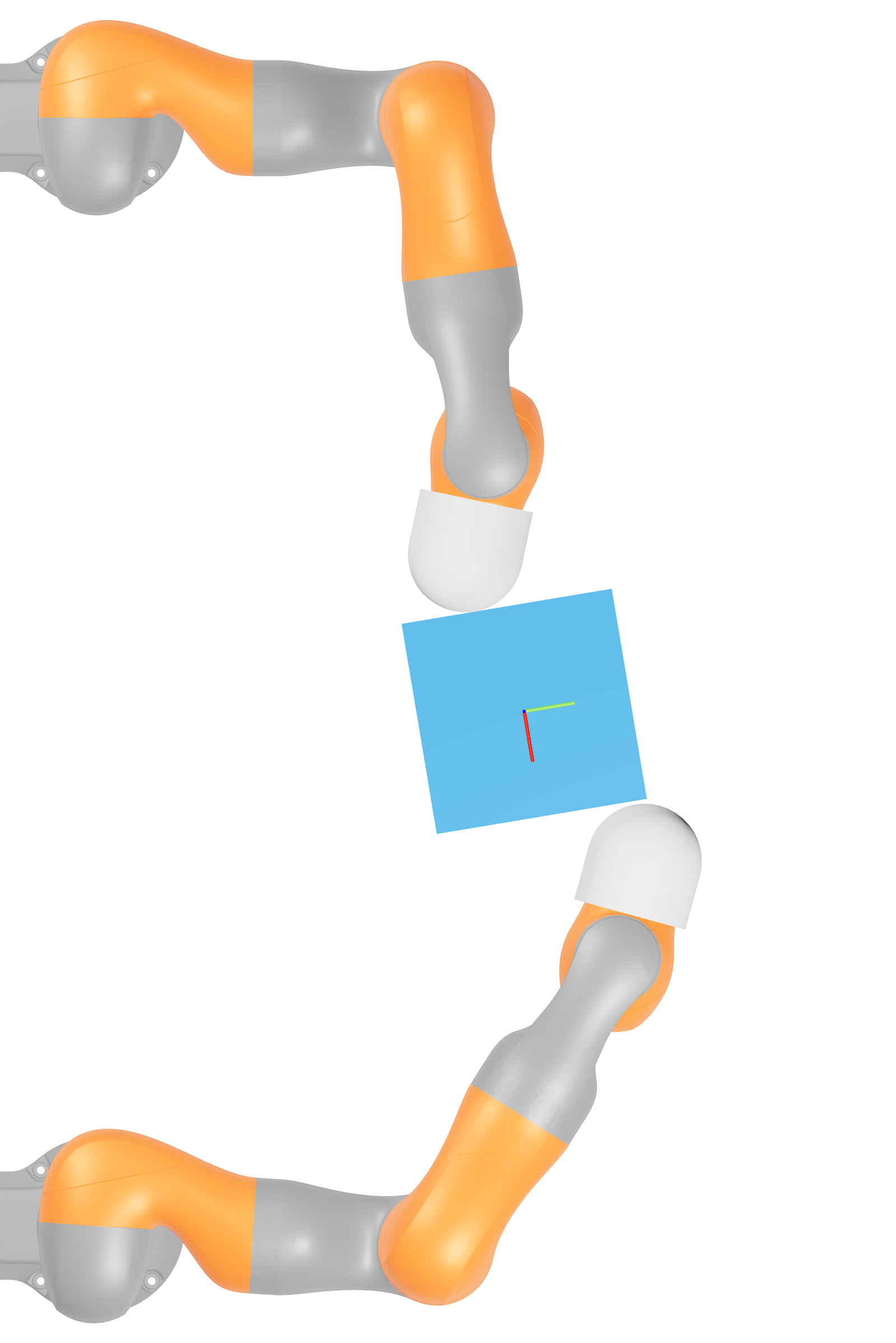}
        \adjincludegraphics[height=1.7cm, trim={0 0 {.17\width} 0}, clip]{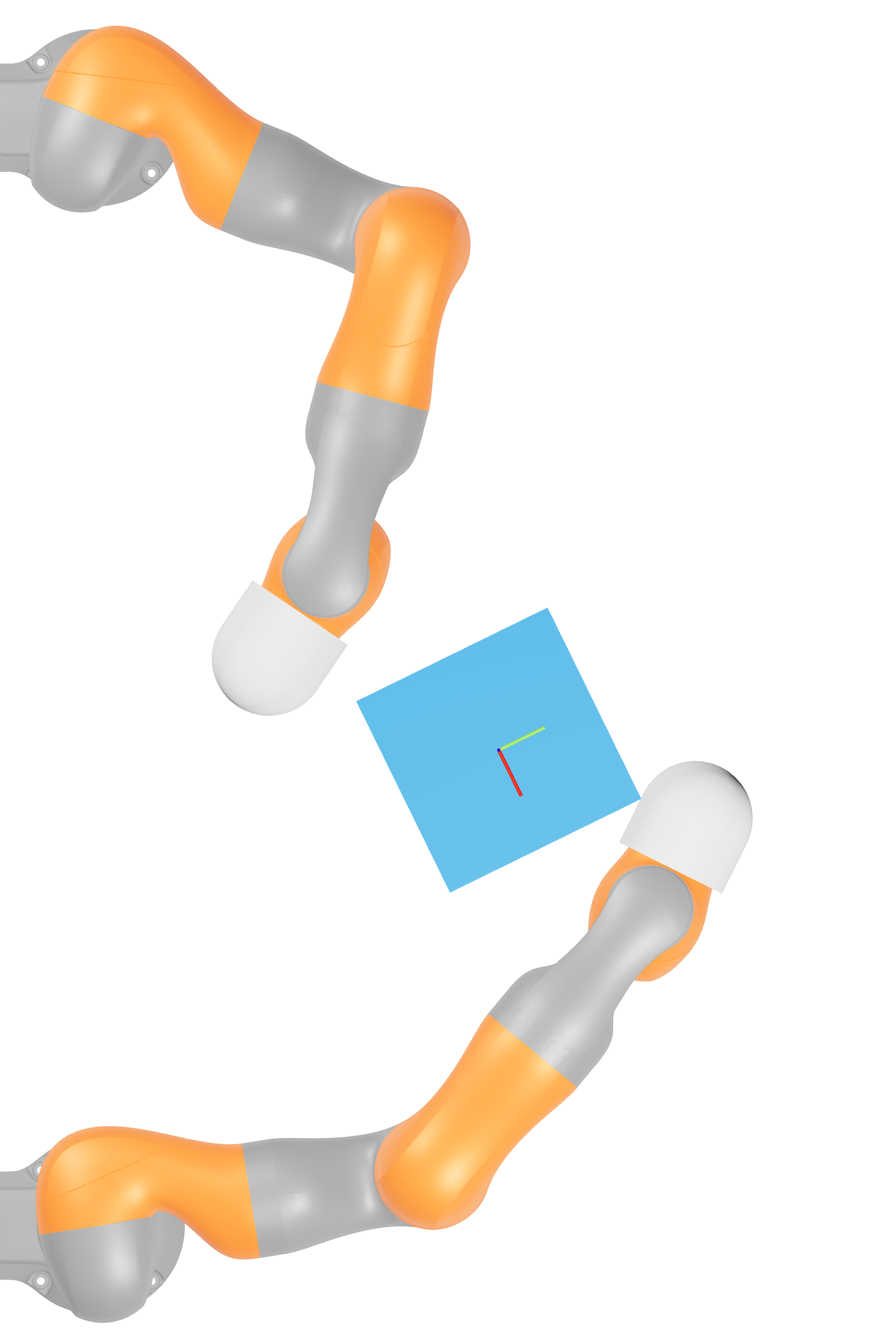}
        \adjincludegraphics[height=1.7cm, trim={0 0 {.17\width} 0}, clip]{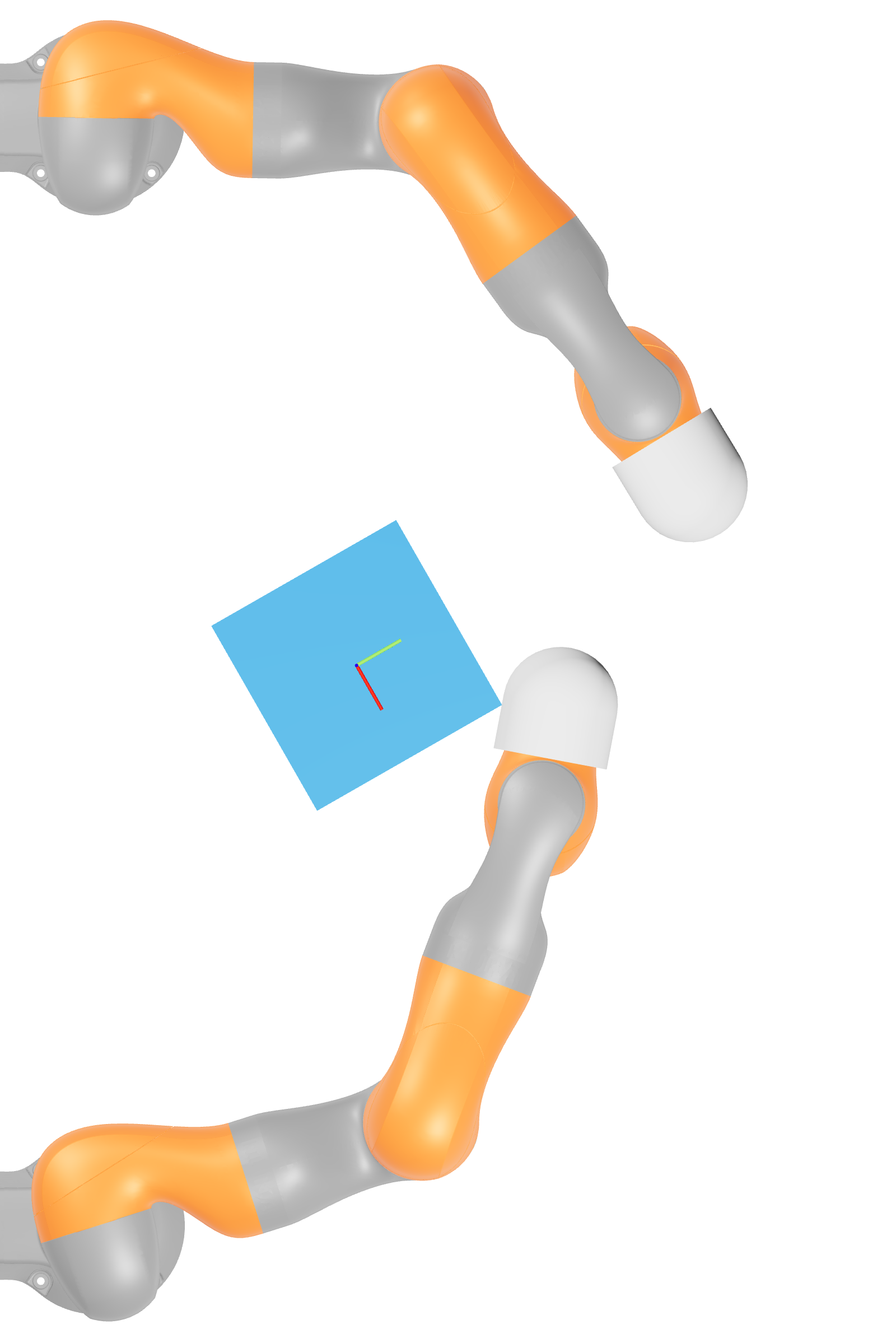}
    }\\
    \subfloat[]{%
        \hspace*{-2mm}
        \begin{tikzpicture}

\definecolor{color1}{RGB}{228,26,28}
\definecolor{color2}{RGB}{55,126,184}
\definecolor{color3}{RGB}{77,175,74}
\definecolor{color4}{RGB}{152,78,163}

\begin{axis}[
width=3.6cm, height=3.4cm,
xlabel={\footnotesize $t$},
ylabel={\footnotesize $x$},
xlabel shift=-1mm,
ylabel shift=-3mm,
xmin=0, xmax=2,
tick label style={
    font=\scriptsize,
    /pgf/number format/fixed,
},
legend cell align={left},
legend style={
    font=\footnotesize,
    at={(-0.3,1.25)},
    anchor=north west,
    draw=none,
    fill=none,
    row sep=-2pt,
    inner sep=0pt,
},
legend image post style={scale=0.4},
]

\pgfplotstableread[row sep=crcr]{
t     Xo_nominal  Xo_true     Xo_lb       Xo_ub      \\
0.00  0.00000000  0.00000000  0.00000000  0.00000000 \\
0.02 -0.00000000  0.00000000  0.00040382 -0.00040382 \\
0.04  0.00533011  0.00551678  0.00643854  0.00462195 \\
0.06  0.01451570  0.01464656  0.01711118  0.01352682 \\
0.08  0.02322427  0.02268556  0.02365861  0.02112208 \\
0.10  0.03120634  0.03012783  0.03139294  0.02843370 \\
0.12  0.03832031  0.03647165  0.03800721  0.03461825 \\
0.14  0.04534588  0.04305337  0.04473796  0.04085805 \\
0.16  0.05289983  0.05044144  0.05221147  0.04782214 \\
0.18  0.06100342  0.05854939  0.06030679  0.05558543 \\
0.20  0.06947821  0.06702880  0.06884091  0.06396366 \\
0.22  0.07813264  0.07547794  0.07773697  0.07282652 \\
0.24  0.08681609  0.08450174  0.08699912  0.08245848 \\
0.26  0.09492258  0.09572781  0.09626820  0.09246092 \\
0.28  0.10021018  0.10283261  0.10480237  0.09957849 \\
0.30  0.09788982  0.10182604  0.10769645  0.09646026 \\
0.32  0.08672365  0.08959791  0.09970821  0.08255568 \\
0.34  0.07548736  0.07680778  0.08841116  0.06789689 \\
0.36  0.06982956  0.07278286  0.08050516  0.06088094 \\
0.38  0.06463156  0.06847068  0.07318037  0.05723374 \\
0.40  0.05929778  0.06148045  0.06552543  0.05188107 \\
0.42  0.05590440  0.05789962  0.06192690  0.04837442 \\
0.44  0.05303411  0.05468492  0.05916211  0.04555025 \\
0.46  0.05187525  0.05377115  0.05822397  0.04436359 \\
0.48  0.05026720  0.05212956  0.05708840  0.04217440 \\
0.50  0.04785274  0.04924677  0.05549392  0.03842971 \\
0.52  0.04681446  0.04809559  0.05538752  0.03584135 \\
0.54  0.04589868  0.04719363  0.05515795  0.03379373 \\
0.56  0.04484459  0.04616636  0.05473262  0.03160779 \\
0.58  0.04382420  0.04521504  0.05439483  0.02938470 \\
0.60  0.04315484  0.04460707  0.05464150  0.02727304 \\
0.62  0.04258880  0.04419303  0.05505389  0.02534813 \\
0.64  0.04217356  0.04401292  0.05568734  0.02367435 \\
0.66  0.04227947  0.04436796  0.05706274  0.02242002 \\
0.68  0.04291879  0.04541398  0.05905855  0.02171983 \\
0.70  0.04414372  0.04697203  0.06194992  0.02129002 \\
0.72  0.04306896  0.04614128  0.06194804  0.01918591 \\
0.74  0.03958045  0.04247807  0.05891374  0.01541316 \\
0.76  0.03593202  0.03857649  0.05575170  0.01146257 \\
0.78  0.03276631  0.03523590  0.05327248  0.00772737 \\
0.80  0.03013442  0.03247532  0.05154573  0.00425833 \\
0.82  0.02769340  0.02996456  0.05005615  0.00092784 \\
0.84  0.02506119  0.02734194  0.04823304 -0.00244571 \\
0.86  0.02274904  0.02491423  0.04707920 -0.00580141 \\
0.88  0.01968299  0.02172212  0.04498234 -0.00965634 \\
0.90  0.01678533  0.01878577  0.04309085 -0.01333310 \\
0.92  0.01439827  0.01651611  0.04172364 -0.01646381 \\
0.94  0.01273819  0.01485390  0.04147392 -0.01910748 \\
0.96  0.01134822  0.01338305  0.04181947 -0.02163577 \\
0.98  0.00887446  0.01097113  0.04059690 -0.02477022 \\
1.00  0.00596755  0.00811946  0.03887333 -0.02817727 \\
1.02  0.00319822  0.00525650  0.03766144 -0.03160621 \\
1.04  0.00025427  0.00216827  0.03651127 -0.03524852 \\
1.06 -0.00363863 -0.00114061  0.03332819 -0.03893763 \\
1.08 -0.00644024 -0.00383501  0.03180005 -0.04170698 \\
1.10 -0.00871378 -0.00607331  0.03112115 -0.04395503 \\
1.12 -0.01108583 -0.00841732  0.03049653 -0.04624526 \\
1.14 -0.01388687 -0.01119771  0.02948526 -0.04889829 \\
1.16 -0.01710990 -0.01442782  0.02809786 -0.05197118 \\
1.18 -0.02058713 -0.01792194  0.02643870 -0.05530231 \\
1.20 -0.02414026 -0.02152120  0.02481265 -0.05881668 \\
1.22 -0.02809560 -0.02549818  0.02246925 -0.06266586 \\
1.24 -0.03230183 -0.02971505  0.01968920 -0.06669698 \\
1.26 -0.03684149 -0.03428846  0.01652491 -0.07103677 \\
1.28 -0.04149246 -0.03898244  0.01328470 -0.07548171 \\
1.30 -0.04615734 -0.04366394  0.00982512 -0.07979911 \\
1.32 -0.05064833 -0.04810914  0.00619534 -0.08372790 \\
1.34 -0.05511553 -0.05176766  0.00233826 -0.08746925 \\
1.36 -0.05930571 -0.05440707 -0.00121612 -0.09093056 \\
1.38 -0.06358032 -0.05669932 -0.00479255 -0.09453735 \\
1.40 -0.06874293 -0.05966452 -0.00970475 -0.09891971 \\
1.42 -0.07427705 -0.06305460 -0.01501360 -0.10382302 \\
1.44 -0.07962411 -0.06653814 -0.01996160 -0.10882457 \\
1.46 -0.08336271 -0.06968599 -0.02304412 -0.11256685 \\
1.48 -0.08352117 -0.07181248 -0.02255649 -0.11293726 \\
1.50 -0.08336960 -0.07259475 -0.02206387 -0.11298503 \\
1.52 -0.08337834 -0.07160947 -0.02228528 -0.11295202 \\
1.54 -0.08352048 -0.07141682 -0.02321458 -0.11271395 \\
1.56 -0.08361213 -0.07141682 -0.02443464 -0.11213953 \\
1.58 -0.08369913 -0.07141683 -0.02573265 -0.11138601 \\
1.60 -0.08385580 -0.07141685 -0.02706019 -0.11062265 \\
1.62 -0.08435237 -0.07176956 -0.02862583 -0.11020857 \\
1.64 -0.08509271 -0.07241706 -0.03013864 -0.11016292 \\
1.66 -0.08548197 -0.07270273 -0.03125088 -0.10984329 \\
1.68 -0.08565671 -0.07270275 -0.03205683 -0.10940957 \\
1.70 -0.08573907 -0.07270275 -0.03273531 -0.10898068 \\
1.72 -0.08578649 -0.07270275 -0.03333957 -0.10862101 \\
1.74 -0.08582007 -0.07270276 -0.03384737 -0.10836326 \\
1.76 -0.08584811 -0.07270276 -0.03432779 -0.10820470 \\
1.78 -0.08588344 -0.07270276 -0.03477273 -0.10816424 \\
1.80 -0.08591998 -0.07270279 -0.03517752 -0.10822914 \\
1.82 -0.08593607 -0.07270279 -0.03542948 -0.10841847 \\
1.84 -0.08594558 -0.07270279 -0.03528109 -0.10885929 \\
1.86 -0.08595749 -0.07270279 -0.03608162 -0.10920105 \\
1.88 -0.08597289 -0.07270279 -0.03570498 -0.11000468 \\
1.90 -0.08599843 -0.07270279 -0.03542962 -0.11080900 \\
1.92 -0.08602990 -0.07270279 -0.03505215 -0.11161638 \\
1.94 -0.08606267 -0.07270279 -0.03457092 -0.11241244 \\
1.96 -0.08609215 -0.07270279 -0.03401871 -0.11319354 \\
1.98 -0.08611700 -0.07270280 -0.03342444 -0.11399667 \\
2.00 -0.08613451 -0.07270280 -0.03336167 -0.11458050 \\
}\datatable

\addplot[semithick, densely dashed, forget plot] {0.0};

\addplot[thick, color2]
    table[x=t, y=Xo_nominal]{\datatable};
\addlegendentry{Nominal trajectory}

\addplot[thick, color1]
    table[x=t, y=Xo_true]{\datatable};

\addplot[name path=lower, draw=none, forget plot]
    table[x=t, y=Xo_lb]{\datatable};

\addplot[name path=upper, draw=none, forget plot]
    table[x=t, y=Xo_ub]{\datatable};

\addplot[fill=color3, fill opacity=0.4, draw=none, legend image code/.code={
        \draw[fill=color3, fill opacity=0.4, draw=none] (0cm,-0.1cm) rectangle (0.6cm,0.1cm);
    }]
    fill between[of=lower and upper];

\end{axis}

\end{tikzpicture} \hspace*{-5mm}
        \begin{tikzpicture}

\definecolor{color1}{RGB}{228,26,28}
\definecolor{color2}{RGB}{55,126,184}
\definecolor{color3}{RGB}{77,175,74}
\definecolor{color4}{RGB}{152,78,163}

\begin{axis}[
width=3.6cm, height=3.4cm,
xlabel={\footnotesize $t$},
ylabel={\footnotesize $y$},
xlabel shift=-1mm,
ylabel shift=-2mm,
xmin=0, xmax=2,
tick label style={
    font=\scriptsize,
    /pgf/number format/fixed,
},
legend cell align={left},
legend style={
    font=\footnotesize,
    at={(-0.3,1.25)},
    anchor=north west,
    draw=none,
    fill=none,
    row sep=-2pt,
    inner sep=0pt,
},
legend image post style={scale=0.4},
]

\pgfplotstableread[row sep=crcr]{
t     Xo_nominal  Xo_true     Xo_lb       Xo_ub      \\
0.00  0.52000000  0.52000000  0.52000000  0.52000000 \\
0.02  0.52001122  0.52000000  0.51999794  0.51997961 \\
0.04  0.51738797  0.51724161  0.51733603  0.51703743 \\
0.06  0.51294420  0.51269708  0.51273699  0.51179246 \\
0.08  0.50899225  0.50876527  0.50978675  0.50795031 \\
0.10  0.50579428  0.50527386  0.50704598  0.50375349 \\
0.12  0.50366556  0.50249412  0.50546491  0.50041954 \\
0.14  0.50279379  0.50159477  0.50518066  0.49908063 \\
0.16  0.50257970  0.50190696  0.50554339  0.49908995 \\
0.18  0.50219903  0.50211115  0.50566276  0.49910081 \\
0.20  0.50134995  0.50168067  0.50522894  0.49854419 \\
0.22  0.50014252  0.50045769  0.50432629  0.49743912 \\
0.24  0.49873802  0.49935700  0.50347190  0.49539955 \\
0.26  0.49762607  0.49860651  0.50375544  0.49489979 \\
0.28  0.49577873  0.49676888  0.50385859  0.49403450 \\
0.30  0.49164434  0.49539569  0.50219282  0.49070466 \\
0.32  0.48458695  0.49029683  0.49840303  0.48398321 \\
0.34  0.47893444  0.48436673  0.49582859  0.47876440 \\
0.36  0.47474619  0.48003756  0.49217710  0.47390472 \\
0.38  0.46984562  0.47517894  0.48796328  0.46783566 \\
0.40  0.46726363  0.47122607  0.48585286  0.46411453 \\
0.42  0.46745224  0.46991219  0.48642619  0.46392156 \\
0.44  0.46930943  0.47052662  0.48841236  0.46566377 \\
0.46  0.47280437  0.47299819  0.49173809  0.46901620 \\
0.48  0.47667866  0.47589836  0.49513965  0.47241498 \\
0.50  0.48148165  0.47996337  0.49925774  0.47626801 \\
0.52  0.48757878  0.48559736  0.50459195  0.48110335 \\
0.54  0.49424205  0.49177530  0.51044429  0.48646953 \\
0.56  0.50091130  0.49795768  0.51651552  0.49179155 \\
0.58  0.50736586  0.50393382  0.52257170  0.49684851 \\
0.60  0.51352120  0.50971633  0.52858273  0.50150913 \\
0.62  0.51935104  0.51516270  0.53440335  0.50584721 \\
0.64  0.52482213  0.52023152  0.54001402  0.50983151 \\
0.66  0.52994792  0.52502675  0.54547105  0.51343584 \\
0.68  0.53481467  0.52946375  0.55072127  0.51682046 \\
0.70  0.53944686  0.53379869  0.55580700  0.51992312 \\
0.72  0.54408044  0.53804576  0.56057792  0.52324496 \\
0.74  0.54905148  0.54272478  0.56561692  0.52698294 \\
0.76  0.55445276  0.54789703  0.57103319  0.53119651 \\
0.78  0.56015824  0.55343001  0.57664582  0.53578037 \\
0.80  0.56606731  0.55922911  0.58241744  0.54057487 \\
0.82  0.57206546  0.56512377  0.58823962  0.54550680 \\
0.84  0.57806001  0.57095281  0.59398839  0.55055512 \\
0.86  0.58421533  0.57711547  0.60005664  0.55560127 \\
0.88  0.59045022  0.58335645  0.60622593  0.56079248 \\
0.90  0.59663220  0.58950270  0.61236294  0.56598516 \\
0.92  0.60272824  0.59544021  0.61844119  0.57116782 \\
0.94  0.60886462  0.60159489  0.62477693  0.57626728 \\
0.96  0.61518453  0.60808534  0.63148372  0.58147429 \\
0.98  0.62132338  0.61421784  0.63793928  0.58668307 \\
1.00  0.62731231  0.62018150  0.64431550  0.59172663 \\
1.02  0.63343988  0.62645812  0.65101955  0.59676335 \\
1.04  0.63969661  0.63293665  0.65801001  0.60179373 \\
1.06  0.64563674  0.63825020  0.66451644  0.60675754 \\
1.08  0.65126632  0.64378665  0.67106344  0.61104082 \\
1.10  0.65730433  0.64991204  0.67813808  0.61569381 \\
1.12  0.66385693  0.65662754  0.68569059  0.62094615 \\
1.14  0.67086669  0.66381943  0.69358449  0.62674651 \\
1.16  0.67832686  0.67147603  0.70180256  0.63302445 \\
1.18  0.68620819  0.67953904  0.71032051  0.63970027 \\
1.20  0.69454380  0.68806995  0.71922495  0.64671151 \\
1.22  0.70298836  0.69661709  0.72810657  0.65376164 \\
1.24  0.71146188  0.70514622  0.73701033  0.66068178 \\
1.26  0.71996007  0.71369697  0.74597165  0.66748100 \\
1.28  0.72833702  0.72214332  0.75490598  0.67397124 \\
1.30  0.73664773  0.73048645  0.76385245  0.68024250 \\
1.32  0.74477367  0.73858612  0.77273397  0.68611429 \\
1.34  0.75365462  0.74665719  0.78255574  0.69242130 \\
1.36  0.76346745  0.75502378  0.79354142  0.69923998 \\
1.38  0.77367443  0.76369349  0.80501154  0.70621249 \\
1.40  0.78535977  0.77384961  0.81777448  0.71489352 \\
1.42  0.79809118  0.78507367  0.83135972  0.72491248 \\
1.44  0.81099747  0.79667466  0.84489116  0.73544607 \\
1.46  0.82275706  0.80805484  0.85707418  0.74516158 \\
1.48  0.83090591  0.81860016  0.86551609  0.75154078 \\
1.50  0.83668851  0.82662532  0.87164633  0.75558929 \\
1.52  0.84006246  0.83060484  0.87557652  0.75709056 \\
1.54  0.84141141  0.83143448  0.87771267  0.75646456 \\
1.56  0.84194487  0.83143449  0.87920265  0.75504860 \\
1.58  0.84240725  0.83143450  0.88069790  0.75372732 \\
1.60  0.84292652  0.83143454  0.88222938  0.75275595 \\
1.62  0.84397328  0.83187255  0.88422094  0.75260033 \\
1.64  0.84552829  0.83264223  0.88669159  0.75294760 \\
1.66  0.84673979  0.83298050  0.88869244  0.75322835 \\
1.68  0.84753595  0.83298053  0.89016039  0.75329859 \\
1.70  0.84801249  0.83298054  0.89118457  0.75327739 \\
1.72  0.84832302  0.83298054  0.89193652  0.75324877 \\
1.74  0.84856768  0.83298055  0.89254391  0.75320071 \\
1.76  0.84878487  0.83298055  0.89304702  0.75321136 \\
1.78  0.84901338  0.83298057  0.89350795  0.75327074 \\
1.80  0.84926612  0.83298062  0.89395618  0.75337112 \\
1.82  0.84948021  0.83298062  0.89435978  0.75334175 \\
1.84  0.84965790  0.83298063  0.89477309  0.75298857 \\
1.86  0.84982127  0.83298063  0.89489228  0.75423646 \\
1.88  0.84998039  0.83298063  0.89559941  0.75382065 \\
1.90  0.85013447  0.83298063  0.89635439  0.75363808 \\
1.92  0.85028359  0.83298063  0.89710728  0.75361015 \\
1.94  0.85042779  0.83298063  0.89780179  0.75377382 \\
1.96  0.85056915  0.83298065  0.89843954  0.75417709 \\
1.98  0.85071142  0.83298066  0.89908216  0.75484504 \\
2.00  0.85087432  0.83298066  0.89976205  0.75623887 \\
}\datatable

\addplot[semithick, densely dashed, forget plot] {0.3};

\addplot[thick, color2, forget plot]
    table[x=t, y=Xo_nominal]{\datatable};

\addplot[thick, color1]
    table[x=t, y=Xo_true]{\datatable};
\addlegendentry{Closed-loop rollout}

\addplot[name path=lower, draw=none, forget plot]
    table[x=t, y=Xo_lb]{\datatable};

\addplot[name path=upper, draw=none, forget plot]
    table[x=t, y=Xo_ub]{\datatable};

\addplot[fill=color3, fill opacity=0.4, draw=none, legend image code/.code={
        \draw[fill=color3, fill opacity=0.4, draw=none] (0cm,-0.1cm) rectangle (0.6cm,0.1cm);
    }]
    fill between[of=lower and upper];

\end{axis}

\end{tikzpicture} \hspace*{-5mm}
        \begin{tikzpicture}

\definecolor{color1}{RGB}{228,26,28}
\definecolor{color2}{RGB}{55,126,184}
\definecolor{color3}{RGB}{77,175,74}
\definecolor{color4}{RGB}{152,78,163}

\begin{axis}[
width=3.6cm, height=3.4cm,
xlabel={\footnotesize $t$},
ylabel={\footnotesize $\theta$},
xlabel shift=-1mm,
ylabel shift=-2mm,
xmin=0, xmax=2,
tick label style={
    font=\scriptsize,
    /pgf/number format/fixed,
},
legend cell align={left},
legend style={
    font=\footnotesize,
    at={(-0.25,1.25)},
    anchor=north west,
    draw=none,
    fill=none,
    row sep=-2pt,
    inner sep=0pt,
},
legend image post style={scale=0.4},
]

\pgfplotstableread[row sep=crcr]{
t     Xo_nominal  Xo_true     Xo_lb       Xo_ub      \\
0.00  0.00000000  0.00000000  0.00000000  0.00000000 \\
0.02  0.00000000 -0.00000000  0.00003778 -0.00003778 \\
0.04  0.00182013  0.00178499  0.00198977  0.00157951 \\
0.06  0.00873183  0.00876164  0.00899229  0.00752718 \\
0.08  0.02458415  0.02499633  0.02624501  0.02480801 \\
0.10  0.05016630  0.05062811  0.05231987  0.05014911 \\
0.12  0.08541699  0.08568239  0.08803841  0.08504696 \\
0.14  0.12912471  0.12929160  0.13214119  0.12903218 \\
0.16  0.17875304  0.17890865  0.18215799  0.17930371 \\
0.18  0.23196623  0.23200810  0.23597143  0.23274408 \\
0.20  0.28710973  0.28690913  0.29175857  0.28769311 \\
0.22  0.34289730  0.34227624  0.34786737  0.34314729 \\
0.24  0.39808528  0.39694537  0.40257251  0.39794216 \\
0.26  0.45263484  0.44988032  0.45979235  0.45205426 \\
0.28  0.50470359  0.50171037  0.51453584  0.50388968 \\
0.30  0.54957272  0.55232675  0.56134545  0.54901675 \\
0.32  0.57980470  0.58641868  0.59418650  0.57987565 \\
0.34  0.60202198  0.60865855  0.61912285  0.60307257 \\
0.36  0.62291163  0.63311619  0.63859052  0.62397314 \\
0.38  0.63943170  0.65258612  0.65699796  0.64043008 \\
0.40  0.64744719  0.65746480  0.66435577  0.64776896 \\
0.42  0.65415205  0.66184529  0.67113783  0.65442405 \\
0.44  0.66008015  0.66588493  0.67727236  0.66051989 \\
0.46  0.66710598  0.67153304  0.68498380  0.66750814 \\
0.48  0.67316605  0.67614882  0.69197108  0.67286180 \\
0.50  0.67737825  0.67878396  0.69677893  0.67552401 \\
0.52  0.68216163  0.68242907  0.70146067  0.67821960 \\
0.54  0.68696268  0.68597483  0.70523525  0.68091460 \\
0.56  0.69157164  0.68929753  0.70860368  0.68341453 \\
0.58  0.69624013  0.69268917  0.71198978  0.68593747 \\
0.60  0.70106518  0.69640150  0.71569072  0.68850341 \\
0.62  0.70607446  0.70027650  0.71953709  0.69130635 \\
0.64  0.71133359  0.70438961  0.72362506  0.69443999 \\
0.66  0.71699495  0.70902543  0.72823419  0.69796415 \\
0.68  0.72347183  0.71435498  0.73364647  0.70230965 \\
0.70  0.73014397  0.72011855  0.73938339  0.70666180 \\
0.72  0.73409224  0.72307060  0.74225481  0.70862715 \\
0.74  0.73455748  0.72264185  0.74172895  0.70776619 \\
0.76  0.73433507  0.72159398  0.74063211  0.70650261 \\
0.78  0.73425911  0.72079090  0.73974040  0.70540019 \\
0.80  0.73435240  0.72024504  0.73907827  0.70434114 \\
0.82  0.73430719  0.71959589  0.73828779  0.70315628 \\
0.84  0.73384373  0.71848895  0.73703241  0.70171867 \\
0.86  0.73305905  0.71722312  0.73562406  0.69960904 \\
0.88  0.73100538  0.71465100  0.73302371  0.69631285 \\
0.90  0.72883710  0.71194304  0.73024730  0.69302517 \\
0.92  0.72715656  0.70960730  0.72783389  0.69043559 \\
0.94  0.72561305  0.70760907  0.72566493  0.68783953 \\
0.96  0.72367962  0.70536863  0.72328464  0.68474686 \\
0.98  0.72037302  0.70158220  0.71950490  0.68066589 \\
1.00  0.71618249  0.69692977  0.71488911  0.67589075 \\
1.02  0.71137161  0.69185221  0.70985558  0.67039485 \\
1.04  0.70569171  0.68598875  0.70416236  0.66395209 \\
1.06  0.69996057  0.67917164  0.69807839  0.65838719 \\
1.08  0.69393111  0.67282767  0.69192695  0.65246995 \\
1.10  0.68770309  0.66652925  0.68578177  0.64640776 \\
1.12  0.68074399  0.65960497  0.67906693  0.63968475 \\
1.14  0.67271800  0.65164674  0.67139261  0.63194902 \\
1.16  0.66359744  0.64262104  0.66270693  0.62310171 \\
1.18  0.65355212  0.63267772  0.65313059  0.61331511 \\
1.20  0.64274664  0.62201586  0.64286140  0.60263279 \\
1.22  0.63084979  0.61021243  0.63140597  0.59095513 \\
1.24  0.61801939  0.59748936  0.61900413  0.57836037 \\
1.26  0.60418591  0.58380403  0.60567908  0.56463461 \\
1.28  0.58955445  0.56939042  0.59168954  0.54991829 \\
1.30  0.57431351  0.55438150  0.57713715  0.53457057 \\
1.32  0.55869915  0.53901279  0.56217506  0.51898281 \\
1.34  0.54452817  0.52393800  0.54856077  0.50505067 \\
1.36  0.53249261  0.51011481  0.53695222  0.49345759 \\
1.38  0.52278160  0.49861767  0.52741045  0.48458038 \\
1.40  0.51442164  0.48846445  0.51867625  0.47807139 \\
1.42  0.50762388  0.47974312  0.51123448  0.47365015 \\
1.44  0.50269670  0.47272848  0.50557725  0.47126236 \\
1.46  0.49976497  0.46778280  0.50197537  0.47067828 \\
1.48  0.49741850  0.46537931  0.49903039  0.47035524 \\
1.50  0.49600629  0.46476930  0.49707878  0.47049202 \\
1.52  0.49515646  0.46441670  0.49569668  0.47070975 \\
1.54  0.49478422  0.46432949  0.49478084  0.47100811 \\
1.56  0.49470227  0.46432949  0.49419930  0.47124717 \\
1.58  0.49462738  0.46432948  0.49370029  0.47122267 \\
1.60  0.49444682  0.46432945  0.49318574  0.47097244 \\
1.62  0.49372860  0.46380694  0.49223697  0.47014779 \\
1.64  0.49258607  0.46285944  0.49101354  0.46869338 \\
1.66  0.49191974  0.46244163  0.49033515  0.46790890 \\
1.68  0.49160671  0.46244160  0.49007960  0.46768340 \\
1.70  0.49148882  0.46244159  0.49008285  0.46793006 \\
1.72  0.49145629  0.46244159  0.49025247  0.46848945 \\
1.74  0.49145510  0.46244159  0.49056077  0.46918752 \\
1.76  0.49146474  0.46244159  0.49098353  0.46004955 \\
1.78  0.49145576  0.46244158  0.49150481  0.46098834 \\
1.80  0.49142558  0.46244155  0.49211468  0.46197447 \\
1.82  0.49142444  0.46244154  0.49291958  0.46288754 \\
1.84  0.49144292  0.46244154  0.49400818  0.46345749 \\
1.86  0.49147068  0.46244154  0.49460322  0.45624532 \\
1.88  0.49149988  0.46244154  0.49568692  0.45752258 \\
1.90  0.49152816  0.46244154  0.49661129  0.45921323 \\
1.92  0.49155506  0.46244154  0.49736499  0.46104606 \\
1.94  0.49157750  0.46244154  0.49788226  0.46287921 \\
1.96  0.49159272  0.46244154  0.49812698  0.46449694 \\
1.98  0.49160110  0.46244154  0.49820906  0.46566313 \\
2.00  0.49161066  0.46244154  0.49827152  0.46654691 \\
}\datatable

\addplot[semithick, densely dashed, forget plot] {0.785398};

\addplot[thick, color2, forget plot]
    table[x=t, y=Xo_nominal]{\datatable};

\addplot[thick, color1, forget plot]
    table[x=t, y=Xo_true]{\datatable};

\addplot[name path=lower, draw=none, forget plot]
    table[x=t, y=Xo_lb]{\datatable};

\addplot[name path=upper, draw=none, forget plot]
    table[x=t, y=Xo_ub]{\datatable};

\addplot[fill=color3, fill opacity=0.4, draw=none, legend image code/.code={
        \draw[fill=color3, fill opacity=0.4, draw=none] (0cm,-0.1cm) rectangle (0.6cm,0.1cm);
    }]
    fill between[of=lower and upper];
\addlegendentry{Predicted tube}

\end{axis}

\end{tikzpicture}
    }\\
    \subfloat[]{%
        \hspace*{-2mm}
        \begin{tikzpicture}

\definecolor{color1}{RGB}{228,26,28}
\definecolor{color2}{RGB}{55,126,184}
\definecolor{color3}{RGB}{77,175,74}
\definecolor{color4}{RGB}{152,78,163}

\begin{axis}[
width=3.6cm, height=3.4cm,
xlabel={\footnotesize $t$},
ylabel={\footnotesize $x$},
xlabel shift=-1mm,
ylabel shift=-2mm,
xmin=0, xmax=2,
tick label style={
    font=\scriptsize,
    /pgf/number format/fixed,
},
legend cell align={left},
legend style={
    font=\footnotesize,
    at={(-0.3,1.25)},
    anchor=north west,
    draw=none,
    fill=none,
    row sep=-2pt,
    inner sep=0pt,
},
legend image post style={scale=0.4},
]

\pgfplotstableread[row sep=crcr]{
t     Xo_nominal  Xo_true     Xo_lb       Xo_ub      \\
0.00  0.00000000  0.00000000  0.00000000  0.00000000 \\
0.02 -0.00000000  0.00000000  0.00057288 -0.00057288 \\
0.04  0.00529351  0.00523987  0.00714232  0.00442112 \\
0.06  0.01310890  0.01548168  0.01449412  0.01195550 \\
0.08  0.02092521  0.02361766  0.02125832  0.01914837 \\
0.10  0.02759970  0.03068750  0.02760298  0.02465755 \\
0.12  0.03356907  0.03758509  0.03326880  0.02940992 \\
0.14  0.04022178  0.04292175  0.03978346  0.03536139 \\
0.16  0.04836359  0.04766869  0.04814160  0.04309092 \\
0.18  0.05829217  0.05296095  0.05855781  0.05285443 \\
0.20  0.07005274  0.06007208  0.07081409  0.06436944 \\
0.22  0.08334972  0.07105850  0.08442502  0.07723757 \\
0.24  0.09665290  0.08841810  0.09779463  0.09019190 \\
0.26  0.10646129  0.10520542  0.10738578  0.10047099 \\
0.28  0.10800551  0.10738371  0.10886368  0.10381620 \\
0.30  0.10050317  0.09727021  0.10248070  0.09874637 \\
0.32  0.09323226  0.08706046  0.09933080  0.09188895 \\
0.34  0.08934599  0.07634305  0.09893419  0.08815125 \\
0.36  0.08653823  0.06740672  0.09931761  0.08503332 \\
0.38  0.08428583  0.06525533  0.10027149  0.08182928 \\
0.40  0.08805834  0.06660661  0.10535416  0.08475358 \\
0.42  0.09440438  0.06779780  0.11281519  0.08892927 \\
0.44  0.10119357  0.07127452  0.11294570  0.09624560 \\
0.46  0.10888983  0.08017719  0.11217743  0.10262883 \\
0.48  0.11264633  0.09869717  0.11823646  0.10706293 \\
0.50  0.11454767  0.11561059  0.12308954  0.10928095 \\
0.52  0.11660847  0.11860984  0.12542635  0.11208946 \\
0.54  0.11782907  0.10424501  0.12767546  0.11268959 \\
0.56  0.11991756  0.09001368  0.12753387  0.11441806 \\
0.58  0.10636058  0.07535230  0.11305371  0.10155189 \\
0.60  0.09213332  0.06460341  0.09856425  0.08771210 \\
0.62  0.07982150  0.05391182  0.08677381  0.07523011 \\
0.64  0.06867038  0.04387420  0.07635806  0.06368823 \\
0.66  0.05873722  0.03476999  0.06715074  0.05331233 \\
0.68  0.04997147  0.02665598  0.05905759  0.04409907 \\
0.70  0.04232722  0.01951181  0.05202372  0.03602039 \\
0.72  0.03575844  0.01330257  0.04600328  0.02903988 \\
0.74  0.03019844  0.00800691  0.04095843  0.02308668 \\
0.76  0.02558989  0.00354989  0.03681882  0.01811554 \\
0.78  0.02185611 -0.00014140  0.03351268  0.01405436 \\
0.80  0.01890485 -0.00315092  0.03095826  0.01081012 \\
0.82  0.01662879 -0.00557854  0.02906019  0.00826534 \\
0.84  0.01489566 -0.00753058  0.02771279  0.00624936 \\
0.86  0.01355263 -0.00917362  0.02676742  0.00454681 \\
0.88  0.01258157 -0.01048591  0.02620705  0.00318437 \\
0.90  0.01195655 -0.01153049  0.02599374  0.00221023 \\
0.92  0.01157880 -0.01237248  0.02606337  0.00156943 \\
0.94  0.01132649 -0.01295382  0.02639163  0.00109665 \\
0.96  0.01111401 -0.01302457  0.02690656  0.00064739 \\
0.98  0.01091680 -0.01302457  0.02734795  0.00029244 \\
1.00  0.01075674 -0.01302457  0.02761435  0.00007013 \\
1.02  0.01063764 -0.01302457  0.02779479 -0.00006908 \\
1.04  0.01054173 -0.01302457  0.02794119 -0.00017304 \\
1.06  0.01045977 -0.01302457  0.02807021 -0.00026000 \\
1.08  0.01038723 -0.01302457  0.02818830 -0.00033742 \\
1.10  0.01032159 -0.01302457  0.02829870 -0.00040893 \\
1.12  0.01026117 -0.01302457  0.02840332 -0.00047651 \\
1.14  0.01020484 -0.01302457  0.02850346 -0.00054132 \\
1.16  0.01015176 -0.01302457  0.02860001 -0.00060409 \\
1.18  0.01010125 -0.01302457  0.02869368 -0.00066529 \\
1.20  0.01005282 -0.01302457  0.02878499 -0.00072526 \\
1.22  0.01000607 -0.01302457  0.02887436 -0.00078426 \\
1.24  0.00996068 -0.01302458  0.02896214 -0.00084248 \\
1.26  0.00991640 -0.01302458  0.02904861 -0.00090010 \\
1.28  0.00987300 -0.01302458  0.02913401 -0.00095726 \\
1.30  0.00983028 -0.01302458  0.02921855 -0.00101407 \\
1.32  0.00978805 -0.01302458  0.02930244 -0.00107066 \\
1.34  0.00974614 -0.01302458  0.02938584 -0.00112711 \\
1.36  0.00970439 -0.01302458  0.02946891 -0.00118352 \\
1.38  0.00966264 -0.01302458  0.02955181 -0.00123998 \\
1.40  0.00962074 -0.01302458  0.02963469 -0.00129659 \\
1.42  0.00957851 -0.01302458  0.02971770 -0.00135343 \\
1.44  0.00953578 -0.01302458  0.02980098 -0.00141063 \\
1.46  0.00949234 -0.01302458  0.02988471 -0.00146827 \\
1.48  0.00944797 -0.01302458  0.02996903 -0.00152648 \\
1.50  0.00940242 -0.01302458  0.03005413 -0.00158540 \\
1.52  0.00935541 -0.01302458  0.03014020 -0.00164525 \\
1.54  0.00930662 -0.01302458  0.03022741 -0.00170628 \\
1.56  0.00925567 -0.01302458  0.03031598 -0.00176882 \\
1.58  0.00920207 -0.01302459  0.03040612 -0.00183335 \\
1.60  0.00914519 -0.01302459  0.03049800 -0.00190055 \\
1.62  0.00908416 -0.01302459  0.03059171 -0.00197146 \\
1.64  0.00901779 -0.01302459  0.03068729 -0.00204765 \\
1.66  0.00894426 -0.01302459  0.03078450 -0.00213180 \\
1.68  0.00886067 -0.01302459  0.03088241 -0.00222883 \\
1.70  0.00876171 -0.01302459  0.03097895 -0.00234879 \\
1.72  0.00863571 -0.01302459  0.03107168 -0.00251484 \\
1.74  0.00845004 -0.01302459  0.03115473 -0.00278667 \\
1.76  0.00811690 -0.01302459  0.03111697 -0.00328672 \\
1.78  0.00758254 -0.01318182  0.03084948 -0.00403532 \\
1.80  0.00694502 -0.01288275  0.03058416 -0.00486514 \\
1.82  0.00633173 -0.01240516  0.03051851 -0.00562518 \\
1.84  0.00580935 -0.01177312  0.03073597 -0.00625992 \\
1.86  0.00532568 -0.01104769  0.03117475 -0.00684279 \\
1.88  0.00488684 -0.01029388  0.03177925 -0.00737732 \\
1.90  0.00453718 -0.00954622  0.03241702 -0.00781943 \\
1.92  0.00425759 -0.00887435  0.03311695 -0.00820526 \\
1.94  0.00400884 -0.00833106  0.03377535 -0.00857516 \\
1.96  0.00374289 -0.00789690  0.03418384 -0.00898540 \\
1.98  0.00344934 -0.00751927  0.03437253 -0.00944076 \\
2.00  0.00304233 -0.00737920  0.03428684 -0.01006170 \\
}\datatable

\addplot[semithick, densely dashed, forget plot] {0.0};

\addplot[thick, color2]
    table[x=t, y=Xo_nominal]{\datatable};
\addlegendentry{Nominal trajectory}

\addplot[thick, color1]
    table[x=t, y=Xo_true]{\datatable};

\addplot[name path=lower, draw=none, forget plot]
    table[x=t, y=Xo_lb]{\datatable};

\addplot[name path=upper, draw=none, forget plot]
    table[x=t, y=Xo_ub]{\datatable};

\addplot[fill=color3, fill opacity=0.4, draw=none, legend image code/.code={
        \draw[fill=color3, fill opacity=0.4, draw=none] (0cm,-0.1cm) rectangle (0.6cm,0.1cm);
    }]
    fill between[of=lower and upper];

\end{axis}

\end{tikzpicture} \hspace*{-5mm}
        \begin{tikzpicture}

\definecolor{color1}{RGB}{228,26,28}
\definecolor{color2}{RGB}{55,126,184}
\definecolor{color3}{RGB}{77,175,74}
\definecolor{color4}{RGB}{152,78,163}

\begin{axis}[
width=3.6cm, height=3.4cm,
xlabel={\footnotesize $t$},
ylabel={\footnotesize $y$},
xlabel shift=-1mm,
ylabel shift=-2mm,
xmin=0, xmax=2,
tick label style={
    font=\scriptsize,
    /pgf/number format/fixed,
},
legend cell align={left},
legend style={
    font=\footnotesize,
    at={(-0.3,1.25)},
    anchor=north west,
    draw=none,
    fill=none,
    row sep=-2pt,
    inner sep=0pt,
},
legend image post style={scale=0.4},
]

\pgfplotstableread[row sep=crcr]{
t     Xo_nominal  Xo_true     Xo_lb       Xo_ub      \\
0.00  0.52000000  0.52000000  0.52000000  0.52000000 \\
0.02  0.52002589  0.52000000  0.51999884  0.51994658 \\
0.04  0.51756680  0.51738007  0.51753159  0.51671642 \\
0.06  0.51398580  0.51228085  0.51398199  0.51301487 \\
0.08  0.51082263  0.50828494  0.51161328  0.50941595 \\
0.10  0.50891916  0.50494590  0.51059871  0.50668733 \\
0.12  0.50869891  0.50262693  0.51142898  0.50590677 \\
0.14  0.50955796  0.50165210  0.51327865  0.50710853 \\
0.16  0.50967367  0.50241350  0.51388436  0.50798638 \\
0.18  0.50778497  0.50336850  0.51190338  0.50687144 \\
0.20  0.50370513  0.50284786  0.50771377  0.50320992 \\
0.22  0.49748610  0.50087204  0.50184084  0.49681013 \\
0.24  0.48923496  0.49681746  0.49421829  0.48802809 \\
0.26  0.47958564  0.49036801  0.48516951  0.47771484 \\
0.28  0.46973386  0.48331416  0.47559133  0.46712636 \\
0.30  0.46039669  0.47634359  0.46614672  0.45722849 \\
0.32  0.45328827  0.46863193  0.45885334  0.45020448 \\
0.34  0.44934061  0.46045575  0.45491573  0.44649076 \\
0.36  0.44867246  0.45715182  0.45478113  0.44496651 \\
0.38  0.45030614  0.45459803  0.45757553  0.44706337 \\
0.40  0.45271738  0.45447413  0.46002702  0.45130847 \\
0.42  0.45732635  0.45960397  0.46392062  0.45609060 \\
0.44  0.46448031  0.46840821  0.47096511  0.45891149 \\
0.46  0.47127511  0.47878176  0.47887401  0.46260829 \\
0.48  0.47831002  0.49318278  0.48502406  0.46949971 \\
0.50  0.48474666  0.50370377  0.49101261  0.47536334 \\
0.52  0.49032965  0.50549561  0.49612005  0.48039042 \\
0.54  0.49761214  0.49536327  0.50371305  0.48442681 \\
0.56  0.49859322  0.48597304  0.50532301  0.48091238 \\
0.58  0.48384110  0.47395349  0.49275836  0.46139345 \\
0.60  0.47070202  0.46027634  0.47903960  0.44964593 \\
0.62  0.45697973  0.44718304  0.46511652  0.43677867 \\
0.64  0.44357723  0.43454148  0.45177945  0.42391165 \\
0.66  0.43057816  0.42227142  0.43892815  0.41136255 \\
0.68  0.41813647  0.41047824  0.42665973  0.39934824 \\
0.70  0.40638222  0.39929531  0.41508360  0.38801793 \\
0.72  0.39542083  0.38882624  0.40429846  0.37747726 \\
0.74  0.38534923  0.37912780  0.39441228  0.36779023 \\
0.76  0.37621708  0.37028978  0.38545808  0.35903858 \\
0.78  0.36806551  0.36236887  0.37746879  0.35127661 \\
0.80  0.36092586  0.35539787  0.37047605  0.34453681 \\
0.82  0.35482260  0.34938508  0.36452266  0.33882546 \\
0.84  0.34978373  0.34429988  0.35969759  0.33410634 \\
0.86  0.34581051  0.34011867  0.35607689  0.33034055 \\
0.88  0.34286231  0.33677479  0.35362579  0.32747611 \\
0.90  0.34083155  0.33429538  0.35215572  0.32551818 \\
0.92  0.33953204  0.33270233  0.35138292  0.32441980 \\
0.94  0.33867199  0.33194273  0.35094958  0.32401178 \\
0.96  0.33799588  0.33185093  0.35065387  0.32401450 \\
0.98  0.33739771  0.33185089  0.35050757  0.32408217 \\
1.00  0.33690085  0.33185089  0.35052432  0.32402840 \\
1.02  0.33652257  0.33185089  0.35065681  0.32392366 \\
1.04  0.33622630  0.33185089  0.35083714  0.32382209 \\
1.06  0.33598146  0.33185089  0.35102836  0.32373401 \\
1.08  0.33577021  0.33185089  0.35121644  0.32365849 \\
1.10  0.33558216  0.33185089  0.35139706  0.32359299 \\
1.12  0.33541093  0.33185089  0.35156942  0.32353528 \\
1.14  0.33525240  0.33185089  0.35173399  0.32348362 \\
1.16  0.33510371  0.33185088  0.35189162  0.32343671 \\
1.18  0.33496284  0.33185088  0.35204319  0.32339355 \\
1.20  0.33482829  0.33185087  0.35218955  0.32335340 \\
1.22  0.33469886  0.33185087  0.35233146  0.32331565 \\
1.24  0.33457356  0.33185087  0.35246963  0.32327986 \\
1.26  0.33445160  0.33185087  0.35260466  0.32324563 \\
1.28  0.33433231  0.33185087  0.35273709  0.32321263 \\
1.30  0.33421511  0.33185087  0.35286740  0.32318060 \\
1.32  0.33409948  0.33185086  0.35299605  0.32314930 \\
1.34  0.33398495  0.33185086  0.35312343  0.32311851 \\
1.36  0.33387108  0.33185086  0.35324993  0.32308802 \\
1.38  0.33375747  0.33185086  0.35337591  0.32305765 \\
1.40  0.33364369  0.33185086  0.35350172  0.32302720 \\
1.42  0.33352932  0.33185086  0.35362769  0.32299648 \\
1.44  0.33341395  0.33185086  0.35375417  0.32296529 \\
1.46  0.33329714  0.33185086  0.35388146  0.32293339 \\
1.48  0.33317841  0.33185086  0.35400990  0.32290052 \\
1.50  0.33305721  0.33185085  0.35413979  0.32286640 \\
1.52  0.33293291  0.33185085  0.35427145  0.32283066 \\
1.54  0.33280471  0.33185085  0.35440517  0.32279288 \\
1.56  0.33267167  0.33185085  0.35454115  0.32275251 \\
1.58  0.33253260  0.33185085  0.35467946  0.32270884 \\
1.60  0.33238594  0.33185084  0.35481982  0.32266091 \\
1.62  0.33222960  0.33185084  0.35496111  0.32260739 \\
1.64  0.33206061  0.33185084  0.35510109  0.32254637 \\
1.66  0.33187449  0.33185084  0.35523474  0.32247497 \\
1.68  0.33166392  0.33185084  0.35535078  0.32238854 \\
1.70  0.33141555  0.33185084  0.35542330  0.32227902 \\
1.72  0.33110052  0.33185084  0.35539231  0.32213011 \\
1.74  0.33063993  0.33185082  0.35510885  0.32189293 \\
1.76  0.32980484  0.33185082  0.35411619  0.32136560 \\
1.78  0.32824737  0.33168384  0.35170302  0.32018490 \\
1.80  0.32586138  0.33017780  0.34785149  0.31823436 \\
1.82  0.32277355  0.32781943  0.34300906  0.31558687 \\
1.84  0.31916545  0.32483449  0.33757198  0.31237936 \\
1.86  0.31525560  0.32156803  0.33188192  0.30885727 \\
1.88  0.31136119  0.31832660  0.32636487  0.30532547 \\
1.90  0.30781513  0.31524667  0.32148575  0.30202741 \\
1.92  0.30469330  0.31257673  0.31729333  0.29910291 \\
1.94  0.30207675  0.31048052  0.31393136  0.29664420 \\
1.96  0.30010710  0.30885120  0.31167622  0.29468421 \\
1.98  0.29860265  0.30746721  0.31024278  0.29305644 \\
2.00  0.29695562  0.30555364  0.30851428  0.29123279 \\
}\datatable

\addplot[semithick, densely dashed, forget plot] {0.3};

\addplot[thick, color2, forget plot]
    table[x=t, y=Xo_nominal]{\datatable};

\addplot[thick, color1]
    table[x=t, y=Xo_true]{\datatable};
\addlegendentry{Closed-loop rollout}

\addplot[name path=lower, draw=none, forget plot]
    table[x=t, y=Xo_lb]{\datatable};

\addplot[name path=upper, draw=none, forget plot]
    table[x=t, y=Xo_ub]{\datatable};

\addplot[fill=color3, fill opacity=0.4, draw=none, legend image code/.code={
        \draw[fill=color3, fill opacity=0.4, draw=none] (0cm,-0.1cm) rectangle (0.6cm,0.1cm);
    }]
    fill between[of=lower and upper];

\end{axis}

\end{tikzpicture} \hspace*{-5mm}
        \begin{tikzpicture}

\definecolor{color1}{RGB}{228,26,28}
\definecolor{color2}{RGB}{55,126,184}
\definecolor{color3}{RGB}{77,175,74}
\definecolor{color4}{RGB}{152,78,163}

\begin{axis}[
width=3.6cm, height=3.4cm,
xlabel={\footnotesize $t$},
ylabel={\footnotesize $\theta$},
xlabel shift=-1mm,
ylabel shift=-2mm,
xmin=0, xmax=2,
tick label style={
    font=\scriptsize,
    /pgf/number format/fixed,
},
legend cell align={left},
legend style={
    font=\footnotesize,
    at={(-0.25,1.25)},
    anchor=north west,
    draw=none,
    fill=none,
    row sep=-2pt,
    inner sep=0pt,
},
legend image post style={scale=0.4},
]

\pgfplotstableread[row sep=crcr]{
t     Xo_nominal  Xo_true     Xo_lb       Xo_ub      \\
0.00  0.00000000  0.00000000  0.00000000  0.00000000 \\
0.02  0.00000000 -0.00000000  0.00006611 -0.00006611 \\
0.04  0.00243145  0.00169539  0.00268345  0.00182236 \\
0.06  0.01133366  0.00711220  0.01185481  0.01109952 \\
0.08  0.02984722  0.02240715  0.03144495  0.03002091 \\
0.10  0.05920492  0.04822564  0.06186285  0.05930352 \\
0.12  0.09872216  0.08305655  0.10228696  0.09911478 \\
0.14  0.14557649  0.12663272  0.14999637  0.14646919 \\
0.16  0.19590923  0.17752333  0.20101287  0.19682974 \\
0.18  0.24663203  0.23174741  0.25205518  0.24717683 \\
0.20  0.29591770  0.28572299  0.30159096  0.29606189 \\
0.22  0.34257463  0.33726427  0.34866533  0.34248942 \\
0.24  0.38611617  0.38260190  0.39277704  0.38598920 \\
0.26  0.42590410  0.41996428  0.43313456  0.42579453 \\
0.28  0.45894297  0.44367530  0.46659617  0.45860579 \\
0.30  0.48142483  0.45162231  0.48950523  0.48007517 \\
0.32  0.49980481  0.45948872  0.50863543  0.49615854 \\
0.34  0.51895507  0.46781432  0.52927249  0.51148450 \\
0.36  0.53856700  0.47836485  0.55197075  0.52583739 \\
0.38  0.55180695  0.47314489  0.56896920  0.54159191 \\
0.40  0.55881376  0.46287751  0.57757288  0.55635530 \\
0.42  0.56616903  0.45849145  0.58863015  0.56796480 \\
0.44  0.57253621  0.44921321  0.59522240  0.57363175 \\
0.46  0.57065364  0.44359970  0.57413826  0.55991414 \\
0.48  0.56452457  0.45272841  0.56674646  0.55242552 \\
0.50  0.55289759  0.45253458  0.55472478  0.53929404 \\
0.52  0.53551739  0.44790088  0.53666711  0.52268539 \\
0.54  0.51777879  0.45916971  0.51939127  0.50579864 \\
0.56  0.50956534  0.47073910  0.51013531  0.50655911 \\
0.58  0.52073182  0.48203245  0.52158282  0.51600685 \\
0.60  0.53336577  0.48885187  0.53397971  0.52898071 \\
0.62  0.54388613  0.49592161  0.54463852  0.53911787 \\
0.64  0.55333964  0.50256339  0.55447424  0.54790585 \\
0.66  0.56165091  0.50843919  0.56332812  0.55544737 \\
0.68  0.56887774  0.51349961  0.57116249  0.56190102 \\
0.70  0.57506956  0.51778469  0.57796291  0.56737533 \\
0.72  0.58027283  0.52134218  0.58375467  0.57192888 \\
0.74  0.58455953  0.52419363  0.58861103  0.57560122 \\
0.76  0.58798471  0.52642927  0.59256676  0.57847248 \\
0.78  0.59062826  0.52813051  0.59569152  0.58062289 \\
0.80  0.59259039  0.52938648  0.59808307  0.58214360 \\
0.82  0.59399121  0.53029900  0.59987373  0.58313951 \\
0.84  0.59498364  0.53096817  0.60127084  0.58372093 \\
0.86  0.59574141  0.53156231  0.60253874  0.58404340 \\
0.88  0.59627458  0.53203798  0.60364224  0.58410609 \\
0.90  0.59658984  0.53246859  0.60447825  0.58395540 \\
0.92  0.59676412  0.53292883  0.60503825  0.58366513 \\
0.94  0.59688104  0.53334112  0.60538254  0.58328367 \\
0.96  0.59698912  0.53339153  0.60574060  0.58273241 \\
0.98  0.59709526  0.53339151  0.60601741  0.58226623 \\
1.00  0.59717885  0.53339151  0.60616792  0.58198769 \\
1.02  0.59724003  0.53339151  0.60625126  0.58180390 \\
1.04  0.59729296  0.53339151  0.60631058  0.58166189 \\
1.06  0.59734238  0.53339151  0.60635889  0.58154212 \\
1.08  0.59738982  0.53339151  0.60640080  0.58143601 \\
1.10  0.59743599  0.53339151  0.60643854  0.58133912 \\
1.12  0.59748127  0.53339151  0.60647338  0.58124889 \\
1.14  0.59752586  0.53339151  0.60650603  0.58116330 \\
1.16  0.59756991  0.53339151  0.60653696  0.58108101 \\
1.18  0.59761347  0.53339151  0.60656652  0.58100109 \\
1.20  0.59765660  0.53339151  0.60659498  0.58092285 \\
1.22  0.59769936  0.53339151  0.60662258  0.58084590 \\
1.24  0.59774184  0.53339151  0.60664954  0.58077010 \\
1.26  0.59778412  0.53339151  0.60667605  0.58069529 \\
1.28  0.59782628  0.53339151  0.60670221  0.58062130 \\
1.30  0.59786839  0.53339151  0.60672816  0.58054797 \\
1.32  0.59791051  0.53339151  0.60675398  0.58047514 \\
1.34  0.59795270  0.53339151  0.60677976  0.58040264 \\
1.36  0.59799504  0.53339151  0.60680559  0.58033033 \\
1.38  0.59803759  0.53339151  0.60683153  0.58025805 \\
1.40  0.59808041  0.53339151  0.60685768  0.58018566 \\
1.42  0.59812357  0.53339151  0.60688412  0.58011301 \\
1.44  0.59816716  0.53339151  0.60691094  0.58003992 \\
1.46  0.59821122  0.53339151  0.60693824  0.57996626 \\
1.48  0.59825582  0.53339151  0.60696614  0.57989181 \\
1.50  0.59830103  0.53339151  0.60699478  0.57981636 \\
1.52  0.59834700  0.53339151  0.60702433  0.57973965 \\
1.54  0.59839385  0.53339151  0.60705504  0.57966139 \\
1.56  0.59844179  0.53339151  0.60708720  0.57958125 \\
1.58  0.59849106  0.53339151  0.60712125  0.57949883 \\
1.60  0.59854196  0.53339151  0.60715777  0.57941366 \\
1.62  0.59859494  0.53339151  0.60719769  0.57932520 \\
1.64  0.59865062  0.53339151  0.60724235  0.57923267 \\
1.66  0.59870998  0.53339151  0.60729404  0.57913513 \\
1.68  0.59877473  0.53339151  0.60735691  0.57903135 \\
1.70  0.59884812  0.53339151  0.60743939  0.57891934 \\
1.72  0.59893766  0.53339151  0.60756071  0.57879342 \\
1.74  0.59906531  0.53339151  0.60776759  0.57864419 \\
1.76  0.59928824  0.53339151  0.60813728  0.57855029 \\
1.78  0.59960149  0.53351373  0.60860967  0.57853838 \\
1.80  0.59987201  0.53277902  0.60898943  0.57831418 \\
1.82  0.59997829  0.53161684  0.60915608  0.57768305 \\
1.84  0.59987885  0.53011291  0.60908736  0.57660076 \\
1.86  0.59967672  0.52842604  0.60890801  0.57520616 \\
1.88  0.59942596  0.52671034  0.60868375  0.57363723 \\
1.90  0.59913931  0.52504127  0.60843683  0.57210655 \\
1.92  0.59885312  0.52356459  0.60821783  0.57060217 \\
1.94  0.59862830  0.52238529  0.60808739  0.56927160 \\
1.96  0.59854837  0.52145408  0.60814796  0.56837974 \\
1.98  0.59859016  0.52065237  0.60836786  0.56784516 \\
2.00  0.59873484  0.51994980  0.60873778  0.56758442 \\
}\datatable

\addplot[semithick, densely dashed, forget plot] {0.785398};

\addplot[thick, color2, forget plot]
    table[x=t, y=Xo_nominal]{\datatable};

\addplot[thick, color1, forget plot]
    table[x=t, y=Xo_true]{\datatable};

\addplot[name path=lower, draw=none, forget plot]
    table[x=t, y=Xo_lb]{\datatable};

\addplot[name path=upper, draw=none, forget plot]
    table[x=t, y=Xo_ub]{\datatable};

\addplot[fill=color3, fill opacity=0.4, draw=none, legend image code/.code={
        \draw[fill=color3, fill opacity=0.4, draw=none] (0cm,-0.1cm) rectangle (0.6cm,0.1cm);
    }]
    fill between[of=lower and upper];
\addlegendentry{Predicted tube}

\end{axis}

\end{tikzpicture}
    }
    \caption{Rollout of bimanual planar box manipulation. Keyframes under \textbf{(a)} no geometry smoothing, and \textbf{(b)} has geometry smoothing. Trajectory and predicted tube per object DoF under \textbf{(c)} no geometry smoothing, and \textbf{(d)} has geometry smoothing.}
    \label{fig:iiwa-box-rollout}
\end{figure}

\looseness-1Next, we evaluate our method on a planar box manipulation problem to motivate the need for geometry smoothing (RQ4). In this setting, the discontinuous gradients that arise from the discontinuous contact normals between the actuated arms and the box destabilize the method without geometry smoothing (Fig. \ref{fig:iiwa-box-rollout}a,c). While the closed-loop rollout remains within the reachable tube, once again validating \textbf{RQ1}, the box fails to approach the goal, leading to a large final position error, which is in fact further from the goal than the initial state, due to the unstable gradients. In contrast, with geometry smoothing, we find a solution that gets much closer to the target, leading to a goal error of 9~mm. In this case, the closed-loop rollouts do not stay within the reachable tubes, as the geometry smoothing-error is not accounted for in the bound $E_\kappa(x,u)$. However, as noted earlier, the resulting error only leads to a small constant offset from the tubes and does not grow over time. Overall, this result demonstrates the need for contact geometry smoothing for manipulating objects with discontinuous surface normals, validating \textbf{RQ4}, though this is at the cost of tube validity (\textbf{RQ1}).

\subsection{In-hand Cube Reorientation}

Finally, we evaluate our method on a high-dimensional dexterous in-hand manipulation task using an Allegro hand manipulating a cube. The system consists of 16 actuated degrees of freedom and 6 unactuated object degrees of freedom. The objective is to rotate the cube by $90^\circ$ (i.e., achieve a roll angle of $\pi/2$) while maintaining its center of mass position and the remaining Euler angles. This yields a 22-dimensional state and 16-dimensional control input system, providing a challenging test of scalability (\textbf{RQ5}). Despite the high dimensionality and contact-rich dynamics, Alg.~\ref{alg:scp} synthesizes a policy in 610 seconds with rollout result shown in Fig.~\ref{fig:inhand-cube-rollout}. We again validate \textbf{RQ1}: the true closed-loop trajectory remains within the predicted reachable tubes. The synthesized policy is able to stabilize the nonsmooth hybrid system about the nominal trajectory, and the resulting goal error is $9$ mm and $2.2^\circ$. Overall, this experiment demonstrates that Alg.~\ref{alg:scp} scales to high-dimensional, complex, contact-rich manipulation tasks while retaining a reachability-based certificate that guarantees robust constraint satisfaction.

\subsection{Ablation on Smoothing Parameter}
We perform an ablation study on $\kappa$ and plot the goal error and tube size as functions of $\kappa$ (\cref{fig:kappa-ablation}). As expected, the tube size scales proportionally with $\kappa$.
For small $\kappa$, insufficient smoothing prevents convergence to a good mode sequence, resulting in larger goal error. For sufficiently large $\kappa$, the solver converges to a favorable contact mode sequence, yielding a smaller goal error that varies monotonically with $\kappa$. 
Therefore, the smoothing parameter must exceed a task-dependent threshold to ensure reliable convergence, while remaining sufficiently small to avoid unnecessarily large tube sizes.
\cref{tab:task-parameters} reports the nominal smoothing parameter $\kappa$ used for each task. 

\begin{figure}[!t]
    \centering
    \vspace*{-2.5mm}
    \subfloat[]{
        \begin{tikzpicture}

\definecolor{color1}{RGB}{231,41,138}
\definecolor{color2}{RGB}{117,112,179}

\pgfplotstableread{
kappa       goal_error tube_size
0.01        0.275484   0.56782113
0.00630957  0.224854   0.29030676
0.00398107  0.178721   0.15111045
0.00251189  0.143012   0.08501434 
0.00158489  0.118597   0.04026456
0.001       0.101928   0.02481901
0.00063096  0.091447   0.01996683
0.00039811  0.704834   0.01002528
0.00025119  0.699487   0.00486300
0.00015849  0.697862   0.00234068
0.0001      0.694946   0.00145865
}\datatable

\begin{axis}[
width=4cm, height=3.2cm,
xmode=log,
xmin=1e-4, xmax=1e-2,
xlabel={\small $\kappa$},
xlabel shift=-1mm,
ylabel={\footnotesize Goal error},
ylabel style={color=color1},
ylabel shift=-1mm,
axis y line*=left,
tick label style={font=\scriptsize},
every y tick scale label/.style={
    at={(axis description cs:0.15,1.1)},
    anchor=north east,
    inner sep=0,
},
]
\addplot[color1] table[x=kappa, y=goal_error] {\datatable};
\end{axis}

\begin{axis}[
width=4cm, height=3.2cm,
xmode=log,
xmin=1e-4, xmax=1e-2,
ylabel={\footnotesize Tube size},
ylabel style={color=color2},
ylabel shift=-1mm,
axis y line*=right,
axis x line=none,
tick label style={font=\scriptsize},
]
\addplot[color2] table[x=kappa, y=tube_size] {\datatable};
\end{axis}

\end{tikzpicture}
    }
    \subfloat[]{
        \begin{tikzpicture}

\definecolor{color1}{RGB}{231,41,138}
\definecolor{color2}{RGB}{117,112,179}

\pgfplotstableread{
kappa        goal_error tube_size
0.001        0.199318   3.0759348
0.000630957  0.159829   1.9843291
0.000398107  0.138845   1.0090343
0.000251189  0.123510   0.5221930 
0.000158489  0.103138   0.3003248
0.0001       0.098661   0.1723498
0.000063096  0.123427   0.1252384
0.000039811  0.348508   0.0911908
0.000025119  1.254760   0.0213093
0.000015849  2.813783   0.0129022
0.00001      2.282993   0.0060983
}\datatable

\begin{axis}[
width=4cm, height=3.2cm,
xmode=log,
xmin=1e-5, xmax=1e-3,
xlabel={\small $\kappa$},
xlabel shift=-1mm,
ylabel={\footnotesize Goal error},
ylabel style={color=color1},
ylabel shift=-1mm,
axis y line*=left,
tick label style={font=\scriptsize},
every y tick scale label/.style={
    at={(axis description cs:0.15,1.1)},
    anchor=north east,
    inner sep=0,
},
]
\addplot[color1] table[x=kappa, y=goal_error] {\datatable};
\end{axis}

\begin{axis}[
width=4cm, height=3.2cm,
xmode=log,
xmin=1e-5, xmax=1e-3,
ylabel={\footnotesize Tube size},
ylabel style={color=color2},
ylabel shift=-1mm,
axis y line*=right,
axis x line=none,
tick label style={font=\scriptsize},
]
\addplot[color2] table[x=kappa, y=tube_size] {\datatable};
\end{axis}

\end{tikzpicture}
    }
    \caption{Goal error and tube size as a function of smoothing parameter. \textbf{(a)} Bimanual planar bucket manipulation task. \textbf{(b)} In-hand cube reorientation task.}
    \label{fig:kappa-ablation}
\end{figure}

\section{Limitations}
\label{sec:limitations}

Despite enabling gradient-based policy synthesis for contact-rich manipulation with formal guarantees on the nonsmooth hybrid dynamics, our method has several limitations.
First, our approach is susceptible to poor local minima, which we believe is partly due to the use of shooting method for optimization. In contrast, direct transcription often appears more robust to initialization and allows dynamically infeasible warm starts. Combining a more robust optimization method with stochastic exploration strategies similar to those used in \ac{RL} may improve motion planning performance.
Second, we account only for errors induced by contact dynamics smoothing and do not model uncertainty in physical parameters such as mass, friction, or geometry. Quantifying how parametric uncertainty propagates through convex programs \cite{gould2026over} and incorporating such uncertainty into our framework is an important direction for future work.
Third, our formulation relies on a quasistatic model and assumes contact only between actuated and unactuated bodies, limiting its applicability to dynamic manipulation and tool-use scenarios.
Finally, our algorithm can benefit from GPU acceleration to reduce solve time \cite{fang2026safe}.

\section{Conclusion}
We presented a principled method for handling hybrid contact dynamics that overcomes the nonsmoothness that has traditionally hindered gradient-based trajectory optimization, while retaining guarantees of closed-loop constraint satisfaction. 
The method builds on characterizing the smoothing error, propagating it to compute reachable tubes, and synthesizing affine feedback policies. 
Because the feedback gains are constructed from a one-sided deviation bound, they appropriately capture the unilateral nature of contact, while ensuring that the system remains within a region where those gains are valid by construction. 
The resulting policy enables significantly cheaper online execution than \ac{MPC}, making formal feedback synthesis practical for high-dimensional contact-rich systems.

\clearpage
\bibliographystyle{plainnat}
\bibliography{contents/references}

@article{schulman2017proximal,
  title={\href{https://doi.org/10.48550/arXiv.1707.06347}{Proximal policy optimization algorithms}},
  author={Schulman, John and Wolski, Filip and Dhariwal, Prafulla and Radford, Alec and Klimov, Oleg},
  journal={arXiv preprint arXiv:1707.06347},
  year={2017}
}

@inproceedings{freeman2021brax,
  title={\href{https://openreview.net/forum?id=VdvDlnnjzIN}{Brax--A differentiable physics engine for large scale rigid body simulation}},
  author={C. Daniel Freeman and Erik Frey and Anton Raichuk and Sertan Girgin and Igor Mordatch and Olivier Bachem},
  booktitle={35th Conference on Neural Information Processing Systems},
  year={2021}
}

@inproceedings{suh2022do,
  title = {\href{https://proceedings.mlr.press/v162/suh22b.html}{Do differentiable simulators give better policy gradients?}},
  author = {Suh, H.J. Terry and Simchowitz, Max and Zhang, Kaiqing and Tedrake, Russ},
  booktitle = {Proceedings of the 39th International Conference on Machine Learning},
  pages = {20668-20696},
  year = {2022}
}

@article{goulart2024clarabel,
   author = {Goulart, Paul J. and Chen, Yuwen},
   title = {\href{https://doi.org/10.48550/arXiv.2405.12762}{Clarabel: An interior-point solver for conic programs with quadratic objectives}},
   journal = {arXiv preprint arXiv.2405.12762},
   year = {2024}
}

@book{boyd2004convex,
   author = {Boyd, Stephen and Vandenberghe, Lieven},
   title = {\href{https://stanford.edu/~boyd/cvxbook/}{Convex Optimization}},
   publisher = {Cambridge University Press},
   year = {2004}
}

@manual{vandenberghe2010cvxopt,
   author = {Vandenberghe, Lieven},
   title = {\href{http://www.seas.ucla.edu/~vandenbe/publications/coneprog.pdf}{The CVXOPT linear and quadratic cone program solvers}},
   year = {2010}
}

@article{agrawal2019differentiable,
   author = {Agrawal, Akshay and Amos, Brandon and Barratt, Shane and Boyd, Stephen and Diamond, Steven and Kolter, J Zico},
   title = {\href{https://proceedings.neurips.cc/paper_files/paper/2019/file/9ce3c52fc54362e22053399d3181c638-Paper.pdf}{Differentiable convex optimization layers}},
   journal = {Advances in Neural Information Processing Systems},
   volume = {32},
   year = {2019}
}

@article{hansen2016cma,
  title={\href{https://doi.org/10.48550/arXiv.1604.00772}{The CMA evolution strategy: A tutorial}},
  author={Hansen, Nikolaus},
  journal={arXiv preprint arXiv:1604.00772},
  year={2016}
}

@article{pezzato2025sampling,
  author={Pezzato, Corrado and Salmi, Chadi and Trevisan, Elia and Spahn, Max and Alonso-Mora, Javier and Hernández Corbato, Carlos},
  title={\href{https://doi.org/10.1109/LRA.2025.3535185}{Sampling-based model predictive control leveraging parallelizable physics simulations}}, 
  journal={IEEE Robotics and Automation Letters}, 
  year={2025},
  volume={10},
  number={3},
  pages={2750-2757},
}

@article{howell2022predictive,
  title={\href{https://doi.org/10.48550/arXiv.2212.00541}{Predictive sampling: Real-time behaviour synthesis with {MuJoCo}}},
  author={Howell, Taylor and Gileadi, Nimrod and Tunyasuvunakool, Saran and Zakka, Kevin and Erez, Tom and Tassa, Yuval},
  journal={arXiv preprint arXiv:2212.00541},
  year={2022}
}

@article{jankowski2022vp,
  title={\href{https://doi.org/10.48550/arXiv.2210.04067}{{VP-STO}: Via-point-based stochastic trajectory optimization for reactive robot behavior}},
  author={Jankowski, Julius and Bruderm{\"u}ller, Lara and Hawes, Nick and Calinon, Sylvain},
  journal={arXiv preprint arXiv:2210.04067},
  year={2022}
}

@inproceedings{li2025drop,
  author={Li, Albert H. and Culbertson, Preston and Kurtz, Vince and Ames, Aaron D.},
  booktitle={2025 IEEE International Conference on Robotics and Automation}, 
  title={\href{https://doi.org/10.1109/ICRA55743.2025.11128433}{DROP: Dexterous reorientation via online planning}}, 
  year={2025},
  pages={14299-14306}
}

@inproceedings{peng2018simtoreal,
  author={Peng, Xue Bin and Andrychowicz, Marcin and Zaremba, Wojciech and Abbeel, Pieter},
  booktitle={2018 IEEE International Conference on Robotics and Automation}, 
  title={\href{https://doi.org/10.1109/ICRA.2018.8460528}{Sim-to-real transfer of robotic control with dynamics randomization}}, 
  year={2018},
  pages={3803-3810}
}

@inproceedings{rajeswaran2017epopt,
    title={\href{https://openreview.net/forum?id=SyWvgP5el}{{EPO}pt: Learning robust neural network policies using model ensembles}},
    author={Aravind Rajeswaran and Sarvjeet Ghotra and Balaraman Ravindran and Sergey Levine},
    booktitle={International Conference on Learning Representations},
    year={2017}
}

@article{suh2022bundled,
  author={Suh, H.J. Terry and Pang, Tao and Tedrake, Russ},
  journal={IEEE Robotics and Automation Letters}, 
  title={\href{https://doi.org/10.1109/LRA.2022.3146931}{Bundled gradients through contact via randomized smoothing}}, 
  year={2022},
  volume={7},
  number={2},
  pages={4000-4007}
}

@article{pang2023global,
  author={Pang, Tao and Suh, H.J. Terry and Yang, Lujie and Tedrake, Russ},
  journal={IEEE Transactions on Robotics}, 
  title={\href{https://doi.org/10.1109/TRO.2023.3300230}{Global planning for contact-rich manipulation via local smoothing of quasi-dynamic contact models}}, 
  year={2023},
  volume={39},
  number={6},
  pages={4691-4711}
}

@inproceedings{shirai2025is,
   author={Shirai, Yuki and Zhao, Tong and Suh, H.J. Terry and Zhu, Huaijiang and Ni, Xinpei and Wang, Jiuguang and Simchowitz, Max and Pang, Tao},
   title = {\href{https://doi.org/10.1109/ICRA55743.2025.11127776}{Is linear feedback on smoothed dynamics sufficient for stabilizing contact-rich plans?}},
   booktitle = {2025 IEEE International Conference on Robotics and Automation},
   pages = {11926–11932},
   year = {2025}
}

@article{suh2025dexterous,
   author = {H. J. Terry Suh and Tao Pang and Tong Zhao and Russ Tedrake},
   title ={\href{https://doi.org/10.1177/02783649251398875}{Dexterous contact-rich manipulation via the contact trust region}},
   journal = {The International Journal of Robotics Research},
   year = {2026},
}

@article{leeman2025robust,
   author={Leeman, Antoine P. and Köhler, Johannes and Zanelli, Andrea and Bennani, Samir and Zeilinger, Melanie N.},
   title = {\href{https://doi.org/10.1109/TAC.2025.3552482}{Robust nonlinear optimal control via system level synthesis}},
   journal = {IEEE Transactions on Automatic Control},
   volume = {70},
   number = {7},
   pages = {4780–4787},
   year = {2025}
}

@inproceedings{pang2021convex,
  author={Pang, Tao and Tedrake, Russ},
  booktitle={2021 IEEE International Conference on Robotics and Automation}, 
  title={\href{https://doi.org/10.1109/ICRA48506.2021.9560941}{A convex quasistatic time-stepping scheme for rigid multibody systems with contact and friction}}, 
  year={2021},
  pages={6614-6620}
}

@InProceedings{zhang2023adaptive,
  title = {\href{https://proceedings.mlr.press/v202/zhang23s.html}{Adaptive barrier smoothing for first-order policy gradient with contact dynamics}},
  author = {Zhang, Shenao and Jin, Wanxin and Wang, Zhaoran},
  booktitle = {Proceedings of the 40th International Conference on Machine Learning},
  pages = {41219-41243},
  year = 	{2023}
}

@inproceedings{sutton1999policy,
 author = {Sutton, Richard S. and McAllester, David and Singh, Satinder and Mansour, Yishay},
 title = {\href{https://proceedings.neurips.cc/paper_files/paper/1999/file/464d828b85b0bed98e80ade0a5c43b0f-Paper.pdf}{Policy gradient methods for reinforcement learning with function approximation}},
 booktitle = {Advances in Neural Information Processing Systems},
 volume = {12},
 year = {1999}
}

@inproceedings{montaut2023differentiable,
  author={Montaut, Louis and Lidec, Quentin Le and Bambade, Antoine and Petrik, Vladimir and Sivic, Josef and Carpentier, Justin},
  booktitle={2023 IEEE International Conference on Robotics and Automation}, 
  title={\href{https://doi.org/10.1109/ICRA48891.2023.10160251}{Differentiable collision detection: A randomized smoothing approach}}, 
  year={2023},
  pages={3240-3246}
}

@article{le2023single,
  author={Le Cleac'h, Simon and Schwager, Mac and Manchester, Zachary and Sindhwani, Vikas and Florence, Pete and Singh, Sumeet},
  journal={IEEE Robotics and Automation Letters}, 
  title={\href{https://doi.org/10.1109/LRA.2023.3268824}{Single-level differentiable contact simulation}}, 
  year={2023},
  volume={8},
  number={7},
  pages={4012-4019}
}

@inproceedings{mora2021pods,
  title ={\href{https://proceedings.mlr.press/v139/mora21a.html}{PODS: Policy optimization via differentiable simulation}},
  author={Mora, Miguel Angel Zamora and Peychev, Momchil and Ha, Sehoon and Vechev, Martin and Coros, Stelian},
  booktitle={Proceedings of the 38th International Conference on Machine Learning},
  pages={7805-7817},
  year={2021},
}

@inproceedings{qiao2021efficient,
  title={\href{https://proceedings.mlr.press/v139/qiao21a.html}{Efficient differentiable simulation of articulated bodies}},
  author={Qiao, Yi-Ling and Liang, Junbang and Koltun, Vladlen and Lin, Ming C},
  booktitle={Proceedings of the 38th International Conference on Machine Learning},
  pages={8661-8671},
  year={2021},
}

@inproceedings{xu2022accelerated,
    title={\href{https://openreview.net/forum?id=ZSKRQMvttc}{Accelerated policy learning with parallel differentiable simulation}},
    author={Jie Xu and Viktor Makoviychuk and Yashraj Narang and Fabio Ramos and Wojciech Matusik and Animesh Garg and Miles Macklin},
    booktitle={International Conference on Learning Representations},
    year={2022}
}

@inproceedings{georgiev2023adaptive,
  author = {Georgiev, Ignat and Srinivasan, Krishnan and Xu, Jie and Heiden, Eric and Garg, Animesh},
  title = {\href{https://doi.org/10.5555/3692070.3692688}{Adaptive horizon actor-critic for policy learning in contact-rich differentiable simulation}},
  booktitle = {Proceedings of the 41st International Conference on Machine Learning},
  year = {2024}
}

@article{le2024fast,
  author={Le Cleac'h, Simon and Howell, Taylor A. and Yang, Shuo and Lee, Chi-Yen and Zhang, John and Bishop, Arun and Schwager, Mac and Manchester, Zachary},
  journal={IEEE Transactions on Robotics}, 
  title={\href{https://doi.org/10.1109/TRO.2024.3351554}{Fast contact-implicit model predictive control}}, 
  year={2024},
  volume={40},
  pages={1617-1629}
}

@article{kim2025contact,
   author = {Kim, Gijeong and Kang, Dongyun and Kim, Joon-Ha and Hong, Seungwoo and Park, Hae-Won},
   title = {\href{https://doi.org/10.1177/02783649241273645}{Contact-implicit model predictive control: Controlling diverse quadruped motions without pre-planned contact modes or trajectories}},
   journal = {The International Journal of Robotics Research},
   volume = {44},
   number = {3},
   pages = {486–510},
   year = {2025}
}

@inproceedings{kang2025global,
   author = {Kang, Shucheng and Liu, Guorui and Yang, Heng},
   title = {\href{https://www.roboticsproceedings.org/rss21/p046.pdf}{Global contact-rich planning with sparsity-rich semidefinite relaxations}},
   booktitle = {Proceedings of Robotics: Science and Systems},
   year      = {2025}
}

@article{kurtz2026inverse,
  author = {Vince Kurtz and Alejandro Castro and Aykut Özgün Önol and Hai Lin},
  title = {\href{https://doi.org/10.1177/02783649251344635}{Inverse dynamics trajectory optimization for contact-implicit model predictive control}},
  journal = {The International Journal of Robotics Research},
  volume = {45},
  number = {1},
  pages = {23-40},
  year = {2026}
}

@article{manchester2019contact,
   author = {Manchester, Zachary and Doshi, Neel and Wood, Robert J. and Kuindersma, Scott},
   title = {\href{https://doi.org/10.1177/0278364919849235}{Contact-implicit trajectory optimization using variational integrators}},
   journal = {The International Journal of Robotics Research},
   volume = {38},
   number = {12-13},
   pages = {1463–1476},
   year = {2019}
}

@article{aydinoglu2024consensus,
  author={Aydinoglu, Alp and Wei, Adam and Huang, Wei-Cheng and Posa, Michael},
  journal={IEEE Transactions on Robotics}, 
  title={\href{https://doi.org/10.1109/TRO.2024.3435423}{Consensus complementarity control for multicontact {MPC}}}, 
  year={2024},
  volume={40},
  pages={3879-3896}
}

@inproceedings{graesdal2024towards,
  title     = {\href{https://www.roboticsproceedings.org/rss20/p132.pdf}{Towards tight convex relaxations for contact-rich manipulation}},
  author    = {Graesdal, Bernhard Paus and Chia, Shao Yuan Chew and Marcucci, Tobia and Morozov, Savva and Amice, Alexandre and Parrilo, Pablo A and Tedrake, Russ},
  booktitle = {Proceedings of Robotics: Science and Systems},
  year      = {2024}
}

@inproceedings{li2025surprising,
   author = {Li, Yulin and Han, Haoyu and Kang, Shucheng and Ma, Jun and Yang, Heng},
   title = {\href{https://www.roboticsproceedings.org/rss21/p047.pdf}{On the surprising robustness of sequential convex optimization for contact-implicit motion planning}},
   booktitle = {Proceedings of Robotics: Science and Systems},
   year      = {2025}
}

@inproceedings{jin2025complementarity,
   author = {Jin, Wanxin},
   title = {\href{https://www.roboticsproceedings.org/rss21/p111.pdf}{Complementarity-free multi-contact modeling and optimization for dexterous manipulation}},
   booktitle = {Proceedings of Robotics: Science and Systems},
   year      = {2025}
}

@article{posa2013direct,
   author = {Posa, Michael and Cantu, Cecilia and Tedrake, Russ},
   title = {\href{https://doi.org/10.1177/0278364913506757}{A direct method for trajectory optimization of rigid bodies through contact}},
   journal = {The International Journal of Robotics Research},
   volume = {33},
   number = {1},
   pages = {69–81},
   year = {2013},
}

@article{hogan2020reactive,
  author = {Francois R Hogan and Alberto Rodriguez},
  title ={\href{https://doi.org/10.1177/0278364920913938}{Reactive planar non-prehensile manipulation with hybrid model predictive control}},
  journal = {The International Journal of Robotics Research},
  volume = {39},
  number = {7},
  pages = {755-773},
  year = {2020}
}

@article{agrawal2019differentiating,
  title={\href{https://doi.org/10.48550/arXiv.1904.09043}{Differentiating through a cone program}},
  author={Agrawal, Akshay and Barratt, Shane and Boyd, Stephen and Busseti, Enzo and Moursi, Walaa M},
  journal={arXiv preprint arXiv:1904.09043},
  year={2019}
}

@article{anitescu2006optimization,
  title={\href{https://doi.org/10.1007/s10107-005-0590-7}{Optimization-based simulation of nonsmooth rigid multibody dynamics}},
  author={Anitescu, Mihai},
  journal={Mathematical Programming},
  volume={105},
  number={1},
  pages={113-143},
  year={2006}
}

@book{mason2021mechanics,
    author = {Mason, Matthew T.},
    title = {\href{https://doi.org/10.7551/mitpress/4527.001.0001}{Mechanics of Robotic Manipulation}},
    publisher = {The MIT Press},
    year = {2001}
}

@article{anderson2019system,
    title = {\href{https://doi.org/10.1016/j.arcontrol.2019.03.006}{System level synthesis}},
    author = {James Anderson and John C. Doyle and Steven H. Low and Nikolai Matni},
    journal = {Annual Reviews in Control},
    volume = {47},
    pages = {364-393},
    year = {2019}
}

@article{leeman2024fast,
    title = {\href{https://doi.org/10.1016/j.ifacol.2024.09.027}{Fast system level synthesis: Robust model predictive control using {Riccati} recursions}},
    author = {Antoine P. Leeman and Johannes Köhler and Florian Messerer and Amon Lahr and Moritz Diehl and Melanie N. Zeilinger},
    journal = {IFAC-PapersOnLine},
    volume = {58},
    number = {18},
    pages = {173-180},
    year = {2024},
}

@article{newbury2024review,
  author={Newbury, Rhys and Collins, Jack and He, Kerry and Pan, Jiahe and Posner, Ingmar and Howard, David and Cosgun, Akansel},
  journal={IEEE Access}, 
  title={\href{https://doi.org/10.1109/ACCESS.2024.3425448}{A review of differentiable simulators}}, 
  year={2024},
  volume={12},
  number={},
  pages={97581-97604}
}

@inproceedings{murthy2021gradsim,
   title={\href{https://openreview.net/forum?id=c_E8kFWfhp0}{gradSim: Differentiable simulation for system identification and visuomotor control}},
   author={J. Krishna Murthy and Miles Macklin and Florian Golemo and Vikram Voleti and Linda Petrini and Martin Weiss and Breandan Considine and J{\'e}r{\^o}me Parent-L{\'e}vesque and Kevin Xie and Kenny Erleben and Liam Paull and Florian Shkurti and Derek Nowrouzezahrai and Sanja Fidler},
   booktitle={International Conference on Learning Representations},
   year={2021}
}

@article{betts1998survey,
   author = {Betts, John T.},
   title = {\href{https://doi.org/10.2514/2.4231}{Survey of numerical methods for trajectory optimization}},
   journal = {Journal of Guidance, Control, and Dynamics},
   volume = {21},
   number = {2},
   pages = {193-207},
   year = {1998},
}

@article{schulman2014motion,
   author = {John Schulman and Yan Duan and Jonathan Ho and Alex Lee and Ibrahim Awwal and Henry Bradlow and Jia Pan and Sachin Patil and Ken Goldberg and Pieter Abbeel},
   title ={\href{https://doi.org/10.1177/0278364914528132}{Motion planning with sequential convex optimization and convex collision checking}},
   journal = {The International Journal of Robotics Research},
   volume = {33},
   number = {9},
   pages = {1251-1270},
   year = {2014}
}

@article{malyuta2022convex,
  author={Malyuta, Danylo and Reynolds, Taylor P. and Szmuk, Michael and Lew, Thomas and Bonalli, Riccardo and Pavone, Marco and Açıkmeşe, Behçet},
  journal={IEEE Control Systems Magazine}, 
  title={\href{https://doi.org/10.1109/MCS.2022.3187542}{Convex optimization for trajectory generation: A tutorial on generating dynamically feasible trajectories reliably and efficiently}}, 
  year={2022},
  volume={42},
  number={5},
  pages={40-113}
}

@article{sorensen1982newtons,
   author = {Sorensen, D. C.},
   title = {\href{https://doi.org/10.1137/0719026}{Newton’s method with a model trust region modification}},
   journal = {SIAM Journal on Numerical Analysis},
   volume = {19},
   number = {2},
   pages = {409-426},
   year = {1982},
}

@inproceedings{chou2022safe,
  title={\href{https://doi.org/10.1007/978-3-031-21090-7_21}{Safe output feedback motion planning from images via learned perception modules and contraction theory}},
  author={Chou, Glen and Ozay, Necmiye and Berenson, Dmitry},
  booktitle={International Workshop on the Algorithmic Foundations of Robotics},
  pages={349-367},
  year={2023}
}

@INPROCEEDINGS{knuth2023statistical,
  author={Knuth, Craig and Chou, Glen and Reese, Jamie and Moore, Joseph},
  booktitle={2023 IEEE International Conference on Robotics and Automation}, 
  title={\href{https://doi.org/10.1109/ICRA48891.2023.10161001}{Statistical safety and robustness guarantees for feedback motion planning of unknown underactuated stochastic systems}}, 
  year={2023},
  pages={12700-12706}
}

@article{knuth2021planning,
  author={Knuth, Craig and Chou, Glen and Ozay, Necmiye and Berenson, Dmitry},
  journal={IEEE Robotics and Automation Letters}, 
  title={\href{https://doi.org/10.1109/LRA.2021.3068889}{Planning with learned dynamics: Probabilistic guarantees on safety and reachability via {Lipschitz} constants}}, 
  year={2021},
  volume={6},
  number={3},
  pages={5129-5136},
}

@article{limon2005robust,
  author = {D. Limon  and J.M. Bravo  and T. Alamo  and E.F. Camacho },
  title = {\href{https://doi.org/10.1049/ip-cta:20040480}{Robust {MPC} of constrained nonlinear systems based on interval arithmetic}},
  journal = {IEE Proceedings - Control Theory and Applications},
  volume = {152},
  issue = {3},
  pages = {325-332},
  year = {2005},
}

@article{antoine2026visionsls,
   author = {Leeman, Antoine P and Zhan, Shuyu and Zeilinger, Melanie N and Chou, Glen},
   title = {\href{https://doi.org/10.48550/arXiv.2604.24894}{VISION-SLS: Safe perception-based control from learned visual representations via system level synthesis}},
   journal = {arXiv preprint arXiv:2604.24894},
   year = {2026},
}

@article{gould2026over,
  author={Gould, Brendan and Chiu, Chih-Yuan and Leeman, Antoine P and Vamvoudakis, Kyriakos G and Coogan, Samuel and Chou, Glen},
  title={\href{https://doi.org/10.48550/arXiv.2604.27355}{Over-approximating minimizer sets of constrained convex programs with parametric uncertainty via reachability analysis}},
  journal={arXiv preprint arXiv:2604.27355},
  year={2026}
}

@article{fang2026safe,
  author={Fang, Jeffrey and Chou, Glen},
  title={\href{https://doi.org/10.48550/arXiv.2604.07644}{Safe large-scale robust nonlinear MPC in milliseconds via reachability-constrained system level synthesis on the GPU}},
  journal={arXiv preprint arXiv:2604.07644},
  year={2026}
}

@article{duchi2012randomized,
   author = {Duchi, John C. and Bartlett, Peter L. and Wainwright, Martin J.},
   title = {\href{https://doi.org/10.1137/110831659}{Randomized smoothing for stochastic optimization}},
   journal = {SIAM Journal on Optimization},
   volume = {22},
   number = {2},
   pages = {674–701},
   year = {2012},
}

@article{lidec2024contact,
   author = {Lidec, Q. Le and Jallet, W. and Montaut, L. and Laptev, I. and Schmid, C. and Carpentier, J.},
   title = {\href{https://doi.org/10.1109/TRO.2024.3434208}{Contact models in robotics: A comparative analysis}},
   journal = {IEEE Transactions on Robotics},
   volume = {40},
   pages = {3716–3733},
   year = {2024},
}

@article{scholtes2001convergence,
   author = {Scholtes, Stefan},
   title = {\href{https://doi.org/10.1137/S1052623499361233}{Convergence properties of a regularization scheme for mathematical programs with complementarity constraints}},
   journal = {SIAM Journal on Optimization},
   volume = {11},
   number = {4},
   pages = {918–936},
   year = {2001},
}

\clearpage
\appendices
\section{Implicit Differentiation of Conic Programs}
\label{app:conic-program-implicit-differentiation}

The \ac{KKT} conditions for \eqref{eq:conic-program} and \eqref{eq:conic-logbarrier-program} are
\begin{subequations}  \label{eq:conic-program-kkt-full}
\begin{align}
    & P x + q + A^\top z = 0, \\
    & A x + s - b = 0, \\
    & s \circ z = \kappa\, e,  \\
    & s \in \mathcal{K}, \quad z \in \mathcal{K}^*.
\end{align}
\end{subequations}
The product $s \circ z$ is defined blockwise as $(s \circ z)_i = s_i \circ z_i$, where each block uses the Jordan product associated with its symmetric cone \cite{vandenberghe2010cvxopt}:
\begin{equation*}
    u \circ v =
    \begin{cases}
        \bigl[ \begin{matrix} u_1 v_1 & \cdots & u_p v_p \end{matrix} \bigr]^\top
        & \hspace{-3mm} \text{if $\mathcal{K}_i$ is a non-negative cone}, \\[4pt]
        \bigl[ \begin{matrix} u^\top v & u_0 v_1^\top + v_0 u_1^\top \end{matrix} \bigr]^\top
        & \hspace{-3mm} \text{if $\mathcal{K}_i$ is a second-order cone},
    \end{cases}
\end{equation*}
and the vector $e = \begin{bmatrix} e_1^\top & \cdots & e_I^\top \end{bmatrix}^\top$ is defined by
\begin{equation*}
    e_i =
    \begin{cases}
        \bigl[ \begin{matrix} 1 & 1 & \cdots & 1 \end{matrix} \bigr]^\top
        & \text{if $\mathcal{K}_i$ is a non-negative cone}, \\[4pt]
        \bigl[ \begin{matrix} 1 & 0 & \cdots & 0 \end{matrix} \bigr]^\top
        & \text{if $\mathcal{K}_i$ is a second-order cone}.
    \end{cases}
\end{equation*}
The logarithmic barrier in \eqref{eq:conic-logbarrier-program} is defined as
\begin{equation*}
    \psi_i(u) =
    \begin{cases}
        -\textstyle\sum_{j=1}^p \log(u_j)
        & \text{if $\mathcal{K}_i$ is a non-negative cone}, \\[4pt]
        -\tfrac{1}{2} \log(u_0^2 - u_1^\top u_1)
        & \text{if $\mathcal{K}_i$ is a second-order cone}.
    \end{cases}
\end{equation*}

Implicit differentiation of the \ac{KKT} system \eqref{eq:conic-program-kkt-full} with respect to problem data $\theta$ yields
\begin{equation*}
    \underbrace{
    \begin{bmatrix}
        P & A^\top & 0 \\
        A & 0 & I \\
        0 & L(s) & L(z)
    \end{bmatrix}}_{\mathrm{K}}
    \begin{bmatrix}
        \tfrac{\partial x}{\partial \theta} \\
        \tfrac{\partial z}{\partial \theta} \\
        \tfrac{\partial s}{\partial \theta}
    \end{bmatrix}
    +
    \begin{bmatrix}
        \tfrac{\partial P}{\partial \theta} x
        + \tfrac{\partial q}{\partial \theta}
        + \tfrac{\partial A^\top}{\partial \theta} z \\
        \tfrac{\partial A}{\partial \theta} x
        - \tfrac{\partial b}{\partial \theta} \\
        0
    \end{bmatrix}
    = 0,
\end{equation*}
where $L(s)$ denotes the linear operator satisfying $L(s)\,z = s \circ z$.

Similarly, differentiating \eqref{eq:conic-program-kkt-full} with respect to the complementarity parameter $\kappa$ gives
\begin{equation*}
    \underbrace{
    \begin{bmatrix}
        P & A^\top & 0 \\
        A & 0 & I \\
        0 & L(s) & L(z)
    \end{bmatrix}}_{\mathrm{K}}
    \begin{bmatrix}
        \tfrac{\partial x}{\partial \kappa} \\
        \tfrac{\partial z}{\partial \kappa} \\
        \tfrac{\partial s}{\partial \kappa}
    \end{bmatrix}
    =
    \begin{bmatrix}
        0 \\
        0 \\
        e
    \end{bmatrix}.
\end{equation*}
The mixed derivatives with respect to $\theta$ and $\kappa$ can then be obtained by solving
\begin{equation*}
    \mathrm{K}
    \begin{bmatrix}
        \tfrac{\partial^2 x}{\partial \theta \partial \kappa} \\
        \tfrac{\partial^2 z}{\partial \theta \partial \kappa} \\
        \tfrac{\partial^2 s}{\partial \theta \partial \kappa}
    \end{bmatrix}
    =
    -
    \frac{\partial \mathrm{K}}{\partial \theta}
    \begin{bmatrix}
        \tfrac{\partial x}{\partial \kappa} \\
        \tfrac{\partial z}{\partial \kappa} \\
        \tfrac{\partial s}{\partial \kappa}
    \end{bmatrix}.
\end{equation*}

\section{Quasistatic Model of Contact Dynamics}
\label{app:quasistatic-contact-model}

In the quasistatic model of contact dynamics, the state is defined as
\begin{equation*}
    x = \begin{bmatrix} x_a^\top & x_o^\top \end{bmatrix}^\top,
\end{equation*}
where $x_a \in \mathbb{R}^{n_a}$ represents the actuated \acp{DoF} and $x_o \in \mathbb{R}^{n_o}$ represents the unactuated object \acp{DoF}. The current timestep state $x$ and next timestep state $x^+$ must satisfy:
\begin{subequations} \label{eq:dynamics}
\begin{align}
    K_a\, (x_a^+ - u) &= \tau_a(x) + \sum_{i=1}^{n_c} J_{a_i}(x)^\top\, \lambda_i,  \label{eq:dynamics-a} \\
    M_o(x)\, \frac{x_o^+ - x_o}{\delta t^2} &= \tau_o(x) + \sum_{i=1}^{n_c} J_{o_i}(x)^\top\, \lambda_i,  \label{eq:dynamics-b} 
\end{align} \vspace*{-2em}
\begin{align}
    & \lambda_i \in \mathcal{F}_i,  \label{eq:dynamics-c} \\
    & \nu_{i,n} := J_{i,n}(x)\, (x^+ - x) + \phi_i(x) \geq 0,  \label{eq:dynamics-d} \\
    & \nu_{i,n}\, \lambda_{i,n} = 0.  \label{eq:dynamics-e} 
\end{align}
\end{subequations}
Equations \eqref{eq:dynamics-a} and \eqref{eq:dynamics-b} describes the force balance. Here, $u \in \mathbb{R}^{n_a}$ is the control input applied to the actuated \acp{DoF} via a stiffness controller parameterized by $K_a \in \mathbb{S}_{++}^{n_a}$, and $M_o \in \mathbb{S}_{++}^{n_o}$ is the mass matrix for the unactuated object \acp{DoF}. The generalized forces $\tau_a \in \mathbb{R}^{n_a}$ and $\tau_o \in \mathbb{R}^{n_o}$ include effects such as gravity.
The contact forces $\lambda_i = \begin{bmatrix} \lambda_{i,n} & \lambda_{i,t}^\top \end{bmatrix}^\top \in \mathbb{R}^d$ enter the force balance equations through the contact Jacobian
\begin{equation*}
    J_i(x) = \begin{bmatrix} J_{a_i}(x) & J_{o_i}(x) \end{bmatrix} \in \mathbb{R}^{d \times (n_a + n_o)},
\end{equation*}
where $d$ is the dimension of the contact force ($d=1$ for purely 1D contact, $d=2$ for planar contact, and $d=3$ for general spatial contact). The total number of contact pairs is denoted by $n_c$.
Constraint \eqref{eq:dynamics-c} enforces that the contact force lies within the friction cone:
\begin{equation*}
    \mathcal{F}_i = \Bigl\{ \begin{bmatrix} \lambda_{i,n} & \lambda_{i,t}^\top \end{bmatrix}^\top \,\Big|\, \|\lambda_{i,t}\|_2 \leq \mu_i \lambda_{i,n} \Bigr\},
\end{equation*}
where $\mu_i$ is the coefficient of friction for the $i$th contact.
In \eqref{eq:dynamics-d}, $\phi_i(x)$ denotes the signed distance of the $i$th contact at the current timestep, $\nu_{i,n}$ is its linear approximation at the next timestep. This constraint enforces non-penetration. Finally, \eqref{eq:dynamics-e} encodes complementarity, i.e., contact force can only be applied when the signed distance is zero.

We can rewrite \eqref{eq:dynamics} compactly as
\begin{equation} \label{eq:dynamics-compact}
\begin{aligned}
    & P(x)\, x^+ + q(x,u) - \sum_{i=1}^{n_c} J_i(x)^\top\, \lambda_i = 0,  \\
    & \nu_{i,n} = J_{i,n}(x)\, (x^+ - x) + \phi_i(x),  \\
    & \lambda_i \in \mathcal{F}_i, \quad \nu_{i,n} \geq 0, \quad \nu_{i,n}\, \lambda_{i,n} = 0,
\end{aligned}    
\end{equation}
where
\begin{align*}
    P(x) &= \begin{bmatrix}
        K_a & 0 \\
        0 & M_o(x) / \delta t^2
    \end{bmatrix},  \\
    q(x,u) &= \begin{bmatrix}
        -K_a\, u - \tau_a(x) \\
        -M_o(x)\, x_o / \delta t^2 - \tau_o(x)
    \end{bmatrix}.
\end{align*}
For general 3D systems involving rotation, the governing equations are more complicated but can still be expressed in the compact form of \eqref{eq:dynamics-compact}.

Conveniently, the \ac{KKT} conditions of
\begin{equation} \label{eq:dynamics-cvxopt}
\begin{aligned}
    & \minimize_{x^+,\nu}  && \tfrac{1}{2} (x^+)^\top\, P(x)\, x^+ + q(x,u)^\top\, x^+  \\
    & \subjectto && \nu_i = J_i(x)\, (x^+ - x) + \begin{bmatrix} \phi_i(x) & 0 & 0 \end{bmatrix}^\top,  \\
    &            && \nu_i \in \mathcal{F}_i^*, \qquad \forall i=1,\dots,n_c,
\end{aligned}
\end{equation}
are
\begin{subequations}
\begin{align}
    & P(x)\, x^+ + q(x,u) - \sum_{i=1}^{n_c} J_i(x)^\top\, \lambda_i = 0,  \\
    & \nu_i = J_i(x)\, (x^+ - x) + \begin{bmatrix} \phi_i(x) & 0 & 0 \end{bmatrix}^\top,  \\
    & \lambda_i \in \mathcal{F}_i,  \quad
      \nu_i \in \mathcal{F}_i^*,  \quad
      \lambda_i \circ \nu_i = 0, \nonumber \\
    & \hspace{10.7em} \forall i=1,\dots,n_c.
\end{align}
\end{subequations}
which satisfy \eqref{eq:dynamics-compact}. Here, $\mathcal{F}_i^*$ denotes the dual cone of $\mathcal{F}_i$, defined as
\begin{equation*}
    \mathcal{F}_i^* = \Bigl\{ \begin{bmatrix} \nu_{i,n} & \nu_{i,t}^\top \end{bmatrix}^\top \,\Big|\, \|\nu_{i,t}\|_2 \leq \tfrac{1}{\mu_i} \nu_{i,n} \Bigr\}.
\end{equation*}
Thus, solving the convex conic program \eqref{eq:dynamics-cvxopt} yields the primal and dual solutions corresponding to the next timestep state $x^+$ and contact forces $\lambda_i$, respectively.

The smoothed dynamics is defined implicitly by
\begin{equation}
\begin{aligned}
    & \minimize_{x_\kappa^+, \nu}  && \tfrac{1}{2} (x_\kappa^+)^\top\, P(x)\, x_\kappa^+ + q(x,u)^\top\, x_\kappa^+  + \kappa \sum_{i=1}^{I} \psi_i(\nu_{\kappa,i})\\
    & \subjectto && \nu_{\kappa,i} = J_i(x)\, (x_\kappa^+ - x) + \begin{bmatrix} \phi_i(x) & 0 & 0 \end{bmatrix}^\top, \\
    &            && \hspace{12em} \forall i = 1,\dots,n_c,
\end{aligned}\hspace{-25pt}
\end{equation}
which has \ac{KKT} conditions
\begin{subequations}  \label{eq:smoothed-dynamics-kkt}
\begin{align}
    & P(x)\, x_\kappa^+ + q(x,u) - \sum_{i=1}^{n_c} J_i(x)^\top\, \lambda_{\kappa,i} = 0,  \label{eq:smoothed-dynamics-kkt-a} \\
    & \nu_{\kappa,i} = J_i(x)\, (x_\kappa^+ - x) + \begin{bmatrix} \phi_i(x) & 0 & 0 \end{bmatrix}^\top,  \label{eq:smoothed-dynamics-kkt-b} \\
    & \lambda_{\kappa,i} \in \mathcal{F}_i,  \quad
      \nu_{\kappa,i} \in \mathcal{F}_i^*,  \quad
      \lambda_{\kappa,i} \circ \nu_{\kappa,i} = \kappa\, e, \nonumber \\
    & \hspace{11.5em} \forall i=1,\dots,n_c.
\end{align}
\end{subequations}

\section{Proof of Theorem~\ref{thm:smoothing-error-bounds}}
\label{app:proof-smoothing-error-bounds}

Notice that the deviation between the nonsmooth dynamics $f_0(x,u)$ and the smoothed dynamics $f_\kappa(x,u)$ is given by
\begin{equation}  \label{eq:smoothing-deviation}
    f_0(x,u) = f_\kappa(x,u) + \int_\kappa^0 \frac{\partial f_\upkappa(x,u)}{\partial\upkappa}\, d\upkappa.
\end{equation}
To compute $\partial f_\upkappa(x,u) / \partial\upkappa$, which is equivalent to $\partial x_\upkappa / \partial\upkappa$, we differentiate \eqref{eq:smoothed-dynamics-kkt-a} with respect to $\kappa$, yielding
\begin{equation*}
    P(x)\, \frac{\partial x_\kappa^+}{\partial\kappa} - \sum_{i=1}^{n_c} J_i(x)^\top\, \frac{\partial\lambda_{\kappa,i}}{\partial\kappa} = 0.
\end{equation*}
Substituting into \eqref{eq:smoothing-deviation} yields
\begin{equation*}
    f_0(x,u) = f_\kappa(x,u) - P(x)^{-1} \sum_{i=1}^{n_c} J_i(x)^\top \int_0^\kappa \frac{\partial \lambda_{\upkappa,i}}{\partial\upkappa}\, d\upkappa.
\end{equation*}
Hence, we need to show that
\begin{equation}  \label{eq:dlambda-dkappa-set}
    \Bigl( \int_0^\kappa \frac{\partial \lambda_{\upkappa,i}}{\partial\upkappa}\, d\upkappa \Bigr)
    \in
    \Bigl( \frac{\partial \lambda_{\kappa,i}}{\partial\kappa}\, \kappa\, w_i \Bigr),
\end{equation}
where $w_i \in [1,2]$.

From
\begin{equation*}
    \lambda_{\kappa,i} = \frac{\kappa}{\nu_{\kappa,i,n}^2 / \mu_i^2 - \|\nu_{\kappa,i,t}\|_2^2} 
                         \begin{bmatrix} \nu_{\kappa,i,n} / \mu_i^2 \\ -\nu_{\kappa,i,t} \end{bmatrix},
\end{equation*}
we have
\begin{equation*}
    \frac{\partial\lambda_{\kappa,i}}{\partial\kappa}
    = \bigl( \diag(\nu_{\kappa,i}) + c_{\kappa,i}\, J_i P^{-1} J_i^\top \bigr)^\dag\, \tfrac{1}{\kappa} \diag(\nu_{\kappa,i})\, \lambda_{\kappa,i},
\end{equation*}
where
\begin{equation*}
    c_{\kappa,i} = \tfrac{2}{\kappa} \diag(\nu_{\kappa,i})\, \lambda_{\kappa,i}\, \lambda_{\kappa,i}^\top - \diag(\lambda_{\kappa,i}).
\end{equation*}
In the normal direction, we simply have
\begin{equation*}
    \frac{\partial\lambda_{\kappa,i,n}}{\partial\kappa}
    = \bigl( \nu_{\kappa,i,n} + \lambda_{\kappa,i,n}\, J_{i,n} P^{-1} J_{i,n}^\top \bigr)^{-1}.
\end{equation*}
We will show separately that \eqref{eq:dlambda-dkappa-set} satisfies in the normal and tangential direction.
But before that, we need the following lemma.
\begin{lemma}  \label{thm:smoothing-vector-invariant}
With the assumption that $J_i P^{-1} J_j^\top = 0$ for all $i \neq j$, the vector quantity
\begin{equation*}
    \nu_{\kappa,i} - J_i P^{-1} J_i^\top \lambda_{\kappa,i}
\end{equation*}
is independent of $\kappa$.
\end{lemma}
\begin{proof}
Left-multiplying \eqref{eq:smoothed-dynamics-kkt-a} by $J_j P^{-1}$ yields
\begin{equation*}
    J_j\, x_\kappa^+ + J_j P^{-1}\, q(x,u) - \sum_{i=1}^{n_c} J_j P^{-1} J_i^\top\, \lambda_{\kappa,i} = 0.
\end{equation*}
Due to the assumption that $J_i P^{-1} J_j^\top = 0$ for all $i \neq j$, we have
\begin{equation*}
    J_i\, x_\kappa^+ + J_i P^{-1}\, q(x,u) - J_i P^{-1} J_i^\top\, \lambda_{\kappa,i} = 0.    
\end{equation*}
Substituting in $J_i\, x_\kappa^+ = \nu_{\kappa,i} + J_i\, x - \begin{bmatrix} \phi_i(x) & 0 & 0 \end{bmatrix}^\top$ from \eqref{eq:smoothed-dynamics-kkt-b}, we obtain
\begin{equation*}
    \nu_{\kappa,i}^+ - J_i P^{-1} J_i^\top\, \lambda_{\kappa,i} = -J_i P^{-1}\, q(x,u) - J_i\, x + \begin{bmatrix} \phi_i(x) & 0 & 0 \end{bmatrix}^\top.
\end{equation*}
The right-hand side of the equation is independent of $\kappa$. Hence, the left-hand side is also independent of $\kappa$.
\end{proof}
\begin{corollary}  \label{thm:smoothing-scalar-invariant}
With the assumption that $J_i P^{-1} J_j^\top = 0$ for all $i \neq j$, the scalar quantity
\begin{equation*}
    \nu_{\kappa,i,n} - \lambda_{\kappa,i,n}\, J_{i,n} P^{-1} J_{i,n}^\top
\end{equation*}
is independent of $\kappa$.
\end{corollary}

\subsection{Bounds for Normal Component}
We show that the normal component of \eqref{eq:dlambda-dkappa-set} is true.
In the normal direction, the left-hand side of \eqref{eq:dlambda-dkappa-set} is
\begin{equation} \label{eq:dlambda-dkappa-normal-integral}
    \int_0^\kappa \bigl( \nu_{\upkappa,i,n} + \lambda_{\upkappa,i,n}\, J_{i,n} P^{-1} J_{i,n}^\top \bigr)^{-1}\, d\upkappa .
\end{equation}
From the \ac{AM-GM} inequality, we have
\begin{equation*}
\begin{split}
    \tfrac{1}{2} \bigl( \nu_{\upkappa,i,n} + \lambda_{\upkappa,i,n}\, J_{i,n} P^{-1} J_{i,n}^\top \bigr)
    &\geq \bigl( \nu_{\upkappa,i,n}\, \lambda_{\upkappa,i,n}\, J_{i,n} P^{-1} J_{i,n}^\top \bigr)^{\frac{1}{2}}  \\
    &= \bigl( \upkappa\, J_{i,n} P^{-1} J_{i,n}^\top \bigr)^{\frac{1}{2}},
\end{split}
\end{equation*}
with equality satisfied when
\begin{equation*}
    \nu_{\upkappa,i,n} = \lambda_{\upkappa,i,n}\, J_{i,n} P^{-1} J_{i,n}^\top.
\end{equation*}
Hence, we have bounds on the integrand of \eqref{eq:dlambda-dkappa-normal-integral}:
\begin{equation}  \label{eq:smoothing-integrand-normal-bound}
    0 
    \leq \bigl( \nu_{\upkappa,i,n} + \lambda_{\upkappa,i,n}\, J_{i,n} P^{-1} J_{i,n}^\top \bigr)^{-1} 
    \leq \tfrac{1}{2} \bigl( \upkappa\, J_{i,n} P^{-1} J_{i,n}^\top \bigr)^{-\frac{1}{2}}.
\end{equation}

The upper bound of \eqref{eq:smoothing-integrand-normal-bound} is achieved when an $(x,u)$ pair results in
\begin{equation}  \label{eq:smoothing-integrand-normal-upper-bound-condition}
    \nu_{\upkappa,i,n} = \lambda_{\upkappa,i,n}\, J_{i,n} P^{-1} J_{i,n}^\top = \bigl( \upkappa\, J_{i,n} P^{-1} J_{i,n}^\top \bigr)^{\frac{1}{2}}.
\end{equation}
From the fact that
\begin{equation*}
    \nu_{\upkappa,i,n} - \lambda_{\upkappa,i,n}\, J_{i,n} P^{-1} J_{i,n}^\top = 0
\end{equation*}
and utilizing \cref{thm:smoothing-scalar-invariant}, \eqref{eq:smoothing-integrand-normal-upper-bound-condition} must satisfy for all $\upkappa$ values under the same $(x,u)$ pair.
Hence, \eqref{eq:dlambda-dkappa-normal-integral} is bounded above tightly as follows:
\begin{equation*}
\begin{split}
    &
    \int_0^\kappa \bigl( \nu_{\upkappa,i,n} + \lambda_{\upkappa,i,n}\, J_{i,n} P^{-1} J_{i,n}^\top \bigr)^{-1}\, d\upkappa  \\
    &\leq
    \int_0^\kappa \tfrac{1}{2} \bigl( \upkappa\, J_{i,n} P^{-1} J_{i,n}^\top \bigr)^{-\frac{1}{2}}\, d\upkappa  \\
    &= \tfrac{1}{2} \bigl( J_{i,n} P^{-1} J_{i,n}^\top \bigr)^{-\frac{1}{2}}\, 2 \kappa^{\frac{1}{2}}  \\
    &= \bigl( \nu_{\kappa,i,n} + \lambda_{\kappa,i,n}\, J_{i,n} P^{-1} J_{i,n}^\top \bigr)^{-1}\, 2\kappa.
\end{split}
\end{equation*}
This establishes the upper bound of the normal component of \eqref{eq:dlambda-dkappa-set} with $w_i=2$.

The lower bound of \eqref{eq:smoothing-integrand-normal-bound} is achieved when an $(x,u)$ pair results in
\begin{equation}  \label{eq:smoothing-integrand-normal-lower-bound-condition}
    \nu_{\upkappa,i,n} + \lambda_{\upkappa,i,n}\, J_{i,n} P^{-1} J_{i,n}^\top = C,
\end{equation}
in the limit of $C \to \infty$.
From the fact that
\begin{equation*}
    \nu_{\upkappa,i,n} - \lambda_{\upkappa,i,n}\, J_{i,n} P^{-1} J_{i,n}^\top = \pm C
\end{equation*}
and utilizing \cref{thm:smoothing-scalar-invariant}, \eqref{eq:smoothing-integrand-normal-lower-bound-condition} must satisfy for all $\upkappa$ values under the same $(x,u)$ pair.
Hence, \eqref{eq:dlambda-dkappa-normal-integral} is bounded below tightly as follows:
\begin{equation*}
\begin{split}
    &
    \int_0^\kappa \bigl( \nu_{\upkappa,i,n} + \lambda_{\upkappa,i,n}\, J_{i,n} P^{-1} J_{i,n}^\top \bigr)^{-1}\, d\upkappa  \\
    &\geq
    \int_0^\kappa C^{-1}\, d\upkappa
     = C^{-1}\, \kappa  \\
    &= \bigl( \nu_{\kappa,i,n} + \lambda_{\kappa,i,n}\, J_{i,n} P^{-1} J_{i,n}^\top \bigr)^{-1}\, \kappa.
\end{split}
\end{equation*}
This establishes the lower bound of the normal component of \eqref{eq:dlambda-dkappa-set} with $w_i=1$.

\subsection{Bounds for Tangential Component}
We show that the tangential component of \eqref{eq:dlambda-dkappa-set} is true.
The left-hand side of \eqref{eq:dlambda-dkappa-set} is
\begin{equation}  \label{eq:dlambda-dkappa-integral}
    \int_0^\kappa \bigl( \diag(\nu_{\upkappa,i}) + c_{\upkappa,i}\, J_i P^{-1} J_i^\top \bigr)^\dag\, \tfrac{1}{\upkappa} \diag(\nu_{\upkappa,i})\, \lambda_{\upkappa,i}\, d\upkappa.
\end{equation}

The upper bound of the tangential component of \eqref{eq:dlambda-dkappa-integral} is achieved when an $(x,u)$ pair results in
\begin{equation}  \label{eq:smoothing-integrand-tangential-upper-bound-condition}
    \|\nu_{\upkappa,i,t}\|_2 = \nu_{\upkappa,i,n} / \mu_i,  \quad
    \nu_{\upkappa,i,n} = \sqrt{\tfrac{\upkappa \mu_i^2}{\mu_i^2 + 1}}.
\end{equation}
From the fact that
\begin{equation*}
    \nu_{\kappa,i} - J_i P^{-1} J_i^\top \lambda_{\kappa,i} = \infty
\end{equation*}
and utilizing \cref{thm:smoothing-vector-invariant}, \eqref{eq:smoothing-integrand-tangential-upper-bound-condition} must satisfy for all $\upkappa$ values under the same $(x,u)$ pair.
It follows that
\begin{multline*}
    \bigl( \diag(\nu_{\upkappa,i}) + c_{\upkappa,i}\, J_i P^{-1} J_i^\top \bigr)^\dag\, \tfrac{1}{\upkappa} \diag(\nu_{\upkappa,i})\, \lambda_{\upkappa,i}  \\
    =
    \begin{bmatrix}
        \tfrac{1}{2} \mu_i (\mu_i^2 + 1)^{-\frac{1}{2}}\, \upkappa^{-\frac{1}{2}}  \\
        \xi_{\upkappa,i,t}
    \end{bmatrix},
\end{multline*}
with $\|\xi_{\upkappa,i,t}\|_2 = \tfrac{1}{2} (\mu_i^2 + 1)^{-\frac{1}{2}}\, \upkappa^{-\frac{1}{2}}$.
Hence, the tangential component of bounded above tightly by
\begin{equation*}
\begin{split}
    \int_0^\kappa \tfrac{1}{2} (\mu_i^2 + 1)^{-\frac{1}{2}}\, \upkappa^{-\frac{1}{2}}\, d\upkappa
    &= \tfrac{1}{2} (\mu_i^2 + 1)^{-\frac{1}{2}}\, 2 \kappa^{\frac{1}{2}}  \\
    &= \tfrac{1}{2} (\mu_i^2 + 1)^{-\frac{1}{2}}\, \kappa^{-\frac{1}{2}}\, 2\kappa.
\end{split}
\end{equation*}
This establishes the upper bound of the tangential component of \eqref{eq:dlambda-dkappa-set} with $w_i=2$.

The lower bound of the tangential component of \eqref{eq:dlambda-dkappa-integral} is achieved when an $(x,u)$ pair results in
\begin{equation}  \label{eq:smoothing-integrand-tangential-lower-bound-condition}
    \nu_{\upkappa,i,t} = 0,  \quad
    \lambda_{\upkappa,i,t} = 0.
\end{equation}
From the fact that
\begin{equation*}
    \nu_{\kappa,i} - J_i P^{-1} J_i^\top \lambda_{\kappa,i} = 0
\end{equation*}
and utilizing \cref{thm:smoothing-vector-invariant}, \eqref{eq:smoothing-integrand-tangential-lower-bound-condition} must satisfy for all $\upkappa$ values under the same $(x,u)$ pair.
It follows that
\begin{equation*}
    \bigl( \diag(\nu_{\upkappa,i}) + c_{\upkappa,i}\, J_i P^{-1} J_i^\top \bigr)^\dag\, \tfrac{1}{\upkappa} \diag(\nu_{\upkappa,i})\, \lambda_{\upkappa,i}
    =
    \begin{bmatrix}
        \xi_{\upkappa,i,n}  \\
        0
    \end{bmatrix}.
\end{equation*}
Hence, the tangential component of bounded below tightly by $0$.
This establishes the lower bound of the tangential component of \eqref{eq:dlambda-dkappa-set} with $w_i=1$.

\section{Proof of Theorem~\ref{thm:policy-parameterization}}
\label{app:proof-policy-parameterization}

The proof here is similar to \cite[Thm~2.1]{anderson2019system}.
For the \ac{LTV} approximation of the dynamics \eqref{eq:dynamics-bounds} rewritten here as
\begin{equation*}
    x_{k+1} = f_\kappa(x_k,u_k) + E_\kappa(x_k,u_k)\, w,
\end{equation*}
about the nominal trajectory  $\vec{z} = \begin{bmatrix} z_0^\top & \cdots & z_N^\top \end{bmatrix}^\top$ and $\vec{v} = \begin{bmatrix} v_0^\top & \cdots & v_{N-1}^\top \end{bmatrix}^\top$, we have
\begin{equation*}
    x_{k+1} - z_k = A_k\, (x_k - z_k) + B_k\, (u_k - v_k) + E_k\, w,
\end{equation*}
where $A_k = \partial f_\kappa(z_k, v_k) / \partial z$,
$B_k = \partial f_\kappa(z_k, v_k) / \partial v$,
and $E_k = E_\kappa(z_k, v_k)$.
This can be written is block matrix form as
\begin{equation*}
    \vec{x} - \vec{z} = \mat{Z} \mat{A} (\vec{x} - \vec{z}) + \mat{Z} \mat{B} (\vec{u} - \vec{v}) + \mat{Z} \mat{E} \vec{w},
\end{equation*}
where $\mat{Z}$ is the block-downshift operator, i.e., a matrix with identity matrices along its first block sub-diagonal and zeros elsewhere, and $\mat{A} = \blkdiag(A_1, \dots, A_{N-1}, 0)$, $\mat{B} = \blkdiag(B_1, \dots, B_{N-1}, 0)$.
Furthermore, with causal feedback policies of the form
\begin{equation*}
    \vec{u} - \vec{v} = \mat{K} (\vec{x} - \vec{z}),
\end{equation*}
where $\mat{K}$ is a block lower-triangular matrix, we have
\begin{equation*}
    \begin{bmatrix} \vec{x} - \vec{z} \\ \vec{u} - \vec{v} \end{bmatrix} :=
    \begin{bmatrix} \mat{\Phi}^x \\ \mat{\Phi}^u \end{bmatrix} \vec{w} =
    \begin{bmatrix} 
                (\mat{I} - \mat{Z} (\mat{A} + \mat{B} \mat{K}))^{-1} \mat{Z} \mat{E} \\ 
        \mat{K} (\mat{I} - \mat{Z} (\mat{A} + \mat{B} \mat{K}))^{-1} \mat{Z} \mat{E}
    \end{bmatrix} \vec{w}.
\end{equation*}
It is then easily seen that
\begin{equation*}
\begin{aligned}
    &
    \begin{bmatrix} \mat{I} - \mat{Z} \mat{A} & -\mat{Z} \mat{B} \end{bmatrix}
    \begin{bmatrix} \mat{\Phi}^x \\ \mat{\Phi}^u \end{bmatrix}  \\
    &=
    \begin{bmatrix} \mat{I} - \mat{Z} \mat{A} & -\mat{Z} \mat{B} \end{bmatrix}
    \begin{bmatrix} 
                (\mat{I} - \mat{Z} (\mat{A} + \mat{B} \mat{K}))^{-1} \mat{Z} \mat{E} \\ 
        \mat{K} (\mat{I} - \mat{Z} (\mat{A} + \mat{B} \mat{K}))^{-1} \mat{Z} \mat{E}
    \end{bmatrix}  \\
    &=
    (\mat{I} - \mat{Z} \mat{A} - \mat{Z} \mat{B} \mat{K})\,  (\mat{I} - \mat{Z} (\mat{A} + \mat{B} \mat{K}))^{-1}\, \mat{Z} \mat{E}  \\
    &=
    \mat{Z} \mat{E},
\end{aligned}
\end{equation*}
which is equivalent to
\begin{equation*}
\begin{aligned}
    \Phi_{k+1,j}^x &= A_k \Phi_{k,j}^x + B_k \Phi_{k,j}^u,
    \quad \forall k = j+1, \dots, N-1, \\
    \Phi_{j+1,j}^x &= E_j,
\end{aligned}
\end{equation*}
for all $j = 0, \dots, N-1$.

We next show that $\mat{K} = \mat{\Phi}^u (\mat{\Phi}^x)^{-1}$ achieves the desired response. Such a $\mat{K}$ results in
\begin{equation*}
\begin{aligned}
    \vec{x} - \vec{z}
    &= (\mat{I} - \mat{Z} (\mat{A} + \mat{B} \mat{\Phi}^u (\mat{\Phi}^x)^{-1})^{-1} \mat{Z} \mat{E} \vec{w}  \\
    &= \bigl( ((\mat{I} - \mat{Z} \mat{A}) \mat{\Phi}^x - \mat{Z} \mat{B} \mat{\Phi}^u) (\mat{\Phi}^x)^{-1} \bigr)^{-1} \mat{Z} \mat{E} \vec{w}  \\
    &= \mat{\Phi}^x ((\mat{I} - \mat{Z} \mat{A}) \mat{\Phi}^x - \mat{Z} \mat{B} \mat{\Phi}^u)^{-1} \mat{Z} \mat{E} \vec{w}  \\
    &= \mat{\Phi}^x (\mat{Z} \mat{E})^{-1} \mat{Z} \mat{E} \vec{w}  \\
    &= \mat{\Phi}^x \vec{w}.
\end{aligned}
\end{equation*}
We also have
\begin{equation*}
\begin{aligned}
    \vec{u} - \vec{v} 
    &= \mat{\Phi}^u (\mat{\Phi}^x)^{-1} (\vec{x} - \vec{z})  \\
    &= \mat{\Phi}^u (\mat{\Phi}^x)^{-1} (\mat{\Phi}^x \vec{w})  \\
    &= \mat{\Phi}^u \vec{w}.
\end{aligned}
\end{equation*}

\section{Efficient Solver for Problem (\ref{eq:inner-cvx-robust})}
\label{app:fast-sls}

\begin{figure*}[b]
\centering
\begin{minipage}{0.7\textwidth}
\begin{subequations}  \label{eq:optim-SLS}
\begin{alignat}{2}
    & \min_{\substack{\delta\vec{z}, \delta\vec{v},\\\mat{\Phi}^x, \mat{\Phi}^u}}\ && J(\vec{z} + \delta\vec{z}, \vec{v} + \delta\vec{v}) + \mathcal{J}(\mat{\Phi}^x, \mat{\Phi}^u)  \label{eq:optim-SLS-a} \\
    & \hspace{2.2mm}\text{s.t.}
        && \delta z_{k+1} = A_k\, \delta z_k + B_k\, \delta v_k, \quad 
           \delta z_0 = 0,  \label{eq:optim-SLS-b} \\
    & 
        && \Phi_{k+1,j}^x = A_k \Phi_{k,j}^x + B_k \Phi_{k,j}^u, \quad
           \Phi_{j+1,j}^x = E_j + \tfrac{\partial E_j}{\partial z}\, \delta z_j + \tfrac{\partial E_j}{\partial v}\, \delta v_j,  \label{eq:optim-SLS-c} \\
    &
        && \sum_{j=0}^{k-1} G_k \begin{bmatrix} \Phi_{k,j}^x \\ \Phi_{k,j}^u \end{bmatrix} w_c 
           + w_r\, \big\| G_k \begin{bmatrix} \Phi_{k,j}^x \\ \Phi_{k,j}^u \end{bmatrix} \big\|_{\text{row},q}
           + G_k \begin{bmatrix} z_k + \delta z_k \\ v_k + \delta v_k \end{bmatrix} + g_k \leq 0,  \label{eq:optim-SLS-d} \\
    & 
        && \sum_{j=0}^{N-1} G_f\, \Phi_{N,j}^x\, w_c 
           + w_r\, \big\| G_f \Phi_{N,j}^x \big\|_{\text{row},q}
           + G_f\, (z_N + \delta z_N) + g_f \leq 0,  \label{eq:optim-SLS-e} \\
    &
        && \|\delta\vec{z}\| \leq \varepsilon, \quad
           \|\delta\vec{v}\| \leq \varepsilon.  \label{eq:optim-SLS-f}
\end{alignat}
\end{subequations}
\end{minipage}
\end{figure*}

We derive an efficient algorithm for solving \eqref{eq:optim-SLS}.
The algorithm derived here is based on \cite{leeman2024fast}, with some modifications to allow for state/input-dependent disturbance matrix and nonzero-centered disturbance
\begin{equation*}
    w \in \mathcal{W} \subseteq
    \bigl\{ w \in \mathbb{R}^{n_c} \,\big|\, \|w - w_c\|_p \leq w_r \bigr\}.
\end{equation*}
We outer approximate the disturbance set $\mathcal{W}$ with a $\ell_2$-norm ball centered around $w_c$ so that $p=2$. From $1/p + 1/q = 1$, the dual norm exponent $q$ is equal to $2$.

To write the optimization problem \eqref{eq:optim-SLS} concisely, we introduce auxiliary variables $\beta_{k,j}$ defined as
\begin{subequations} \label{eq:beta}
\begin{align}
    \beta_{k,j} &= H_{k,j}(\mat{\Phi}) = \bigl( \| G_k \begin{bmatrix} \Phi^x_{k,j} \\ \Phi^u_{k,j} \end{bmatrix} \|_{\text{row},2} \bigr)^2 \in \mathbb{R}^{n_c},  \\
    \beta_{f,j} &= H_{f,j}(\mat{\Phi}) = \bigl( \| G_f \Phi^x_{N,j} \|_{\text{row},2} \bigr)^2 \in \mathbb{R}^{n_f}.
\end{align}
\end{subequations}
Here, the operation $(\cdot)^2$ is defined elementwise. We define
\begin{subequations} \label{eq:constraint-tightening}
\begin{align}
    h^\text{ct}_k(\mat{\beta}) &= \sum_{j=0}^{k-1} w_r\, (\beta_{k,j} + \epsilon 1_{n_c})^\frac{1}{2} \in \mathbb{R}^{n_c} ,  \\
    h^\text{ct}_f(\mat{\beta}) &= \sum_{j=0}^{N-1} w_r\, (\beta_{f,j} + \epsilon 1_{n_f})^\frac{1}{2} \in \mathbb{R}^{n_f} .
\end{align}
\end{subequations}
Here, the operation $(\cdot)^\frac{1}{2}$ is defined elementwise, and a fixed constant $\epsilon > 0$ is to circumvent point s of non-differentiability.
We also define
\begin{subequations} \label{eq:constraint-shifting}
\begin{align}
    h^\text{cs}_k(\mat{\Phi}) &= \sum_{j=0}^{k-1} G_k \begin{bmatrix} \Phi^x_{k,j} \\ \Phi^u_{k,j} \end{bmatrix} w_c \in \mathbb{R}^{n_c} ,  \\
    h^\text{cs}_f(\mat{\Phi}) &= \sum_{j=0}^{N-1} G_f\, \Phi^x_{N,j}\, w_c \in \mathbb{R}^{n_f} .
\end{align}
\end{subequations}

We define $\vec{y} = (\delta\vec{z}, \delta\vec{v})$ and summarize the optimization problem \eqref{eq:optim-SLS} as
\begin{subequations} \label{eq:optim-withbeta}
\begin{alignat}{2}
    &\min_{\vec{y},\mat{\Phi},\mat{\beta}}\quad && J(\vec{y}) + \mathcal{J}(\mat{\Phi}) ,  \\ 
    &\hspace{1.5mm}\text{s.t.}
           && f(\vec{y}) = 0 ,  \label{eq:optim-withbeta-b} \\
    &      && D(\mat{\Phi}, \vec{y}) = 0 ,  \label{eq:optim-withbeta-c} \\
    &      && h^\text{cs}(\mat{\Phi}) + h^\text{ct}(\mat{\beta}) + h(\vec{y}) \leq 0 ,  \label{eq:optim-withbeta-d} \\
    &      && H(\mat{\Phi}) - \mat{\beta} = 0 ,  \label{eq:optim-withbeta-e}  \\
    &      && g(\vec{y}) \leq 0.  \label{eq:optim-withbeta-f}
\end{alignat}    
\end{subequations}
The constraint \eqref{eq:optim-withbeta-b} encodes the nominal dynamics \eqref{eq:optim-SLS-b}, and the constraint \eqref{eq:optim-withbeta-c} encodes the disturbance propagation \eqref{eq:optim-SLS-c}. Inequality \eqref{eq:optim-withbeta-d} captures the tightened constraints \eqref{eq:optim-SLS-d}--\eqref{eq:optim-SLS-e}. Constraint \eqref{eq:optim-withbeta-e} captures the nonlinear relationship \eqref{eq:beta} between $\mat{\beta}$ and $\mat{\Phi}$. Constraint \eqref{eq:optim-withbeta-f} encodes additional constraints such as \eqref{eq:optim-SLS-f}.

We further define $\tilde{\mathcal{J}}$ as a composition function condensing the disturbance propagation \eqref{eq:optim-withbeta-c}, such that \mbox{$\tilde{\mathcal{J}}(\mat{\Phi}^u, \vec{y}) = \mathcal{J}(\mat{\Phi})$} for any $\mat{\Phi}$ with \mbox{$D(\mat{\Phi}, \vec{y}) = 0$}. We do the same for $\tilde{h}^\text{cs}$ and $\tilde{H}$. The results in a more compact optimization problem:
\begin{equation} \label{eq:optim-compact}
\begin{alignedat}{2}
    &\min_{\vec{y},\mat{\Phi}^u,\mat{\beta}}\quad && J(\vec{y}) + \tilde{\mathcal{J}}(\mat{\Phi}^u, \vec{y}) ,  \\ 
    &\hspace{2.3mm}\text{s.t.}
           && f(\vec{y}) = 0 ,  \\
    &      && \tilde{h}^\text{cs}(\mat{\Phi}^u) + h^\text{ct}(\mat{\beta}) + h(\vec{y}) \leq 0 ,  \\
    &      && \tilde{H}(\mat{\Phi}^u) - \mat{\beta} = 0 ,  \\
    &      && g(\vec{y}) \leq 0.
\end{alignedat}    
\end{equation}
The Lagrangian of \eqref{eq:optim-compact} is given by
\begin{multline} \label{eq:optim-compact-lagrangian}
    \mathcal{L}(\vec{y}, \mat{\Phi}^u, \mat{\beta}, \lambda, \mu, \eta)
    = J(\vec{y}) + \tilde{\mathcal{J}}(\mat{\Phi}^u, \vec{y})
    + \lambda^\top f(\vec{y})  \\
    + \mu^\top \bigl( \tilde{h}^\text{cs}(\mat{\Phi}^u) + h^\text{ct}(\mat{\beta}) + h(\vec{y}) \bigr)  \\
    + \eta^\top \bigl( \tilde{H}(\mat{\Phi}^u) - \mat{\beta} \bigr)
    + \gamma^\top g(\vec{y}) .
\end{multline}
with $\lambda$, $\mu$, and $\eta$ the dual variables.
The \ac{KKT} conditions of \eqref{eq:optim-compact} are given by
\begin{subequations} \label{eq:optim-compact-KKT}
\begin{align}
    \nabla J(\vec{y}) + \nabla_{\vec{y}} \tilde{\mathcal{J}}(\mat{\Phi}^u, \vec{y}) + \nabla f(\vec{y})\, \lambda + \nabla h(\vec{y}) \mu &= 0 ,  \label{eq:optim-compact-KKT-a} \\
    \nabla_{\mat{\Phi}^u} \tilde{\mathcal{J}}(\mat{\Phi}^u, \vec{y}) + \nabla \tilde{h}^\text{cs}(\mat{\Phi}^u)\, \mu + \nabla \tilde{H}(\mat{\Phi}^u)\, \eta &= 0 ,  \label{eq:optim-compact-KKT-b} \\
    \nabla h^\text{ct}(\mat{\beta})\, \mu - \eta &= 0 ,  \label{eq:optim-compact-KKT-c} \\
    f(\vec{y}) &= 0 ,  \label{eq:optim-compact-KKT-d} \\
    0 \leq \mu \perp \tilde{h}^\text{cs}(\mat{\Phi}^u) + h^\text{ct}(\mat{\beta}) + h(\vec{y}) &\leq 0 ,  \label{eq:optim-compact-KKT-e} \\
    \tilde{H}(\mat{\Phi}^u) - \mat{\beta} &= 0 .  \label{eq:optim-compact-KKT-f} \\
    0 \leq \gamma \perp g(\vec{y}) &\leq 0 .  \label{eq:optim-compact-KKT-g}
\end{align}
\end{subequations}

The key idea of the algorithm is to partition the \ac{KKT} conditions \eqref{eq:optim-compact-KKT} into two subsets and solve them alternately.
In particular, the subset of necessary conditions \eqref{eq:optim-compact-KKT-a}, \eqref{eq:optim-compact-KKT-c}, \eqref{eq:optim-compact-KKT-d}, \eqref{eq:optim-compact-KKT-e}, \eqref{eq:optim-compact-KKT-g} with fixed $\bar{\mat{\Phi}}^u$ and $\bar{\mat{\beta}}$, corresponds to a nominal trajectory optimization:
\begin{subequations} \label{eq:optim-traj-compact}
\begin{align}
    \nabla J(\vec{y}) + \nabla_{\vec{y}} \tilde{\mathcal{J}}(\bar{\mat{\Phi}}^u, \vec{y}) + \nabla f(\vec{y})\, \lambda + \nabla h(\vec{y})\, \mu &= 0 ,  \\
    \nabla h^\text{ct}(\bar{\mat{\beta}})\, \mu - \eta &= 0 ,  \label{eq:optim-traj-compact-b} \\
    f(\vec{y}) &= 0 ,  \\
    0 \leq \mu \perp \tilde{h}^\text{cs}(\bar{\mat{\Phi}}^u) + h^\text{ct}(\bar{\mat{\beta}}) + h(\vec{y}) &\leq 0 ,  \\
    0 \leq \gamma \perp g(\vec{y}) &\leq 0 .
\end{align}
\end{subequations}
The solution to \eqref{eq:optim-traj-compact} yields $\bar{\vec{y}}$, $\bar{\mu}$, and $\bar{\eta}$, which are used to solve the remaining necessary conditions \eqref{eq:optim-compact-KKT-b} and \eqref{eq:optim-compact-KKT-f}, corresponding to an uncertainty tube optimization:
\begin{subequations} \label{eq:optim-tube-compact}
\begin{align}
    \nabla_{\mat{\Phi}^u} \tilde{\mathcal{J}}(\mat{\Phi}^u, \bar{\vec{y}}) + \nabla \tilde{h}^\text{cs}(\mat{\Phi}^u)\, \bar{\mu} + \nabla \tilde{H}(\mat{\Phi}^u)\, \bar{\eta} &= 0 ,  \\
    \tilde{H}(\mat{\Phi}) - \mat{\beta} &= 0 .  \label{eq:optim-tube-compact-b}
\end{align}
\end{subequations}
The solution to \eqref{eq:optim-tube-compact} yields a new fixed $\bar{\mat{\Phi}}^u$ and $\bar{\mat{\beta}}$, which are used to solve \eqref{eq:optim-traj-compact} again. The alternation continues until the \ac{KKT} conditions \eqref{eq:optim-compact-KKT} are all satisfied.

\subsection{Uncertainty Tube Optimization}
The solution to \eqref{eq:optim-tube-compact} can be obtained by solving the following optimization problem:
\begin{equation*} \label{eq:optim-tube}
\begin{alignedat}{2}
    &\min_{\mat{\Phi}}\quad && \mathcal{J}(\mat{\Phi}) + \bar{\mu}^\top h^\text{cs}(\mat{\Phi}) + \bar{\eta}^\top H(\mat{\Phi}) ,  \\
    &\hspace{0.6mm}\text{s.t.} && D(\mat{\Phi}, \bar{\vec{y}}) = 0 ,
\end{alignedat}
\end{equation*}
or equivalently
\begin{equation*}  \label{eq:Phi-governing}
\begin{alignedat}{2}
    &\min_{\mat{\Phi}}\quad && \sum_{j=0}^{N-1} \biggl( \sum_{k=j+1}^{N-1} \trace\Bigl( \begin{bmatrix} \Phi^x_{k,j} \\ \Phi^u_{k,j} \end{bmatrix}^\top P_{k,j} \begin{bmatrix} \Phi^x_{k,j} \\ \Phi^u_{k,j} \end{bmatrix} + 2 p_{k,j}^\top \begin{bmatrix} \Phi^x_{k,j} \\ \Phi^u_{k,j} \end{bmatrix} \Bigr) \nonumber\\
    &                       &&\hspace{7em} + \trace\Bigl( \Phi_{N,j}^{x^\top} P_{N,j} \Phi^x_{N,j} + 2 p_{N,j}^\top \Phi^x_{N,j} \Bigr) \biggr) ,  \\
    &\hspace{0.6mm}\text{s.t.} && \Phi^x_{k+1,j} = A_k \Phi^x_{k,j} + B_k \Phi^u_{k,j} ,  \\
    &            && \Phi^x_{j+1,j} =  E_j + \tfrac{\partial E_j}{\partial z}\, \overline{\delta z}_j + \tfrac{\partial E_j}{\partial v}\, \overline{\delta v}_j ,
\end{alignedat}    
\end{equation*}
where
\begin{align*}
    P_{k,j} &= \begin{bmatrix} \bar{Q} & 0 \\ 0 & \bar{R} \end{bmatrix} + G_k^\top \diag(\bar{\eta}_{k,j})\, G_k ,\quad
    p_{k,j} = \frac{1}{2} G_k^\top \bar{\mu}_k w_c^\top , \\
    P_{N,j} &= \bar{Q}_f + G_f^\top \diag(\bar{\eta}_{N,j})\, G_f ,\hspace{2.8em}
    p_{N,j} = \frac{1}{2} G_f^\top \bar{\mu}_N w_c^\top .
\end{align*}
This can be via using $N$ independent Riccati recursions \cite{leeman2024fast}. Each Riccati recursion solves for
\begin{equation}
\begin{alignedat}{2}
    &\min_{\mat{\Phi}}\quad && \sum_{k=j+1}^{N-1} \trace\Bigl( \begin{bmatrix} \Phi^x_{k,j} \\ \Phi^u_{k,j} \end{bmatrix}^\top P_{k,j} \begin{bmatrix} \Phi^x_{k,j} \\ \Phi^u_{k,j} \end{bmatrix} + 2 p_{k,j}^\top \begin{bmatrix} \Phi^x_{k,j} \\ \Phi^u_{k,j} \end{bmatrix} \Bigr) \nonumber\\
    &                       &&\hspace{5em} + \trace\Bigl( \Phi_{N,j}^{x^\top} P_{N,j} \Phi^x_{N,j} + 2 p_{N,j}^\top \Phi^x_{N,j} \Bigr) ,  \\
    &\hspace{0.6mm}\text{s.t.} && \Phi^x_{k+1,j} = A_k \Phi^x_{k,j} + B_k \Phi^u_{k,j} ,  \\
    &            && \Phi^x_{j+1,j} =  E_j + \tfrac{\partial E_j}{\partial z}\, \overline{\delta z}_j + \tfrac{\partial E_j}{\partial v}\, \overline{\delta v}_j .
\end{alignedat}
\end{equation}
 We partition $P_{k,j}$ and $p_{k,j}$ as 
\begin{equation*}
    P_{k,j} = \begin{bmatrix} P_{k,j}^{xx} & P_{k,j}^{xu} \\ P_{k,j}^{ux} & P_{k,j}^{uu} \end{bmatrix} , \quad
    p_{k,j} = \begin{bmatrix} p_{k,j}^{x} \\ p_{k,j}^{u} \end{bmatrix} .
\end{equation*}
The Riccati recursion starts by initializing
\begin{equation*}
    S_{N,j} = P_{N,j} , \quad 
    s_{N,j} = p_{N,j} .
\end{equation*}
Followed by a backward pass for $k = N-1, \dots, j+1$:
\begin{subequations}
\begin{align*}
    M_{k,j} &= P_{k,j}^{uu} + B_k^\top S_{k+1,j} B_k ,  \\
    F_{k,j} &= P_{k,j}^{ux} + B_k^\top S_{k+1,j} A_k ,  \\
    f_{k,j} &= p_{k,j}^u + B_k^\top s_{k+1,j} ,  \\
    H_{k,j} &= P_{k,j}^{xx} + A_k^\top S_{k+1,j} A_k ,  \\
    h_{k,j} &= p_{k,j}^x + A_k^\top s_{k+1,j} ,  \\
    K_{k,j} &= -M_{k,j}^{-1} F_{k,j} ,\hspace{3.6em} k_{k,j} = -M_{k,j}^{-1} f_{k,j} ,  \\
    S_{k,j} &= H_{k,j} + F_{k,j}^\top K_{k,j} ,\hspace{1.5em} s_{k,j} = h_{k,j} + F_{k,j}^\top k_{k,j} .
\end{align*}    
\end{subequations}
Followed by a forward pass starting from $\Phi^x_{j+1,j} = E_j + \tfrac{\partial E_j}{\partial z}\, \overline{\delta z}_j + \tfrac{\partial E_j}{\partial v}\, \overline{\delta v}_j$ and for $k = j+1, \dots, N-1$:
\begin{subequations}
\begin{align*}
    \Phi^u_{k,j} &= K_{k,j} \Phi^x_{k,j} + k_{k,j} ,  \\
    \Phi^x_{k+1,j} &= A_k \Phi^x_{k,j} + B_k \Phi^u_{k,j} .
\end{align*}
\end{subequations}
The resulting optimal value of the objective is
\begin{equation*}
    \sum_{j=0}^{N-1} \trace\bigl( \Phi_{j+1,j}^{x^\top}\, S_{j+1,j}\, \Phi_{j+1,j}^{x} \bigr).
\end{equation*}
Once solved, the value of $\mat{\beta}$ is then obtained by evaluating \eqref{eq:optim-tube-compact-b} and \eqref{eq:beta}.

\subsection{Nominal Trajectory Optimization}
The solution to \eqref{eq:optim-traj-compact} can be obtained by solving a nominal trajectory optimization problem:
\begin{equation*} \label{eq:optim-traj}
\begin{alignedat}{2}
    &\min_{\delta\vec{z}, \delta\vec{v}}\quad && J(\vec{z} + \delta\vec{z}, \vec{v} + \delta\vec{v})  \\
    &    && + \trace\Bigl( \begin{bmatrix} \delta z_j \\ \delta v_j\end{bmatrix}^\top \begin{bmatrix} \frac{\partial E_j}{\partial z} & \frac{\partial E_j}{\partial v} \end{bmatrix}^\top S_{j+1,j} \begin{bmatrix} \frac{\partial E_j}{\partial z} & \frac{\partial E_j}{\partial v} \end{bmatrix} \begin{bmatrix} \delta z_j \\ \delta v_j\end{bmatrix} \Bigr) ,  \\
    &\hspace{1mm}\text{s.t.} 
         && \delta z_{k+1} = A_k\, \delta z_k + B_k\, \delta v_k, \quad \delta z_0 = 0 ,  \\
    &    && h^\text{cs}_k(\mat{\bar{\Phi}}) + h^\text{ct}_k(\bar{\mat{\beta}}) + G_k \begin{bmatrix} z_k + \delta z_k \\ v_k + \delta v_k \end{bmatrix} + g_k \leq 0 ,  \\
    &    && h^\text{cs}_N(\mat{\bar{\Phi}}) + h^\text{ct}_N(\bar{\mat{\beta}}) + G_f\, (z_N + \delta z_N) + g_f \leq 0 ,  \\
    &    && \|\delta\vec{z}\| \leq \varepsilon, \quad
           \|\delta\vec{v}\| \leq \varepsilon.
\end{alignedat}    
\end{equation*}
This problem is a \ac{QP}.
Once solved, the value of $\eta$ is then obtained by evaluating \eqref{eq:optim-traj-compact-b} and \eqref{eq:constraint-tightening},
\begin{subequations}
\begin{align*}
    \eta_{k,j} &= \tfrac{1}{2} w_r\, (\beta_{k,j} + \epsilon 1_{n_c})^{-\frac{1}{2}} \odot \mu_k ,  \\
    \eta_{f,j} &= \tfrac{1}{2} w_r\, (\beta_{f,j} + \epsilon 1_{n_f})^{-\frac{1}{2}} \odot \mu_N .
\end{align*}
\end{subequations}
Here, the operation $(\cdot)^{-\frac{1}{2}}$ is applied elementwise, and $\odot$ denotes elementwise multiplication.

\vspace{\baselineskip}
The algorithm alternates between uncertainty tube optimization, which uses Riccati recursions, and nominal trajectory optimization, which solves a \ac{QP}, until convergence. Both the Riccati recursions and \ac{QP} solves can benefit from GPU acceleration to reduce the algorithm solve time \cite{fang2026safe}.

\end{document}